\documentclass{article}

\usepackage{amsmath,amsthm,amssymb,calc,xcolor}
\usepackage[margin=1in,footskip=0.25in]{geometry}
\usepackage{amsfonts}
\usepackage{latexsym}
\usepackage{mathtools}
\usepackage{stmaryrd}
\usepackage{multicol}
\usepackage{pdfsync}
\usepackage{dsfont}
\usepackage{comment}
\usepackage{algorithmic}
\usepackage{algorithm}
\usepackage[margin=1in,footskip=0.25in]{geometry}
\usepackage{graphicx}
\usepackage{subcaption}

\renewcommand{\P}{\mathbb{P}}
\newcommand{\R}{\mathbb{R}}

\newcommand{\CX}{\mathcal{X}}
\newcommand{\CY}{\mathcal{Y}}
\newcommand{\CM}{\mathcal{M}}

\newcommand{\CF}{\mathcal{F}}
\newcommand{\abs}[1]{\left\lvert#1\right\rvert}
\newcommand{\brac}[1]{\left\{#1\right\}}
\newcommand{\ip}[1]{\left\langle#1\right\rangle}
\newcommand{\norm}[1]{\left\lVert#1\right\rVert}
\newcommand{\paren}[1]{\left(#1\right)}
\newcommand{\floor}[1]{\left\lfloor#1\right\rfloor}

\DeclareMathOperator{\E}{\mathbb{E}}

\newtheorem{theorem}{Theorem}[section]
\newtheorem{lemma}[theorem]{Lemma}
\newtheorem{proposition}[theorem]{Proposition}

\newtheorem{remark}{Remark}[section]
\newtheorem{assumption}{Assumption}[section]

\newcounter{listctr}
\newenvironment{custlist}[1][Case]%
	{\vspace{-5pt}\begin{list}{}{\usecounter{listctr}%
		\renewcommand\makelabel[1]{\hfil\itshape #1\ \arabic{listctr}.}%
		\settowidth\labelwidth{\makelabel{\quad}}%
		\setlength\leftmargin{0pt}}%
		\setlength\itemindent{\labelwidth+\labelsep}}
	{\end{list}}
	
\title{Normalization effects on deep neural networks}

\author{Jiahui Yu\footnote{Department of Mathematics and Statistics, Boston University, Boston, E-mail: jyu32@bu.edu} \phantom{.}  and Konstantinos Spiliopoulos\footnote{Department of Mathematics and Statistics, Boston University, Boston, E-mail: kspiliop@math.bu.edu}
\thanks{K.S. was partially supported by the National Science Foundation (DMS 2107856) and Simons Foundation Award  672441}\\
}

\begin{document}


\date{\today}
\maketitle

\begin{abstract}
We study the effect of normalization on the layers of deep neural networks of feed-forward type. A given layer $i$ with $N_{i}$ hidden units is allowed to be normalized by $1/N_{i}^{\gamma_{i}}$ with $\gamma_{i}\in[1/2,1]$ and we study the effect of the choice of the $\gamma_{i}$ on the statistical behavior of the neural network's output (such as variance) as well as on the test accuracy on the MNIST data set. We find that in terms of variance of the neural network's output and test accuracy the best choice is to choose the $\gamma_{i}$'s to be equal to one, which is the mean-field scaling. We also find that this is particularly true for the outer layer, in that the neural network's behavior is more sensitive in the scaling of the outer layer as opposed to the scaling of the inner layers. The mechanism for the mathematical analysis is an asymptotic expansion for the neural network's output. An important practical consequence of the analysis is that it provides a systematic and mathematically informed way to choose the learning rate hyperparameters. Such a choice guarantees that the neural network behaves in a statistically robust way as the $N_i$ grow to infinity.
\end{abstract}


\section{Introduction}\label{S:Introduction}

The last few years have experienced an explosion in the study of neural networks. Neural networks are parametric models and their coefficients are estimated from data using gradient descent methods. Early classical results regarding the approximation power of neural networks \cite{Barron,Hornik1,Hornik2} set the stage and then advances in technology led to great successes in text, speech and image recognition, see for example \cite{LeCun,Goodfellow,DriverlessCar,FacialRecognition,DeepVoice,GoogleDuplex,SpeechRecognition3} to name a few. Later on, neural networks showed a lot of promise in other fields such as robotics, medicine, finance, and applied mathematics, see for example \cite{Ling1,Ling2,Robotics2,Robotics3,NatureMedicine1,NatureMedicine2,Finance1,Finance2,Finance3}. Their success in applications has made clearer the need for a better understanding of their mathematical properties.

The goal of this paper is to investigate the performance of multilayer neural networks as a function of normalization features. In particular, let us consider the following neural network with two hidden layers:
\begin{equation}\label{Eq:TwoLayerNN}
g_{\theta}^{N_1,N_2}(x) = \frac{1}{N_2^{\gamma_2}} \sum_{i=1}^{N_2} C^i \sigma\left(\frac{1}{N_1^{\gamma_1}}\sum_{j=1}^{N_1} W^{2,j,i}\sigma(W^{1,j}x)\right),
\end{equation}
where $C^i, W^{2,j,i} \in \R$, $x,W^{1,j} \in \R^d$, and $\gamma_1, \gamma_2 \in [1/2,1)$ are fixed scaling parameters. For convenience, we write $W^{1,j}x = \ip{W^{1,j},x}_{l^2}$ as the standard $l^2$ inner product for the vectors. The neural network model has parameters
\begin{equation*}
\theta = \left(C^1, \ldots, C^{N_2}, W^{2,1,1}, \ldots W^{2,N_1,N_2}, W^{1,1},\ldots W^{1,N_1} \right),
\end{equation*}
which are to be estimated from data $(X,Y) \sim \pi(dx,dy)$.

Our goal is to understand the effect of the choice of the values of the scaling parameters $\gamma_1,\gamma_2
\in[1/2,1]$ on the behavior of the neural network. The choice $\gamma_1=\gamma_2=1$ corresponds to the mean field scaling that has been studied in the literature in recent years, see for example \cite{Chizat2018,Montanari,RVE,SirignanoSpiliopoulosNN1,SirignanoSpiliopoulosNN2,SirignanoSpiliopoulosNN3}.
On the other side of the spectrum, i.e, when $\gamma_1=\gamma_2=1/2$, then we have the so-called Xavier normalization \cite{Xavier}, giving rise to the so-called neural tangent kernel, that has been analyzed in a number of works, see for example \cite{NTK,Du1,HuangYau2020, JasonLee,SirignanoSpiliopoulosNN4}. Even though, most of the discussion of this paper is focused on the two-layer neural network, in Section \ref{SS:Three_layers}, see also Section \ref{S:DNN_LearningRates}, we discuss the three-layer neural network case demonstrating that our conclusions extend to general feed-frward multilayer neural networks.

In the case of shallow neural networks (SNN), i.e, when $g_{\theta}^N(x) = \frac{1}{N^{\gamma}} \sum_{i=1}^{N} C^i\sigma(W^i x)$, the question on the effect of $\gamma\in[1/2,1]$ on the performance of the neural network has been recently studied in \cite{SpiliopoulosYu2021}. In \cite{SpiliopoulosYu2021} we developed an asymptotic expansion for the neural network's statistical output $g^{N}$ after training with stochastic gradient descent (SGD) pointwise with respect to the scaling parameter $\gamma\in(1/2,1)$ as the number of hidden units $N$ grows to infinity. Based on this expansion \cite{SpiliopoulosYu2021} demonstrates mathematically that to leading order in $N$, there is no bias-variance trade off, in that both bias and variance (both explicitly characterized) decrease as the number of hidden units increases and time grows. In addition, it is shown there that to leading order in $N$, the variance of the neural network's statistical output $g^{N}$ is monotonically decreasing in $\gamma$ and thus the lowest variance is attained at $\gamma=1$. Numerical studies on the MNIST and CIFAR10 datasets showed that test and train accuracy monotonically improve as the neural network's normalization gets closer to the mean field normalization $\gamma=1$. An additional useful conclusion of the mathematical analysis in \cite{SpiliopoulosYu2021} is that in order for the asymptotic results to be true (without trivial limits) one needs to choose the learning rate in SGD in a very specific way with respect to $N$ and $\gamma$.

The content of this paper is the corresponding analysis in the case of deep neural networks (DNN). As we will see the analysis in the case of DNNs is considerably more complicated than in the case of SNN. However, the end conclusions are of similar nature with the additional interesting observation that the outer layer plays a more special role. In addition, the analysis of this paper offers a mathematically principled way to appropriately choose the learning rates. We base our analysis on a typical two-layer neural network, however as we shall see in Section \ref{S:Numerics}  this is done without loss of generality.

 In particular, we derive an asymptotic expansion of the neural network's output as $N_2\rightarrow\infty$ with $N_1$ fixed. This expansion shows mathematically that to leading order in $N_2$, the variance of the neural network's statistical output is monotonically decreasing with respect to $\gamma_2\in[1/2,1]$. At the same time, the same expression (after appropriately choosing the learning rates) shows that the effect of $\gamma_1$ is perhaps less prominent in the sense that it appears through terms that are averages and are also bounded (for bounded activated functions). The mathematical conclusion is that, at least under our assumptions (as presented in Section  \ref{S:MainResults})  one would optimally choose the outer layer normalization to be $\gamma_2=1$ and subsequently choosing $\gamma_1=1$ would be optimal. This conclusion is also validated numerically. Indeed, in Section \ref{S:Numerics} we study the test accuracy of two and three layer neural networks for different parametrizations in terms of $\gamma_1,\gamma_2
\in[1/2,1]$ (and $\gamma_3\in[1/2,1]$ in the three-layer neural network case) when trained with standard SGD on the MNIST dataset \cite{MNIST}. As we shall see there, the test accuracy is sensitive to the choice of the normalization of the outer layer $\gamma_2$ with the optimal choice being $\gamma_2=1$, but having done that, the effect of the choice of the normalization of the inner layer, i.e., of $\gamma_1$ is less profound. The end optimal choice is to choose $\gamma_2=\gamma_1=1$, i.e., the mean-field normalization in all layers.

An additional important conclusion of this work is that it provides a systematic and mathematically informed way to choose the learning rates hyperparameters, see (\ref{potential_alpha_1}) for the model (\ref{Eq:TwoLayerNN}), Section \ref{SS:Three_layers} for the three-layer case and Section \ref{S:DNN_LearningRates} for the general case. Without choosing the learning rates to be of the indicated order with respect to the $N_{i}$'s and $\gamma_i$'s the neural network as a statistical object will have trivial limits, i.e., it will either converge to zero or to infinity. If however, they are chosen in the indicated way then the neural network will behave nicely as a statistical quantity in the sense of not being trivial and having finite variance at least.

Our analysis is based upon the quadratic error loss function
\begin{equation*}
L(\theta) = \frac{1}{2} \E_{X,Y}\left[\left(Y-g_{\theta}^{N_1,N_2}(x)\right)^2\right],
\end{equation*}
and the model parameters $\theta$ are trained by the stochastic gradient descent algorithm, for $k \in \mathbb{N}$
\begin{equation}\label{SGD}
\begin{aligned}
C^i_{k+1} &= C^i_k + \frac{\alpha_c^{N_1,N_2}}{N_2^{\gamma_2}} \left(y_k - g_k^{N_1,N_2}(x_k)\right)H^{2,i}_k(x_k),\\
W^{1,j}_{k+1} &= W^{1,j}_k + \frac{\alpha_{W,1}^{N_1,N_2}}{N_1^{\gamma_1}}\left(y_k - g_k^{N_1,N_2}(x_k)\right)\left(\frac{1}{N_2^{\gamma_2}}\sum_{i=1}^{N_2} C^i_k \sigma'(Z^{2,i}_k(x_k))W^{2,j,i}_k\right)\sigma'(W^{1,j}_k x_k)x_k,\\
W^{2,j,i}_{k+1} &= W^{2,j,i}_k + \frac{\alpha_{W,2}^{N_1,N_2}}{N_1^{\gamma_1}N_2^{\gamma_2}}\left(y_k - g_k^{N_1,N_2}(x_k)\right)C^i
_k \sigma'(Z^{2,i}_k(x_k))H^{1,j}_k(x_k),\\
\end{aligned}
\end{equation}
where
\begin{equation*}
\begin{aligned}
H^{1,j}_k(x) &= \sigma(W^{1,j}_k x),\quad
Z^{2,i}_k(x) = \frac{1}{N_1^{\gamma_1}}\sum_{j=1}^{N_1} W^{2,j,i}_k H^{1,j}_k(x),\quad
H^{2,i}_k(x) = \sigma(Z^{2,i}_k(x)).
\end{aligned}
\end{equation*}


For fixed $N_1$, we define the empirical measure
\begin{equation*}
\tilde{\gamma}^{N_1,N_2}_k = \frac{1}{N_2}\sum_{i=1}^{N_2} \delta_{C^i_k, W^{2,1,i}_k, \ldots, W^{2,N_1,i}_k, W^{1,1}_k, \ldots, W_k^{1,N_1}},
\end{equation*}
and the time-scaled empirical measure
\begin{equation}
\gamma_t^{N_1,N_2} = \tilde{\gamma}^{N_1,N_2}_{\floor{N_2t}}.\label{Eq:TimeScaledEmpiricalMeasure}
\end{equation}
The neural network output can be rewritten as
\begin{equation*}
g^{N_1,N_2}_{\theta_k}(x) = \ip{c\sigma\left(\frac{1}{N_1^{\gamma_1}} \sum_{j=1}^{N_1} w^{2,j}\sigma(w^{1,j}x)\right), N_2^{1-\gamma_2} \tilde{\gamma}_k^{N_1,N_2}} = \ip{c\sigma(Z^{2,N_1}(x)), N_2^{1-\gamma_2} \tilde{\gamma}_k^{N_1,N_2}},
\end{equation*}
and the time-scaled neural network output is
\begin{align}
h^{N_1,N_2}_t(x)& = g^{N_1,N_2}_{\theta_{\floor{N_2 t}}}(x).\label{Eq:TimeScaledNNoutput}
\end{align}

For a fixed data set $(x^{(i)},y^{(i)})_{i=1}^M$, let $g^{N_1,N_2}_k$ and $h^{N_1,N_2}_t$ denote the M-dimensional vectors whose $i$-th entries are $g^{N_1,N_2}_k(x^{(i)})$ and $h^{N_1,N_2}_t(x^{(i)})$, respectively. In order to emphasize the dependence on $\hat{\gamma}=(\gamma_1,\gamma_2)$ and on $\hat{N}=(N_1,N_2)$ we will instead write sometimes $h^{\hat{N},\hat{\gamma}}_t$.

As it will be demonstrated below, it turns out that in order to understand the main effects of $\gamma_1,\gamma_2\in(1/2,1)$ on the behavior of $h^{\hat{N},\hat{\gamma}}_t$ it is enough to look at its asymptotic behavior as $N_2\rightarrow\infty$ with the $N_1$ being thought of as large but fixed.

In addition, the learning rates need to be chosen to be of the right order with respect to the number of hidden units $N_i$ and network normalization $\gamma_i$ in order for the neural network to behave in a statistically robust way. In particular, for reasons that will become clearer later on, we shall choose the learning rates to be
\begin{equation}\label{potential_alpha_1}
\alpha_C^{N_1,N_2} = \frac{\alpha_C}{N_2^{2-2\gamma_2}}, \quad \alpha_{W,1}^{N_1,N_2} = \frac{\alpha_{W,1}}{N_1^{1-2\gamma_1}N_2^{3-2\gamma_2}}, \quad \alpha_{W,2}^{N_1,N_2} = \frac{\alpha_{W,2}}{N_1^{1-2\gamma_1}N_2^{2-2\gamma_2}},
\end{equation}
where the coefficients $\alpha_C,\alpha_{W,1},\alpha_{W,2}\in(0,\infty)$ are chosen to be of order one with respect to $N_1, N_2$.

Loosely speaking our main mathematical result is that for each fixed $\gamma_2\in(1/2,1)$ one has  that as $N_2\rightarrow\infty$, and  when $\gamma_2 \in \left(\frac{2\nu-1}{2\nu}, \frac{2\nu+1}{2\nu+2}\right)$ for fixed $\nu\in\{1,2,3,\cdots\}$ and fixed $\gamma_1$ and $N_1$:
\begin{align}
h_t^{\hat{N},\hat{\gamma}}&\approx h^{N_1,\gamma_1}_{t}+\sum_{j=1}^{\nu-1} N_2^{-j(1-\gamma_2)}Q^{N_1,\gamma_1}_{j,t} + N_2^{-(\gamma_2-1/2)}e^{-A^{N_1,\gamma_1} t} \mathcal{G}^{N_1}+ \textrm{ lower order terms in }N_2.\label{Eq:FormalExpansion}
\end{align}

In (\ref{Eq:FormalExpansion}), $h^{N_1,\gamma_1}_{t}$ is the limit of  $h_t^{\hat{N},\hat{\gamma}}$ as $N_2\rightarrow\infty$, $Q^{N_1,\gamma_1}_{j,t}$ are deterministic quantities, $A^{N_1,\gamma_1}$ is a positive definite matrix and $\mathcal{G}^{N_1}$ is a Gaussian vector of mean zero and known variance-covariance structure. Noticeably, all of $h^{N_1,\gamma_1}_{t}$, $Q^{N_1,\gamma_1}_{j,t}$, $A^{N_1,\gamma_1}$ and $\mathcal{G}^{N_1}$ are not only independent of $N_2<\infty$ and $\gamma_2>0$, but the dependence on $N_1$ is through explicit averages of the form $\frac{1}{N_1}\sum_{i=1}^{N_1}\left(\cdots\right)$, and the dependence on $\gamma_1$ is only through the terms $\sigma(Z^{2,i}_k(x)),\sigma^{\prime}(Z^{2,i}_k(x))$ which for bounded $\sigma\in C^{1}_{b}(\mathbb{R})$ will be bounded.

Even though we do not show this here, as in \cite{SirignanoSpiliopoulosNN4,SpiliopoulosYu2021}, one gets that for all $\gamma_1,\gamma_2 \in (1/2,1)$ and for all $N_1<\infty$, the limit of the network output recovers the global minimum as $t \to \infty$, i.e.  $h^{N_1,\gamma_1}_t \to \hat{Y}$, where $\hat{Y} = \paren{y^{(1)}, \ldots, y^{(M)}}$. For fixed $j\in\mathbb{N}$, one can also show exactly as in \cite{SpiliopoulosYu2021} that $Q^{N_1,\gamma_1}_{j,t}\rightarrow 0$ exponentially fast as $t\rightarrow\infty$. The Gaussian vector $\mathcal{G}^{N_1}$ is related to the variance of the network at initialization which then propagates forward, see (\ref{limit_gaussian}).


These conclusions immediately suggest that  the variance of $h_t^{\hat{N},\hat{\gamma}}$ to leading order in $N_2$ is monotonically decreasing in $\gamma_2\in[1/2,1]$, with the smallest possible variance when $N_2$ is large, but fixed, when $\gamma_2=1$. In addition, the fact that the dependence of the leading order terms in the right hand side of (\ref{Eq:FormalExpansion}) on $N_1$ and on $\gamma_1$ is through averages of the form $\frac{1}{N_1}\sum_{i=1}^{N_1}\left(\cdots\right)$ for $N_1$ and through bounded terms for $\gamma_1$ (given that the activation function $\sigma\in C^{1}_b(\R)$), demonstrates that $h_t^{\hat{N},\hat{\gamma}}$ is less sensitive on the value of $\gamma_1$. The latter observation is also confirmed numerically in Section \ref{S:Numerics}.

To further validate and demonstrate these conclusions we perform in Section \ref{S:Numerics} extensive numerical studies fitting two and three layer feed-forward neural networks on the MNIST dataset \cite{MNIST}. In all of the examples, the pattern is the same and corroborates the theoretical conclusions. Namely, the test accuracy is sensitive in the choice of the normalization of the outer layer $\gamma_2$ with the optimal choice being $\gamma_2=1$, but having done that, the choice of the normalization of the inner layer, i.e., of $\gamma_1$ has less of an impact on the performance. The end optimal choice is to choose $\gamma_2=\gamma_1=1$, i.e., the mean-field normalization in all layers. 

At this point we want to emphasize that the goal of this paper is not to study the limit as $N_2,N_1
\rightarrow\infty$. We refer the interested reader to \cite{NTK,Du1,SirignanoSpiliopoulosNN3,Araujo2019,Nguyen2019} for related results. Our goal here is to disentangle the effect of different scalings in different layers. With this goal in mind, it turns out that it is enough to fix $N_1$, look at $N_2\rightarrow\infty$ and then observe that at least to leading order in $N_2$ the effect of $N_1$ is only through averages that converge to well defined limtis. In addition, in the process of doing so, we obtain that the effect of $\gamma_2$ is to scale the variance in a very simple and intuitive way as demonstrated by (\ref{Eq:FormalExpansion}). On the other hand, the effect of $\gamma_1$ is through bounded terms when at least the activation function and its derivatives are bounded. Also, we note that in order to obtain expansions like (\ref{Eq:FormalExpansion}) one needs not only to characterize the asymptotic behavior of $h_t^{\hat{N},\hat{\gamma}}$, but also needs to understand the fluctuations (central limit theorem) corrections, corrections to those corrections, etc. Lastly, our numerical studies indicate, see Figures \ref{Fig:mnist_ce_comparingh1h2_b20_test} and \ref{Fig:mnist_ce_gII10_e1000_b20_test}, that test accuracy is better when $N_2>N_1$, which also motivates looking at $N_2\rightarrow\infty$.

The rest of the paper is organized as follows. In Section \ref{S:MainResults} we lay down our main assumptions and present the main mathematical results of the paper. In Section \ref{S:Numerics} we discuss the theoretical results further and we present our numerical studies. In Section \ref{S:DNN_LearningRates} we present for completeness and without proof the mathematically motivated choice of the learning rates for a deep feedforward neural network of arbitrary depth. Conclusions are in Section \ref{S:Conclusions}.  The proof of the main results presented in Section \ref{S:MainResults} are presented in the appendix of this paper. In Appendix \ref{S:AprioriBounds} we establish apriori bounds on the learning parameters as they evolve in time. In Appendix \ref{sec::network_convergence}  we prove   Theorem \ref{LLN:theorem}. In Appendix \ref{sec::convergence_of_K} we prove Theorem \ref{CLT:theorem}. In Appendix \ref{sec::Psi} we prove  Theorem \ref{thm::Psi}. Then in Appendix \ref{sec::higher order}
 we complete the proof of the asymptotic expansion for $h^{N_1,N_2}_t$ for  $\gamma_2\in(1/2,1)$ through an inductive argument.

\section{Assumptions and main results}\label{S:MainResults}

In this section, we describe our main assumptions under which the results of this paper hold and we present our main results. We also establish necessary notation. We work on a filtered probability space $(\Omega,\mathcal{F},\mathbb{P})$  where all the random variables are defined. The probability space is equipped with a filtration $\mathfrak{F}_t$ that is right continuous and $\mathfrak{F}_0$ contains all $\mathbb{P}$-negligible sets.
\begin{assumption}
\begin{enumerate}
\item The activation function $\sigma \in C^{\infty}_b(\R)$, i.e. $\sigma$ is infinitely differentiable and bounded.
\item There is a fixed dataset $\CX \times \CY = (x^{(i)}, y^{(i)})_{i=1}^M$, and we set $\pi(dx,dy) = \frac{1}{M} \sum_{i=1}^M \delta_{(x^{(i)},y^{(i)})}(dx,dy)$.
\item The initialized parameters $\{C_0^i\}_i, \{W_0^{2,j,i}\}_{i,j}, \{W_0^{1,j}\}_j)$ are i.i.d.,generated from  mean-zero random variables and take values in compact sets $\mathcal{C}, \mathcal{W}^1$, and $\mathcal{W}^2$.
\end{enumerate}
\label{assumption}
\end{assumption}

We recall that we shall choose the learning rates to be
\begin{equation*}
\alpha_C^{N_1,N_2} = \frac{\alpha_C}{N_2^{2-2\gamma_2}}, \quad \alpha_{W,1}^{N_1,N_2} = \frac{\alpha_{W,1}}{N_1^{1-2\gamma_1}N_2^{3-2\gamma_2}}, \quad \alpha_{W,2}^{N_1,N_2} = \frac{\alpha_{W,2}}{N_1^{1-2\gamma_1}N_2^{2-2\gamma_2}},
\end{equation*}
where the coefficients $\alpha_C,\alpha_{W,1},\alpha_{W,2}\in(0,\infty)$ are chosen to be of order one with respect to $N_1, N_2$. For notational convenience and without loss of generality we shall set them to be
$\alpha_C=\alpha_{W,1}=\alpha_{W,2}=1$.

Note that the weights in different layers are trained with different rates. This choice of learning rates is necessary for convergence to a non-trivial limit as $N_2 \rightarrow \infty$. If the parameters in all the layers are trained with the same learning rate, it can be mathematically shown that the network will not train as $N_1, N_2$ become large in the sense of having convergence to trivial limits. 

Before presenting our main mathematical results let us first discuss what happens at time $t=0$. By law of large numbers, as $N_2 \to \infty$, we have that $\tilde{\gamma}^{N_1,N_2}_0 \overset{p}{\to}\gamma_0^{N_1}(dw^1,dw^2,dc)$, where
\begin{equation}\label{gamma^N1_0}
\gamma_0^{N_1}(dw^1,dw^2,dc) = \delta_{W_o^{1,1}}(dw^{1,1}) \times \cdots \times \delta_{W_o^{1,N_1}}(dw^{1,N_1})\times \mu_{W^2}(dw^{2,1}) \times \cdots \times \mu_{W^2}(dw^{2,N_1}) \times \mu_C(dc).
\end{equation}
By the central limit theorem, we have in distribution
\begin{align}
N_2^{(\gamma_2-\frac{1}{2})} h_0^{N_1,N_2}(x) =  \ip{c\sigma(Z^{2,N_1}(x)), \sqrt{N_2} \tilde{\gamma}_0^{N_1,N_2}}\overset{d}{\to} \mathcal{G}^{N_1}(x), \text{ as } N_2\rightarrow\infty\label{limit_gaussian}
\end{align}
where $\mathcal{G}^{N_1}$ is a Gaussian random variable and variance $\lambda^{2}_{N_1}(x)=\ip{|c\sigma(Z^{2,N_1}(x))|^{2},\gamma_0^{N_1}}$. From now on, we will use the notation $\mathcal{G}^{N_1}$ to refer to this specific Gaussian random variable.

Hence, when  $\gamma_2=1/2$, one has that $h_0^{N_1,N_2}(x)\overset{d}{\to} \mathcal{G}^{N_1}(x)$, and when $\gamma_2>1/2$, $h_0^{N_1,N_2}(x)\overset{d}{\to} 0$.

\begin{remark}
Notice now that due to the independence assumption from (\ref{assumption}), the sequence of random variables $\{Z^{2,N_1}(x)\}_{N_1}$, which is the input to the assumed bounded activation function $\sigma$, will also converge to a Gaussian with mean zero and finite variance in the limit $N_2\rightarrow\infty$ if $\gamma_1=1/2$ and to the trivial limit $Z^{2,N_1}(x)\to 0$ if $\gamma_1\in(1/2,1)$.
\end{remark}

Certain quantities will appear many times, so let's define them here.
\begin{align}
B^1_{x,x'}(\theta) &= \sigma\left(Z^{2,N_1}(x')\right)\sigma\left(Z^{2,N_1}(x)\right),\nonumber\\
B^{2,j}_{x,x'}(\theta) &= (c)^2 \sigma'\left(Z^{2,N_1}(x')\right)\sigma'\left(Z^{2,N_1}(x)\right)\sigma(w^{1,j}x')\sigma(w^{1,j}x),\nonumber\\
B^{3,j}_{x}(\theta) &= cw^{2,j}\sigma'(w^{1,j}x)\sigma'\left(Z^{2,N_1}(x)\right),\label{Eq:B_def0}
\end{align}
 and set
\begin{align}
A^{N_1}_{x,x'}&=\ip{B^1_{x,x'}(\theta),\gamma_0^{N_1}}
+\frac{1}{N_1}\sum_{j=1}^{N_1}\left[\ip{B^{2,j}_{x,x'}(\theta),\gamma_0^{N_1}}+xx'\ip{B^{3,j}_{x}(\theta),\gamma_0^{N_1}}\ip{B^{3,j}_{x'}(\theta),\gamma_0^{N_1}}\right] \label{Eq:A_def}
\end{align}

In addition, for a given $f \in C_b^2(\R^{1+N_1(1+d)})$ let us define
\begin{align}
C^{N_1,f}_{x'}(\theta)&=\partial_{c}f(\theta)  \sigma(Z^{2,N_1}(x'))+
\frac{1}{N_1^{1-\gamma_1}} c\sigma'(Z^{2,N_1}(x'))\sigma(w^1 x')\cdot \partial_{w^2}f(\theta)\nonumber\\
&\quad+\frac{1}{N_1^{1-\gamma_1}}\ip{c\sigma'(Z^{2,N_1}(x'))\sigma'(w^1x')w^{2},{\gamma}_0^{N_1}}\cdot  \nabla_{w^1}f(\theta)x'
\label{Eq:C_def}
\end{align}

Even though we do not explore this further here, we note that the dependence of $A^{N_1}$ on $N_1$ is through averages of the form $\frac{1}{N_1}\sum_{j=1}^{N_1}\left(\cdots\right)$ and thus by Assumption \ref{assumption} and law of large numbers convergence as $N_1\rightarrow\infty$ is expected to hold. A fully rigorous justification of the latter claim is beyond the scope and purposes of this article and is left for future work.

\begin{remark}
In a snapshot the theorems that follow essentially establish that for large $N_2$ the neural network output behaves as
\begin{itemize}
\item $\gamma \in \left(\frac{1}{2}, \frac{3}{4} \right]$:
$h^{N_1,N_2}_t \approx h^{N_1}_t + \frac{1}{N^{\gamma_2-\frac{1}{2}}} K^{N_1}_t$
where $K_t$ satisfies either of equations \eqref{CLT:evolution} or \eqref{K_t for 1-gamma_2} and has a Gaussian distribution.
\item $\gamma_2 \in \left(\frac{3}{4}, \frac{5}{6} \right]$:
 $h^{N_1,N_2}_t \approx h^{N_1}_t + \frac{1}{N^{1-\gamma_2}} K^{N_1}_t + \frac{1}{N^{\gamma_2-\frac{1}{2}}} \Psi^{N_1}_t,$
where $K^{N_1}_t$ satisfies equation \eqref{K_t for 1-gamma_2} with $K^{N_1}_0(x) =0$, $\Psi^{N_1}_t$ satisfies either equations \eqref{limit_Psi} or \eqref{Psi_2_0} and has a Gaussian distribution.
\end{itemize}
where, under the appropriate assumptions, $h^{N_1}_{t}$ recovers the global minimum as $t\rightarrow\infty$. We note that, as expected this is in parallel to what one observes in the one layer case of \cite{SpiliopoulosYu2021}. However, what is potentially interesting here is that the outer layer dominates the behavior.
\end{remark}

Our first result is related to the convergence of the pair $(\gamma_t^{N_1,N_2},h_t^{N_1,N_2})$ as defined by (\ref{Eq:TimeScaledEmpiricalMeasure}) and (\ref{Eq:TimeScaledNNoutput}) as $N_2\rightarrow\infty$. We study the convergence in the Skorokhod space  $D_E([0,T])$, where   $E = \CM(\R^{1+N_1(1+d)}) \times \R^M$, and $N_1\in\mathbb{N}$ is fixed. Here $\CM(\R^{1+N_1(1+d)})$ is the space of probability measures in $\R^{1+N_1(1+d)}$.

\begin{theorem}\label{LLN:theorem}
Let $T<\infty$ be given. Under Assumption \ref{assumption}, for fixed $\gamma_1,\gamma_2 \in (1/2,1)$ and learning rates chosen via  (\ref{potential_alpha_1}), we get that as $N_2 \to \infty$, the process $(\gamma_t^{N_1,N_2},h_t^{N_1,N_2})$ converges in probability in the space $D_E([0,T])$ to $(\gamma_t^{N_1},h_t^{N_1})$, which for $t\in[0,T]$, satisfies the evolution equation
\begin{equation}\label{LLN:limit_evolution}
\begin{aligned}
h^{N_1}_t(x) &= h^{N_1}_0(x) +  \int^t_0 \int_{\CX \times \CY} \paren{y-h^{N_1}_s(x')} A^{N_1}_{x,x'} \pi(dx',dy) ds,\\
\end{aligned}
\end{equation}
where $h^{N_1}_0(x) = 0$. In addition, we have that for any $f \in C_b^2(\R^{1+N_1(1+d)})$ and $t\in[0,T]$, $\ip{f,\gamma_t^{N_1}} = \ip{f,\gamma_0^{N_1}}$.
\end{theorem}

For some of our results we would need to further assume the following.
\begin{assumption}\label{assumption1}
\begin{enumerate}
\item The activation function $\sigma$ is smooth, non-polynomial and slowly increasing\footnote{A function $\sigma(x)$ is called slowly increasing if $\lim_{x\rightarrow\infty}\frac{\sigma(x)}{x^{a}}=0$ for every $a>0$.}.
\item The fixed dataset $(x^{(i)}, y^{(i)})_{i=1}^M$ from part (ii) of Assumption \ref{assumption} has  data points that are in distinct directions (per definition on page $192$ of \cite{yIto}).
\end{enumerate}
\end{assumption}

In a similar manner now to \cite{SpiliopoulosYu2021} and to \cite{SirignanoSpiliopoulosNN4} we get that under Assumption \ref{assumption1} and for any $N_1\in\mathbb{N}$ the matrix $A^{N_1} \in \mathbb{R}^{M \times M}$, whose elements are $ A^{N_1}_{x,x'}$ with $x, x' \in \mathcal{X}$, is positive definite. The latter immediately says that we have convergence to the global minimum
 \begin{align}
h^{N_1}_t \rightarrow \hat Y \phantom{....} \textrm{as} \phantom{....} t \rightarrow\infty.\label{Eq:ConverenceGM}
\end{align}
 where $h^{N_1}_t = (h^{N_1}_t(x^{(1)}), \ldots, h^{N_1}_t(x^{(M)}))$ and $\hat{Y} = ( y^{(1)}, \ldots,  y^{(M)} )$.

We note that with these choices of learning rates, the aforementioned convergence is true for any $N_1\in\mathbb{N}$.

Since for $\gamma_2\in(1/2,1)$ the first order limit is deterministic it makes sense to investigate the second order convergence. In particular, consider
\[K^{N_1,N_2}_t = N
_2^{\varphi} (h^{N_1,N_2}_t - h^{N_1}_t),\]
where $\varphi$ depends on the scaling parameters $\gamma_1, \gamma_2$ and will be chosen appropriately momentarily. We also denote $\eta^{N_1,N_2}_t = N_2^{\varphi}(\gamma^{N_1,N_2}_t - \gamma^{N_1}_0)$. For $f \in C_b^2(\R^{1+N_1(1+d)})$ let us also define $l^{N_1,N_2}_t(f) = \ip{f,\eta^{N_1,N_2}_t}$.

Then, we have the following results.
\begin{proposition}\label{prop::l_t}
Let Assumption \ref{assumption} hold and choose the learning rates via  (\ref{potential_alpha_1}). Then,  for fixed $\gamma_1,\gamma_2 \in (1/2,1)$ and fixed $f \in C_b^2(\R^{1+N_1(1+d)})$, if $\varphi \le 1-\gamma_2$, the process $\left\{l^{N_1,N_2}_t(f) = \ip{f,\eta^{N_1,N_2}_t}, t\in[0,T]\right\}_{N_2\in\mathbb{N}}$ converges in probability in the space $D_{\R}([0,T])$ as $N_2 \to \infty$, and
\begin{custlist}[Case]
\item If $\varphi < 1-\gamma_2$, $\ip{f,\eta^{N_1,N_2}_t} \rightarrow 0$.
\item If $\varphi = 1-\gamma_2$, $l^{N_1,N_2}_t(f) = \ip{f,\eta^{N_1,N_2}_t} \rightarrow l^{N_1}_t(f)$, where $l^{N_1}_t(f)$ is given by
\begin{equation}\label{l_t limit}
\begin{aligned}
l^{N_1}_{t}(f)&=\int_0^t \int_{\CX\times \CY} \left(y-h_s^{N_1}(x')\right) \ip{C^{N_1,f}_{x'}(\theta),{\gamma}^{N_1}_0} \pi(dx',dy)ds
\end{aligned}
\end{equation}
\end{custlist}
 \end{proposition}

\begin{theorem}\label{CLT:theorem}
Let Assumption \ref{assumption} hold and choose the learning rates via  (\ref{potential_alpha_1}). Let $\mathcal{G}^{N_1}(x)$ be the Gaussian random variable defined in (\ref{limit_gaussian}). Then,
as $N_2 \to \infty$, the sequence of processes $\{K^{N_1,N_2}_t, t\in[0,T]\}_{N_2\in\mathbb{N}}$ converges in distribution in the space $D_{\R^M}([0,T])$ to $K^{N_1}_t$, such that, depending on the values of $\gamma$ and $\phi$, we shall have
\begin{custlist}[Case]
\item When $\gamma \in \paren{\frac{1}{2}, \frac{3}{4}}$ and $\varphi \le \gamma_2 - \frac{1}{2}$, or when $\gamma_2 \in \left[\frac{3}{4}, 1\right)$ and $\varphi < 1-\gamma_2 \le \gamma_2 - \frac{1}{2}$,
\begin{equation}\label{CLT:evolution}
\begin{aligned}
K^{N_1}_t(x) &= K^{N_1}_0(x)- \int^t_0 \int_{\CX \times \CY} K^{N_1}_s(x') A^{N_1}_{x,x'} \pi(dx',dy) ds
\end{aligned}
\end{equation}
where $K^{N_1}_0(x) = 0$ if $\varphi < \gamma_2 - \frac{1}{2}$, and $K^{N_1}_0(x)=\mathcal{G}^{N_1}(x)$  if $\varphi = \gamma_2-\frac{1}{2}$.
\item When $\gamma_2 \in \left[\frac{3}{4}, 1\right)$ and $\varphi = 1-\gamma_2$,
\begin{equation}\label{K_t for 1-gamma_2}
\begin{aligned}
K^{N_1}_t(x) = K^{N_1}_0(x) &+ \int_0^t \int_{\CX\times \CY} \left(y-h^{N_1}_s(x')\right)\left[l^{N_1}_t\left(B^1_{x,x'}(\theta)\right)+\frac{1}{N_1}\sum_{j=1}^{N_1}l^{N_1}_t\left(B^{2,j}_{x,x'}(\theta)\right)\right]\pi(dx',dy)ds \\
&+ \frac{1}{N_1}\sum_{j=1}^{N_1}\int_0^t \int_{\CX\times \CY} \left(y-h^{N_1}_s(x')\right)xx'l^{N_1}_t\left(B^{3,j}_{x}(\theta)\right)\ip{B^{3,j}_{x'}(\theta),\gamma^{N_1}_0}\pi(dx',dy)ds\\
&+ \frac{1}{N_1}\sum_{j=1}^{N_1}\int_0^t \int_{\CX\times \CY} \left(y-h^{N_1}_s(x')\right)xx'\ip{B^{3,j}_{x}(\theta), \gamma^{N_1}_0}l^{N_1}_t\left(B^{3,j}_{x'}(\theta)\right)\pi(dx',dy)ds\\
&- \int^t_0 \int_{\CX \times \CY}K^{N_1}_s(x')A^{N_1}_{x,x'}\pi(dx',dy) ds
\end{aligned}
\end{equation}
where $K^{N_1}_0(x) = 0$ if $\gamma_2 \in \paren{\frac{3}{4},1}$,  $K^{N_1}_0(x) = \mathcal{G}^{N_1}(x)$ if $\gamma_2 = \frac{3}{4}$, and $l^{N_1}_t(f)$ is given by equation \eqref{l_t limit} for any $f\in C_b^2(\R^{1+N_1(1+d)})$.
\end{custlist}
\end{theorem}


Notice that when $\gamma_2 > {3}/{4}$, Theorem \ref{CLT:theorem} shows that the limit of $K^{N_1,N_2}_t$ is deterministic. This motivates us to consider the next order correction. Namely, let us define the second order fluctuations $\Psi^{N_1,N_2}_t = N_2^{\zeta - \varphi} (K^{N_1,N_2}_t - K^{N_1}_t)$ for $\gamma_2 \in \paren{{3}/{4},1}$ and for some $\zeta>\varphi$ to be determined. 

\begin{proposition}\label{prop::L_t}
Let Assumption \ref{assumption} hold and choose the learning rates via  (\ref{potential_alpha_1}). Fix $\gamma_2 \in (3/4,1)$, $\varphi = 1-\gamma_2$, and $f \in C^3_b(\R^{1+N_1(1+d)})$. Letting $\zeta \le 2\varphi$, the process $\{L_t^{N_1,N_2}(f) = N_2^{\zeta-\varphi} [l^{N_1,N_2}_t(f) - l^{N_1}_t(f)], t\in[0,T]\}_{N_2\in\mathbb{N}}$ converges in probability in the space $D_{\R}([0,T])$ as $N_2 \to \infty$, and
\begin{custlist}[Case]
\item If $\zeta < 2\varphi = 2-2\gamma_2$, $L^{N_1,N_2}_t(f) \rightarrow 0$.
\item If $\zeta = 2\varphi = 2-2\gamma_2$, $L^{N_1,N_2}_t(f)  \rightarrow L^{N_1}_t(f)$, where $L^{N_1}_t(f)$ is given by
\begin{equation}\label{limit_Lt_0}
\begin{aligned}
L^{N_1}_t(f)
&=\int_0^t \int_{\CX \times \CY} \paren{y-h^{N_1}_s(x')}  l^{N_1}_s\left(\partial_{c}f(\theta)  \sigma(Z^{2,N_1}(x'))\right) \pi(dx',dy) ds\\
&\quad  + \frac{1}{N_1^{1-\gamma_1}}\int_0^t \int_{\CX \times \CY} \left(y-h_s^{N_1}(x')\right) l^{N_1}_s\left(c\sigma'(Z^{2,N_1}(x'))\sigma(w^1x')\cdot \partial_{w^2}f(\theta) \right) \pi(dx',dy) ds\\
&\quad  + \frac{1}{N_1^{1-\gamma_1}}\int_0^t\int_{\CX\times \CY}\left(y-h_s^{N_1}(x')\right)l^{N_1}_s\left(\ip{c\sigma'(Z^{2,N_1}(x'))\sigma'(w^1x')w^{2},{\gamma}_0^{N_1}}\cdot  \nabla_{w^1}f(\theta)x'\right)\pi(dx',dy)ds\\
&\quad +\frac{1}{N_1^{1-\gamma_1}}\int_0^t\int_{\CX\times \CY}\left(y-h_s^{N_1}(x')\right)\ip{l^{N_1}_s\left(c\sigma'(Z^{2,N_1}(x'))\sigma'(w^1x')w^{2}\right) \cdot  \nabla_{w^1}f(\theta)x',{\gamma}^{N_1}_0}\pi(dx',dy)ds\\
&\quad - \int_0^t\int_{\CX\times \CY}K^{N_1}_s(x')\ip{C_{x'}^{N_1,f}(c,w),{\gamma}^{N_1}_0}\pi(dx',dy)ds
\end{aligned}
\end{equation}
\end{custlist}
\end{proposition}

\begin{theorem}\label{thm::Psi}
Let Assumption \ref{assumption} hold and choose the learning rates via  (\ref{potential_alpha_1}). Let also $\mathcal{G}^{N_1}(x)$ be the Gaussian random variable defined in (\ref{limit_gaussian}). Then, for fixed $\gamma_2 \in (3/4,1)$ and $\varphi = 1-\gamma_2$, the sequence of processes $\{\Psi^{N_1,N_2}_t, t\in[0,T]\}_{N_2\in\mathbb{N}}$ converges in distribution in the space $D_{\R^M}([0,T])$ to $\Psi^{N_1}_t$, which satisfies  the following evolution equations, depending on the values of $\gamma_2$ and $\zeta$:
\begin{custlist}[Case]
\item When $\gamma_2 \in \paren{ \frac{3}{4}, \frac{5}{6}}$ and $\zeta \le \gamma_2 - \frac{1}{2}$, or when $\gamma_2 \in \left[\frac{5}{6}, 1\right)$ and $\zeta < 2-2\gamma_2 \le \gamma_2 - \frac{1}{2}$,
\begin{equation}\label{limit_Psi}
\begin{aligned}
\Psi_t^{N_1}(x) &= \Psi_0^{N_1}(x)-\int^t_0 \int_{\CX \times \CY} \Psi_s^{N_1}(x') A^{N_1}_{x,x'} \pi(dx',dy) ds,
\end{aligned}
\end{equation}
where $\Psi^{N_1}_0(x) = 0$ if $\zeta < \gamma_2 - \frac{1}{2}$, and $\Psi^{N_1}_0(x)=\mathcal{G}^{N_1}(x)$ if $\zeta = \gamma_2-\frac{1}{2}$.
\item When $\gamma_2 \in \left[\frac{5}{6}, 1\right)$ and $\zeta = 2-2\gamma_2$,
\begin{equation}\label{Psi_2_0}
\begin{aligned}
\Psi^{N_1}_t(x)
&= \Psi^{N_1}_0(x) - \int^t_0 \int_{\CX \times \CY}\Psi^{N_1}_s(x')A^{N_1}_{x,x'}\pi(dx',dy) ds\\
&\quad +\int^t_0 \int_{\CX \times \CY} \paren{y-h^{N_1}_s(x')}\left[ L^{N_1}_{s}(B^1_{x,x'}(\theta)) + \frac{1}{N_1}\sum_{j=1}^{N_1}   L^{N_1}_{s}((B^{2,j}_{x,x'}(\theta))\right] \pi(dx',dy) ds\\
&\quad + \frac{1}{N_1}\sum_{j=1}^{N_1}\int_0^t \int_{\CX\times \CY} \left(y-h^{N_1}_s(x')\right)L^{N_1}_{s}(B^{3,j}_{x}(\theta))\ip{xx'B^{3,j}_{x'}(\theta),\gamma^{N_1}_0}\pi(dx',dy)ds\\
&\quad + \frac{1}{N_1}\sum_{j=1}^{N_1}\int_0^t \int_{\CX\times \CY} \left(y-h^{N_1}_s(x')\right)\ip{xx'B^{3,j}_{x}(\theta), \gamma^{N_1}_0}L^{N_1}_{s}(B^{3,j}_{x'}(\theta))\pi(dx',dy)ds\\
&\quad - \int_0^t \int_{\CX\times \CY} K^{N_1}_s(x')\left[l^{N_1}_s\left(B^1_{x,x'}(\theta)\right)+\frac{1}{N_1}\sum_{j=1}^{N_1}l^{N_1}_s \left(B^{2,j}_{x,x'}(\theta)\right)\right]\pi(dx',dy)ds\\
&\quad -\frac{1}{N_1} \sum_{j=1}^{N_1} \int_0^t \int_{\CX\times \CY} K^{N_1}_s(x')xx'l^{N_1}_s\left(B^{3,j}_{x}(\theta)\right)\ip{B^{3,j}_{x'}(\theta),\gamma^{N_1}_0}\pi(dx',dy)ds \\
&\quad -\frac{1}{N_1} \sum_{j=1}^{N_1} \int_0^t \int_{\CX\times \CY} K^{N_1}_s(x')xx'\ip{B^{3,j}_{x}(\theta), \gamma^{N_1}_0}l^{N_1}_s\left(B^{3,j}_{x'}(\theta)\right)\pi(dx',dy)ds.
\end{aligned}
\end{equation}
where $\Psi^{N_1}_0(x) =0$  if $\gamma_2 \in \paren{\frac{5}{6},1}$, $\Psi^{N_1}_0(x) = \mathcal{G}^{N_1}(x)$  if $\gamma_2 = \frac{5}{6}$, $K^{N_1}_s$ satisfies equation \eqref{K_t for 1-gamma_2}, and $L^{N_1}_s$ satisfies \eqref{limit_Lt_0}.
\end{custlist}
\end{theorem}

These results suggest that there is an expansion of  $\ip{f,\gamma^{N_1,N_2}_{t}}$ and $h^{N_1,N_2}_{t}$ as $N_2\rightarrow\infty$ for all  $\gamma_2 \in \left[\frac{2\nu-1}{2\nu}, \frac{2\nu+1}{2\nu+2}\right)$ with $\nu\in \mathbb{N}$. The aforementioned results obtain the leading order of such expansions when $\nu=1$ and $\nu=2$. In Appendix \ref{sec::higher order} we obtain the leading order of such asymptotic expansions for all $\nu\in\mathbb{N}$ and as a consequence for all $\gamma_2\in(1/2,1)$ using an inductive argument.

In particular, when $\gamma_2 \in \left[\frac{2\nu-1}{2\nu}, \frac{2\nu+1}{2\nu+2}\right)$,  we obtain that for any fixed $f\in C^{\infty}_b(\R^{1+N_1(1+d)})$, as $N_2\rightarrow\infty$,
\begin{equation}\label{measure_expansion}
\ip{f,\gamma^{N_1,N_2}_t} \approx \sum_{n=0}^{\nu-1} \frac{1}{N_2^{n(1-\gamma_2)}} l^{N_1}_{n,t}(f)+\text{ lower order terms in } N_2,
\end{equation}
where we have identified $l^{N_1}_{0,t}(f) = \ip{f,\gamma^{N_1}_0}$, $l^{N_1}_{1,t}(f) = l^{N_1}_t(f)$, $l^{N_1}_{2,t}(f) = L^{N_1}_t(f)$. When $\nu\ge 3$, the inductive expressions for $l^{N_1}_{n,t}(f)$ are given in (\ref{Eq:l_equation}).

As $N_2 \rightarrow\infty$ and when $\gamma_2 \in \left(\frac{2\nu-1}{2\nu}, \frac{2\nu+1}{2\nu+2}\right]$, we have the asymptotic expansion
\begin{equation}\label{network_expansion}
h^{N_1,N_2}_t(x) \approx \sum_{n=0}^{\nu-1} \frac{1}{N_2^{n(1-\gamma_2)}}Q^{N_1}_{j,t}(x) + \frac{1}{N_2^{\gamma_2-\frac{1}{2}}}Q^{N_1}_{\nu,t}(x) +\text{ lower order terms in }N_2,
\end{equation}
where $Q^{N_1}_{0,t} = h^{N_1}_t$, $Q^{N_1}_{1,t} = K^{N_1}_t$, $Q^{N_1}_{2,t} = \Psi^{N_1}_t$. For $n = 1,\ldots, \nu-1$, $Q^{N_1}_{n,t}$ satisfy the deterministic evolution equations (\ref{Eq:Qj_formula1}), (\ref{Eq:Qk_formula1}) and (\ref{Eq:Qj_formula2}). We do not show this here, but for fixed $j\in\mathbb{N}$, one can also show exactly as in \cite{SpiliopoulosYu2021} that, under Assumptions \ref{assumption} and \ref{assumption1}, $Q^{N_1}_{j,t}\rightarrow 0$ exponentially fast as $t\rightarrow\infty$.

For the sake of presentation and due to the length of the formulas we present the associated formulas on the right hand side of these expansions  (and their derivation) in Appendix \ref{sec::higher order}.

\section{Numerical studies}\label{S:Numerics}

The goal of this section is to compare the numerical performance of two and three-layer neural networks of the form (\ref{Eq:TwoLayerNN}) for different values of $\gamma_i\in[1/2,1]$. In Section \ref{S:MainResults}, we demonstrated the neural network's output statistical properties can be approximated via the limit to $\infty$ of the hidden layers of the outer layer. This analysis showed that the variance of the neural network's output is minimized when the outer layer is in the mean-field scaling ($\gamma_2=1$ in the case of (\ref{Eq:TwoLayerNN})) while the scaling of the inner layer i.e. the value of $\gamma_1$, plays a less prominent role.

In this section we demonstrate a number of numerical studies to compare test accuracy for two and three layer neural networks for different values of the normalization parameters.  Our numerical studies involve the well known MNIST \cite{MNIST} data sets. The MNIST dataset \cite{MNIST}, which includes 70,000 images of handwritten integers from 0 to 9. For the two layer network case, the learning rats satisfy (\ref{potential_alpha_1}), as suggested by our theoretical analysis. The neural networks are trained to identify the handwritten numbers using the image pixels as an input. In the MNIST dataset, each image has 784 pixels, 60,000 images are used as train images and 10,000 images are test images.

 We find numerically that  test accuracy of the fitted neural networks increases monotonically in $\gamma_2 \in [1/2,1]$, suggesting that the mean-field normalization $1/N_{2}$ for the outer layer that corresponds to $\gamma_2=1$, has  certain advantages over scalings $1/N_{2}^{\gamma_{2}}$ for $\gamma_2\in[1/2,1)$ when it comes to test accuracy. The numerical studies in both the two and the three layer neural networks demonstrate that as long as the outer layer is scaled in the mean-field scaling, the scalings of the inner layers plays a less prominent role. With that being said, the optimal choice, as seen by these numerical studies, is to scale all layers in the mean-field scaling.

\subsection{Numerical results for the two layer case}\label{SS:TwoLayerCase}

In this subsection we fit the model (\ref{Eq:TwoLayerNN}) to the MNIST dataset and we compare the effect of different values of $\gamma_1,\gamma_2$.

In Figure \ref{Fig:mnist_ce_gII_h100_e1000_b20_test}, we fix in each sub-figure the value of $\gamma_2$ and plot test accuracy curves with respect to values of $\gamma_1$. We find that for each $\gamma_2$, after an initial phase,  the behavior is monotonic with respect to $\gamma_1$. We also find that the best behavior is when $\gamma_1=\gamma_2=1$ with the neural network's behavior being more sensitive on the choice of the value for $\gamma_2$.

\begin{figure}[H]
  \centering
  \begin{subfigure}[b]{0.45\linewidth}
    \includegraphics[width=\linewidth]{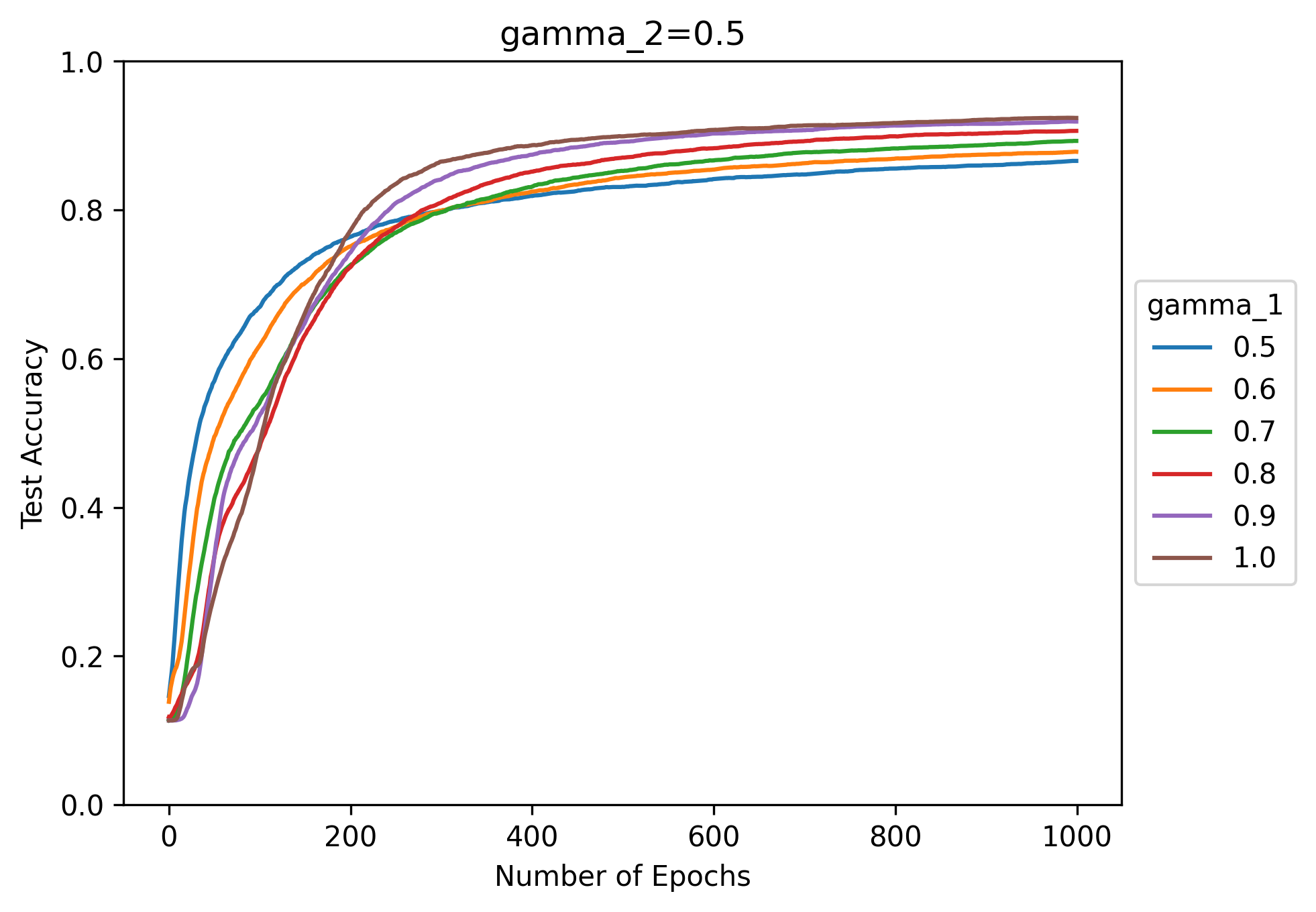}
  \end{subfigure}
  \begin{subfigure}[b]{0.45\linewidth}
    \includegraphics[width=\linewidth]{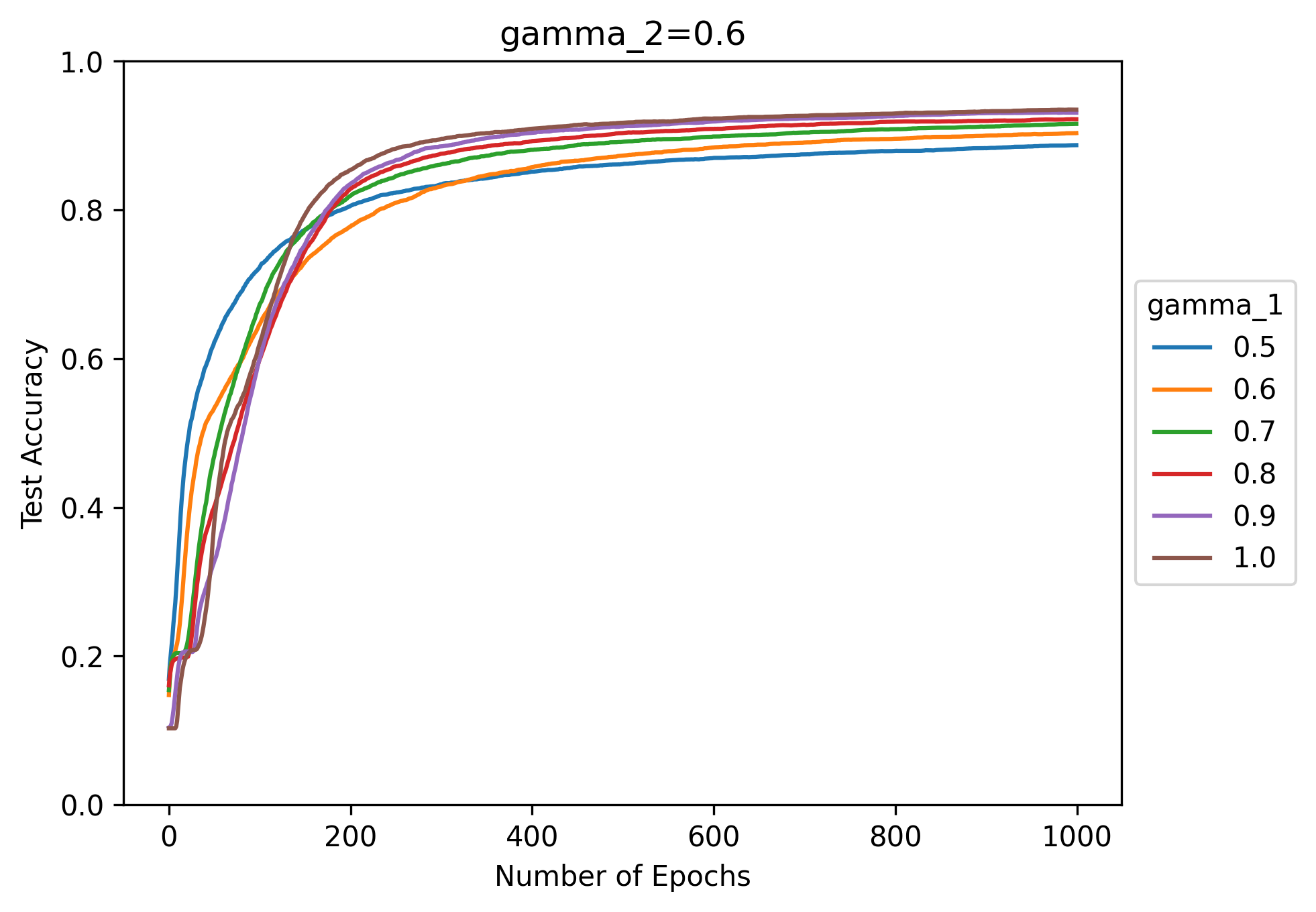}
  \end{subfigure}
  \begin{subfigure}[b]{0.45\linewidth}
    \includegraphics[width=\linewidth]{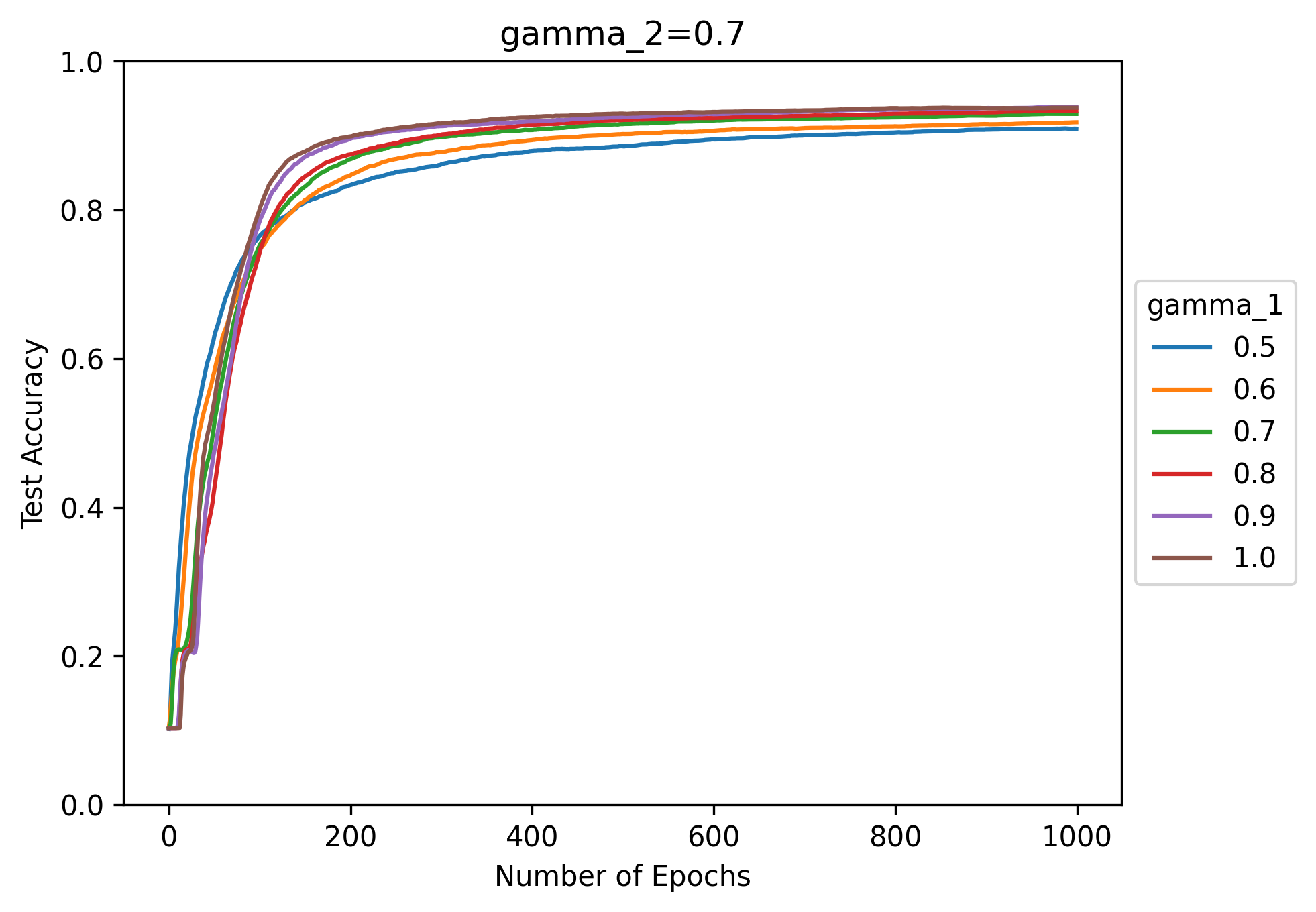}
  \end{subfigure}
  \begin{subfigure}[b]{0.45\linewidth}
    \includegraphics[width=\linewidth]{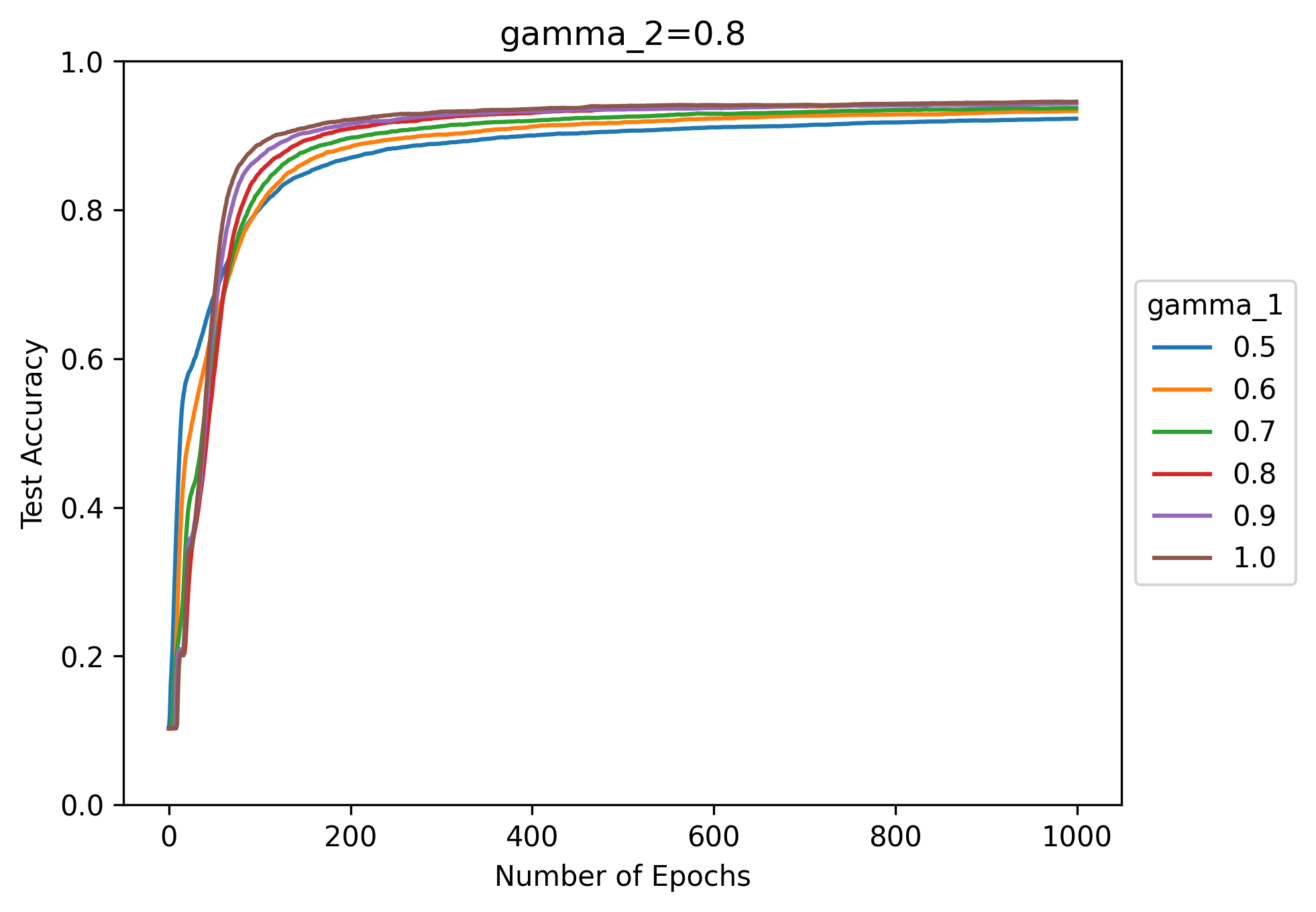}
  \end{subfigure}
    \begin{subfigure}[b]{0.45\linewidth}
    \includegraphics[width=\linewidth]{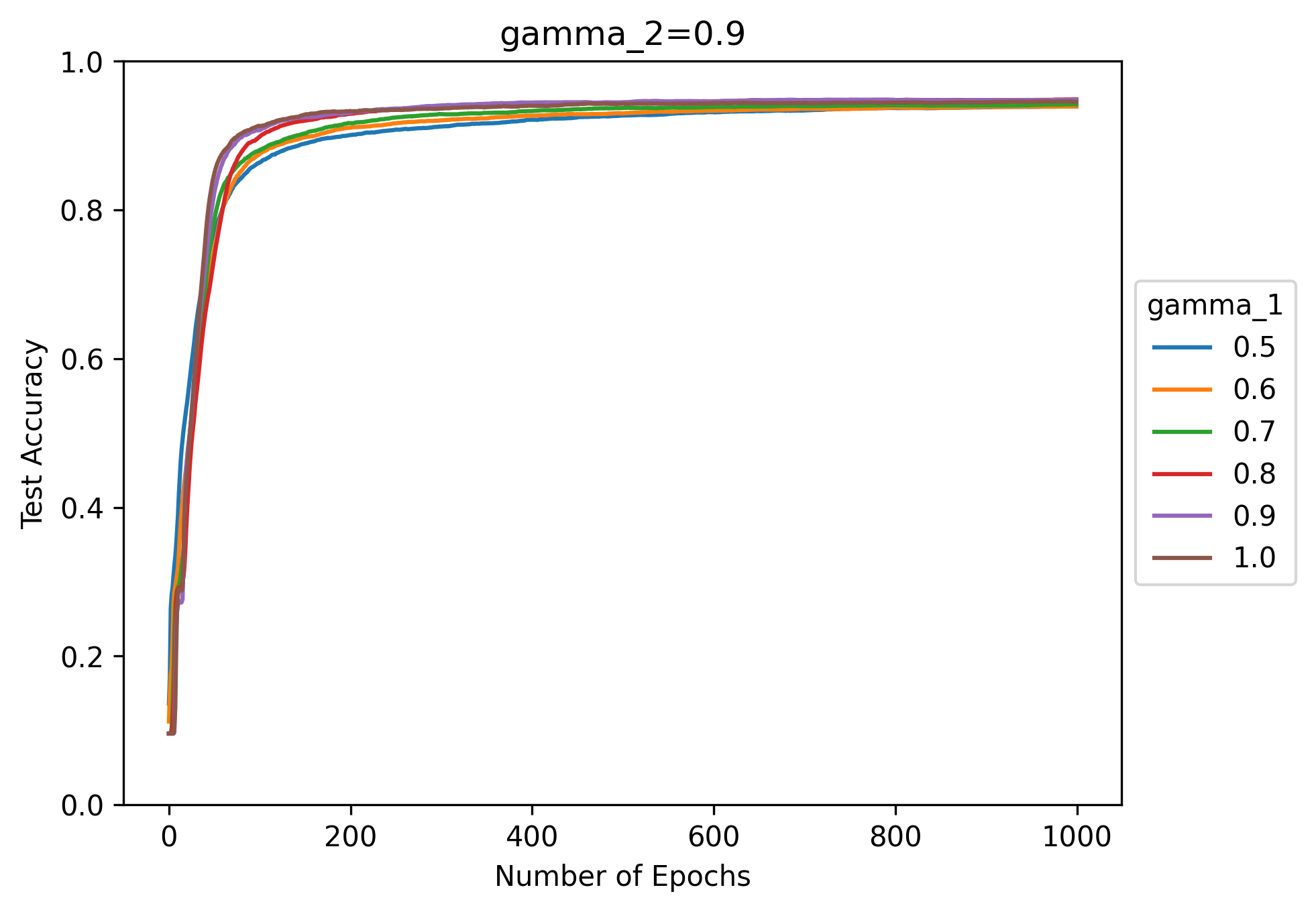}
  \end{subfigure}
    \begin{subfigure}[b]{0.45\linewidth}
    \includegraphics[width=\linewidth]{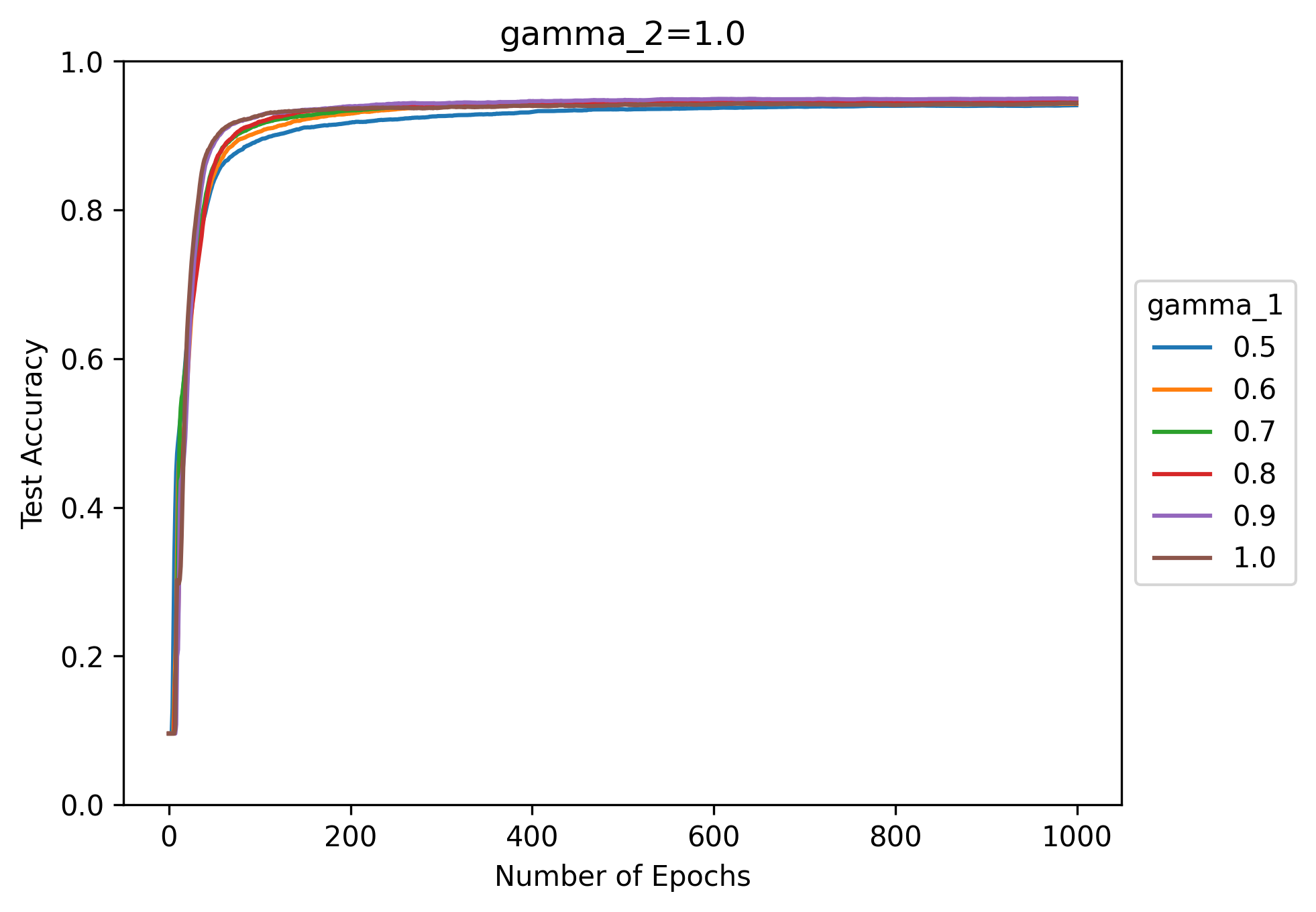}
  \end{subfigure}
  \caption{Performance of scaled neural networks on MNIST test dataset: cross entropy loss, $N_1=N_2=100$, batch size $=20$, Number of Epoch $=1000$. Each subfigure plots various $\gamma_1$ for a fixed $\gamma_2$.}
  \label{Fig:mnist_ce_gII_h100_e1000_b20_test}
  \end{figure}

In Figure \ref{Fig:mnist_ce_gI_h100_e1000_b20_test}, we fix in each sub-figure the value of $\gamma_1$ and plot test accuracy curves with respect to values of $\gamma_2$. We find that for each $\gamma_1$ the test accuracy is clearly monotonic with respect to $\gamma_2$. Independently of the value of $\gamma_1$, the best test accuracy is obtained when $\gamma_2=1$.
  \begin{figure}[H]
  \centering
  \begin{subfigure}[b]{0.45\linewidth}
    \includegraphics[width=\linewidth]{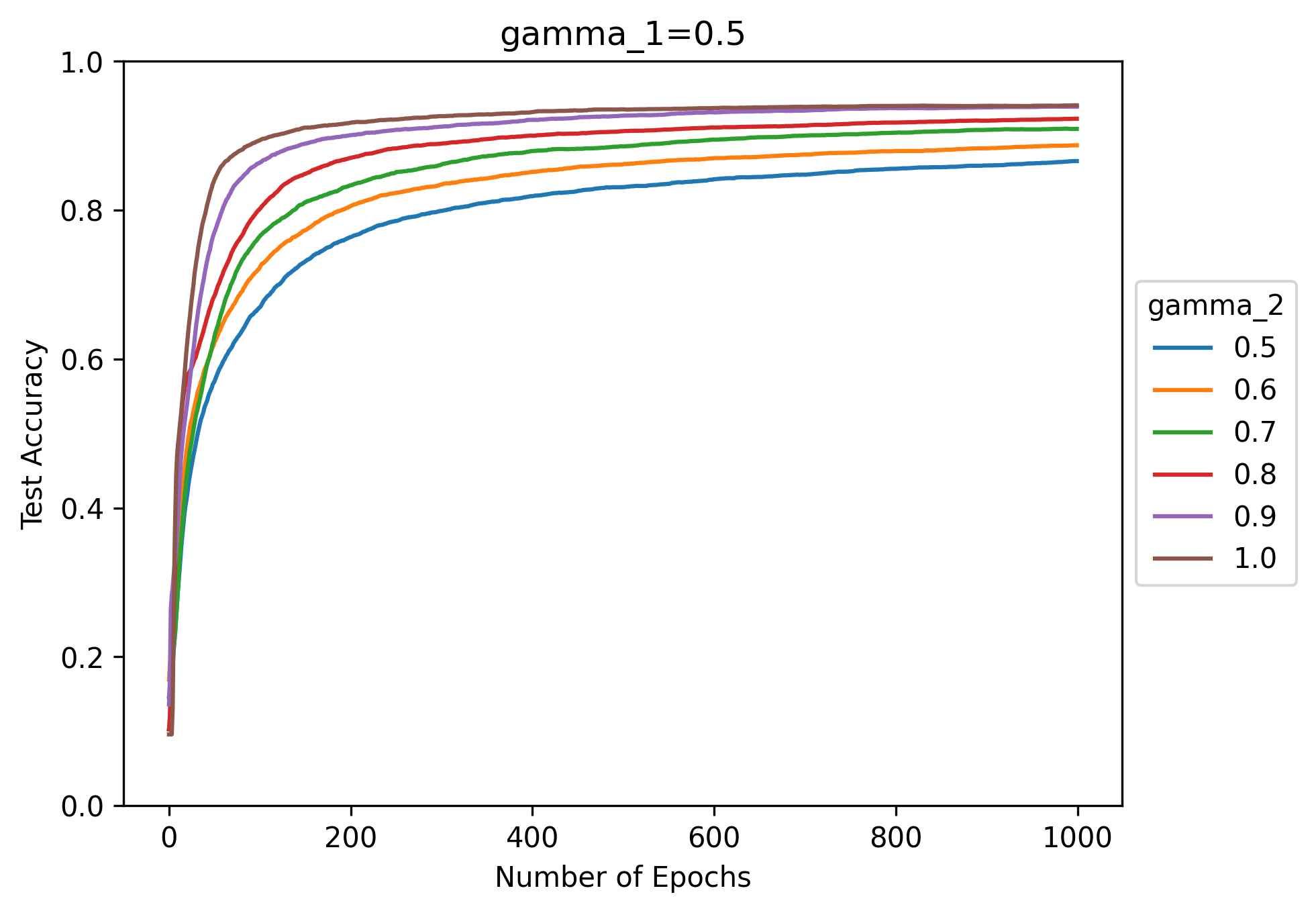}
  \end{subfigure}
  \begin{subfigure}[b]{0.45\linewidth}
    \includegraphics[width=\linewidth]{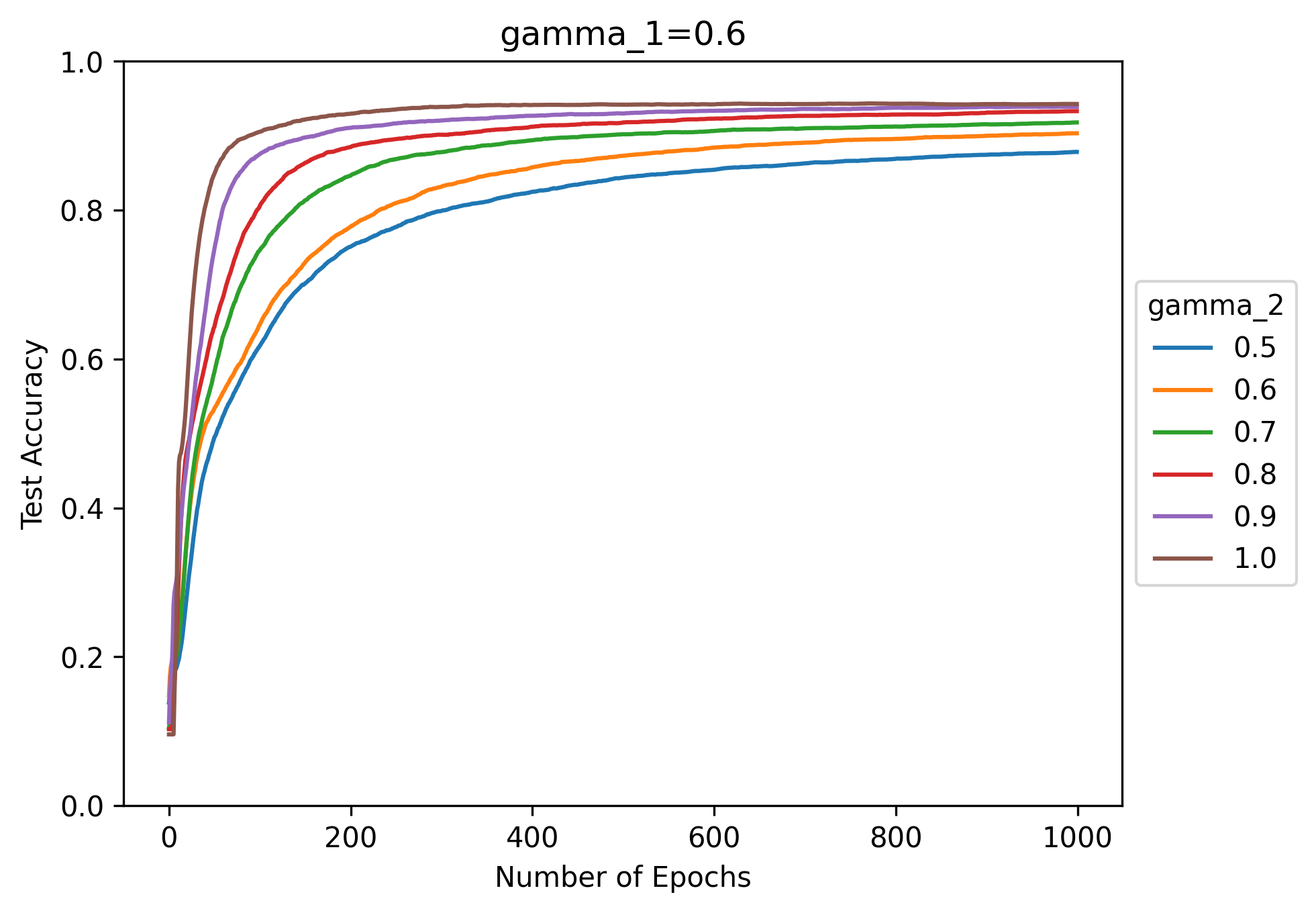}
  \end{subfigure}
  \begin{subfigure}[b]{0.45\linewidth}
    \includegraphics[width=\linewidth]{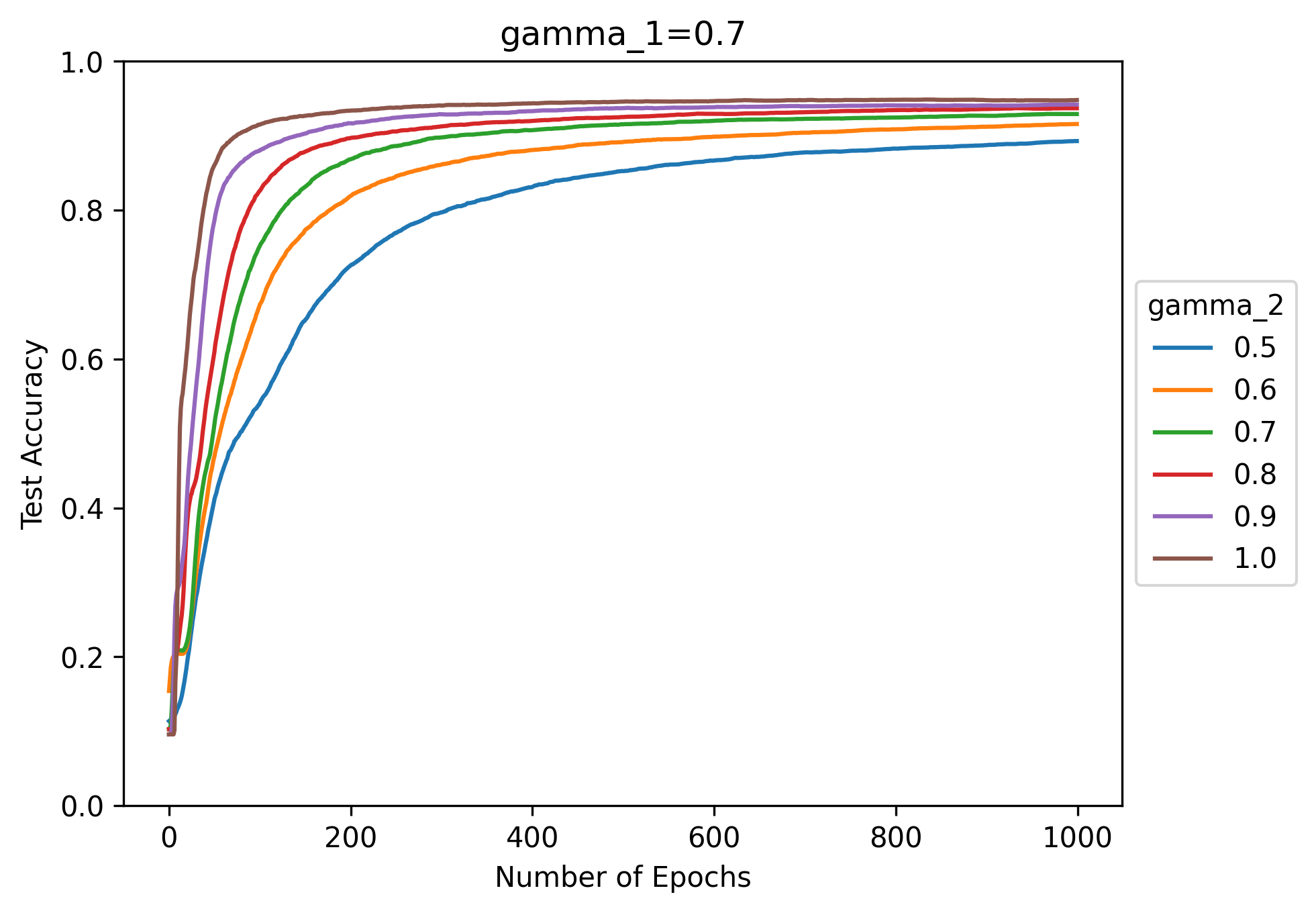}
  \end{subfigure}
  \begin{subfigure}[b]{0.45\linewidth}
    \includegraphics[width=\linewidth]{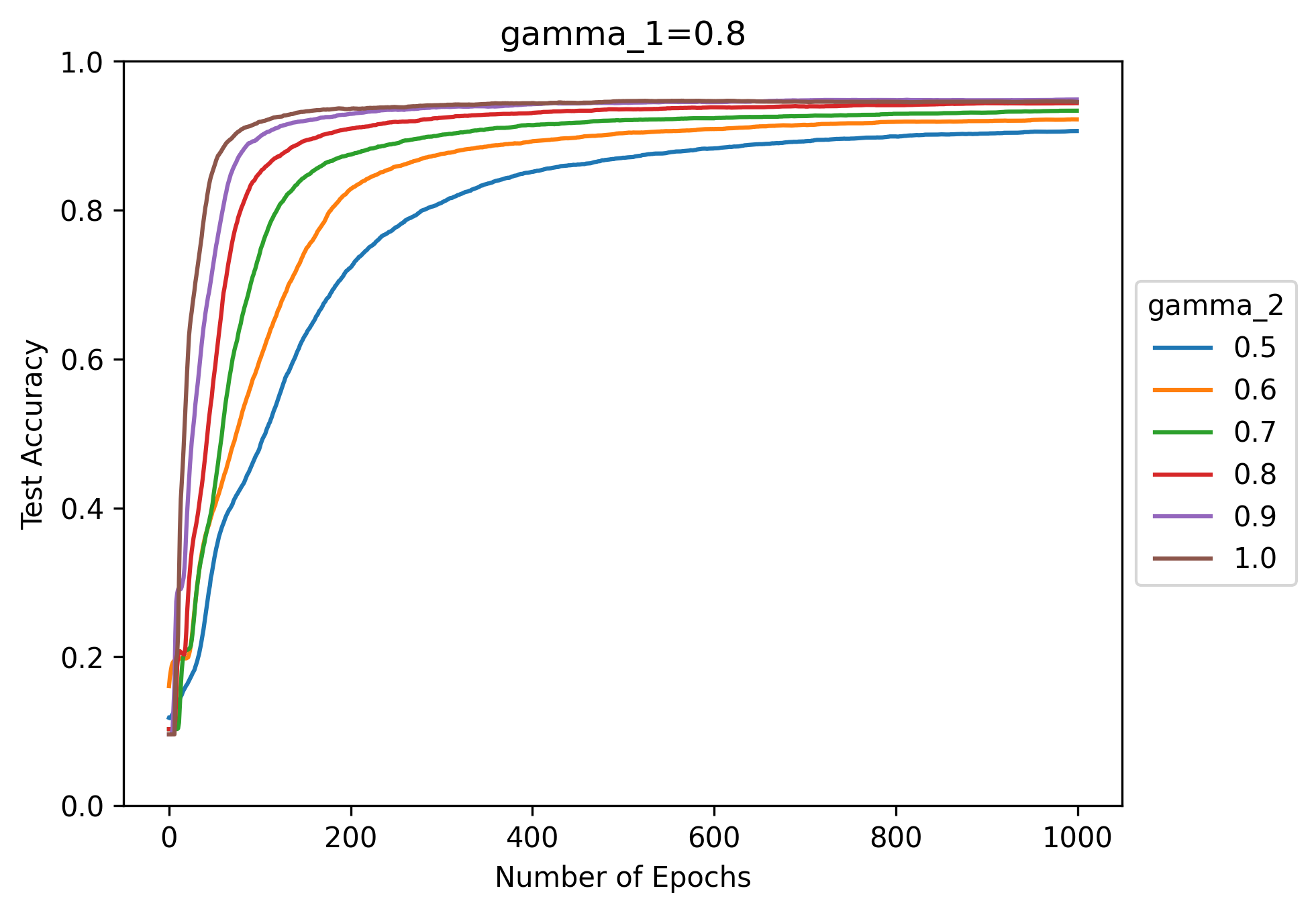}
  \end{subfigure}
    \begin{subfigure}[b]{0.45\linewidth}
    \includegraphics[width=\linewidth]{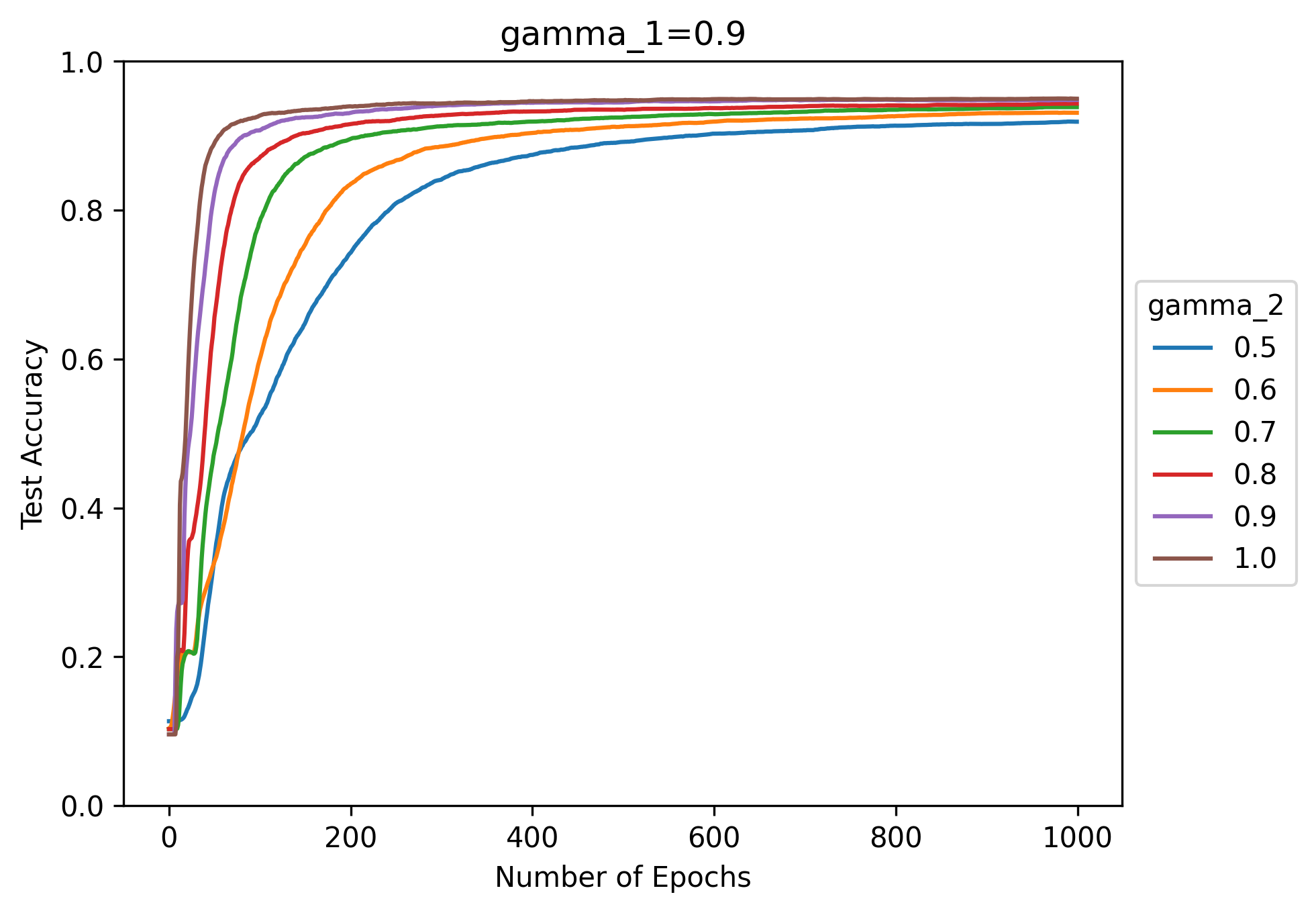}
  \end{subfigure}
    \begin{subfigure}[b]{0.45\linewidth}
    \includegraphics[width=\linewidth]{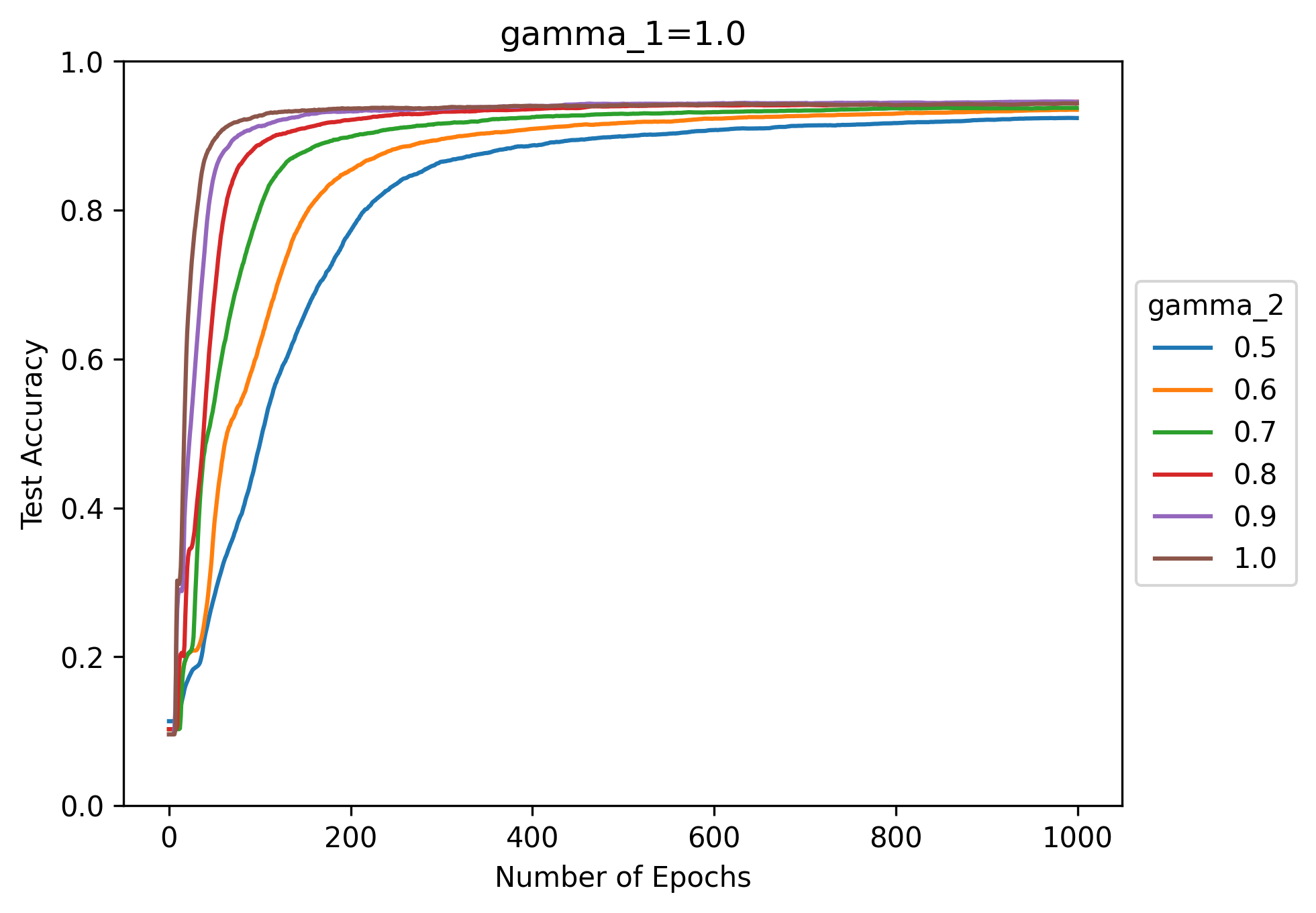}
  \end{subfigure}
  \caption{Performance of scaled neural networks on MNIST test dataset: cross entropy loss, $N_1=N_2=100$, batch size $=20$, Number of Epoch $=1000$. Each subfigure plots various $\gamma_2$ for a fixed $\gamma_1$.}
  \label{Fig:mnist_ce_gI_h100_e1000_b20_test}
\end{figure}

In Figures \ref{Fig:mnist_ce_comparingh1h2_b20_test} and \ref{Fig:mnist_ce_gII10_e1000_b20_test} we illustrate the effect of unequal choices for $N_1$ and $N_2$. We find that the best test accuracy is always when $N_2>N_1$, which also motivates taking first $N_2\rightarrow\infty$ and then $N_1\rightarrow\infty$. Also overall, best test accuracy is also when $\gamma_1=\gamma_2=1$.
\begin{figure}[H]
  \centering
  \begin{subfigure}[b]{0.45\linewidth}
    \includegraphics[width=\linewidth]{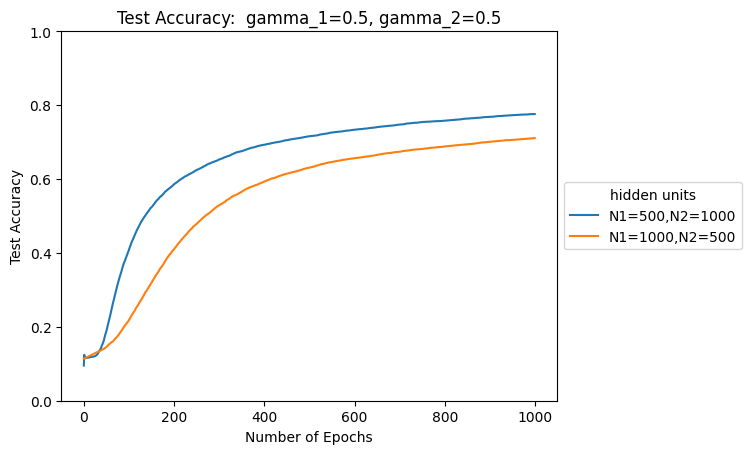}
  \end{subfigure}
  \begin{subfigure}[b]{0.45\linewidth}
    \includegraphics[width=\linewidth]{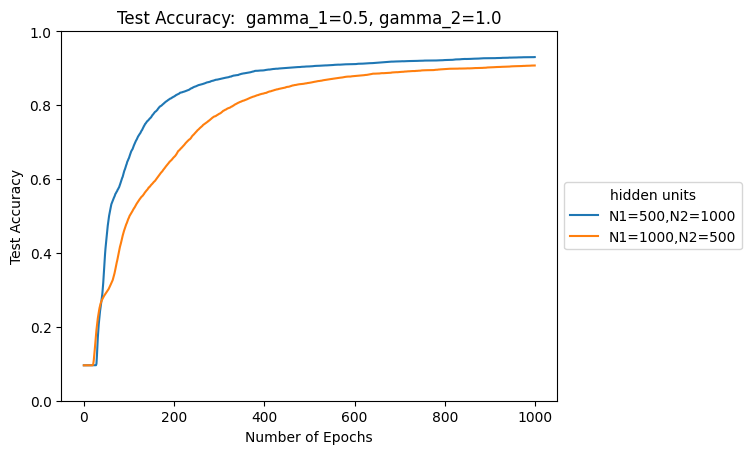}
  \end{subfigure}
  \begin{subfigure}[b]{0.45\linewidth}
    \includegraphics[width=\linewidth]{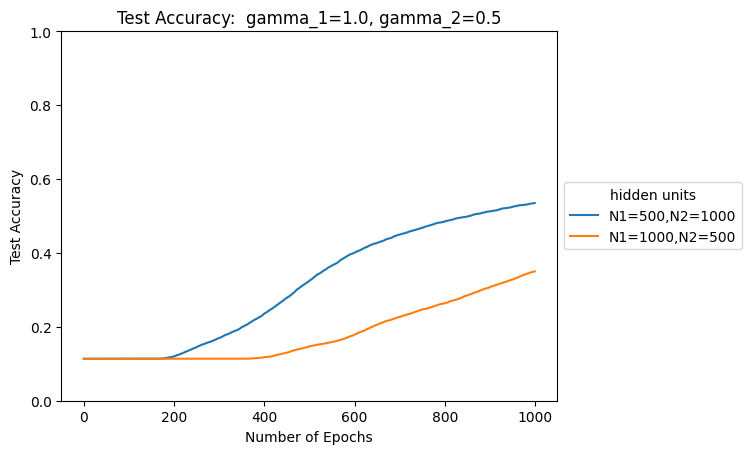}
  \end{subfigure}
  \begin{subfigure}[b]{0.45\linewidth}
    \includegraphics[width=\linewidth]{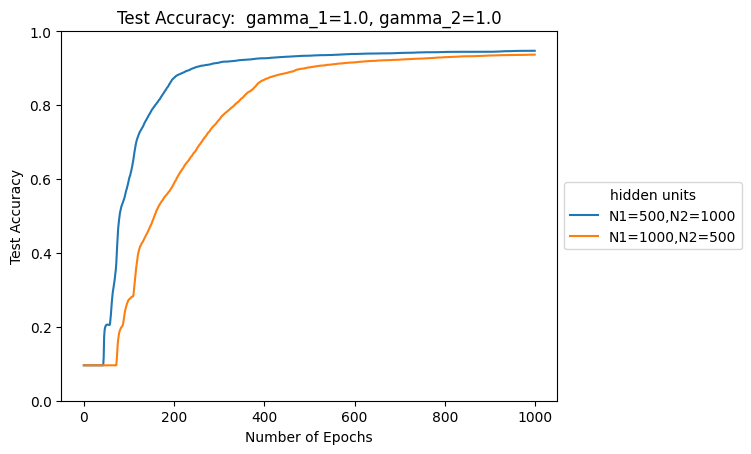}
  \end{subfigure}
  \caption{Performance of scaled neural networks on MNIST test dataset: cross entropy loss, batch size $=20$, Number of Epoch $=1000$. For each fixed sets of $\gamma_1, \gamma_2$, each subfigure compares the performances of models with different $N_1, N_2$. }
  \label{Fig:mnist_ce_comparingh1h2_b20_test}
\end{figure}

\begin{figure}[H]
  \centering
  \begin{subfigure}[b]{0.45\linewidth}
    \includegraphics[width=\linewidth]{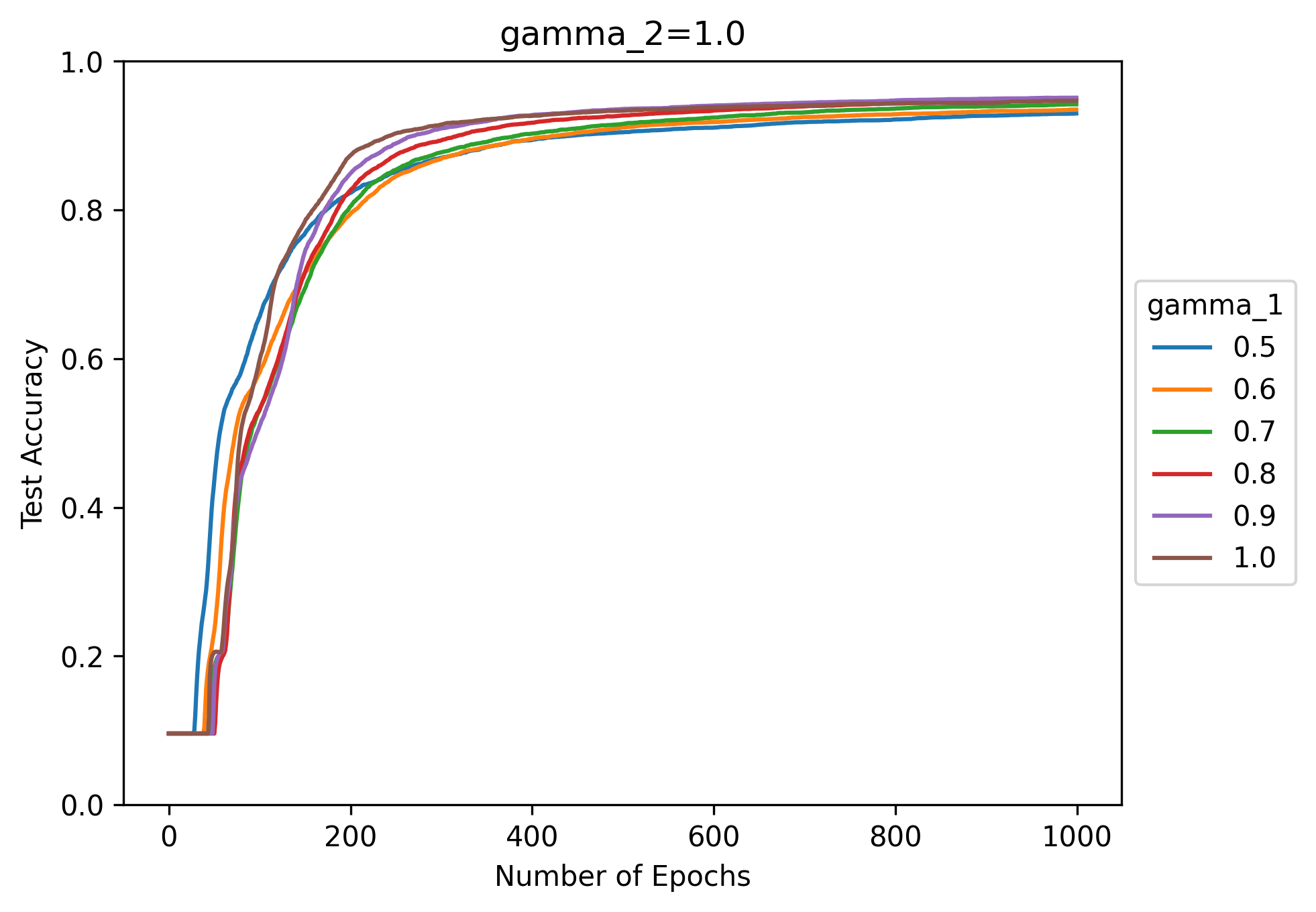}
     \caption{$N_1=500, N_2=1000$}
  \end{subfigure}
  \begin{subfigure}[b]{0.45\linewidth}
    \includegraphics[width=\linewidth]{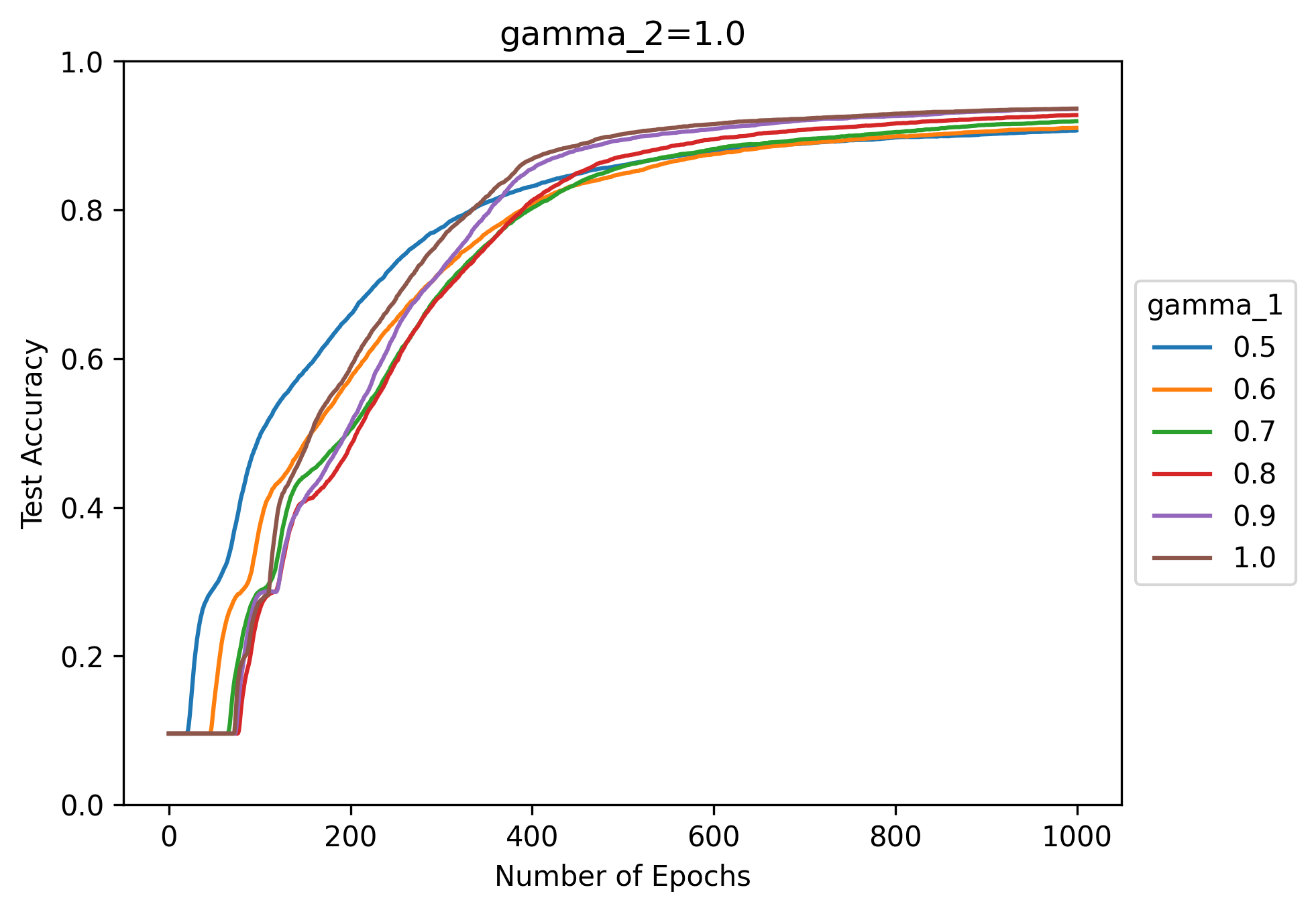}
    \caption{$N_1=1000, N_2=500$}
  \end{subfigure}
  \caption{Performance of scaled neural networks on MNIST test dataset: cross entropy loss, $\gamma_2=1.0$, batch size $=20$, Number of Epoch $=1000$. Each subfigure plots for different sets of hidden units.}
  \label{Fig:mnist_ce_gII10_e1000_b20_test}
\end{figure}

\subsection{The three layer neural network case}\label{SS:Three_layers}

The purpose of this section is to demonstrate that the same qualitative conclusions that hold for the two-layer case also hold for neural networks with more layers. For this purpose, let us consider the following three-layer scaled neural networks:
\begin{equation}\label{three_layer_nn}
g_{\theta}^{N_1,N_2,N_3}(x) = \frac{1}{N_3^{\gamma_3}} \sum_{i=1}^{N_3} C^i \sigma\left(\frac{1}{N_2^{\gamma_2}}\sum_{j=1}^{N_2}W^{3,i,j}\sigma\left(\frac{1}{N_1^{\gamma_1}}\sum_{\nu=1}^{N_1} W^{2,j,\nu}\sigma(W^{1,\nu}x)\right)\right),
\end{equation}
where $C^i, W^{3,i,j}, W^{2,j,\nu} \in \R$, $x,W^{1,\nu} \in \R^d$, and $\gamma_1, \gamma_2 \in [1/2,1)$ are fixed scaling parameters. For convenience, we write $W^{1,\nu}x = \ip{W^{1,\nu},x}_{l^2}$ as the standard $l^2$ inner product for the vectors. The neural network model has parameters
\begin{equation*}
\theta = \left(C^1, \ldots, C^{N_3}, W^{3,1,1}, \ldots, W^{3,N_3,N_2}, W^{2,1,1}, \ldots W^{2,N_2,N_1}, W^{1,1},\ldots W^{1,N_1} \right),
\end{equation*}
which are to be estimated from data $(X,Y) \sim \pi(dx,dy)$. We consider the loss function
\begin{equation*}
L(\theta) = \frac{1}{2} \E_{X,Y}\left[\left(Y-g_{\theta}^{N_1,N_2,N_3}(x)\right)^2\right],
\end{equation*}
and the model parameters $\theta$ are trained by the stochastic gradient descent algorithm, for $k \in \mathbb{N}$, $\nu=1,\ldots,N_1$, $j=1,\ldots, N_2$ and $i=1,\ldots, N_1$,
\begin{equation}\label{SGD}
\begin{aligned}
C^i_{k+1} &= C^i_k + \frac{\alpha_c^{N_1,N_2,N_3}}{N_3^{\gamma_3}} \left(y_k - g_k^{N_1,N_2,N_3}(x_k)\right)H^{3,i}_k(x_k),\\
W^{1,\nu}_{k+1} &= W^{1,\nu}_k + \frac{\alpha_{W,1}^{N_1,N_2,N_3}}{N_1^{\gamma_1}}\left(y_k - g_k^{N_1,N_2,N_3}(x_k)\right)\left(\frac{1}{N_3^{\gamma_3}}\sum_{i=1}^{N_3}C^i_k\sigma'(Z^{3,i}_k(x_k))\left(\frac{1}{N_2^{\gamma_2}}\sum_{j=1}^{N_2} W^{3,i,j}_k\sigma'(Z^{2,j}(x_k))W^{2,j,\nu}_k\right)\right)\\
&\qquad \qquad \times \sigma'(W^{1,\nu}_k x_k)x_k,\\
W^{2,i,\nu}_{k+1} &= W^{2,i,\nu}_k + \frac{\alpha_{W,2}^{N_1,N_2,N_3}}{N_1^{\gamma_1}N_2^{\gamma_2}}\left(y_k - g_k^{N_1,N_2,N_3}(x_k)\right)\frac{1}{N_3^{\gamma_3}}\sum_{i=1}^{N_3}C^i_k \sigma'(Z^{3,i}_k(x_k))W^{3,i,j}_k\sigma'(Z^{2,j}_k(x_k))H^{1,\nu}_k(x_k),\\
W^{3,i,j}_{k+1} &= W^{3,i,j}_k + \frac{\alpha_{W,3}^{N_1,N_2,N_3}}{N_2^{\gamma_2}N_3^{\gamma_3}}\left(y_k - g_k^{N_1,N_2,N_3}(x_k)\right)C^i_k \sigma'(Z^{3,i}_k(x_k))H^{2,j}_k(x_k),
\end{aligned}
\end{equation}
where
\begin{equation*}
\begin{aligned}
H^{1,\nu}_k(x) &= \sigma(W^{1,\nu}_k x),\\
Z^{2,j}_k(x) &= \frac{1}{N_1^{\gamma_1}}\sum_{\nu=1}^{N_1} W^{2,j,\nu}_k H^{1,\nu}_k(x),\\
H^{2,j}_k(x) &= \sigma(Z^{2,j}_k(x)),\\
Z^{3,i}_k(x) &= \frac{1}{N_2^{\gamma_2}} \sum_{j=1}^{N_2} W^{3,i,j}_k H^{2,j}_k(X_k),\\
H^{3,i}_k(x) &= \sigma(Z^{3,i}_k(x)),\\
g_k^{N_1,N_2,N_3}(x) &= g^{N_1,N_2,N_3}_{\theta_k}(x) = \frac{1}{N_3^{\gamma_3}}\sum_{i=1}^{N_3} C^i_k H^{3,i}_k(x).
\end{aligned}
\end{equation*}
We investigate the numerical performance of the neural network \eqref{three_layer_nn} trained by the SGD algorithm \eqref{SGD} with various $\gamma_1, \gamma_2,\gamma_3, N_1, N_2$ and $N_3$. Even though we do not show this here, following the mathematical analysis that led to the choice of the learning rates (\ref{potential_alpha_1}), we get that in the three layer case the learning rates should be given as follows
\begin{equation}\label{potential_alpha_2}
\begin{aligned}
&\alpha_C^{N_1,N_2,N_3} = \frac{1}{N_3^{2-2\gamma_3}}, &&\quad \alpha_{W,1}^{N_1,N_2,N_3} = \frac{1}{N_1^{1-2\gamma_1}N_2^{2-2\gamma_2}N_3^{3-2\gamma_3}},\\
&\alpha_{W,2}^{N_1,N_2,N_3} = \frac{1}{N_1^{1-2\gamma_1}N_2^{1-2\gamma_2}N_3^{3-2\gamma_3}}, &&\quad \alpha_{W,3}^{N_1,N_2,N_3} = \frac{1}{N_2^{1-2\gamma_2}N_3^{2-2\gamma_3}}
\end{aligned}
\end{equation}

Let us now investigate numerically the performance of neural networks scaled by $1/N_1^{\gamma_1}$, $1/N_2^{\gamma_2}$ and $1/N_3^{\gamma_3}$ with $\gamma_1, \gamma_2, \gamma_3 \in[1/2,1]$. The numerical studies are again on the MNIST data set.

In Figure \ref{Fig:mnist_ce_gIII_h100_e1500_b20_test} we fix in each sub-figure the value of $\gamma_3$ and vary the values of $\gamma_1,\gamma_2$. We find that the best results in terms of test accuracy are when $\gamma_i=1$ for all $i$. Importantly, we also find that the neural network's test accuracy is more sensitive on the choice of the outer layer normalization, i.e., on $\gamma_3$.
\begin{figure}[H]
  \centering
  \begin{subfigure}[b]{0.6\linewidth}
    \includegraphics[width=\linewidth]{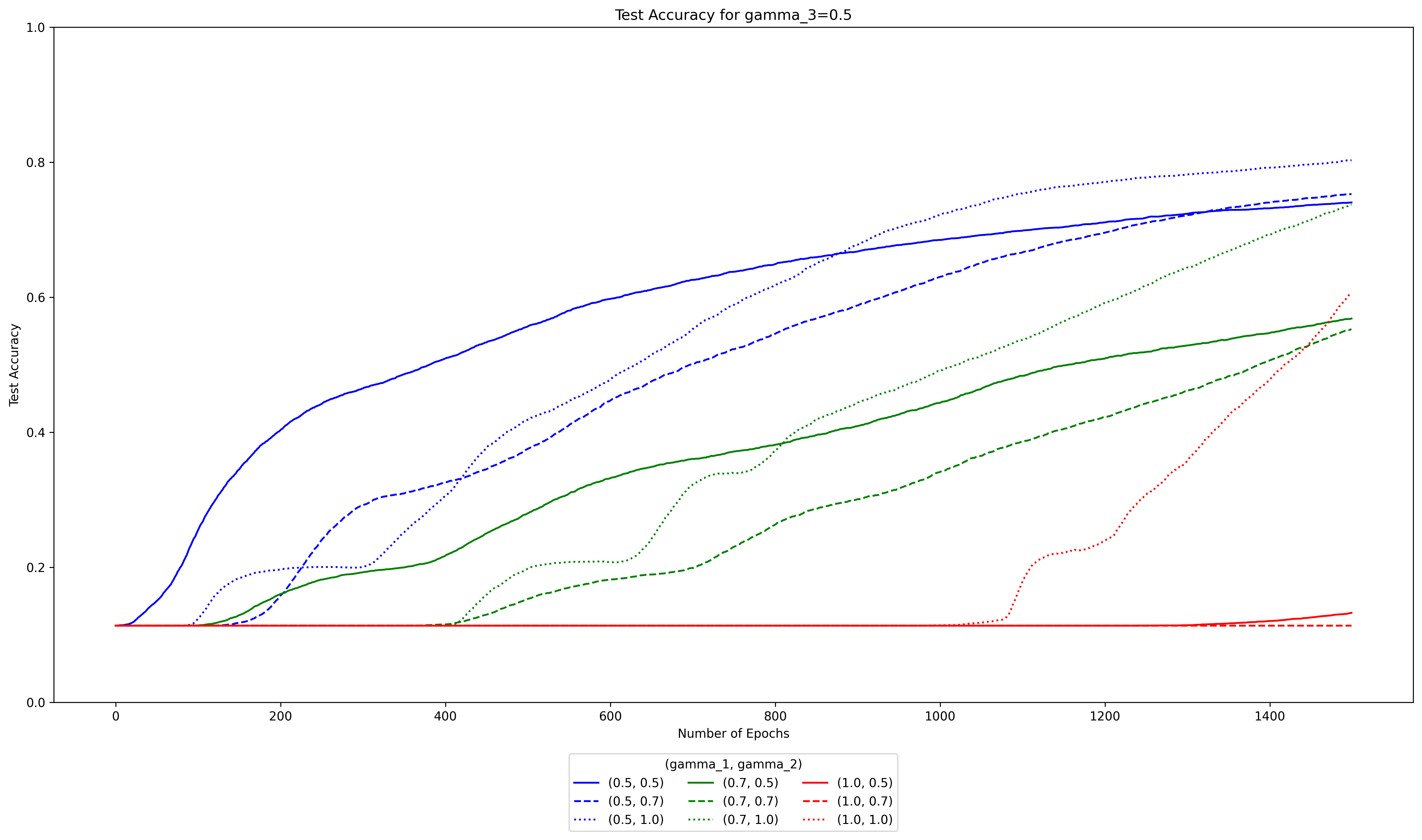}
  \end{subfigure}
  \begin{subfigure}[b]{0.6\linewidth}
    \includegraphics[width=\linewidth]{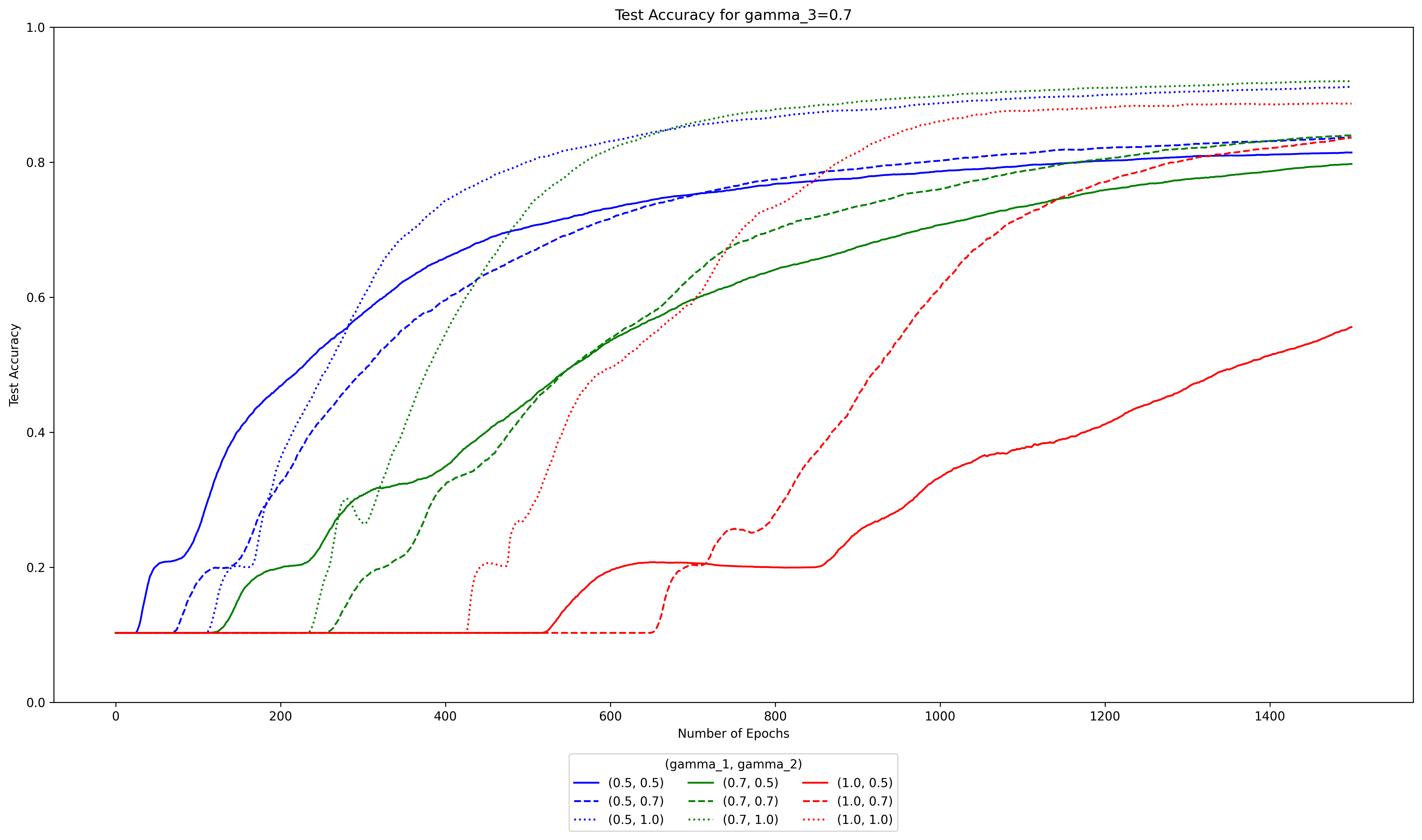}
  \end{subfigure}
  \begin{subfigure}[b]{0.6\linewidth}
    \includegraphics[width=\linewidth]{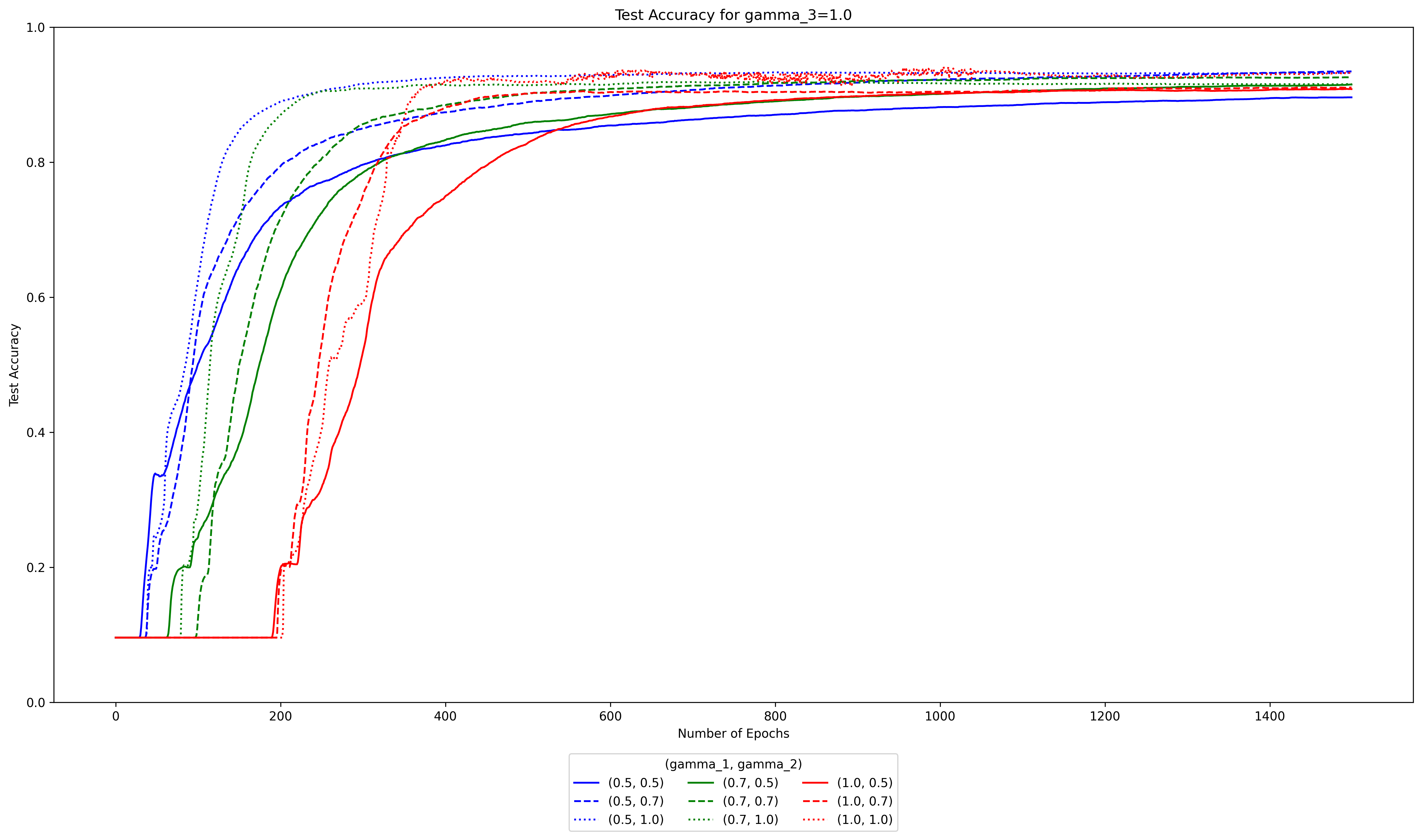}
  \end{subfigure}
  \caption{Performance of scaled neural networks on MNIST test dataset: cross entropy loss, $N_1=N_2=N_3=100$, batch size $=20$, Number of Epoch $=1500$. Each subfigure plots various $\gamma_1, \gamma_2$ for a fixed $\gamma_3$.}
  \label{Fig:mnist_ce_gIII_h100_e1500_b20_test}
\end{figure}

In Figure \ref{Fig:mnist_ce_gII_h100_e1500_b20_test} we fix in each sub-figure the value of $\gamma_2$ and vary the values of $\gamma_1,\gamma_3$. We find that the best results in terms of test accuracy are when $\gamma_i=1$ for all $i$. Again, we find that the neural network is more sensitive on the choice for $\gamma_3$.
\begin{figure}[H]
  \centering
  \begin{subfigure}[b]{0.6\linewidth}
    \includegraphics[width=\linewidth]{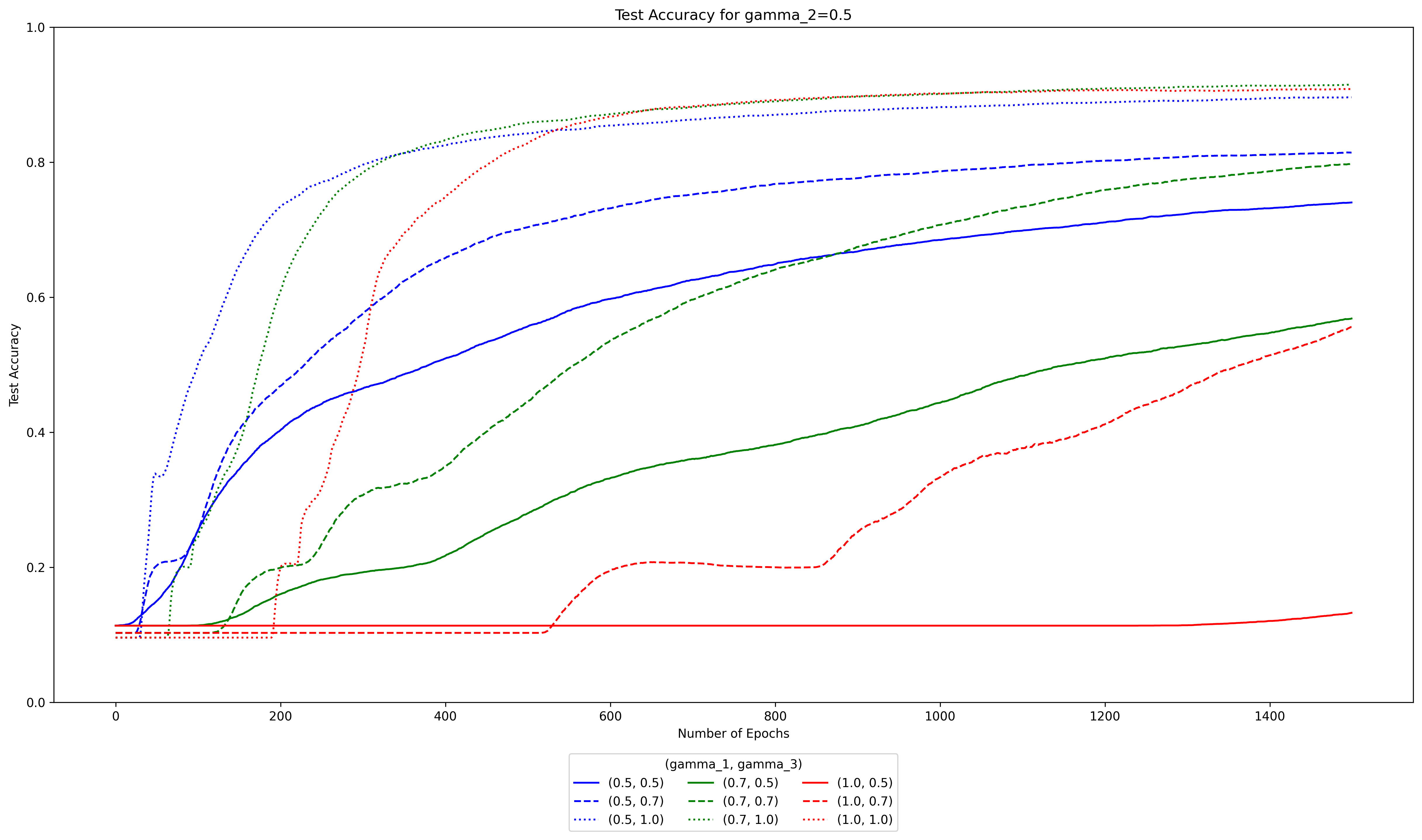}
  \end{subfigure}
  \begin{subfigure}[b]{0.6\linewidth}
    \includegraphics[width=\linewidth]{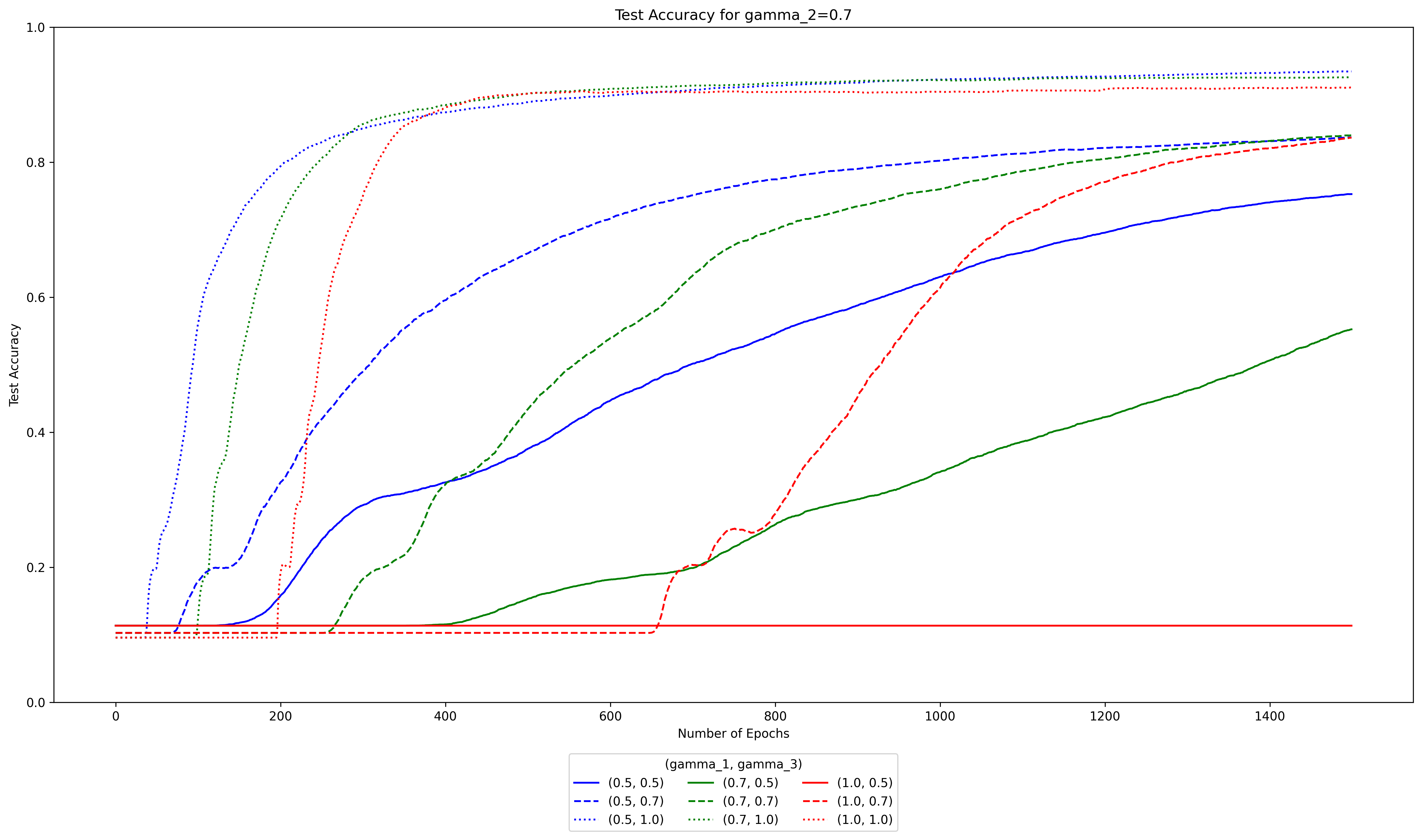}
  \end{subfigure}
  \begin{subfigure}[b]{0.6\linewidth}
    \includegraphics[width=\linewidth]{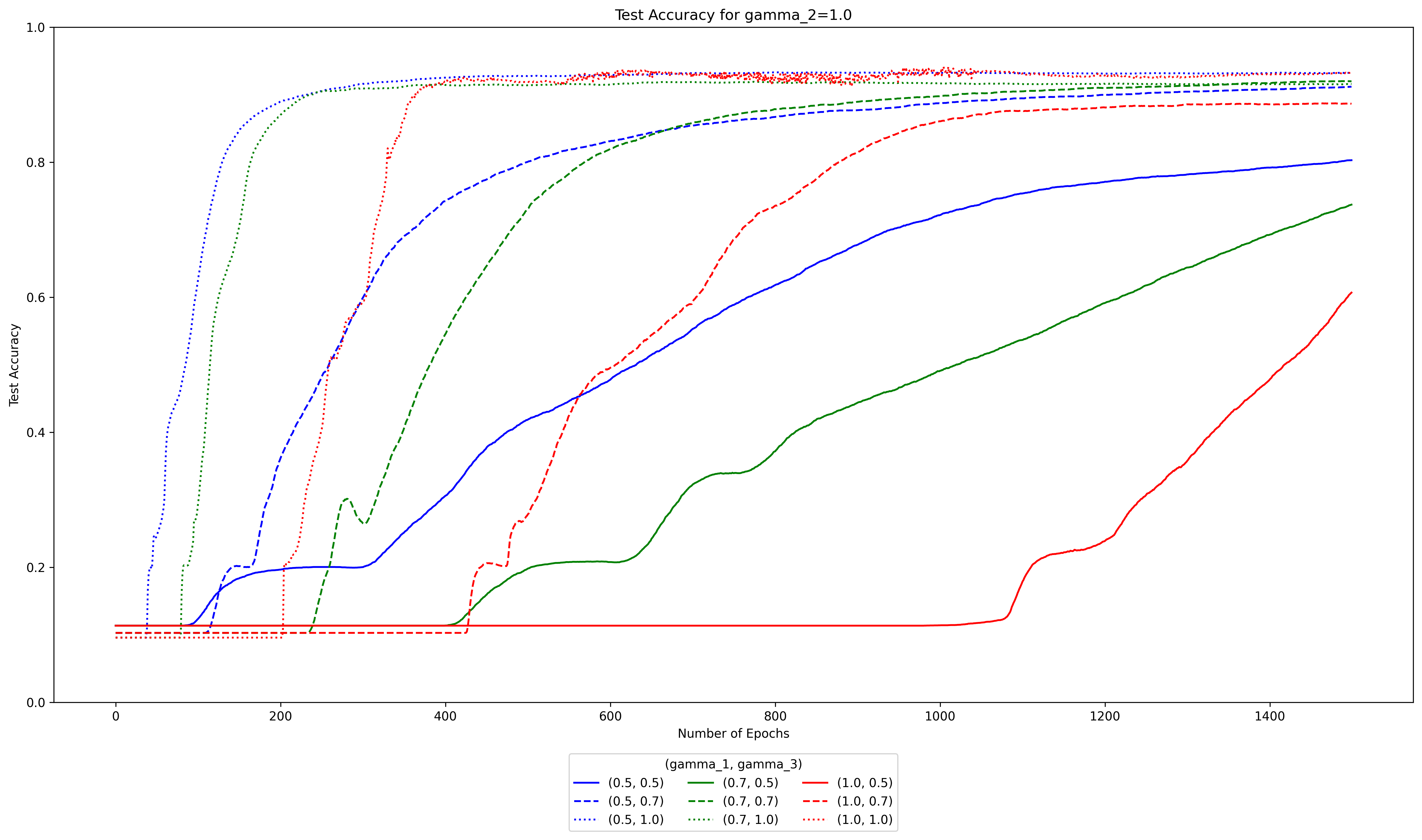}
  \end{subfigure}
  \caption{Performance of scaled neural networks on MNIST test dataset: cross entropy loss, $N_1=N_2=N_3=100$, batch size $=20$, Number of Epoch $=1500$. Each subfigure plots various $\gamma_1, \gamma_3$ for a fixed $\gamma_2$.}
  \label{Fig:mnist_ce_gII_h100_e1500_b20_test}
\end{figure}

In Figure \ref{Fig:mnist_ce_gI_h100_e1500_b20_test} we fix in each sub-figure the value of $\gamma_1$ and vary the values of $\gamma_2,\gamma_3$. The conclusions are the same as before. Namely, the best results in terms of test accuracy are when $\gamma_i=1$ for all $i$. Again, we find that if $\gamma_3=1$, then the neural network behavior is less sensitive on the choice of $\gamma_1,\gamma_2$.
\begin{figure}[H]
  \centering
  \begin{subfigure}[b]{0.6\linewidth}
    \includegraphics[width=\linewidth]{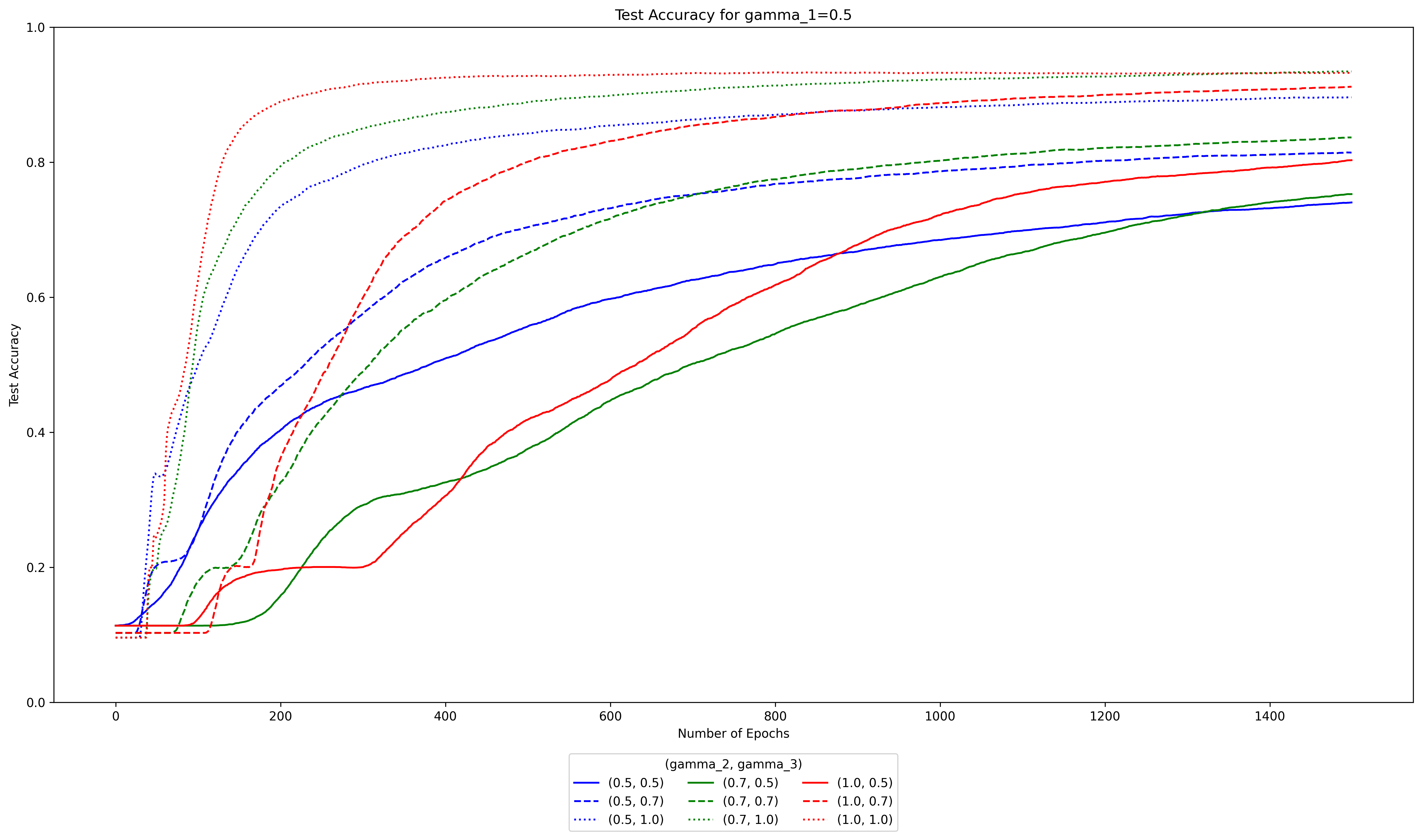}
  \end{subfigure}
  \begin{subfigure}[b]{0.6\linewidth}
    \includegraphics[width=\linewidth]{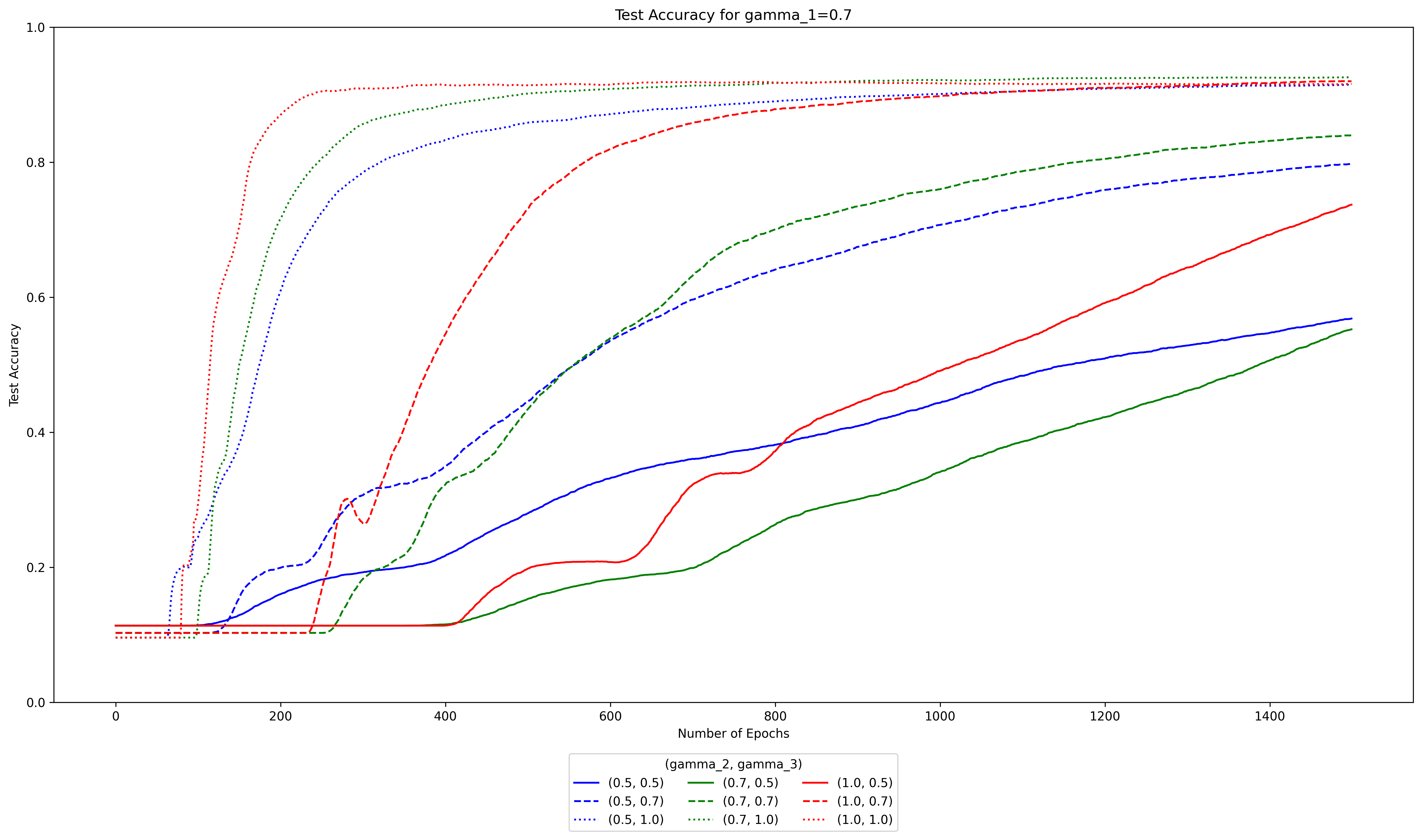}
  \end{subfigure}
  \begin{subfigure}[b]{0.6\linewidth}
    \includegraphics[width=\linewidth]{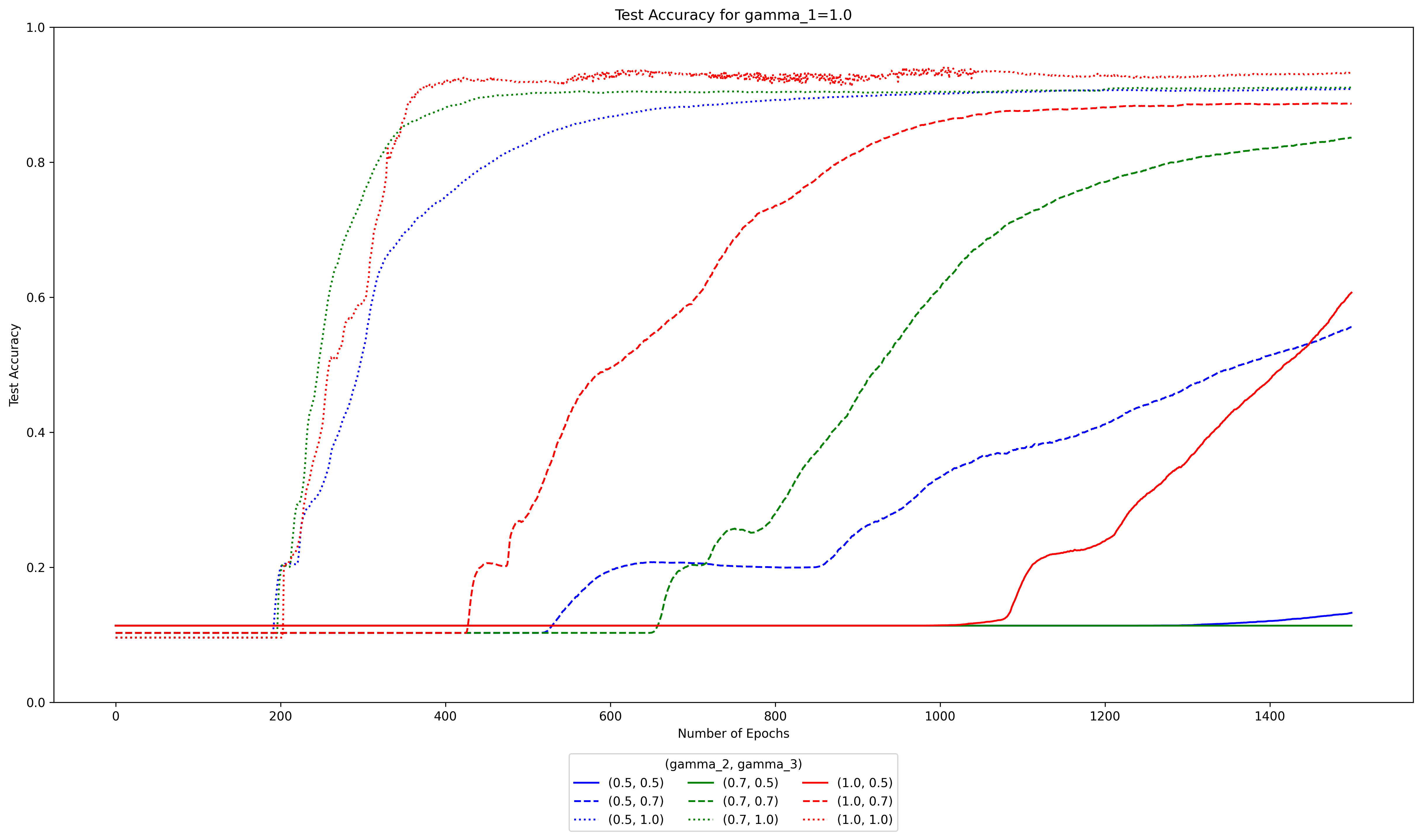}
  \end{subfigure}
  \caption{Performance of scaled neural networks on MNIST test dataset: cross entropy loss, $N_1=N_2=N_3=100$, batch size $=20$, Number of Epoch $=1500$. Each subfigure plots various $\gamma_2, \gamma_3$ for a fixed $\gamma_1$.}
  \label{Fig:mnist_ce_gI_h100_e1500_b20_test}
\end{figure}

\section{Learning rates definitions for deep neural networks of arbitrary depth}\label{S:DNN_LearningRates}

Let us consider a typical deep feed-forward neural network that has depth $m\in\mathbb{N}$ with $\hat{\gamma}=(\gamma_{1},\cdots,\gamma_{m})\in[1/2,1]^{\otimes m}$ scalings  that is defined inductively as follows
\begin{align}
g^{N_1,N_2,\cdots,N_m}_{\theta}(x)&=\frac{1}{N_{m}^{\gamma_{m}}}\sum_{i_m=1}^{N_m}W^{N_m,i_m}\sigma_{i_m}\left(g^{N_1,\cdots,N_{m-1},i_m}_{\theta}(x)\right)\nonumber\\
g^{N_1,\cdots,N_{m-j},i_{m-(j-1)}}_{\theta}(x)&=\frac{1}{N_{m-j}^{\gamma_{m-j}}}\sum_{i_{m-j}=1}^{N_{m-j}}W^{N_{m-j},i_{m-j},i_{m-(j-1)}}\sigma_{i_{m-j}}\left(g^{N_1,\cdots,N_{m-(j+1)},i_{m-j}}_{\theta}(x)\right), j=1,\cdots,m-2\nonumber\\
&\vdots\nonumber\\
g^{N_1,i_{2}}_{\theta}(x)&=\frac{1}{N_{1}^{\gamma_{1}}}\sum_{i_{1}=1}^{N_{1}}W^{N_{1},i_{1},i_{2}}\sigma_{i_{1}}\left(g^{N_0,i_{1}}_{\theta}(x)\right)\nonumber\\
g^{N_0,,i_{1}}_{\theta}(x)&=\sigma_{i_{0}}\left(W^{N_0,i_1}x\right).\nonumber
\end{align}
%
%

Even though $N_0=1$ is redundant, we write it for notational consistency purposes.

The goal of this section is to provide the formulas for the choices of the learning rates as functions of $N_i$ and $\gamma_i$ for $i=1\cdots m$ so that in the end the neural network will be expected to converge to a well defined limit as $N_i\rightarrow\infty$.

We do not repeat the lengthy calculations here, but rather we only provide the formulas for the appropriate choice of the learning rate and leave the rest of the details to the interested reader. In the end, following the exact same procedure as in the two-layer and three-layer case, we obtain that the learning rates should be chosen according to the rules:
\begin{align}
a_{W^{N_m}}&=N_{m}^{2\gamma_{m}-2}\nonumber\\
a_{W^{N_{m-1}}}&=N_{m}^{2\gamma_{m}-2}N_{m-1}^{2\gamma_{m-1}-1}\nonumber\\
a_{W^{N_{m-2}}}&=N_{m}^{2\gamma_{m}-3}N_{m-1}^{2\gamma_{m-1}-1}N_{m-2}^{2\gamma_{m-2}-1}\nonumber\\
a_{W^{N_{m-3}}}&=N_{m}^{2\gamma_{m}-3}N_{m-1}^{2\gamma_{m-1}-2}N_{m-2}^{2\gamma_{m-2}-1}N_{m-3}^{2\gamma_{m-3}-1}\nonumber\\
a_{W^{N_{m-4}}}&=N_{m}^{2\gamma_{m}-3}N_{m-1}^{2\gamma_{m-1}-2}N_{m-2}^{2\gamma_{m-2}-2}N_{m-3}^{2\gamma_{m-3}-1}N_{m-4}^{2\gamma_{m-4}-1}\nonumber\\
&\vdots\nonumber\\
a_{W^{N_{1}}}&=N_{m}^{2\gamma_{m}-3}N_{m-1}^{2\gamma_{m-1}-2}N_{m-2}^{2\gamma_{m-2}-2}N_{m-3}^{2\gamma_{m-3}-2}\cdots N_{3}^{2\gamma_{3}-2}N_{2}^{2\gamma_{2}-1}N_{1}^{2\gamma_{1}-1}\nonumber\\
a_{W^{N_{0}}}&=N_{m}^{2\gamma_{m}-3}N_{m-1}^{2\gamma_{m-1}-2}N_{m-2}^{2\gamma_{m-2}-2}N_{m-3}^{2\gamma_{m-3}-2}\cdots N_{3}^{2\gamma_{3}-2}N_{2}^{2\gamma_{2}-2}N_{1}^{2\gamma_{1}-1}.\nonumber
\end{align}

Such a choice directly generalizes the formulas for the learning rates in the two and three layer case presented before and one can show that lead to formulas of the same type as those obtained in Section \ref{S:MainResults}.

\section{Conclusions}\label{S:Conclusions}
In this work, we have investigated the effect of layer normalization on the statistical behavior and test accuracy of deep neural networks. We have looked at all the scaling regimes between the square root normalization, i.e., the so-called Xavier normalization, see \cite{Xavier}, all the way up to the mean-field normalization \cite{Chizat2018,Montanari,RVE,SirignanoSpiliopoulosNN1,SirignanoSpiliopoulosNN2,SirignanoSpiliopoulosNN3}. Our two key findings are that (a): the mean field normalization leads to lower variance of the neural's network statistical output and better test accuracy, and (b): given that the outer layer's normalization is the mean-field regime, the subsequent choice for the normalization of the inner layers does not affect test accuracy as much (mean field normalization remains the optimal choice, but there is less sensitivity in the inner layers). An important by-product of the mathematical analysis of this paper is a mathematically motivated way to define the learning rates. This is an important conclusion of our work since it gives a principled way to choose the related hyperparameters.

\appendix

\section{A-priori Bound for the Parameters}\label{S:AprioriBounds}
By specifying the learning rates $\alpha_C^{N_1,N_2}, \alpha_{W,1}^{N_1,N_2}, \alpha_{W,2}^{N_1,N_2}$ as in \eqref{potential_alpha_1}, we can establish an important uniform bound for the parameters.
\begin{lemma}\label{lemma_parameterbound}
For $k=0, 1, \ldots, \floor{TN_2}$, $i=1,\ldots,N_1$, and $j=1,\ldots,N_2$, there exist a finite constant $K>0$ such that
\begin{equation*}
\abs{C^i_k} + \norm{W^{1,j}_k} + \abs{W^{2,j,i}_k} < K.
\end{equation*}
Furthermore, as $N_1,N_2$ grow
 \begin{equation*}
 \begin{aligned}
 &\abs{C^i_{k+1} - C^i_k} = O(N_2^{-1}),\\
 &\norm{W^{1,j}_{k+1} - W^{1,j}_k} = O(N_1^{-(1-\gamma_1)}N_2^{-1}),\\
 &\abs{W^{2,j,i}_{k+1} - W^{2,j,i}_k} = O(N_1^{-(1-\gamma_1)}N_2^{-1}).
 \end{aligned}
 \end{equation*}
\end{lemma}
\begin{proof}In this proof, we use $K, K_1$ to represent unimportant constants that may change from line to line. We first establish a bound on $C^i_k$. For $k=0, 1, \ldots, \floor{TN_2}$, since $\sigma(\cdot)$ is bounded, $\abs{H^{2,i}_k} <K$, and by \eqref{SGD}, we have
\begin{equation*}
\begin{aligned}
\abs{C^i_{k+1}} &\le \abs{C^i_k} + \frac{\alpha^{N_1,N_2}_C}{N_2^{\gamma_2}}K \left[K_1 + \frac{1}{N_2^{\gamma_2}} \sum_{m=1}^{N_2}\abs{C^m_k}\right]\\
&\le \abs{C^i_k} + K \left[\frac{\alpha^{N_1,N_2}_C}{N_2^{\gamma_2}} + \frac{\alpha^{N_1,N_2}_C}{N_2^{2\gamma_2}} \sum_{m=1}^{N_2}\abs{C^m_k}\right].
\end{aligned}
\end{equation*}
Since also
\begin{equation*}\label{parabound_1}
\begin{aligned}
\abs{C^i_k} &= \abs{C^i_0} + \sum_{j=1}^k \left(\abs{C^i_j} - \abs{C^i_{j-1}}\right)\le \abs{C^i_0} + K \frac{\alpha^{N_1,N_2}_C }{N_2^{\gamma_2-1}} + K \frac{\alpha^{N_1,N_2}_C}{N_2^{2\gamma_2}}\sum_{j=1}^k \sum_{m=1}^{N_2} \abs{C^m_{j-1}},
\end{aligned}
\end{equation*}
we have
\begin{equation*}
m^{N_2}_k \le b^{N_2} + K \frac{\alpha^{N_1,N_2}_C}{N_2^{2\gamma_2-1}}\sum_{j=1}^k m^{N_2}_{j-1} = b^N_2 + K \frac{\alpha^{N_1,N_2}_C}{N_2^{2\gamma_2-1}}\sum_{j=0}^{k-1} m^{N_2}_{j},
\end{equation*}
where
\[b^{N_2} = \frac{1}{N_2}\sum_{i=1}^{N_2} \abs{C^i_0} + K \frac{\alpha^{N_1,N_2}_C N_2}{N_2^{\gamma_2}}, \quad m^{N_2}_k = \frac{1}{N_2} \sum_{i=1}^{N_2} \abs{C^i_k}.\]
By the discrete Gronwall lemma and $k\le \floor{TN_2}$,
\[m^{N_2}_k \le b^{N_2} \exp{\left(K \frac{\alpha^{N_1,N_2}_C}{N_2^{2\gamma_2-2}} \right)}.\]
Thus, since $C^i_0$ has compact support,
\[\abs{C^i_k} \le \abs{C^i_0} + K \frac{\alpha^{N_1,N_2}_C}{N_2^{\gamma_2-1}} + K \frac{\alpha^{N_1,N_2}_C}{N_2^{2\gamma_2-2}} \left[b^{N_2} \exp\left(K \frac{\alpha^{N_1,N_2}_C}{N_2^{2\gamma_2-2}}\right)\right]\]
is bounded if $\alpha^{N_1,N_2}_C \le {1}/{(N_2^{2-2\gamma_2})}$, for $\gamma \in [1/2,1)$.

Next, let's address the parameters $W^{2,j,i}_k$. By equation \eqref{SGD} and the boundedness of $\sigma(\cdot), \sigma'(\cdot), H^{1,j}_k, C^i_k,m^{N_2}_k$, we have
\begin{equation*}
\begin{aligned}
\abs{W^{2,j,i}_{k+1}} &\le \abs{W^{2,j,i}_k} + \frac{\alpha^{N_1,N_2}_{W,2}}{N_1^{\gamma_1}N_2^{\gamma_2}} \left(\abs{y_k} + \frac{K}{N_2^{\gamma_2}}\sum_{m=1}^{N_2} \abs{C^m_k} \right) \abs{C^i_k \sigma'(Z^{2,i}_k)H^{1,j}_k}\\
&\le \abs{W^{2,j,i}_k} + K\frac{\alpha^{N_1,N_2}_{W,2}}{N_1^{\gamma_1}N_2^{\gamma_2}} \left(\abs{y_k} + \frac{1}{N_2^{\gamma_2}}\sum_{m=1}^{N_2} \abs{C^m_k} \right) \\
&\le \abs{W^{2,j,i}_k} + K\frac{\alpha^{N_1,N_2}_{W,2} N_2^{1-\gamma_2}}{N_1^{\gamma_1}N_2^{\gamma_2}} \left(\frac{1}{N_2^{1-\gamma_2}}\abs{y_k} + \frac{1}{N_2}\sum_{m=1}^{N_2} \abs{C^m_k} \right) \\
&\le \abs{W^{2,j,i}_k} + K\frac{\alpha^{N_1,N_2}_{W,2}}{N_1^{\gamma_1}N_2^{2\gamma_2-1}}.
\end{aligned}
\end{equation*}
Since $k \le \floor{TN_2}$ and $W^{2,j,i}_0$ has compact support, we have
\begin{equation*}
\begin{aligned}
\abs{W^{2,j,i}_{k}} &\le \abs{W^{2,j,i}_{0}} + \sum_{m=1}^k \left( \abs{W^{2,j,i}_{m}} - \abs{W^{2,j,i}_{m-1}} \right)\\
&\le \abs{W^{2,j,i}_{0}} + \sum_{m=1}^k K\frac{\alpha^{N_1,N_2}_{W,2}}{N_1^{\gamma_1}N_2^{2\gamma_2-1}}\\
&\le K_1 +  K\frac{\alpha^{N_1,N_2}_{W,2} N_2}{N_1^{\gamma_1}N_2^{2\gamma_2-1}},
\end{aligned}
\end{equation*}
which is bounded if $\alpha^{N_1,N_2}_{W,2} \le N_1^{\gamma_1}/N_2^{2-2\gamma_2}$.

Lastly, for $W^{1,j}_k$, we have
\begin{equation*}
\begin{aligned}
\norm{W^{1,j}_{k+1}} &\le \norm{W^{1,j}_k} + \frac{\alpha^{N_1,N_2}_{W,1}}{N_1^{\gamma_1}}\abs{y_k - \frac{1}{N_2^{\gamma_2}} \sum_{i=1}^{N_2} C^i_k H^{2,i}_k (x_k)} \abs{\frac{1}{N_2^{\gamma_2}} \sum_{i=1}^{N_2} C^i_k \sigma'(Z^{2,i}_k)W^{2,j,i}_k}\norm{\sigma'(W^{1,j}_k x_k)x_k}\\
&\le \norm{W^{1,j}_k} + K \frac{\alpha^{N_1,N_2}_{W,1} N_2^{2-2\gamma_2}}{N_1^{\gamma_1}}\left(\frac{1}{N_2^{1-\gamma_2}} \abs{y_k} + \frac{1}{N_2} \sum_{i=1}^{N_2} \abs{C_k}\right) \left(\frac{1}{N_2} \sum_{i=1}^{N_2}\abs{C^i_k}\right)\\
&\le \norm{W^{1,j}_k} + K \frac{\alpha^{N_1,N_2}_{W,1} N_2^{2-2\gamma_2}}{N_1^{\gamma_1}}.
\end{aligned}
\end{equation*}
Hence, for $k \le \floor{TN_2}$,
\begin{equation*}
\begin{aligned}
\norm{W^{1,j}_{k}} &\le \norm{W^{1,j}_{0}}+\sum_{m=1}^k \left(\norm{W^{1,j}_{m}}-\norm{W^{1,j}_{m-1}}\right)\\
&\le \norm{W^{1,j}_{0}}+\sum_{m=1}^k K \frac{\alpha^{N_1,N_2}_{W,1} N_2^{2-2\gamma_2}}{N_1^{\gamma_1}}\\
&\le \norm{W^{1,j}_{0}}+ K \frac{\alpha^{N_1,N_2}_{W,1} N_2^{3-2\gamma_2}}{N_1^{\gamma_1}},
\end{aligned}
\end{equation*}
which is bounded since $W^{1,j}_0$ has compact support and if $\alpha^{N_1,N_2}_{W,1} \le N_1^{\gamma_1}/N_2^{3-2\gamma_2}$.

Collecting our results, for all $k \le \floor{TN_2}$ and $i=1,\ldots,N_2$, we have the desired uniform bound for the parameters.
\end{proof}

%

\section{Proof of Theorem \ref{LLN:theorem}}\label{sec::network_convergence}
\subsection{Evolution of the Pre-limit Process}\label{sec::prelimit_evolution}
We first analyze the evolution of the network output $g_k^{N_1,N_2}(x)$. Using Taylor expansion, we have
\begin{equation*}
\begin{aligned}
g_{k+1}^{N_1,N_2}(x) - g_k^{N_1,N_2}(x) &= \frac{1}{N_2^{\gamma_2}} \sum_{i=1}^{N_2} C^i_{k+1} \sigma\left(\frac{1}{N_1^{\gamma_1}}\sum_{j=1}^{N_1} W^{2,j,i}_{k+1}\sigma(W^{1,j}_{k+1}x)\right)-\frac{1}{N_2^{\gamma_2}} \sum_{i=1}^{N_2} C^i_{k} \sigma\left(\frac{1}{N_1^{\gamma_1}}\sum_{j=1}^{N_1} W^{2,j,i}_{k}\sigma(W^{1,j}_{k}x)\right)\\
&= \frac{1}{N_2^{\gamma_2}}\sum_{i=1}^{N_2} \left(C^i_{k+1}-C^i_k \right) \left[\sigma\left(\frac{1}{N_1^{\gamma_1}}\sum_{j=1}^{N_1} W^{2,j,i}_{k+1}\sigma(W^{1,j}_{k+1}x)\right)\right]\\
&\quad + \frac{1}{N_2^{\gamma_2}} \sum_{i=1}^{N_2} C^i_k \left[\sigma\left(\frac{1}{N_1^{\gamma_1}}\sum_{j=1}^{N_1} W^{2,j,i}_{k+1}\sigma(W^{1,j}_{k+1}x)\right) - \sigma\left(\frac{1}{N_1^{\gamma_1}}\sum_{j=1}^{N_1} W^{2,j,i}_{k}\sigma(W^{1,j}_{k}x)\right) \right]\\
&=\frac{1}{N_2^{\gamma_2}}\sum_{i=1}^{N_2} \left(C^i_{k+1}-C^i_k \right) \sigma\left(\frac{1}{N_1^{\gamma_1}}\sum_{j=1}^{N_1} W^{2,j,i}_{k}\sigma(W^{1,j}_{k}x)\right) \\
&\quad + \frac{1}{N_2^{\gamma_2}} \sum_{i=1}^{N_2} C^i_k \left[\sigma'\left(\frac{1}{N_1^{\gamma_1}}\sum_{j=1}^{N_1} W^{2,j,i}_{k}\sigma(W^{1,j}_{k}x)\right)\frac{1}{N_1^{\gamma_1}}\sum_{j=1}^{N_1} \sigma({W}^{1,j}_{k}x)\left(W^{2,j,i}_{k+1} - W^{2,j,i}_k \right)\right.\\
&\qquad  \left.+\sigma'\left(\frac{1}{N_1^{\gamma_1}}\sum_{j=1}^{N_1} W^{2,j,i}_{k}\sigma({W}^{1,j}_{k}x)\right)\frac{1}{N_1^{\gamma_1}}\sum_{j=1}^{N_1} {W}^{2,j,i}_k \sigma'({W}^{1,j}_{k}x)\left(W^{1,j}_{k+1} - W^{1,j}_k \right)x \right] + R^{N_1,N_2}
\end{aligned}
\end{equation*}
where $R^{N_1,N_2} = R^{N_1,N_2}_1 + R^{N_1,N_2}_2$, and
\begin{equation*}
\begin{aligned}
R^{N_1,N_2}_1&=\frac{1}{N_2^{\gamma_2}}\sum_{i=1}^{N_2} \left(C^i_{k+1}-C^i_k \right) \sigma'\left(\frac{1}{N_1^{\gamma_1}}\sum_{j=1}^{N_1} \breve{W}^{2,j,i}_{k}\sigma(\breve{W}^{1,j}_{k}x)\right)\frac{1}{N_1^{\gamma_1}}\sum_{j=1}^{N_1} \sigma(\breve{W}^{1,j}_{k}x)\left(W^{2,j,i}_{k+1} - W^{2,j,i}_k \right)\\
& \quad + \frac{1}{N_2^{\gamma_2}}\sum_{i=1}^{N_2} \left(C^i_{k+1}-C^i_k \right) \sigma'\left(\frac{1}{N_1^{\gamma_1}}\sum_{j=1}^{N_1} \breve{W}^{2,j,i}_{k}\sigma(\breve{W}^{1,j}_{k}x)\right)\frac{1}{N_1^{\gamma_1}}\sum_{j=1}^{N_1} \breve{W}^{2,j,i}_k \sigma'(\breve{W}^{1,j}_{k}x)\left(W^{1,j}_{k+1} - W^{1,j}_k \right)x,
\end{aligned}
\end{equation*}
\begin{equation*}
\begin{aligned}
R^{N_1,N_2}_2 &= \frac{1}{N_2^{\gamma_2}} \sum_{i=1}^{N_2} C^i_k \left\lbrace \frac{1}{2}\sigma''\left(\frac{1}{N_1^{\gamma_1}}\sum_{j=1}^{N_1} \tilde{W}^{2,j,i}_{k}\sigma(\tilde{W}^{1,j}_{k}x)\right)\left[\frac{1}{N_1^{\gamma_1}}\sum_{j=1}^{N_1} \sigma(\tilde{W}^{1,j}_{k}x) \left(W^{2,j,i}_{k+1} - W^{2,j,i}_k \right)\right]^2\right.\\
&\qquad  + \frac{1}{2} \sigma''\left(\frac{1}{N_1^{\gamma_1}}\sum_{j=1}^{N_1} \tilde{W}^{2,j,i}_{k}\sigma(\tilde{W}^{1,j}_{k}x)\right)\left[\frac{1}{N_1^{\gamma_1}}\sum_{j=1}^{N_1} \tilde{W}^{2,j,i}_k \sigma'(\tilde{W}^{1,j}_{k}x)\left(W^{1,j}_{k+1} - W^{1,j}_k \right)x\right]^2\\
&\qquad \qquad \qquad + \frac{1}{2}\sigma'\left(\frac{1}{N_1^{\gamma_1}}\sum_{j=1}^{N_1} \tilde{W}^{2,j,i}_{k}\sigma(\tilde{W}^{1,j}_{k}x)\right)\frac{1}{N_1^{\gamma_1}}\sum_{j=1}^{N_1} \tilde{W}^{2,j,i}_k \sigma''(\tilde{W}^{1,j}_{k}x)\left[\left(W^{1,j}_{k+1} - W^{1,j}_k \right)x\right]^2  \\
&\qquad  + \sigma''\left(\frac{1}{N_1^{\gamma_1}}\sum_{j=1}^{N_1} \tilde{W}^{2,j,i}_{k}\sigma(\tilde{W}^{1,j}_{k}x)\right) \left[\frac{1}{N_1^{\gamma_1}}\sum_{j=1}^{N_1} \tilde{W}^{2,j,i}_k \sigma'(\tilde{W}^{1,j}_{k}x)\left(W^{1,j}_{k+1} - W^{1,j}_k \right)x\right]\\
&\qquad \qquad \qquad \qquad \qquad \qquad \cdot \left[\frac{1}{N_1^{\gamma_1}}\sum_{j=1}^{N_1} \sigma(\tilde{W}^{1,j}_{k}x) \left(W^{2,j,i}_{k+1} - W^{2,j,i}_k \right)\right]\\
&\qquad  \left. + \sigma'\left(\frac{1}{N_1^{\gamma_1}}\sum_{j=1}^{N_1} \tilde{W}^{2,j,i}_{k}\sigma(\tilde{W}^{1,j}_{k}x)\right)\left[\frac{1}{N_1^{\gamma_1}}\sum_{j=1}^{N_1} \sigma(\tilde{W}^{1,j}_{k}x) \left(W^{2,j,i}_{k+1} - W^{2,j,i}_k \right)\left(W^{1,j}_{k+1} - W^{1,j}_k \right)x\right]\right\rbrace,
\end{aligned}
\end{equation*}
for some $(\breve{W}^{2,j,i}_k, \breve{W}^{1,j}_k)$, $(\tilde{W}^{2,j,i}_k, \tilde{W}^{1,j}_k)$ in the line segments connecting $({W}^{2,j,i}_{k+1}, {W}^{1,j}_{k+1})$ and $({W}^{2,j,i}_{k}, {W}^{1,j}_{k})$. By Lemma \ref{lemma_parameterbound}, $R^{N_1,N_2} = O(N_2^{-(1+\gamma_2)})$.
 Using equation \eqref{SGD} and definition of the empirical measure, we have
\begin{equation}
\label{g_evolution}
\begin{aligned}
&g_{k+1}^{N_1,N_2}(x) - g_k^{N_1,N_2}(x)= \frac{\alpha^{N_1,N_2}_C}{N_2^{2\gamma_2-1}}  \left(y_k-g^{N_1,N_2}_k(x_k)\right) \ip{\sigma\left(Z^{2,N_1}(x_k)\right)\sigma\left(Z^{2,N_1}(x)\right), \tilde{\gamma}^{N_1,N_2}_k}\\
& +\frac{\alpha^{N_1,N_2}_{W,2}}{N_1^{2\gamma_1}N_2^{2\gamma_2-1}}\sum_{j=1}^{N_1} \left(y_k-g^{N_1,N_2}_k(x_k)\right)\ip{(c)^2 \sigma'\left(Z^{2,N_1}(x_k)\right)\sigma'\left(Z^{2,N_1}(x)\right)\sigma(w^{1,j}x_k)\sigma(w^{1,j}x), \tilde{\gamma}^{N_1,N_2}_k}\\
& +\frac{\alpha^{N_1,N_2}_{W,1}}{N_1^{2\gamma_1}N_2^{2\gamma_2-2}}\sum_{j=1}^{N_1} \left(y_k-g^{N_1,N_2}_k(x_k)\right)xx_k\ip{cw^{2,j}\sigma'(w^{1,j}x)\sigma'\left(Z^{2,N_1}(x)\right), \tilde{\gamma}^{N_1,N_2}_k}\ip{cw^{2,j}\sigma'(w^{1,j}x_k)\sigma'(Z^{2,N_1}(x_k)),\tilde{\gamma}^{N_1,N_2}_k}\\
& +O(N_2^{-(1+\gamma_2)}),
\end{aligned}
\end{equation}
where $Z^{2,N_1}(x) = \frac{1}{N_1^{\gamma_1}} \sum_{j=1}^{N_1} w^{2,j}\sigma(w^{1,j}x)$.
We can then write the evolution of $h^{N_1,N_2}_t(x)$ for $t\in[0,T]$ as
\begin{equation*}
\begin{aligned}
&h^{N_1,N_2}_t(x)-h^{N_1,N_2}_0(x)= \sum_{k=0}^{\floor{N_2t}-1} \left[g^{N_1,N_2}_{k+1}(x) - g^{N_1,N_2}_k(x)\right]\\
&= \frac{\alpha^{N_1,N_2}_C}{N_2^{2\gamma_2-1}} \sum_{k=0}^{\floor{N_2t}-1} \left(y_k-g^{N_1,N_2}_k(x_k)\right) \ip{\sigma\left(Z^{2,N_1}(x_k)\right)\sigma\left(Z^{2,N_1}(x)\right), \tilde{\gamma}^{N_1,N_2}_k}\\
&\quad +\frac{\alpha^{N_1,N_2}_{W,2}}{N_1^{2\gamma_1}N_2^{2\gamma_2-1}} \sum_{k=0}^{\floor{N_2t}-1} \sum_{j=1}^{N_1} \left(y_k-g^{N_1,N_2}_k(x_k)\right)\ip{(c)^2 \sigma'\left(Z^{2,N_1}(x_k)\right)\sigma'\left(Z^{2,N_1}(x)\right)\sigma(w^{1,j}x_k)\sigma(w^{1,j}x), \tilde{\gamma}^{N_1,N_2}_k}\\
&\quad +\frac{\alpha^{N_1,N_2}_{W,1}}{N_1^{2\gamma_1}N_2^{2\gamma_2-2}}\sum_{k=0}^{\floor{N_2t}-1}\sum_{j=1}^{N_1} \left(y_k-g^{N_1,N_2}_k(x_k)\right)xx_k\ip{cw^{2,j}\sigma'(w^{1,j}x)\sigma'\left(Z^{2,N_1}(x)\right), \tilde{\gamma}^{N_1,N_2}_k}\\
&\qquad \qquad \qquad \qquad \qquad \qquad \cdot \ip{cw^{2,j}\sigma'(w^{1,j}x_k)\sigma'(Z^{2,N_1}(x_k)),\tilde{\gamma}^{N_1,N_2}_k}\\
&\quad +O(N_2^{-\gamma_2}),\nonumber
\end{aligned}
\end{equation*}
and, using now the definitions of the learning rates from (\ref{potential_alpha_1}) we continue the last display as
\begin{equation*}
\begin{aligned}
&= \frac{1}{N_2}\sum_{k=0}^{\floor{N_2t}-1}  \int_{\CX\times \CY} \left(y-g^{N_1,N_2}_k(x')\right)\ip{\sigma\left(Z^{2,N_1}(x')\right)\sigma\left(Z^{2,N_1}(x)\right), \tilde{\gamma}^{N_1,N_2}_k}\pi(dx',dy)\\
&\quad +\frac{1}{N_1N_2}\sum_{k=0}^{\floor{N_2t}-1}\sum_{j=1}^{N_1} \int_{\CX\times\CY}\left(y-g^{N_1,N_2}_k(x')\right)\ip{(c)^2 \sigma'\left(Z^{2,N_1}(x')\right)\sigma'\left(Z^{2,N_1}(x)\right)\sigma(w^{1,j}x')\sigma(w^{1,j}x), \tilde{\gamma}^{N_1,N_2}_k} \pi(dx',dy)\\
&\quad +\frac{1}{N_1N_2}\sum_{k=0}^{\floor{N_2t}-1} \sum_{j=1}^{N_1} \int_{\CX\times \CY} \left(y-g^{N_1,N_2}_k(x')\right)xx'\ip{cw^{2,j}\sigma'(w^{1,j}x)\sigma'\left(Z^{2,N_1}(x)\right), \tilde{\gamma}^{N_1,N_2}_k}\\
&\qquad \qquad \qquad \qquad \qquad \qquad \cdot \ip{cw^{2,j}\sigma'(w^{1,j}x')\sigma'(Z^{2,N_1}(x')),\tilde{\gamma}^{N_1,N_2}_k}\pi(dx',dy)\\
&\quad + M^{N_1,N_2}_{t} +O(N_2^{-\gamma_2}),
\end{aligned}
\end{equation*}
where $M^{N_1,N_2}_t = M^{N_1,N_2}_{1,t} + M^{N_1,N_2}_{2,t}+M^{N_1,N_2}_{3,t}$ is a martingale term given by
\begin{equation}\label{M_1}
\begin{aligned}
M^{N_1,N_2}_{1,t} = \frac{1}{N_2} \sum_{k=0}^{\floor{N_2t}-1} &\left\lbrace \left(y_k-g^{N_1,N_2}_k(x_k)\right) \ip{\sigma\left(Z^{2,N_1}(x_k)\right)\sigma\left(Z^{2,N_1}(x)\right), \tilde{\gamma}^{N_1,N_2}_k}\right.\\
& \quad \left. -\int_{\CX\times \CY} \left(y-g^{N_1,N_2}_k(x')\right)\ip{\sigma\left(Z^{2,N_1}(x')\right)\sigma\left(Z^{2,N_1}(x)\right), \tilde{\gamma}^{N_1,N_2}_k}\pi(dx',dy)\right\rbrace,\\
\end{aligned}
\end{equation}

\begin{equation}\label{M_2}
\begin{aligned}
M^{N_1,N_2}_{2,t}= &\frac{1}{N_1N_2} \sum_{k=0}^{\floor{N_2t}-1} \sum_{j=1}^{N_1}\left\lbrace \left(y_k-g^{N_1,N_2}_k(x_k)\right)\ip{(c)^2 \sigma'\left(Z^{2,N_1}(x_k)\right)\sigma'\left(Z^{2,N_1}(x)\right)\sigma(w^{1,j}x_k)\sigma(w^{1,j}x), \tilde{\gamma}^{N_1,N_2}_k}\right.\\
&- \left. \int_{\CX\times\CY}\left(y-g^{N_1,N_2}_k(x')\right)\ip{(c)^2 \sigma'\left(Z^{2,N_1}(x')\right)\sigma'\left(Z^{2,N_1}(x)\right)\sigma(w^{1,j}x')\sigma(w^{1,j}x), \tilde{\gamma}^{N_1,N_2}_k} \pi(dx',dy)\right \rbrace,
\end{aligned}
\end{equation}

\begin{equation}\label{M_3}
\begin{aligned}
M^{N_1,N_2}_{3,t}=\frac{1}{N_1N_2}\sum_{k=0}^{\floor{N_2t}-1}\sum_{j=1}^{N_1}& \left\lbrace\left(y_k-g^{N_1,N_2}_k(x_k)\right)xx_k\ip{cw^{2,j}\sigma'(w^{1,j}x)\sigma'\left(Z^{2,N_1}(x)\right), \tilde{\gamma}^{N_1,N_2}_k}\right.\\
&\qquad \qquad \qquad \qquad \qquad  \cdot \ip{cw^{2,j}\sigma'(w^{1,j}x_k)\sigma'(Z^{2,N_1}(x_k)),\tilde{\gamma}^{N_1,N_2}_k}\\
&-\int_{\CX\times \CY} \left(y-g^{N_1,N_2}_k(x')\right)xx'\ip{cw^{2,j}\sigma'(w^{1,j}x)\sigma'\left(Z^{2,N_1}(x)\right), \tilde{\gamma}^{N_1,N_2}_k}\\
&\qquad \qquad \qquad \qquad \qquad  \left. \cdot \ip{cw^{2,j}\sigma'(w^{1,j}x')\sigma'(Z^{2,N_1}(x')),\tilde{\gamma}^{N_1,N_2}_k}\pi(dx',dy)\right\rbrace
\end{aligned}
\end{equation}

Recall that learning rates are as given in \eqref{potential_alpha_1}.
As $N_2 \to \infty$, $h^{N_1,N_2}_t$ can further be re-written in terms of Riemann integrals and the scaled empirical measure $\gamma^{N_1,N_2}_t$,
\begin{equation}\label{h_N1N2_evolution}
\begin{aligned}
&h^{N_1,N_2}_t(x)-h^{N_1,N_2}_0(x)= \int_0^t \int_{\CX\times \CY} \left(y-h^{N_1,N_2}_s(x')\right)\ip{\sigma\left(Z^{2,N_1}(x')\right)\sigma\left(Z^{2,N_1}(x)\right), \gamma^{N_1,N_2}_s}\pi(dx',dy)ds\\
&\quad +\frac{1}{N_1}\sum_{j=1}^{N_1} \int_0^t \int_{\CX\times\CY}\left(y-h^{N_1,N_2}_s(x')\right)\ip{(c)^2 \sigma'\left(Z^{2,N_1}(x')\right)\sigma'\left(Z^{2,N_1}(x)\right)\sigma(w^{1,j}x')\sigma(w^{1,j}x), \gamma^{N_1,N_2}_s} \pi(dx',dy)ds\\
&\quad +\frac{1}{N_1} \sum_{j=1}^{N_1} \int_0^t \int_{\CX\times \CY} \left(y-h^{N_1,N_2}_s(x')\right)xx'\ip{cw^{2,j}\sigma'(w^{1,j}x)\sigma'\left(Z^{2,N_1}(x)\right), \gamma^{N_1,N_2}_s}\\
&\qquad \qquad \qquad \qquad \qquad \qquad \cdot \ip{cw^{2,j}\sigma'(w^{1,j}x')\sigma'(Z^{2,N_1}(x')),\gamma^{N_1,N_2}_s}\pi(dx',dy)ds\\
&\quad + M^{N_1,N_2}_{t} +O(N_2^{-\gamma_2}).
\end{aligned}
\end{equation}

Finally, we analyze the evolution of the empirical measure $\tilde{\gamma}^{N_1,N_2}_k$ in terms of test functions $f \in C^2_b(\R^{1+N_1(1+d)})$. Denote
$\theta^i_k = (C^i_k,  W^{2,1,i}_k, \ldots, W^{2,N_1,i}_k, W^{1,1}_k, \ldots, W^{1,N_1}_k)$, first order Taylor expansion gives
\begin{equation*}
\begin{aligned}
\ip{f,\tilde{\gamma}^{N_1,N_2}_{k+1}} - \ip{ f,\tilde{\gamma}^{N_1,N_2}_{k}} &= \frac{1}{N_2} \sum_{i=1}^{N_2} \left[ f(\theta^i_{k+1}) - f(\theta^i_k)\right]\\
&= \frac{1}{N_2} \sum_{i=1}^{N_2}  \partial_{c} f(\theta^i_k)\left(C^i_{k+1} - C^i_k\right) + \frac{1}{N_2} \sum_{i=1}^{N_2} \sum_{j=1}^{N_1} \partial_{w^{2,j}} f(\theta^i_k)\left(W^{2,j,i}_{k+1} - W^{2,j,i}_k\right)\\
&\quad + \frac{1}{N_2} \sum_{i=1}^{N_2} \sum_{j=1}^{N_1} \nabla_{w^{1,j}} f(\theta^i_k) \left(W^{1,j}_{k+1} - W^{1,j}_k\right) + O\left(\frac{1}{N_2^2}\right)
\end{aligned}
\end{equation*}
Using \eqref{SGD}, we have
\begin{equation*}
\begin{aligned}
&\ip{f,\tilde{\gamma}^{N_1,N_2}_{k+1}} - \ip{ f,\tilde{\gamma}^{N_1,N_2}_{k}} = \frac{\alpha_C^{N_1,N_2}}{N_2^{\gamma_2}} \left(y_k-g_k^{N_1,N_2}(x_k)\right) \ip{\partial_{c}f(\theta)  \sigma(Z^{2,N_1}(x_k)),\tilde{\gamma}^{N_1,N_2}_k}\\
&\quad +\frac{\alpha^{N_1,N_2}_{W,2}}{N_1^{\gamma_1}N_2^{\gamma_2}}\left(y_k-g_k^{N_1,N_2}(x_k)\right) \ip{c\sigma'(Z^{2,N_1}(x_k))\sigma(w^1x_k)\cdot \partial_{w^2}f(\theta) ,\tilde{\gamma}^{N_1,N_2}_k}\\
&\quad + \frac{\alpha^{N_1,N_2}_{W,1}}{N_1^{\gamma_1}}\left(y_k-g_k^{N_1,N_2}(x_k)\right) \ip{\ip{c\sigma'(Z^{2,N_1}(x_k))\sigma'(w^1x_k)w^{2},N_2^{1-\gamma_2}\tilde{\gamma}_k^{N_1,N_2}}\cdot  \nabla_{w^1}f(\theta)x_k,\tilde{\gamma}^{N_1,N_2}_k}\\
&\quad + O\left(\frac{1}{N_2^2}\right).
\end{aligned}
\end{equation*}

In order to write the evolution in terms of the scaled measure $\gamma^{N_1,N_2}_t$, for $t\in [0,1]$, we have
\begin{equation}\label{measure_evolution}
\begin{aligned}
&\ip{f,\gamma^{N_1,N_2}_{t}} - \ip{ f,\gamma^{N_1,N_2}_{0}}= \sum_{k=0}^{\floor{N_2t}-1}\ip{f,\nu^{N_1,N_2}_{k+1}} - \ip{ f,\nu^{N_1,N_2}_{k}}\\
&=\frac{\alpha_C^{N_1,N_2}}{N_2^{\gamma_2}}\sum_{k=0}^{\floor{N_2t}-1} \left(y_k-g_k^{N_1,N_2}(x_k)\right) \ip{\partial_{c}f(\theta)  \sigma(Z^{2,N_1}(x_k)),\tilde{\gamma}^{N_1,N_2}_k}\\
&\quad +\frac{\alpha^{N_1,N_2}_{W,2}}{N_1^{\gamma_1}N_2^{\gamma_2}}\sum_{k=0}^{\floor{N_2t}-1}\left(y_k-g_k^{N_1,N_2}(x_k)\right) \ip{c\sigma'(Z^{2,N_1}(x_k))\sigma(w^1x_k)\cdot \partial_{w^2}f(\theta) ,\tilde{\gamma}^{N_1,N_2}_k}\\
&\quad + \frac{\alpha^{N_1,N_2}_{W,1}}{N_1^{\gamma_1}}\sum_{k=0}^{\floor{N_2t}-1}\left(y_k-g_k^{N_1,N_2}(x_k)\right)\ip{\ip{c\sigma'(Z^{2,N_1}(x_k))\sigma'(w^1x_k)w^{2},N_2^{1-\gamma_2}\tilde{\gamma}_k^{N_1,N_2}}\cdot  \nabla_{w^1}f(\theta)x_k,\tilde{\gamma}^{N_1,N_2}_k}\\
&\quad + O\left(\frac{1}{N_2}\right)\\
&=\frac{\alpha_C^{N_1,N_2}}{N_2^{\gamma_2-1}}\int_0^t\int_{\CX\times \CY} \left(y-h_s^{N_1,N_2}(x)\right) \ip{\partial_{c}f(\theta)  \sigma(Z^{2,N_1}(x)),\gamma^{N_1,N_2}_s}\pi(dx,dy)ds\\
&\quad +\frac{\alpha^{N_1,N_2}_{W,2}}{N_1^{\gamma_1}N_2^{\gamma_2-1}}\int_0^t\int_{\CX\times \CY}\left(y-h_s^{N_1,N_2}(x)\right) \ip{c\sigma'(Z^{2,N_1}(x))\sigma(w^1x)\cdot \partial_{w^2}f(\theta) ,\gamma^{N_1,N_2}_s}\pi(dx,dy)ds\\
&\quad + \frac{ \alpha^{N_1,N_2}_{W,1}}{N_1^{\gamma_1}N_2^{\gamma_2-2}}\int_0^t\int_{\CX\times \CY}\left(y-h_s^{N_1,N_2}(x)\right)\ip{\ip{c\sigma'(Z^{2,N_1}(x))\sigma'(w^1x)w^{2},\gamma_s^{N_1,N_2}}\cdot  \nabla_{w^1}f(\theta)x,\gamma^{N_1,N_2}_s}\pi(dx,dy)ds\\
&\quad + M_{f,t}^{N_1,N_2}+ O\left(\frac{1}{N_2}\right),\\
\end{aligned}
\end{equation}
where $M_{f,t}^{N_1,N_2} = M^{N_1,N_2}_{f,1,t} +M^{N_1,N_2}_{f,2,t} + M^{N_1,N_2}_{f,3,t} $ is a martingale term, and
\begin{equation*}
\begin{aligned}
M^{N_1,N_2}_{f,1,t} = \frac{\alpha_C^{N_1,N_2}}{N_2^{\gamma_2}}&\sum_{k=0}^{\floor{N_2t}-1} \left\lbrace\left(y_k-g_k^{N_1,N_2}(x_k)\right) \ip{\partial_{c}f(\theta)  \sigma(Z^{2,N_1}(x_k)),\tilde{\gamma}^{N_1,N_2}_k}\right.\\
&\qquad \left.-\int_{\CX\times \CY} \left(y-g_k^{N_1,N_2}(x)\right) \ip{\partial_{c}f(\theta)  \sigma(Z^{2,N_1}(x)),\tilde{\gamma}^{N_1,N_2}_k}\pi(dx,dy)\right\rbrace,
\end{aligned}
\end{equation*}

\begin{equation*}
\begin{aligned}
M^{N_1,N_2}_{f,2,t} = \frac{\alpha^{N_1,N_2}_{W,2}}{N_1^{\gamma_1}N_2^{\gamma_2}}&\sum_{k=0}^{\floor{N_2t}-1}\left\lbrace \left(y_k-g_k^{N_1,N_2}(x_k)\right) \ip{c\sigma'(Z^{2,N_1}(x_k))\sigma(w^1x_k)\cdot \partial_{w^2}f(\theta) ,\tilde{\gamma}^{N_1,N_2}_k}\right.\\
&\qquad \left. -\int_{\CX\times \CY}\left(y-g_k^{N_1,N_2}(x)\right) \ip{c\sigma'(Z^{2,N_1}(x))\sigma(w^1x)\cdot \partial_{w^2}f(\theta) ,\tilde{\gamma}^{N_1,N_2}_k}\pi(dx,dy)\right\rbrace,
\end{aligned}
\end{equation*}

\begin{equation*}
\begin{aligned}
 M^{N_1,N_2}_{f,3,t} &= \frac{\alpha^{N_1,N_2}_{W,1}}{N_1^{\gamma_1}N_2^{\gamma_2-1}}\sum_{k=0}^{\floor{N_2t}-1}\left\lbrace\left(y_k-g_k^{N_1,N_2}(x_k)\right)\ip{\ip{c\sigma'(Z^{2,N_1}(x_k))\sigma'(w^1x_k)w^{2},\tilde{\gamma}_k^{N_1,N_2}}\cdot  \nabla_{w^1}f(\theta)x_k,\tilde{\gamma}^{N_1,N_2}_k}\right.\\
&\qquad -\left.\int_{\CX\times \CY}\left(y-g_k^{N_1,N_2}(x)\right)\ip{\ip{c\sigma'(Z^{2,N_1}(x))\sigma'(w^1x)w^{2},\tilde{\gamma}_k^{N_1,N_2}}\cdot  \nabla_{w^1}f(\theta)x,\tilde{\gamma}^{N_1,N_2}_k}\pi(dx,dy) \right\rbrace.
\end{aligned}
\end{equation*}

Using learning rates as specified in \eqref{potential_alpha_1}, we have
\begin{equation}\label{mu_N1N2_evolution}
\begin{aligned}
&\ip{f,\gamma^{N_1,N_2}_{t}} - \ip{ f,\gamma^{N_1,N_2}_{0}}\\
&=\frac{1}{N_2^{1-\gamma_2}}\int_0^t\int_{\CX\times \CY} \left(y-h_s^{N_1,N_2}(x)\right) \ip{\partial_{c}f(\theta)  \sigma(Z^{2,N_1}(x)),\gamma^{N_1,N_2}_s}\pi(dx,dy)ds\\
&\quad +\frac{1}{N_1^{1-\gamma_1}N_2^{1-\gamma_2}}\int_0^t\int_{\CX\times \CY}\left(y-h_s^{N_1,N_2}(x)\right) \ip{c\sigma'(Z^{2,N_1}(x))\sigma(w^1x)\cdot \partial_{w^2}f(\theta) ,\gamma^{N_1,N_2}_s}\pi(dx,dy)ds\\
&\quad + \frac{ 1}{N_1^{1-\gamma_1}N_2^{1-\gamma_2}}\int_0^t\int_{\CX\times \CY}\left(y-h_s^{N_1,N_2}(x)\right)\ip{\ip{c\sigma'(Z^{2,N_1}(x))\sigma'(w^1x)w^{2},\gamma_s^{N_1,N_2}}\cdot  \nabla_{w^1}f(\theta)x,\gamma^{N_1,N_2}_s}\pi(dx,dy)ds\\
&\quad + M_{f,t}^{N_1,N_2}+ O\left(\frac{1}{N_2}\right)
\end{aligned}
\end{equation}

In the following lemma, we prove a uniform bound for $\E \left(\abs{g^N_k(x)}^4\right)$.

\begin{lemma}\label{lemma_g}
For any $k\le N_2T$ and any $x \in \CX$,
\[\sup_{N_1, N_2 \in \mathbb{N}, k\le \floor{N_2T}} \E \paren{\abs{g^{N_1,N_2}_k(x)}^4} < C,\]
for some finite constant $C<\infty$.
\end{lemma}
\begin{proof}
By equation \eqref{g_evolution}, we have the following bound
\begin{equation}\label{network_evolution_sq}
\begin{aligned}
\abs{g^{N_1,N_2}_{k+1}(x)}  &\le  \abs{g^{N_1,N_2}_{k}(x)} + \frac{C\alpha^{N_1,N_2}_C}{N_2^{2\gamma_2-1}}  \abs{y_k-g^{N_1,N_2}_k(x_k)}  +\frac{C\alpha^{N_1,N_2}_{W,2}}{N_1^{2\gamma_1}N_2^{2\gamma_2-1}} \sum_{j=1}^{N_1} \abs{y_k-g^{N_1,N_2}_k(x_k)}\\
&\quad  + \frac{C\alpha^{N_1,N_2}_{W,1}}{N_1^{2\gamma_1}N_2^{2\gamma_2-2}}\sum_{j=1}^{N_1} \abs{y_k-g^{N_1,N_2}_k(x_k)} + \frac{C}{N_2^{\gamma_2+1}}\\
&\le \abs{g^{N_1,N_2}_{k}(x)} + \frac{C}{N_2}\abs{g^{N_1,N_2}_k(x_k)} + \frac{C}{N_2}\paren{1+\frac{1}{N_2^{\gamma_2}}}\\
&\le \abs{g^{N_1,N_2}_{k}(x)} + \frac{C}{N_2}\abs{g^{N_1,N_2}_k(x_k)} + \frac{C}{N_2},
\end{aligned}
\end{equation}

where the last inequality holds because $N_2>0$ is large. Squaring both sides of the \eqref{network_evolution_sq} gives
\begin{equation*}
\begin{aligned}
\abs{g^{N_1,N_2}_{k+1}(x)}^2  &\le \abs{g^{N_1,N_2}_{k}(x)}^2 + 2\abs{g^{N_1,N_2}_{k}(x)} \paren{\frac{C}{N_2}\abs{g^{N_1,N_2}_k(x_k)}+\frac{C}{N_2}} + \paren{\frac{C}{N_2}}^2\paren{\abs{g^{N_1,N_2}_k(x_k)}+1}^2\\
&\le  \abs{g^{N_1,N_2}_{k}(x)}^2 + \frac{C}{N_2} \abs{g^{N_1,N_2}_k(x)}^2 + \frac{C}{N_2} \abs{g^{N_1,N_2}_k(x_k)}^2+\frac{C}{N_2},
\end{aligned}
\end{equation*}
where the last inequality follows from the Young's inequality ($ab \le \frac{a^2}{2\epsilon} + \frac{\epsilon b^2}{2}$, for $\epsilon= \frac{1}{N_2}$). Similarly, squaring both sides one more time gives
\begin{equation*}
\begin{aligned}
\abs{g^{N_1,N_2}_{k+1}(x)}^4
&\le  \abs{g^{N_1,N_2}_{k}(x)}^4 + \frac{C}{N_2} \abs{g^{N_1,N_2}_k(x)}^4 + \frac{C}{N_2} \abs{g^{N_1,N_2}_k(x_k)}^4+\frac{C}{N_2}.
\end{aligned}
\end{equation*}
Therefore, for $k\le N_2T$,
\begin{equation*}
\begin{aligned}
\abs{g^{N_1,N_2}_k(x)}^4 &= \abs{g^{N_1,N_2}_0(x)}^4 + \sum_{j=1}^k \paren{\abs{g^{N_1,N_2}_{j}(x)}^4 - \abs{g^{N_1,N_2}_{j-1}(x)}^4}\\
&\le \abs{g^{N_1,N_2}_0(x)}^4 + \sum_{j=1}^k \paren{\frac{C}{N_2} \abs{g^{N_1,N_2}_{j-1}(x)}^4 + \frac{C}{N_2} \abs{g^{N_1,N_2}_{j-1}(x_{j-1})}^4+\frac{C}{N_2}}\\
&\le \abs{g^{N_1,N_2}_0(x)}^4 + C + \frac{C}{N_2} \sum_{j=1}^k  \abs{g^{N_1,N_2}_{j-1}(x)}^4 + \frac{C}{N_2}\sum_{j=1}^k  \abs{g^{N_1,N_2}_{j-1}(x_{j-1})}^4.
\end{aligned}
\end{equation*}
We then take expectation on both sides and get
\begin{equation}\label{expectation_g_1}
\begin{aligned}
\E\paren{\abs{g^{N_1,N_2}_k(x)}^4 }
&\le \E \paren{\abs{g^{N_1,N_2}_0(x)}^4} + C + \frac{C}{N_2} \sum_{j=1}^k  \E \paren{\abs{g^{N_1,N_2}_{j-1}(x)}^4} + \frac{C}{N_2}\sum_{j=1}^k  \E \paren{\abs{g^{N_1,N_2}_{j-1}(x_{j-1})}^4}\\
&\le \E \paren{\abs{g^{N_1,N_2}_0(x)}^4} + C + \frac{C}{N_2} \sum_{j=1}^k  \E \paren{\abs{g^{N_1,N_2}_{j-1}(x)}^4} + \frac{C}{N_2}\sum_{j=1}^k  \sum_{x' \in \CX} \E \paren{\abs{g^{N_1,N_2}_{j-1}(x')}^4},
\end{aligned}
\end{equation}
where the last term in the last inequality holds because $x_j$ are sampled from a fixed data set $\CX$ of size $M$.

Therefore, summing both side of \eqref{expectation_g_1} with respect to $x$ gives
\begin{equation}\label{int_Eg}
\begin{aligned}
\sum_{x\in\CX} \E\paren{\abs{g^{N_1,N_2}_k(x)}^4}
&\le \sum_{x\in \CX}\E \paren{\abs{g^{N_1,N_1}_0(x)}^4} + CM + \frac{C}{N_2} \sum_{j=1}^k \sum_{x\in\CX} \E \paren{\abs{g^{N_1,N_2}_{j-1}(x)}^4}\\
&\quad + \frac{CM}{N_2}\sum_{j=1}^k  \sum_{x'\in \CX}\E \paren{\abs{g^{N_1,N_2}_{j-1}(x')}^4}\\
&\le \sum_{x\in\CX}\E \paren{\abs{g^{N_1,N_2}_0(x)}^4} + C + \frac{C}{N_2} \sum_{j=1}^k \sum_{x\in \CX} \E \paren{\abs{g^{N_1,N_2}_{j-1}(x)}^4}.
\end{aligned}
\end{equation}
Since $(C^i_0,W^{1,j},W^{2,j,i}_0)$ are i.i.d. mean zero random variables, we have
\begin{align*}
\E\paren{\abs{g^N_0(x)}^4} &= \E \left[ \abs{\frac{1}{N_2^{\gamma_2}}\sum_{i=1}^{N_2} C^i_0 \sigma\left(\frac{1}{N_1^{\gamma_1}} \sum_{j=1}^{N_1} W^{2,j,i}_0\sigma(W^{1,j}_0x)\right)}^4\right] \le \frac{C}{N_2^{4\gamma_2}} \sum_{i=1}^{N_2} \E \paren{\abs{C^i_0}^4}\le C.
\end{align*}
Then, by applying the discrete Gr\"onwall lemma to equation \eqref{int_Eg}, for any $0\le k\le \floor{N_2T}$ and $N_2 \in \mathbb{N}$
\[\sum_{x\in\CX} \E\paren{\abs{g^{N_1,N_2}_k(x)}^4 } \le C. \]
The result in the lemma follows.

\end{proof}

Next, using conditional independence of the terms in the series for $M^{N_1,N_2}_t$ and $M^{N_1,N_2}_{f,t}$ as well as the bounds from Lemmas \ref{lemma_parameterbound} and \ref{lemma_g}, we can establish the following $L^2$ bounds for the martingale terms $M^{N_1,N_2}_t$ and $M^{N_1,N_2}_{f,t}$, which implies that  they converge to zero as $N_2 \to \infty$. The proof is similar to that for Lemma 3.1 in \cite{SirignanoSpiliopoulosNN3} and thus it is omitted.
\begin{lemma}\label{LLN:lemma:Mt_bound} For large $N_1,N_2 \in \mathbb{N}$ and some finite constant $C>0$, we have
\begin{equation*}
\begin{aligned}
\E\left[\left(M^{N_1,N_2}_t \right)^2\right] &\le \frac{C}{N_2}, \quad
\E\left[\left(M^{N_1,N_2}_{f,t} \right)^2\right] \le \frac{C}{N_2^{3-2\gamma_2}}.\\
\end{aligned}
\end{equation*}
\end{lemma}
\subsection{Relative Compactness}\label{sec:relative_compactness}
In this section, we prove the relative compactness of the family $\{\gamma^{N_1,N_2}, h^{N_1,N_2}\}_{N_2\in \mathbb{N}}$ in $D_E([0,T])$, where   $E = \CM(\R^{1+N_1(1+d)}) \times \R^M$, and $N_1\in\mathbb{N}$ is fixed. Using Lemmas \ref{lemma_parameterbound}, \ref{lemma_g}, and Markov's inequality, we get the following lemma which shows compact containment for $\{(\gamma_t^{N_1,N_2}, h_t^{N_1,N_2}), t\in[0,T]\}_{N_2 \in \mathbb{N}}$. The proof is analogous to that for Lemma 3.3 in \cite{SirignanoSpiliopoulosNN4} and thus omitted.
\begin{lemma}\label{LLN:lemma:compact containment}
For each $\eta >0$, there is a compact subset $\mathcal{K}$ of $E$ such that
\[\sup_{N_1 \in \mathbb{N}, t\in[0,T]} \P \left[ \paren{\gamma^{N_1,N_2}_t, h^{N_1,N_2}_t} \notin \mathcal{K}\right] < \eta.\]
\end{lemma}

We now show the regularity of the process $\gamma^{N_1,N_2}$ in $D_{\CM(\R^{1+N_1(1+d)})}([0,T])$. For $z_1,z_2 \in \R$, define the function $q(z_1,z_2) = \min\{\vert z_1-z_2\vert,1\}$. Let $\CF^{N_1,N_2}_t$ be the $\sigma$-algebra generated by $(C^i_0,W^{2,j,i}_0, W^{1,j}_0)_{i,j}$ and $(x_j,y_j)_{j=0}^{\floor{N_2t}-1}$.
\begin{lemma}\label{LLN:lemma:mu_regularity}
For any $f \in C^2_b(\R^{1+N_1(1+d)})$ and $\delta \in (0,1)$, there is a constant $C<\infty$ such that for $0\le u \le \delta$, $0\le v\le \delta \wedge t$, and $t \in [0,T]$,
\[\E \left[q\paren{\ip{f,\gamma^{N_1,N_2}_{t+u}},\ip{f,\gamma^{N_1,N_2}_{t}}}q\paren{\ip{f,\gamma^{N_1,N_2}_{t}},\ip{f,\gamma^{N_1,N_2}_{t-v}}}\vert \CF^{N_1,N_2}_t\right]\le \frac{C\delta}{N_2^{1-\gamma_2}} + \frac{C}{N_2^{2-\gamma_2}}.\]
\end{lemma}
\begin{proof}For $0\le s< t \le T$,  using a Taylor expansion, we have
\begin{equation}\label{LLN:mu_cont}
\begin{aligned}
&\left\vert \ip{f,\gamma^{N_1,N_2}_t} - \ip{f,\gamma^{N_1,N_2}_s} \right\vert = \left \vert \ip{f,\tilde{\gamma}^{N_1,N_2}_{\floor{N_2t}}} - \ip{f,\tilde{\gamma}^{N_1,N_2}_{\floor{N_2s}}} \right \vert\\
&\quad \le \frac{1}{N_2} \sum_{i=1}^{N_2} \left\vert f(\theta^i_{\floor{N_2t}}) -f(\theta^i_{\floor{N_2s}})\right\vert\\
&\quad\le \frac{1}{N_2} \sum_{i=1}^{N_2}  \abs{\partial_{c}  f(\bar{\theta}^i_{\floor{N_2t}})}\abs{C^i_{\floor{N_2t}} - C^i_{\floor{N_2s}}} + \frac{1}{N_2} \sum_{i=1}^{N_2} \sum_{j=1}^{N_1} \abs{ \partial_{w^{2,j}} f(\bar{\theta}^i_{\floor{N_2t}})} \abs{ W^{2,j,i}_{\floor{N_2t}} - W^{2,j,i}_{\floor{N_2s}}}\\
&\qquad + \frac{1}{N_2} \sum_{i=1}^{N_2} \sum_{j=1}^{N_1} \norm{\nabla_{w^{1,j}} f(\bar{\theta}^i_{\floor{N_2t}})} \norm{ W^{1,j}_{\floor{N_2t}} - W^{1,j}_{\floor{N_2s}}}
\end{aligned}
\end{equation}
for some $\bar{\theta}^i_{\floor{N_2t}}$ in the line segments between $\theta^i_{\floor{N_2s}}$ and $\theta^i_{\floor{N_2t}}$. With $0<t-s \le \delta <1$, by Lemmas \ref{lemma_parameterbound}, \ref{lemma_g}, we have
\begin{align*}
\E \paren{\abs{C^i_{\floor{N_2t}}-C^i_{\floor{N_2s}}} \big\vert \CF^{N_1,N_2}_s} &\le \sum_{k=\floor{N_2s}}^{\floor{N_2t}-1}\E \paren{\abs{ C^i_{k+1}-C^i_{k}} \big\vert \CF^{N_1,N_2}_s}\\
&\le \frac{1}{N_2^{2-\gamma_2}}\sum_{k=\floor{N_2s}}^{\floor{N_2t}-1} C \\
&\le \frac{C\delta}{N_2^{1-\gamma_2}} + \frac{C}{N_2^{2-\gamma_2}}.
\end{align*}
Similar analysis shows
\begin{align*}
\E \paren{\abs{W^{2,j,i}_{\floor{N_2t}}-W^{2,j,i}_{\floor{N_2s}}}\big\vert \CF^{N_1,N_2}_s} &\le \frac{C\delta}{N_1^{1-\gamma_1}N_2^{1-\gamma_2}} + \frac{C}{N_1^{1-\gamma_1}N_2^{2-\gamma_2}},\\
\E \paren{\norm{W^{1,j}_{\floor{N_2t}}-W^{1,j}_{\floor{N_2s}}}\big\vert \CF^{N_1,N_2}_s} &\le \frac{C\delta}{N_1^{1-\gamma_1}N_2^{1-\gamma_2}} + \frac{C}{N_1^{1-\gamma_1}N_2^{2-\gamma_2}}.
\end{align*}

By Lemma \ref{lemma_parameterbound}, $\bar{\theta}^i_{\floor{N_2t}}$ is bounded in expectation for $0<s<t\le T$.
Taking conditional expectation on both sides of \eqref{LLN:mu_cont} and using bounds we derived above yields
\[\E \left[\ip{f,\gamma^{N_1,N_2}_{t}}-\ip{f,\gamma^{N_1,N_2}_{s}}\vert \CF^{N_1,N_2}_s\right]\le \frac{C\delta}{N_2^{1-\gamma_2}} + \frac{C}{N_2^{2-\gamma_2}},\]
for $0<s<t\le T$ with $0<t-s\le \delta <1$, and some unimportant positive constant $C<\infty$. Therefore, the statement of the lemma follows.
\end{proof}

We next establish the regularity of the process $h^{N_1,N_2}$ in $D_{\R^M}([0,T])$ in the following lemma. For the purpose of this lemma, we denote $q(z_1,z_2) = \min\{\norm{ z_1-z_2}_{l^1},1\}$ for $z_1,z_2 \in \R^M$.
\begin{lemma}\label{LLN:lemma:network_regularity}
For any $\delta \in (0,1)$, there is a constant $C<\infty$ such that for $0\le u \le \delta$, $0\le v\le \delta \wedge t$, and $t \in [0,T]$,
\[\E \left[q\paren{h^{N_1,N_2}_{t+u},h^{N_1,N_2}_{t}}q\paren{h^{N_1,N_2}_{t},h^{N_1,N_2}_{t-v}}\vert \CF^{N_1,N_2}_t\right]\le {C\delta}+\frac{C}{N_2}.\]
\end{lemma}
\begin{proof}For $0<s<t\le T$, by the Taylor expansion of the network output $g^{N_1,N_2}_k(x)$, we have

\begin{equation}\label{network_cont}
\begin{aligned}
&\abs{h^{N_1,N_2}_t(x) - h^{N_1,N_2}_s(x)}
 \le \sum_{k=\floor{N_2s}}^{\floor{N_2t}-1}\abs{g^{N_1,N_2}_{k+1}(x) - g^{N_1,N_2}_{k}(x)}\\
& \le \frac{1}{N_2^{\gamma_2}}\sum_{k=\floor{N_2s}}^{\floor{N_2t}-1} \sum_{i=1}^{N_2} \left\vert C^i_{k+1}-C^i_k \right\vert \abs{\sigma\left(\frac{1}{N_1^{\gamma_1}}\sum_{j=1}^{N_1} W^{2,j,i}_{k}\sigma(W^{1,j}_{k}x)\right)} \\
&\quad + \frac{1}{N_1^{\gamma_1}N_2^{\gamma_2}}\sum_{k=\floor{N_2s}}^{\floor{N_2t}-1} \sum_{i=1}^{N_2}\sum_{j=1}^{N_1} \abs{C^i_k \sigma'\left(\frac{1}{N_1^{\gamma_1}}\sum_{j=1}^{N_1} {W}^{2,j,i}_{k}\sigma({W}^{1,j}_{k}x)\right) \sigma({W}^{1,j}_{k}x)\left(W^{2,j,i}_{k+1} - W^{2,j,i}_k \right)}\\
&\quad +\frac{1}{N_1^{\gamma_1}N_2^{\gamma_2}}\sum_{k=\floor{N_2s}}^{\floor{N_2t}-1} \sum_{i=1}^{N_2}\sum_{j=1}^{N_1} \abs{C^i_k\sigma'\left(\frac{1}{N_1^{\gamma_1}}\sum_{j=1}^{N_1} {W}^{2,j,i}_{k}\sigma({W}^{1,j}_{k}x)\right) {W}^{2,j,i}_k \sigma'({W}^{1,j}_{k}x)\left(W^{1,j}_{k+1} - W^{1,j}_k \right)x}\\
& \le \frac{C}{N_2^{\gamma_2}}\sum_{k=\floor{N_2s}}^{\floor{N_2t}-1} \sum_{i=1}^{N_2} \abs{ C^i_{k+1}-C^i_k } + \frac{C}{N_1^{\gamma_1}N_2^{\gamma_2}}\sum_{k=\floor{N_2s}}^{\floor{N_2t}-1} \sum_{i=1}^{N_2}\sum_{j=1}^{N_1} \left(\abs{W^{2,j,i}_{k+1} - W^{2,j,i}_k } + \norm{W^i_{k+1} - W^i_k}\right).\\
\end{aligned}
\end{equation}

By taking conditional expectation on both sides of \eqref{network_cont} and using the bounds we derived in the proof of Lemma \ref{LLN:lemma:mu_regularity},
\begin{align*}
&\E\left[\abs{h^{N_1,N_2}_t(x)-h^{N_1,N_2}_s(x)} \vert \CF^{N_1,N_2}_s\right] \le \frac{C}{N_2^{\gamma_2}}\sum_{k=\floor{N_2s}}^{\floor{N_2t}-1} \sum_{i=1}^{N_2} \E\left[\abs{ C^i_{k+1}-C^i_k } \vert \CF^{N_1,N_2}_s\right]\\
&\qquad + \frac{C}{N_1^{\gamma_1}N_2^{\gamma_2}}\sum_{k=\floor{N_2s}}^{\floor{N_2t}-1} \sum_{i=1}^{N_2}\sum_{j=1}^{N_1} \E\left[\abs{W^{2,j,i}_{k+1} - W^{2,j,i}_k } + \norm{W^i_{k+1} - W^i_k}\vert \CF^{N_1,N_2}_s \right]\\
&\quad \le C\delta + \frac{C}{N_2}.
\end{align*}
Since $x\in \mathcal{X}$ is arbitrary, the bound above implies that
\[\E\left[\norm{h^{N_1,N_2}_t-h^{N_1,N_2}_s}_{l^1} \vert \CF^{N_1,N_2}_s\right] \le C\delta + \frac{C}{N_2}. \]
The statement of the lemma then follows.
\end{proof}

Combining Lemmas \ref{LLN:lemma:compact containment} to \ref{LLN:lemma:network_regularity}, we have the following lemma for the relative compactness of the processes $\{\gamma^{N_1,N_2}, h^{N_1,N_2}\}_{N_2\in\mathbb{N}}$ for fixed $N_1$. The proof is similar to that of Lemma 3.6 in \cite{SirignanoSpiliopoulosNN1}, which is omitted here.
\begin{lemma}\label{LLN:lemma:relative_compact}
The sequence of processes $\{\gamma^{N_1,N_2}, h^{N_1,N_2}\}_{N_2\in\mathbb{N}}$ is relatively compact in $D_E([0,T])$, where   $E = \CM(\R^{1+N_1(1+d)}) \times \R^M$.
\end{lemma}

\subsection{Identification of the Limit}\label{sec::Identification of the Limit}
In this section, we show that for fixed $N_1$ and as $N_2 \to \infty$, the process $(\gamma^{N_1,N_2}_t, h^{N_1,N_2}_t)$ converges in distribution in the space $D_E([0,T])$ to $(\gamma^{N_1}_t,h^{N_1}_t)$, which satisfies the evolution equation
\begin{equation}\label{h_N1_evolution}
\begin{aligned}
h^{N_1}_t(x) &= h^{N_1}_0(x) +  \int^t_0 \int_{\CX \times \CY} \paren{y-h^{N_1}_s(x')} A^{N_1}_{x,x'} \pi(dx',dy) ds,
\end{aligned}
\end{equation}
%
where $\gamma^{N_1}_0$ is given by \eqref{gamma^N1_0}, and if $\gamma_2=1/2$, $h^{N_1}_0(x) = \mathcal{G}^{N_1}(x)$, where $\mathcal{G}^{N_1}$ is Gaussian, and if $\gamma_2>1/2$, $h^{N_1}_0(x) = 0$.

Let $\pi^{N_1,N_2}\in \CM(D_{E}([0,T])$ be the probability measure corresponding to $(\gamma^{N_1,N_2}, h^{N_1,N_2})$. Relative compactness implies that there is a subsequence $\pi^{N_{1},N_{2_k}}$ that converges weakly. We must show that any limit point $\pi^{N_1}$ of a convergent subsequence $\pi^{N_{1},N_{2_k}}$ is a Dirac measure concentrated on $(\gamma^{N_1}, h^{N_1}) \in D_{E}([0,T])$, which satisfies equation \eqref{h_N1_evolution} and  $\langle f,\gamma^{N_1}_t\rangle = \langle f, \gamma^{N_1}_0\rangle$ for any test function $f\in C^2_b(\R^{1+N_1(1+d)})$.

We define the map $F(\gamma^{N_1},h^{N_1}):D_{E}([0,T]) \to \R_+$ for each $t \in [0,T]$, $f\in C^2_b(\R^{1+N_1(1+d)})$, $g_1,\ldots,g_p \in C_b(\R^{1+N_1(1+d)})$, $q_{1,1}, \ldots, q_{1,p}, q_{2,1}, \ldots q_{N_2,p} \in C_b(\R^M)$, $m_1,\ldots, m_p \in C_b(\R^M)$ and $0 \le s_1 < \cdots < s_p \le t$:
\begin{equation}\label{LNN:identify_limit_eq}
\begin{aligned}
&F(\gamma,h)= \left\vert \paren{\ip{f,\gamma^{N_1}_t}-\ip{f,\gamma^{N_1}_0}} \times \ip{g_1,\gamma^{N_1}_{s_1}}\times \cdots \times \ip{g_p,\gamma^{N_1}_{s_p}} \right\vert\\
& \quad + \sum_{x\in \mathcal{X}} \left\vert \left(h^{N_1}_t(x) - h^{N_1}_0(x) - \int_0^t \int_{\CX\times \CY} \left(y-h^{N_1}_s(x')\right)\ip{\sigma\left(Z^{2,N_1}(x')\right)\sigma\left(Z^{2,N_1}(x)\right), \gamma^{N_1}_s}\pi(dx',dy)ds \right.\right.\\
& \quad  - \frac{1}{N_1}\sum_{j=1}^{N_1} \int_0^t \int_{\CX\times\CY}\left(y-h^{N_1}_s(x')\right)\ip{(c)^2 \sigma'\left(Z^{2,N_1}(x')\right)\sigma'\left(Z^{2,N_1}(x)\right)\sigma(w^{1,j}x')\sigma(w^{1,j}x), \gamma^{N_1}_s} \pi(dx',dy)ds\\
& \quad  -\frac{1}{N_1} \sum_{j=1}^{N_1} \int_0^t \int_{\CX\times \CY} \left(y-h^{N_1}_s(x')\right)xx'\ip{cw^{2,j}\sigma'(w^{1,j}x)\sigma'\left(Z^{2,N_1}(x)\right), \gamma^{N_1}_s}\\
& \qquad \qquad \qquad  \left. \cdot \ip{cw^{2,j}\sigma'(w^{1,j}x')\sigma'(Z^{2,N_1}(x')),\gamma^{N_1}_s}\pi(dx',dy)ds \times m_1(h^{N_1}_{s_1}) \times \cdots \times m_p(h^{N_1}_{s_p}) \right \vert.
\end{aligned}
\end{equation}

By equations \eqref{mu_N1N2_evolution}, \eqref{h_N1N2_evolution}, Lemma \ref{LLN:lemma:Mt_bound} and the Cauchy-Schwarz inequality, we have
\begin{align*}
\E_{\pi^{N_1,N_2}} \left[F(\gamma^{N_1},h^{N_1})\right] &= \E\left[F(\gamma^{N_1,N_2},h^{N_1,N_2})\right]\\
&= \E\left[\abs{O\paren{N_2^{-(1-\gamma_2)}} + M^{N_1,N_2}_{f,t} + O\paren{N_2^{-1}} \times \prod_{i=1}^{p}\ip{g_i,\gamma^{N_1,N_2}_{s_i}}}\right]\\
&\quad + \sum_{x \in \CX} \E\left[\abs{\paren{M^{N_1,N_2}_t + O(N_2^{-\gamma_2})} \times \prod_{i=1}^p m_i(h^{N_1,N_2}_{s_i})} \right]\\
&\le C \paren{\E\left[\abs{M^{N_1,N_2}_{f,t}}^2\right]^{\frac{1}{2}}+ \E\left[\abs{M^{N_1,N_2}_{t}}^2\right]^{\frac{1}{2}}} + O\paren{N_2^{-(1-\gamma_2)}}\\
&\le C\paren{\frac{1}{N_2^{1-\gamma_2}}}.
\end{align*}
Therefore, $\lim_{N_2 \to \infty} \E_{\pi^{N_1,N_2}} \left[F(\gamma^{N_1},h^{N_1})\right] = 0$. Since $F(\cdot)$ is continuous and $F(\gamma^{N_1,N_2},h^{N_1,N_2})$ is uniformly bounded, we have $\E_{\pi^{N_1}} \left[F(\gamma^{N_1},h^{N_1})\right] = 0$. Hence, $(\gamma^{N_1},h^{N_1})$ satisfies the evolution equation \eqref{h_N1_evolution} and  $\langle f,\gamma^{N_1}_t\rangle = \langle f, \gamma^{N_1}_0\rangle$ for any test function $f\in C^2_b(\R^{1+N_1(1+d)})$.

Since equation \eqref{h_N1_evolution} is a finite-dimensional, linear equation, it has a unique solution. By Prokhorov's theorem, $\pi^{N_1,N_2}$ converges weakly to $\pi^{N_1}$, which is the distribution of $(\gamma^{N_1},h^{N_1})$, the unique solution of \eqref{h_N1_evolution}. Hence, for fixed $N_1$, $(\gamma^{N_1,N_2},h^{N_1,N_2})$ converges in distribution to $(\gamma^{N_1},h^{N_1})$ as $N_2 \to \infty$.

\section{Proof of Theorem \ref{CLT:theorem}}\label{sec::convergence_of_K}
In this section, we look at the convergence of the first order fluctuation process of the network output for fixed $N_1$ and study its limiting behavior as $N_2 \to \infty$. In particular, consider
\[K^{N_1,N_2}_t = N
_2^{\varphi} (h^{N_1,N_2}_t - h^{N_1}_t),\]
where $\varphi$ is dependent on the scaling parameters $\gamma_1, \gamma_2$. We also denote $\eta^{N_1,N_2}_t = N_2^{\varphi}(\gamma^{N_1,N_2}_t - \gamma^{N_1}_0)$.

For $t\in [0,T]$ and $x\in \CX$, by equations \eqref{h_N1N2_evolution} and \eqref{h_N1_evolution}, the evolution of $K^{N_1,N_2}_t(x)$ can be written as
\begin{equation*}
\begin{aligned}
K^{N_1,N_2}_t(x) &= N_2^{\varphi} \left[ \left( h^{N_1,N_2}_t -h^{N_1,N_2}_0 \right) + h^{N_1,N_2}_0-h^{N_1}_t\right]\\
&= N_2^{\varphi} \left\lbrace \int_0^t \int_{\CX\times \CY} \left(y-h^{N_1,N_2}_s(x')\right)\ip{B^1_{x,x'}(\theta), \gamma^{N_1,N_2}_s}\pi(dx',dy)ds\right.\\
&\qquad +\frac{1}{N_1}\sum_{j=1}^{N_1} \int_0^t \int_{\CX\times\CY}\left(y-h^{N_1,N_2}_s(x')\right)\ip{B^{2,j}_{x,x'}(\theta), \gamma^{N_1,N_2}_s} \pi(dx',dy)ds\\
&\qquad \left. +\frac{1}{N_1} \sum_{j=1}^{N_1} \int_0^t \int_{\CX\times \CY} \left(y-h^{N_1,N_2}_s(x')\right)xx'\ip{B^{3,j}_{x}(\theta), \gamma^{N_1,N_2}_s}\ip{B^{3,j}_{x'}(\theta),\gamma^{N_1,N_2}_s}\pi(dx',dy)ds\right\rbrace\\
&\quad -N_2^{\varphi} \left\lbrace\int_0^t \int_{\CX\times \CY} \left(y-h^{N_1}_s(x')\right)\ip{B^1_{x,x'}(\theta), \gamma^{N_1}_0}\pi(dx',dy)ds\right.\\
&\qquad +\frac{1}{N_1}\sum_{j=1}^{N_1} \int_0^t \int_{\CX\times\CY}\left(y-h^{N_1}_s(x')\right)\ip{B^{2,j}_{x,x'}(\theta), \gamma^{N_1}_0} \pi(dx',dy)ds\\
&\qquad \left.+\frac{1}{N_1} \sum_{j=1}^{N_1} \int_0^t \int_{\CX\times \CY} \left(y-h^{N_1}_s(x')\right)xx'\ip{B^{3,j}_{x}(\theta), \gamma^{N_1}_0} \ip{B^{3,j}_{x'}(\theta),\gamma^{N_1}_0}\pi(dx',dy)ds\right\rbrace\\
&\quad + K^{N_1,N_2}_0 + N_2^{\varphi}M^{N_1,N_2}_{t} +O(N_2^{-\gamma_2+\varphi})\\
\end{aligned}
\end{equation*}
By rearranging terms, we obtain
\begin{equation}\label{K_evolution}
\begin{aligned}
&K^{N_1,N_2}_t(x)=\\
&=\int_0^t \int_{\CX\times \CY} \left(y-h^{N_1}_s(x')\right)\ip{B^1_{x,x'}(\theta), \eta^{N_1,N_2}_s}\pi(dx',dy)ds - \int_0^t \int_{\CX\times \CY} K^{N_1,N_2}_s(x')\ip{B^1_{x,x'}(\theta), \gamma^{N_1}_0}\pi(dx',dy)ds\\
&\quad + \frac{1}{N_1}\sum_{j=1}^{N_1} \left\lbrace\int_0^t \int_{\CX\times \CY} \left(y-h^{N_1}_s(x')\right)\ip{B^{2,j}_{x,x'}(\theta), \eta^{N_1,N_2}_s}\pi(dx',dy)ds\right.\\
&\qquad \left. - \int_0^t \int_{\CX\times \CY} K^{N_1,N_2}_s(x')\ip{B^{2,j}_{x,x'}(\theta), \gamma^{N_1}_0}\pi(dx',dy)ds\right\rbrace\\
&\quad  + \frac{1}{N_1}\sum_{j=1}^{N_1}\left\lbrace\int_0^t \int_{\CX\times \CY} \left(y-h^{N_1}_s(x')\right)xx'\ip{B^{3,j}_{x}(\theta), \eta^{N_1,N_2}_s}\ip{B^{3,j}_{x'}(\theta),\gamma^{N_1}_0}\pi(dx',dy)ds\right.\\
&\qquad + \int_0^t \int_{\CX\times \CY} \left(y-h^{N_1}_s(x')\right)xx'\ip{B^{3,j}_{x}(\theta), \gamma^{N_1}_0}\ip{B^{3,j}_{x'}(\theta),\eta^{N_1,N_2}_s}\pi(dx',dy)ds\\
&\qquad \left.- \int_0^t \int_{\CX\times \CY} K^{N_1,N_2}_s(x')xx'\ip{B^{3,j}_{x}(\theta), \gamma^{N_1}_0}\ip{B^{3,j}_{x'}(\theta),\gamma^{N_1}_0}\pi(dx',dy)ds\right\rbrace\\
&\quad+\Gamma^{N_1,N_2}_t(x)+ K^{N_1,N_2}_0 + N_2^{\varphi}M^{N_1,N_2}_{t} +O(N_2^{-\gamma_2+\varphi})
\end{aligned}
\end{equation}
where
$\Gamma^{N_1,N_2}_t(x) =\Gamma^{N_1,N_2}_{1,t}(x) + \Gamma^{N_1,N_2}_{2,t}(x)+ \Gamma^{N_1,N_2}_{3,t}(x)$, and
\begin{equation*}
\begin{aligned}
\Gamma^{N_1,N_2}_{1,t}(x) &= -\frac{1}{N_2^{\varphi}}\int_0^t \int_{\CX\times \CY} K^{N_1,N_2}_s(x')\ip{B^1_{x,x'}(\theta), \eta^{N_1,N_2}_s}\pi(dx',dy)ds,\\
\Gamma^{N_1,N_2}_{2,t}(x) &= -\frac{1}{N_1N_2^{\varphi}}\sum_{j=1}^{N_1} \int_0^t \int_{\CX\times \CY} K^{N_1,N_2}_s(x')\ip{B^{2,j}_{x,x'}(\theta), \eta^{N_1,N_2}_s}\pi(dx',dy)ds\\
\Gamma^{N_1,N_2}_{3,t}(x) &= -\frac{1}{N_1N_2^{\varphi}} \sum_{j=1}^{N_1} \left\lbrace\int_0^t \int_{\CX\times \CY} K^{N_1,N_2}_s(x')xx'\ip{B^{3,j}_{x}(\theta), \eta^{N_1,N_2}_s}\ip{B^{3,j}_{x'}(\theta),\gamma^{N_1,N_2}_s}\pi(dx',dy)ds \right.\\
&\qquad \qquad \qquad + \int_0^t \int_{\CX\times \CY} K^{N_1,N_2}_s(x')xx'\ip{B^{3,j}_{x}(\theta), \gamma^{N_1}_0}\ip{B^{3,j}_{x'}(\theta),\eta^{N_1,N_2}_s}\pi(dx',dy)ds\\
&\qquad \qquad \qquad \left. - \int_0^t \int_{\CX\times \CY} \left(y-h^{N_1}_s(x')\right)xx'\ip{B^{3,j}_{x}(\theta), \eta^{N_1,N_2}_s}\ip{B^{3,j}_{x'}(\theta),\eta^{N_1,N_2}_s}\pi(dx',dy)ds \right\rbrace.
\end{aligned}
\end{equation*}
Recall that when $\gamma_2>1/2$, $h^{N_1}_0(x) = 0$. Therefore,
\[K^{N_1,N_2}_0(x) = N_2^{\varphi}h^{N_1,N_2}_0(x) = N_2^{-(\gamma_2-\frac{1}{2}-\varphi)} \ip{c\sigma(Z^{2,N_1}(x)), \sqrt{N_2} \tilde{\gamma}_0^{N_1,N_2}},\]
which, by the central limit theorem, converges to the Gaussian random variable $\mathcal{G}^{N_1}(x)$  if $\varphi = \gamma_2-(1/2)$ and to 0 if $\varphi < \gamma_2-(1/2)$.

We also need to consider the evolution of $l^{N_1,N_2}_t(f) = \ip{f,\eta^{N_1,N_2}_t}$ for a fixed function $f \in C^2_b(\R^{1+N_1(1+d)})$. By \eqref{measure_evolution}, for $N_2$ large enough, we have
\begin{equation}\label{eta_evolution}
\begin{aligned}
&\ip{f, \eta^{N_1,N_2}_t} - \ip{f, \eta^{N_1,N_2}_0} = N_2^{\varphi} \left(\ip{f,\gamma^{N_1,N_2}_{t}} - \ip{ f,\gamma^{N_1,N_2}_{0}}\right)\\
&=\frac{1}{N_2^{2-\gamma_2-\varphi}}\sum_{k=0}^{\floor{N_2t}-1} \left(y_k-g_k^{N_1,N_2}(x_k)\right) \ip{\partial_{c}f(\theta)  \sigma(Z^{2,N_1}(x_k)),\tilde{\gamma}^{N_1,N_2}_k}\\
&\quad +\frac{1}{N_1^{1-\gamma_1}N_2^{2-\gamma_2-\varphi}}\sum_{k=0}^{\floor{N_2t}-1}\left(y_k-g_k^{N_1,N_2}(x_k)\right) \ip{c\sigma'(Z^{2,N_1}(x_k))\sigma(w^1x_k)\cdot \partial_{w^2}f(\theta) ,\tilde{\gamma}^{N_1,N_2}_k}\\
&\quad + \frac{1}{N_1^{1-\gamma_1}N_2^{2-\gamma_2-\varphi}}\sum_{k=0}^{\floor{N_2t}-1}\left(y_k-g_k^{N_1,N_2}(x_k)\right)\ip{\ip{c\sigma'(Z^{2,N_1}(x_k))\sigma'(w^1x_k)w^{2},\tilde{\gamma}_k^{N_1,N_2}}\cdot  \nabla_{w^1}f(\theta)x_k,\tilde{\gamma}^{N_1,N_2}_k}\\
&\quad + O\left(\frac{1}{N_2^{1-\varphi}}\right).\\
\end{aligned}
\end{equation}
The evolution equations \eqref{K_evolution} and \eqref{eta_evolution} suggest that we consider the convergence of $K^{N_1,N_2}_t$ and $l^{N_1,N_2}_t(f)$ for $\varphi \le \min\{1-\gamma_2, \gamma_2-(1/2)\}$. If $\gamma_2 < \frac{3}{4}$, we can take $\varphi = \gamma_2 -\frac{1}{2} < 1-\gamma_2$ in order to obtain a limiting Gaussian process for $K^{N_1,N_2}_t$. If $\gamma_2 \ge \frac{3}{4}$, the limiting process for $K^{N_1,N_2}_t$ is Gaussian only if $\gamma_2 = \frac{3}{4}$ and $\varphi = 1-\gamma_2 = \gamma_2 - \frac{1}{2}$.

\subsection{Convergence of $l^{N_1,N_2}_t(f) = \ip{f,\eta^{N_1,N_2}_t}$}\label{CLT:sec:eta_properties}
In this section, we establish the convergence of the process $l^{N_1,N_2}_t(f)$ as $N_2 \to \infty$ in $D_{\R}([0,T])$ for a fixed function $f\in C^2_b(\R^{1+N_1(1+d)})$.

Following the same idea as in Section \ref{sec::network_convergence}, we first show that relative compactness holds.
The following lemma implies compact containment of the process $\{l^{N_1,N_2}_t(f)\}$.
\begin{lemma}\label{CLT:lemma:eta_compact_contatinment}
For any fixed $f\in C^2_b(\R^{1+N_1(1+d)})$, when $\varphi \le 1-\gamma_2$, there exist a constant $C<\infty$, such that
\[\sup_{N_2 \in \mathbb{N}, 0\le t\le T} \E \left[\abs{\ip{f,\eta^{N_1,N_2}_t}}^4 \right] < C.\]
Furthermore, for any $\epsilon>0$, there exist a compact subset $U \subset \R$ such that
\[\sup_{N_2\in \mathbb{N}, 0\le t\le T} \P\paren{\ip{f,\eta^{N_1,N_2}_t} \notin U} < \epsilon. \]
\end{lemma}
\begin{proof}
By equation \eqref{eta_evolution}, we have
\begin{align*}
\abs{\ip{f,\eta^{N_1,N_2}_t}}  &\le \abs{\ip{f,\eta^{N_1,N_2}_0}} + \frac{C}{N_2^{2- \gamma_2 -\varphi}}\sum_{k=0}^{\floor{N_2t}-1} \abs{y_k - g^{N_1,N_2}_k(x_k)} + \frac{C}{N_2^{1-\varphi}}\\
&\le  \abs{\ip{f,\eta^{N_1,N_2}_0}} + \frac{C}{N_2^{1- \gamma_2 -\varphi}} \abs{g^{N_1,N_2}_k(x_k)} + \frac{C}{N_2^{1- \gamma_2 -\varphi}}.
\end{align*}

Raising to the forth power on both sides, by H\"older's inequality, we have
\begin{equation}\label{eta_compact}
\begin{aligned}
\abs{\ip{f,\eta^{N_1,N_2}_t}}^4  &\le 9 \paren{\abs{\ip{f,\eta^{N_1,N_2}_0}}^4 + \frac{C}{N_2^{4(1- \gamma_2 -\varphi)}} \abs{g^{N_1,N_2}_k(x_k)}^4 + \frac{C}{N_2^{4(1- \gamma_2 -\varphi)}}}.
\end{aligned}
\end{equation}

Since $\ip{f,\eta^{N_1,N_2}_0} = N_2^{\varphi} \ip{f,  \gamma^{N_1,N_2}_0 - \gamma^{N_1}_0}$, and by independence,
\begin{equation*}
\begin{aligned}
\E\left[\abs{\ip{f,  \gamma^{N_1,N_2}_0 - \gamma^{N_1}_0}}^4\right] &= \E\left[\abs{\frac{1}{N_2} \sum_{i=1}^{N_2} f(\theta^{i}_0)-\ip{f,\gamma^{N_1}_0}}^4 \right]\\
&= \frac{1}{N_2^4}\sum_{i=1}^{N_2} \E\left[\abs{f(\theta^{i}_0)-\ip{f,\gamma^{N_1}_0}}^4\right]<\frac{C}{N_2^3},
\end{aligned}
\end{equation*}
we have
$\E\left[\abs{\ip{f,\eta^{N_1,N_2}_0}}^4\right] \le C(N_2^{4\varphi-3})$.
Taking expectation on both sides of equation \eqref{eta_compact}, by Lemma \ref{lemma_g} and $4\varphi-3<0$, we have
\[\sup_{N_2 \in \mathbb{N}, 0\le t\le T} \E \left[\abs{\ip{f,\eta^{N_1,N_2}_t}}^4 \right] < C,\]
for some $C<\infty$. By Markov's inequality, the compact containment condition of $\ip{f,\eta^{N_1,N_2}_t}$ follows.
\end{proof}

Next, we establish the regularity of $\ip{f,\eta^{N_1,N_2}_t}$. For the following lemma, we define the function $q(z_1,z_2) = \min\{\abs{z_1 - z_2},1\}$, where $z_1,z_2 \in \R$.
\begin{lemma}\label{CLT:eta_regularity}
For $f\in C^2_b(\R^{1+N_1(1+d)})$, $\delta \in (0,1)$, there exist a constant $C <\infty$ such that for any $0\le u \le \delta$, $0\le v\le \delta \wedge t$, and $t \in [0,T]$,
\[\E \left[q\paren{\ip{f,\eta^{N_1,N_2}_{t+u}},\ip{f,\eta^{N_1,N_2}_{t}}}q\paren{\ip{f,\eta^{N_1,N_2}_{t}},\ip{f,\eta^{N_1,N_2}_{t-v}}}\vert \CF^{N_1,N_2}_t\right]\le \frac{C\delta}{N_2^{1-\gamma-\varphi}}+\frac{C}{N_2^{2-\gamma-\varphi}},\]
where $\varphi \le  1-\gamma_2$.
\end{lemma}
\begin{proof}Recall that $\ip{f,\gamma^{N_1}_t} = \ip{f,\gamma^{N_1}_0}$ for any $t \in [0,T]$ and $f\in C^2_b(\R^{1+N_1(1+d)})$. For any $0 \le s < t \le T$,
by the regularity result for $\gamma^{N_1,N_2}_t$ proved in Lemma \ref{LLN:lemma:mu_regularity}, we have
\begin{align*}
\E \left[\abs{\ip{f,\eta^{N_1,N_2}_t} - \ip{f,\eta^{N_1,N_2}_s}} \vert \CF^{N_1,N_2}_s\right] = N_2^{\varphi}\E \left[\abs{\ip{f,\gamma^{N_1,N_2}_t} - \ip{f,\gamma^{N_1,N_2}_s}}\big\vert \CF^{N_1,N_2}_s\right] \le \frac{C\delta}{N_2^{1-\gamma_2-\varphi}}+\frac{C}{N_2^{2-\gamma_2-\varphi}},
\end{align*}
for $0 <s <t \le T$ with $0 < t-s \le \delta <1$. If $\varphi \le 1-\gamma_2$, both terms in the last inequality above are bounded as $N_2$ grows. The statement of the lemma follows.
\end{proof}
Using Lemmas \ref{CLT:lemma:eta_compact_contatinment} and \ref{CLT:eta_regularity}, we are now ready to present the proof of the convergence of $l^{N_1,N_2}_t(f)$. We first show the case when $\varphi < 1-\gamma_2$.
For fixed $f\in C^2_b(\R^{1+N_1(1+d)})$, when $\varphi \le 1-\gamma_2 $, the family of processes $\{\langle f,\eta^{N_1,N_2}_t\rangle, t\in[0,T]\}_{N_2\in \mathbb{N}}$ is relatively compact in $D_{\R}([0,T])$ due to Lemmas \ref{CLT:lemma:eta_compact_contatinment}, \ref{CLT:eta_regularity}, and Theorem 8.6 of Chapter 3 of \cite{EthierAndKurtz}.
For simplicity, we denote $l^{N_1,N_2}_t = \langle {f,\eta^{N_1,N_2}_t}\rangle$.  Let $\pi^{N_1,N_2} \in \mathcal{M}\paren{D_{\R}([0,T]}$ be the probability measure corresponding to $l^{N_1,N_2}_t$. Relative compactness implies that there is a subsequence $\pi^{N_1,N_{2_k}}$ that converges weakly to a limit point $\pi^{N_1}$. We show that $\pi^{N_1}$ is a Dirac measure concentrated on zero when $\varphi < 1-\gamma_2$.

For $t \in [0,T]$, $g_1, \ldots, g_p \in C_b(\R)$, and $0 \le s_1< \cdots < s_p \le t$, define a map $F(l):D_{\R}([0,T]) \to \R_+$ as
\[F(l) = \abs{\paren{l_t - 0} \times g_1(l_{s_1})\times \cdots \times g_p(l_{s_p})}.\]
By equation \eqref{eta_evolution} and the fact that $\ip{f,\eta^{N_1,N_2}_0} = N_2^{\varphi - \frac{1}{2}} \ip{f, \sqrt{N_2} \paren{\gamma^{N_1,N_2}_0 - \gamma^{N_1}_0}} = O_p(N_2^{\varphi -\frac{1}{2}})$, we have
\begin{align*}
\E_{\pi^{N_1,N_2}} \left[F(l)\right] &= \E\left[ F(l^{N_1,N_2}) \right] \\
&= \E \left[ \abs{\paren{\ip{f,\eta^{N_1,N_2}_t} - \ip{f,\eta^{N_1,N_2}_0} + \ip{f,\eta^{N_1,N_2}_0}} \times \prod_{i=1}^p g_i(l^{N_1,N_2}_{s_i})}\right]\\
&\le \E \left[ \abs{\paren{\ip{f,\eta^{N_1,N_2}_t} - \ip{f,\eta^{N_1,N_2}_0}} \times \prod_{i=1}^p g_i(l^{N_1,N_2}_{s_i})}\right] + \E \left[ \abs{ \ip{f,\eta^{N_1,N_2}_0} \times \prod_{i=1}^p g_i(l^{N_1,N_2}_{s_i})}\right]\\
&\le C \paren{\frac{1}{N_2^{1-\gamma_2 - \varphi}} + \frac{1}{N_2^{1-\varphi}} + \frac{1}{N_2^{\frac{1}{2}-\varphi}}}.
\end{align*}
Since $F(\cdot)$ is continuous and $F(l^{N_1,N_2})$ is uniformly bounded, we have
\[\lim_{N_2\to \infty} \E_{\pi^{N_1,N_2}} \left[F(l) \right] = \E_{\pi^{N_1}} \left[F(l)\right] = 0,\]
where $\pi^{N_1}$ is the Dirac measure concentrated on 0.
We have shown that the limit point $\pi$ of any convergence subsequence, which exists due to relative compactness, is the Dirac measure concentrated on 0. Therefore, by Prokhorov's theorem, $\pi^{N_1,N_2}$ weakly converges to 0. As $N_2\to \infty$, $l^{N_1,N_2}(f) \xrightarrow{d} 0$ and thus the limit is in probability. This concludes the proof for case 1: $\varphi < 1-\gamma_2$.

The proof for case 2: $\varphi = 1-\gamma_2$ is more subtle and is given in different steps below. We see that the evolution of $l^{N_1,N_2}_t(f)$  becomes
\begin{equation*}
\begin{aligned}
&\ip{f, \eta^{N_1,N_2}_t} - \ip{f, \eta^{N_1,N_2}_0} \\
&=\frac{1}{N_2}\sum_{k=0}^{\floor{N_2t}-1} \int_{\CX\times \CY} \left(y-g_k^{N_1,N_2}(x')\right) \ip{\partial_{c}f(\theta)  \sigma(Z^{2,N_1}(x')),\tilde{\gamma}^{N_1,N_2}_k} \pi(dx',dy)\\
&\quad +\frac{1}{N_1^{1-\gamma_1}N_2}\sum_{k=0}^{\floor{N_2t}-1}\int_{\CX\times \CY}\left(y-g_k^{N_1,N_2}(x')\right) \ip{c\sigma'(Z^{2,N_1}(x'))\sigma(w^1x')\cdot \partial_{w^2}f(\theta) ,\tilde{\gamma}^{N_1,N_2}_k}\pi(dx',dy)\\
&\quad + \frac{1}{N_1^{1-\gamma_1}N_2}\sum_{k=0}^{\floor{N_2t}-1}\int_{\CX\times \CY}\left(y-g_k^{N_1,N_2}(x')\right)\ip{\ip{c\sigma'(Z^{2,N_1}(x'))\sigma'(w^1x')w^{2},\tilde{\gamma}_k^{N_1,N_2}}\cdot  \nabla_{w^1}f(\theta)x',\tilde{\gamma}^{N_1,N_2}_k}\pi(dx',dy)\\
&\quad + M^{N_1,N_2}_{\eta,1,t} + M^{N_1,N_2}_{\eta,2,t} +  M^{N_1,N_2}_{\eta,3,t}+O\left(\frac{1}{N_2^{\gamma_2}}\right).\\
\end{aligned}
\end{equation*}
where
\begin{align*}
M^{N_1,N_2}_{\eta,1,t} &=  \frac{1}{N_2}\left \lbrace \sum_{k=0}^{\floor{N_2t}-1}\left(y_k-g_k^{N_1,N_2}(x_k)\right) \ip{\partial_{c}f(\theta)  \sigma(Z^{2,N_1}(x_k)),\tilde{\gamma}^{N_1,N_2}_k}\right.\\
&\qquad \quad \left. - \int_{\CX\times \CY} \left(y-g_k^{N_1,N_2}(x')\right) \ip{\partial_{c}f(\theta)  \sigma(Z^{2,N_1}(x')),\tilde{\gamma}^{N_1,N_2}_k} \pi(dx',dy) \right \rbrace, \\
M^{N_1,N_2}_{\eta,2,t} &=  \frac{1}{N_1^{1-\gamma_1}N_2} \left \lbrace \sum_{k=0}^{\floor{N_2t}-1} \left(y_k-g_k^{N_1,N_2}(x_k)\right) \ip{c\sigma'(Z^{2,N_1}(x_k))\sigma(w^1x_k)\cdot \partial_{w^2}f(\theta) ,\tilde{\gamma}^{N_1,N_2}_k}\right.\\
&\qquad \qquad \left. -\int_{\CX\times \CY}\left(y-g_k^{N_1,N_2}(x')\right) \ip{c\sigma'(Z^{2,N_1}(x'))\sigma(w^1x')\cdot \partial_{w^2}f(\theta) ,\tilde{\gamma}^{N_1,N_2}_k}\pi(dx',dy)  \right \rbrace,\\
M^{N_1,N_2}_{\eta,3,t} &=  \frac{1}{N_1^{1-\gamma_1}N_2} \left \lbrace \sum_{k=0}^{\floor{N_2t}-1} \left(y_k-g_k^{N_1,N_2}(x_k)\right)\ip{\ip{c\sigma'(Z^{2,N_1}(x_k))\sigma'(w^1x_k)w^{2},\tilde{\gamma}_k^{N_1,N_2}}\cdot  \nabla_{w^1}f(\theta)x_k,\tilde{\gamma}^{N_1,N_2}_k}\right.\\
&\qquad \qquad \left. -\int_{\CX\times \CY}\left(y-g_k^{N_1,N_2}(x')\right)\ip{\ip{c\sigma'(Z^{2,N_1}(x'))\sigma'(w^1x')w^{2},\tilde{\gamma}_k^{N_1,N_2}}\cdot  \nabla_{w^1}f(\theta)x',\tilde{\gamma}^{N_1,N_2}_k}\pi(dx',dy) \right \rbrace.
\end{align*}

As $N_2$ grows, we can rewrite this equation in terms of Riemann integrals and scaled measure $\gamma^{N_1,N_2}_t$,
\begin{equation}\label{CLT:eta_evolution2}
\begin{aligned}
&\ip{f, \eta^{N_1,N_2}_t} - \ip{f, \eta^{N_1,N_2}_0} \\
&=\int_0^t \int_{\CX\times \CY} \left(y-h_s^{N_1,N_2}(x')\right) \ip{\partial_{c}f(\theta)  \sigma(Z^{2,N_1}(x')),{\gamma}^{N_1,N_2}_s} \pi(dx',dy)ds\\
&\quad +\frac{1}{N_1^{1-\gamma_1}}\int_0^t\int_{\CX\times \CY}\left(y-h_s^{N_1,N_2}(x')\right) \ip{c\sigma'(Z^{2,N_1}(x'))\sigma(w^1x')\cdot \partial_{w^2}f(\theta) ,{\gamma}^{N_1,N_2}_s}\pi(dx',dy)ds\\
&\quad + \frac{1}{N_1^{1-\gamma_1}}\int_0^t\int_{\CX\times \CY}\left(y-h_s^{N_1,N_2}(x')\right)\ip{\ip{c\sigma'(Z^{2,N_1}(x'))\sigma'(w^1x')w^{2},{\gamma}_s^{N_1,N_2}}\cdot  \nabla_{w^1}f(\theta)x',{\gamma}^{N_1,N_2}_s}\pi(dx',dy)ds\\
&\quad + M^{N_1,N_2}_{\eta,1,t} + M^{N_1,N_2}_{\eta,2,t} +  M^{N_1,N_2}_{\eta,3,t}+O\left(\frac{1}{N_2^{\gamma_2}}\right).\\
\end{aligned}
\end{equation}

Fir any fixed $f \in C^2_b(\R^{1+N_1(1+d)})$, similar analysis as in Lemma 3.1 in \cite{SirignanoSpiliopoulosNN1}, we have the following bound for terms $ M^{N_1,N_2}_{\eta,i,t}, i=1,2,3$.
\begin{lemma}\label{CLT:lemma:martingale2}
For any $N \in \mathbb{N}$, there is a constant $C<\infty$ such that
\begin{align*}
\E\left[\sup_{t\in [0,T]} \left(\abs{M^{N_1,N_2}_{\eta,1,t}}^2 + \abs{M^{N_1,N_2}_{\eta,2,t}}^2 + \abs{M^{N_1,N_2}_{\eta,3,t}}^2\right) \right] \le \frac{C}{N_2}.
\end{align*}
\end{lemma}
From equation \eqref{CLT:eta_evolution2}, we see that the evolution of $l^{N_1,N_2}_t(f)$ involves the evolution of $\gamma^{N_1,N_2}_t$ and $h^{N_1,N_2}_t$. In the next lemma, we prove the convergence of the processes $(\gamma^{N_1,N_2}_t, h^{N_1,N_2}_t,l^{N_1,N_2}_t(f))$ in distribution in the space $D_{E'}([0,T])$, where $E' = \CM(\R^{1+N_1(1+d)}) \times \R^M \times \R$.
The convergence of $l^{N_1,N_2}_t(f)$ case 2: $\varphi = 1-\gamma$ then follows from Lemma \ref{CLT:lemma:lt_1-gamma}.
\begin{lemma}\label{CLT:lemma:lt_1-gamma}
For any fixed $f\in C^2_b(\R^{1+N_1(1+d)})$, if $\varphi = 1-\gamma_2$, the processes $(\gamma^{N_1,N_2}_t, h^{N_1,N_2}_t,l^{N_1,N_2}_t(f))$ converges in distribution in $D_{E'}([0,T])$ to $(\gamma^{N_1}_0, h^{N_1}_t,l^{N_1}_t(f))$, where $h^{N_1}_t$ satisfies equation \eqref{h_N1_evolution} and $l^{N_1}_t(f)$ is given by

\begin{equation}\label{CLT:evolution_l}
\begin{aligned}
l^{N_1}_{t}(f)&=\int_0^t \int_{\CX\times \CY} \left(y-h_s^{N_1}(x')\right) \ip{\partial_{c}f(\theta)  \sigma(Z^{2,N_1}(x')),{\gamma}^{N_1}_0} \pi(dx',dy)ds\\
&\quad +\frac{1}{N_1^{1-\gamma_1}}\int_0^t\int_{\CX\times \CY}\left(y-h_s^{N_1}(x')\right) \ip{c\sigma'(Z^{2,N_1}(x'))\sigma(w^1x')\cdot \partial_{w^2}f(\theta) ,{\gamma}^{N_1}_0}\pi(dx',dy)ds\\
&\quad + \frac{1}{N_1^{1-\gamma_1}}\int_0^t\int_{\CX\times \CY}\left(y-h_s^{N_1}(x')\right)\ip{\ip{c\sigma'(Z^{2,N_1}(x'))\sigma'(w^1x')w^{2},{\gamma}_0^{N_1}}\cdot  \nabla_{w^1}f(\theta)x',{\gamma}^{N_1}_0}\pi(dx',dy)ds\\
\end{aligned}
\end{equation}
\end{lemma}
\begin{proof}By Lemmas \ref{CLT:lemma:eta_compact_contatinment} and \ref{CLT:eta_regularity}, $\{l^{N_1,N_2}(f)\}_{N_2\in\mathbb{N}}$ is relatively compact in $D_{\R}([0,T])$. By Lemma \ref{LLN:lemma:relative_compact}, $\{\gamma^{N_1,N_2}, h^{N_1,N_2}\}_{N_2 \in \mathbb{N}}$ is relatively compact in $D_E([0,T])$, where $E= \CM(\R^{1+N_1(1+d)}) \times \R^M $. Since relative compactness is equivalent to tightness, we have that the probability measures of the family of processes $\{l^{N_1,N_2}(f)\}_{N_2\in\mathbb{N}}$ and the probability measures of the family of processes $\{\gamma^{N_1,N_2}, h^{N_1,N_2}\}_{N_2 \in \mathbb{N}}$ are tight. Therefore, $\{\gamma^{N_1,N_2}, h^{N_1,N_2},l^{N_1,N_2}(f)\}_{N_2 \in \mathbb{N}}$ is tight, hence it is also relatively compact.

Denote $\pi^{N_1,N_2} \in \CM(D_{E'}([0,T])$ the probability measure corresponding to $(\gamma^{N_1,N_2}, h^{N_1,N_2},l^{N_1,N_2}(f))$. Relative compactness implies that there is a subsequence $\pi^{N_1,N_{2_k}}$ that converges weakly.
We now show that any limit point $\pi$ of a convergent subsequence  $\pi^{N_1,N_{2_k}}$ is a Dirac measure concentrated on $(\gamma^{N_1},h^{N_1},l^{N_1}(f))\in D_{E'}([0,T])$, where $(\gamma^{N_1},h^{N_1},l^{N_1}(f))$ satisfies equations \eqref{h_N1_evolution} and \eqref{CLT:evolution_l}.
Define a map $F_1(\gamma^{N_1},h^{N_1},l^{N_1}(f)):D_{E'}([0,T]) \to \R_+$ for each $t \in [0,T]$, $m_1,\ldots,m_p \in C_b(\R)$, and $0 \le s_1 < \cdots < s_p \le t$.
\begin{equation}\label{F_1}
\begin{aligned}
&F_1(\gamma,h,l(f)) \\
&= F(\gamma^{N_1},h^{N_1}) + \left\vert \left(l^{N_1}_{t}(f)-\int_0^t \int_{\CX\times \CY} \left(y-h_s^{N_1}(x')\right) \ip{\partial_{c}f(\theta)  \sigma(Z^{2,N_1}(x')),{\gamma}^{N_1}_s} \pi(dx',dy)ds \right.\right. \\
&\qquad  \qquad-\frac{1}{N_1^{1-\gamma_1}}\int_0^t\int_{\CX\times \CY}\left(y-h_s^{N_1}(x')\right) \ip{c\sigma'(Z^{2,N_1}(x'))\sigma(w^1x')\cdot \partial_{w^2}f(\theta) ,{\gamma}^{N_1}_s}\pi(dx',dy)ds\\
&\qquad  \qquad- \left.\frac{1}{N_1^{1-\gamma_1}}\int_0^t\int_{\CX\times \CY}\left(y-h_s^{N_1}(x')\right)\ip{\ip{c\sigma'(Z^{2,N_1}(x'))\sigma'(w^1x')w^{2},{\gamma}_s^{N_1}}\cdot  \nabla_{w^1}f(\theta)x',{\gamma}^{N_1}_s}\pi(dx',dy)ds\right)\\
&\qquad \qquad \left. \times m_1(l^{N_1}_{s_1}(f)) \times \cdots \times  m_p(l^{N_1}_{s_p}(f))\right\vert,
\end{aligned}
\end{equation}
where $F(\gamma^{N_1},h^{N_1})$ is as given in equation \eqref{LNN:identify_limit_eq}. Using equation \eqref{CLT:eta_evolution2}, Lemma \ref{CLT:lemma:martingale2}, the analysis of $F(\gamma^{N_1},h^{N_1})$ in Section \ref{sec::Identification of the Limit} and the fact that $\ip{f,\eta^{N_1,N_2}_0} = O_p(N_2^{\varphi -\frac{1}{2}})$, we obtain
\begin{align*}
&\E_{\pi^{N_1,N_2}} \left[F_1(\gamma,h,l(f))\right] = \E\left[F(\gamma^{N_1,N_2},h^{N_1,N_2})\right] \\
&\qquad + \E\left[\abs{ \paren{\ip{f,\eta^{N_1,N_2}_0} + M^{N_1,N_2}_{\eta,1,t} + M^{N_1,N_2}_{\eta,2,t} +  M^{N_1,N_2}_{\eta,3,t} + O\paren{N_2^{-\gamma_2}}}\times \prod_{i=1}^p m_i\paren{l^{N_1,N_2}_{s_i}(f)}}\right]\\
&\quad\le C\paren{\frac{1}{N_2^{1-\gamma_2}}} +  C \paren{\E\left[\abs{M^{N_1,N_2}_{\eta,1,t}}^2\right]^{\frac{1}{2}}+\E\left[\abs{M^{N_1,N_2}_{\eta,2,t}}^2\right]^{\frac{1}{2}} +\E\left[\abs{M^{N_1,N_2}_{\eta,3,t}}^2\right]^{\frac{1}{2}}} + C\paren{\frac{1}{N_2^{\frac{1}{2}-\varphi}}} \\
&\quad\le C\paren{\frac{1}{N_2^{1-\gamma_2}} + \frac{1}{N_2^{\frac{1}{2}-\varphi}}}.
\end{align*}

Therefore, $\lim_{N_2\to \infty} \E_{\pi^{N_1,N_2}}[F_1(\gamma^{N_1},h^{N_1},l^{N_1}(f))] = 0$. Since $F(\cdot)$ is continuous and $F(\gamma^{N_1,N_2},h^{N_1,N_2})$ is uniformly bounded, together with analysis in Section \ref{CLT:sec:eta_properties}, we have that $F_1(\cdot)$ is continuous and thus $F_1(\gamma^{N_1,N_2},h^{N_1,N_2},l^{N_1,N_2}(f))$ is uniformly bounded. Hence,
\[\lim_{N_2\to \infty} \E_{\pi^{N_1,N_2}}\left[F_1(\gamma^{N_1},h^{N_1},l^{N_1}(f))\right] = 0.\]

We have shown that any limit point $\pi^{N_1}$ of a convergent subsequence must be a Dirac measure concentrated $(\gamma^{N_1},h^{N_1},l^{N_1}(f))\in D_{E'}([0,T])$, where $(\mu^{N_1},h^{N_1},l^{N_1}(f))$ satisfies equations \eqref{h_N1_evolution}, \eqref{CLT:evolution_l} and $\gamma^{N_1}_t = \gamma^{N_1}_0$ weakly. By Prokhorov's theorem, the processes $(\gamma^{N_1,N_2}_t, h^{N_1,N_2}_t,l^{N_1,N_2}_t(f))$ converges in distribution to $(\gamma^{N_1}_0,h^{N_1}_t,l^{N_1}_{t}(f))$.
\end{proof}

\subsection{Relative Compactness of $K^{N_1,N_2}_t$}\label{sec::relative compactness K}

We begin this section by proving the following lemma for the term $N_2^{\varphi}M^{N_1,N_2}_t$.
\begin{lemma}\label{CLT:lemma:martingale_bound}
For any $N_2 \in \mathbb{N}$ and $x\in\mathcal{X}$, there is a constant $C<\infty$ such that
\begin{align*}
\E\left[\sup_{t\in [0,T]} \abs{N_2^{\varphi} M^{N_1,N_2}_t(x)}^2  \right] \le \frac{C}{N_2^{1-2\varphi}}.
\end{align*}
\end{lemma}
\begin{proof}Recall that $M^{N_1,N_2}_t = M^{N_1,N_2}_{1,t} + M^{N_1,N_2}_{2,t}+M^{N_1,N_2}_{3,t}$, which are defined in \eqref{M_1} to \eqref{M_3}. Let $\mathfrak{F}_t$ be the $\sigma$-algebra generated by $\gamma^{N_1,N_2}_s$, $M^{N_1,N_2}_{1,s}$, $M^{N_1,N_2}_{2,s}$ and $M^{N_1,N_2}_{1,s}$ for $s\le t$. Since for any $t > r$, we have

\begin{align*}
& \E \left[N_2^{\varphi} \paren{M^{N_1,N_2}_{1,t}(x) - M^{N_1,N_2}_{1,r}(x)} \vert \mathfrak{F}_r\right]\\
&\quad=\frac{1}{N_2^{1-\varphi}} \sum_{k=\floor{N_2r}}^{\floor{N_2t}-1}\E \left[\left(y_k-g^{N_1,N_2}_k(x_k)\right) \ip{\sigma\left(Z^{2,N_1}(x_k)\right)\sigma\left(Z^{2,N_1}(x)\right), \tilde{\gamma}^{N_1,N_2}_k} \right. \\
&\qquad\qquad \qquad \qquad \left.-\int_{\CX\times \CY} \left(y-g^{N_1,N_2}_k(x')\right)\ip{\sigma\left(Z^{2,N_1}(x')\right)\sigma\left(Z^{2,N_1}(x)\right), \tilde{\gamma}^{N_1,N_2}_k}\pi(dx',dy)\vert \mathcal{F}^{N_1,N_2}_r\right]\\
&\quad= \frac{1}{N_2^{1-\varphi}} \cdot 0 = 0.
\end{align*}

Therefore, we have
\[\E \left[N_2^{\varphi} M^{N_1,N_2}_{1,t}(x)  \vert \mathfrak{F}_r\right] = \E \left[N_2^{\varphi} \paren{M^{N_1,N_2}_{1,t}(x) - M^{N_1,N_2}_{1,r}(x)} \vert \mathfrak{F}_r\right]+\E \left[N_2^{\varphi} M^{N_1,N_2}_{1,r}(x)  \vert \mathfrak{F}_r\right] = 0 + N_2^{\varphi}M^{N_1,N_2}_{1,r}(x),\]
proving the martingale property for the process $N_2^{\varphi} M^{N_1,N_2}_{1,t}(x)$ and $x\in\mathcal{X}$. 
Hence, by Lemma \ref{LLN:lemma:Mt_bound} and Doob's martingale inequality, we have
\[\E\left[\sup_{t\in [0,T]} \abs{N_2^{\varphi} M^{N_1,N_2}_{1,t}(x)}^2 \right] \le CN_2^{2\varphi} \E\left[ \abs{ M^{N_1,N_2}_{1,T}(x)}^2 \right] \le \frac{C}{N_2^{1-2\varphi}},\]
where the constant $C<\infty$. 
Note that since $\gamma <1$ and $\varphi \le \gamma-\frac{1}{2}$, we have $1-2\varphi >0$.

Similar analysis gives
\[\E\left[\sup_{t\in [0,T]} \abs{N_2^{\varphi} M^{N_1,N_2}_{2,t}(x)}^2 \right]  \le \frac{C}{N_2^{1-2\varphi}}, \quad \E\left[\sup_{t\in [0,T]} \abs{N_2^{\varphi} M^{N_1,N_2}_{3,t}(x)}^2 \right]  \le \frac{C}{N_2^{1-2\varphi}}. \]
Hence,
\begin{align*}
\E\left[\sup_{t\in [0,T]} \abs{N_2^{\varphi} M^{N_1,N_2}_t(x)}^2  \right] \le C\sum_{i=1}^{3} \E\left[\sup_{t\in [0,T]} \abs{N_2^{\varphi} M^{N_1,N_2}_{i,t}(x)}^2  \right]  \le \frac{C}{N_2^{1-2\varphi}}.
\end{align*}
\end{proof}

The next three lemmas prove relative compactness of the family $\{K^{N_1,N_2}_t, t\in [0,T]\}_{N_2\in \mathbb{N}}$ in $D_{\mathbb{R}^{M}}([0,T])$.
\begin{lemma}\label{CLT:lemma:bound of ex_Kt}
There exist a constant $C<\infty$, such that for each $x\in\mathcal{X}$,
\[\sup_{N_2 \in \mathbb{N}, 0\le t\le T} \E \left[\abs{K^{N_1,N_2}_t(x)}^2 \right] < C.\]
In particular, for any $\epsilon >0$, there exist a compact subset $U \subset \R^M$ such that
\[\sup_{N_2\in \mathbb{N}, 0\le t\le T} \P\paren{K^{N_1,N_2}_t \notin U} < \epsilon.\]
\end{lemma}

\begin{proof}By \eqref{K_evolution} and Cauchy-Schwarz inequality, we have
\begin{align*}
\abs{K^{N_1,N_2}_t(x)}^2
&\le C \brac{(I)^2 + (II)^2 + \abs{\Gamma^{N_1,N_2}_t}^2 + \abs{N_2^{1-\gamma_2 + \varphi}\ip{c\sigma(Z^{2,N_1}(x)),\gamma^{N_1,N_2}_0}}^2  + \abs{N_2^{\varphi}M_t^{N_1,N_2}}^2 + O \paren{N_2^{-2(\gamma_2-\varphi)}}},
\end{align*}
where
\begin{align*}
(I) &= \int^t_0 \int_{\CX \times \CY} \abs{y-h^{N_1}_s(x')} \left(\abs{\ip{B^1_{x,x'}(\theta),\eta^{N_1,N_2}_s}} + \frac{1}{N_1}\sum_{j=1}^{N_1} \abs{\ip{B^{2,j}_{x,x'}(\theta),\eta^{N_1,N_2}_s}}\right) \pi(dx',dy) ds \\
&\quad +  \frac{1}{N_1}\sum_{j=1}^{N_1}\int_0^t \int_{\CX\times \CY} \abs{y-h^{N_1}_s(x')}\\
&\qquad \qquad \cdot \left(\abs{\ip{xx'B^{3,j}_{x}(\theta), \eta^{N_1,N_2}_s}}\abs{\ip{B^{3,j}_{x'}(\theta),\gamma^{N_1}_0}} + \abs{\ip{B^{3,j}_{x}(\theta), \gamma^{N_1}_0}}\abs{\ip{xx'B^{3,j}_{x'}(\theta),\eta^{N_1,N_2}_s}}\right)\pi(dx',dy)ds\\
(II) &=  \int^t_0 \int_{\CX \times \CY} \abs{K^{N_1,N_2}_s(x')} \abs{\ip{B^1_{x,x'}(\theta), \gamma^{N_1}_0}} + \frac{1}{N_1}\sum_{j=1}^{N_1} \abs{\ip{B^{2,j}_{x,x'}(\theta), \gamma^{N_1}_0}}\pi(dx',dy) ds\\
&\quad + \frac{1}{N_1}\sum_{j=1}^{N_1}\int^t_0 \int_{\CX \times \CY} \abs{K^{N_1,N_2}_s(x')}\abs{\ip{B^{3,j}_{x}(\theta), \gamma^{N_1}_0}}\abs{\ip{xx'B^{3,j}_{x'}(\theta),\gamma^{N_1}_0}}\pi(dx',dy)ds.
\end{align*}

By Assumption \ref{assumption}, definition of $\gamma^{N_1}_0$, and Lemma \ref{lemma_parameterbound},
there exist some constant $C<\infty$, such that
\begin{equation}\label{B0_bound}
\sup_{x,x'\in\CX} \left\lbrace \abs{\ip{B^1_{x,x'}(\theta), \gamma^{N_1}_0}}+ \frac{1}{N_1}\sum_{j=1}^{N_1} \left(\abs{\ip{B^{2,j}_{x,x'}(\theta), \gamma^{N_1}_0}} + \abs{\ip{B^{3,j}_{x}(\theta), \gamma^{N_1}_0}}\abs{\ip{xx'B^{3,j}_{x'}(\theta),\gamma^{N_1}_0}}\right)\right\rbrace<C.
\end{equation}
Then, by the Cauchy-Schwarz inequality and equation \eqref{h_N1_evolution}, we have
\begin{align*}
\abs{h^{N_1}_t(x)}^2 &\le C \left[ \paren{ \int^t_0 \int_{\CX \times \CY} \abs{y}  \pi(dx',dy) ds}^2 + \paren{\int^t_0 \int_{\CX \times \CY} \abs{h^{N_1}_s(x')}  \pi(dx',dy) ds}^2 \right],\\
&\le C_1 t^2 + C_2t \int^t_0 \int_{\CX \times \CY} \abs{h^{N_1}_s(x')}^2  \pi(dx',dy) ds,
\end{align*}
which implies that,
\[\sup_{t \in [0,T]}\int_{\CX \times \CY} \abs{h^{N_1}_t(x)}^2  \pi(dx,dy) \le C_1 T^2 + C_2 T \int^t_0 \int_{\CX \times \CY} \abs{h^{N_1}_s(x')}^2  \pi(dx',dy) ds.\]
Therefore, by Gr\"onwall's inequality,
\begin{equation}\label{CLT:network_bound}
\sup_{0\le t\le T}\int_{\CX \times \CY} \abs{h^{N_1}_t(x)}^2  \pi(dx,dy) \le  \sup_{0\le t\le T} C_1T^2 \exp(C_2Tt) < C(T),
\end{equation}
for some constant $C(T)<\infty$ depending on $T$. By Cauchy-Schwarz inequality and \eqref{B0_bound}, we also have
\begin{equation}\label{K_II_bound}
(II)^2 \le C_3t  \int^t_0 \int_{\CX \times \CY} \abs{K^{N_1,N_2}_s(x')}^2  \pi(dx',dy) ds.
\end{equation}
 Since $\sigma \in C^{\infty}_b(\R)$, by Lemma \ref{CLT:lemma:eta_compact_contatinment}, there exist some constant $C < \infty$ such that
\begin{equation}\label{CLT:network_eta_bound}
\E\left[\abs{\ip{B^1_{x,x'}(\theta),\eta^{N_1,N_2}_s}}^2\right] < C, \quad \E\left[ \abs{\ip{B^{2,j}_{x,x'}(\theta),\eta^{N_1,N_2}_s}}^2\right]<C, \quad  \E\left[ \abs{\ip{xx'B^{3,j}_{x'}(\theta),\eta^{N_1,N_2}_s}}^2\right]<C
\end{equation}
for $t \in [0,T]$, $j=1,\ldots, N_1$, and $N_2 \in \mathbb{N}$. By the Cauchy-Schwarz inequality, equations \eqref{CLT:network_bound}, \eqref{CLT:network_eta_bound}, and Assumption \ref{assumption}, we have
\begin{equation*}
\begin{aligned}
\E[(I)^2] &\le Ct \int^t_0 \int_{\CX \times \CY} \E\left(\abs{\ip{B^1_{x,x'}(\theta),\eta^{N_1,N_2}_s}}^2 + \frac{1}{N_1}\sum_{j=1}^{N_1} \abs{\ip{B^{2,j}_{x,x'}(\theta),\eta^{N_1,N_2}_s}}^2\right) \pi(dx',dy) ds \\
&\quad +  Ct\int_0^t \int_{\CX\times \CY} \frac{1}{N_1}\sum_{j=1}^{N_1}\E \left(\abs{\ip{xx'B^{3,j}_{x}(\theta), \eta^{N_1,N_2}_s}}^2 + \abs{\ip{xx'B^{3,j}_{x'}(\theta),\eta^{N_1,N_2}_s}}^2\right)\pi(dx',dy)ds\\
&\le C_4t^2
\end{aligned}
\end{equation*}

Since $\abs{\Gamma ^{N_1,N_2}_t}^2 \le C \left(\abs{\Gamma^{N_1,N_2}_{1,t}}^2+ \abs{\Gamma^{N_1,N_2}_{2,t}}^2+ \abs{\Gamma ^{N_1,N_2}_{3,t}}^2\right)$, by Assumption \ref{assumption} and Lemma \ref{lemma_parameterbound}, we have
\begin{align*}
\abs{\Gamma^{N_1,N_2}_{1,t}}^2 & \le C \int^t_0 \int_{\CX \times \CY} \abs{K^{N_1,N_2}_s(x')}^2 \pi(dx',dy) ds \int^t_0 \int_{\CX \times \CY} \abs{\ip{B^1_{x,x'}(\theta),\gamma^{N_1,N_2}_s-\gamma^{N_1}_0}}^2 \pi(dx',dy) ds\\
& \le Ct\int^t_0 \int_{\CX \times \CY} \abs{K^{N_1,N_2}_s(x')}^2 \pi(dx',dy) ds,\\
\abs{\Gamma^{N_1,N_2}_{2,t}}^2 & \le C \int^t_0 \int_{\CX \times \CY} \abs{K^{N_1,N_2}_s(x')}^2 \pi(dx',dy) ds \int^t_0 \int_{\CX \times \CY}\frac{1}{N_1}\sum_{j=1}^{N_1} \abs{\ip{B^{2,j}_{x,x'}(\theta),\gamma^{N_1,N_2}_s-\gamma^{N_1}_0}}^2 \pi(dx',dy) ds\\
& \le Ct\int^t_0 \int_{\CX \times \CY} \abs{K^{N_1,N_2}_s(x')}^2 \pi(dx',dy) ds,\\
\abs{\Gamma^{N_1,N_2}_{3,t}(x)}^2 &\le C  \int_0^t \int_{\CX\times \CY} \abs{K^{N_1,N_2}_s(x')}^2 \pi(dx',dy)ds \\
&\qquad \qquad \qquad \cdot  \int_0^t \int_{\CX\times \CY} \frac{1}{N_1}\sum_{j=1}^{N_1}\abs{\ip{xx'B^{3,j}_{x}(\theta), \gamma^{N_1,N_2}_s-\gamma^{N_1}_0}}^2 \abs{\ip{B^{3,j}_{x'}(\theta),\gamma^{N_1,N_2}_s}}^2\pi(dx',dy)ds \\
&\quad+ C \int_0^t \int_{\CX\times \CY} \abs{K^{N_1,N_2}_s(x')}^2\pi(dx',dy)ds\\
&\qquad \qquad \qquad \cdot \int_0^t \int_{\CX\times \CY}\frac{1}{N_1}\sum_{j=1}^{N_1} \abs{\ip{B^{3,j}_{x}(\theta), \gamma^{N_1}_0}}^2 \abs{\ip{xx'B^{3,j}_{x'}(\theta),\gamma^{N_1,N_2}_s-\gamma^{N_1}_0}}^2\pi(dx',dy)ds\\
&\quad  + C \int_0^t \int_{\CX\times \CY} \abs{y-h^{N_1}_s(x')}^2 \pi(dx',dy)ds\\
&\qquad \qquad \qquad \cdot \int_0^t \int_{\CX\times \CY} \frac{1}{N_1}\sum_{j=1}^{N_1}\abs{\ip{ xx'B^{3,j}_{x}(\theta), \gamma^{N_1,N_2}_s-\gamma^{N_1}_0}}^2 \abs{\ip{B^{3,j}_{x'}(\theta),\eta^{N_1,N_2}_s}}^2\pi(dx',dy)ds\\
&\le Ct\int^t_0 \int_{\CX \times \CY} \abs{K^{N_1,N_2}_s(x')}^2 \pi(dx',dy) ds + Ct\int_0^t \int_{\CX\times \CY} \frac{1}{N_1}\sum_{j=1}^{N_1}\abs{\ip{B^{3,j}_{x'}(\theta),\eta^{N_1,N_2}_s}}^2\pi(dx',dy)ds.
\end{align*}
Hence,
\begin{equation}\label{K_Gamma_bound}
\abs{\Gamma ^{N_1,N_2}_t}^2 \le C_5 t\int^t_0 \int_{\CX \times \CY} \abs{K^{N_1,N_2}_s(x')}^2 \pi(dx',dy) ds + C_5t\int_0^t \int_{\CX\times \CY} \frac{1}{N_1}\sum_{j=1}^{N_1}\abs{\ip{B^{3,j}_{x'}(\theta),\eta^{N_1,N_2}_s}}^2\pi(dx',dy)ds
\end{equation}
By \eqref{K_II_bound} to \eqref{K_Gamma_bound}, and the definition of $\pi(dx,dy)$, we see that
\begin{align*}
\E\left(\abs{K^{N_1,N_2}_t(x)}^2 \right)
&\le  C  \left\lbrace (C_4+C_5) t^2  + \frac{C_3+C_5}{M} t \int^t_0 \sum_{x'\in \mathcal{X}} \E \paren{\abs{K^{N_1,N_2}_t(x)}^2} ds\right.\\
&\qquad \left. + \E\paren{\abs{N_2^{1-\gamma_2 + \varphi}\ip{c\sigma(Z^{2,N_1}(x)),\gamma^{N_1,N_2}_0}}^2} + \E\paren{\abs{N_2^{\varphi}M_t^{N_1,N_2}}^2} + O \paren{N_2^{-2(\gamma_2-\varphi)}}\right\rbrace.
\end{align*}
Summing both side of the above inequality over all $x \in \mathcal{X}$, where $\mathcal{X}$ is a fixed data set of size $M$ gives 
\begin{align}\label{sum_KN}
&\sum_{x\in\mathcal{X}} \E\left(\abs{K^{N_1,N_2}_t(x)}^2 \right)
\le  CMT^2  + CT \int^t_0 \sum_{x'\in \mathcal{X}} \E \paren{\abs{K^{N_1,N_2}_s(x')}^2 }ds\nonumber\\
&\qquad + \sum_{x\in \mathcal{X}}\E\paren{\abs{N_2^{1-\gamma_2 + \varphi}\ip{c\sigma(Z^{2,N_1}(x)),\gamma^{N_1,N_2}_0}}^2}  + \sum_{x\in \mathcal{X}}\E\paren{\abs{N_2^{\varphi}M_t^{N_1,N_2}}^2} +  O \paren{N_2^{-2(\gamma_2-\varphi)}}.
\end{align}
Since for $\varphi \le \gamma_2 - \frac{1}{2}$, $2(\gamma_2 -\varphi) \ge 1$, we have
\begin{align*}
\E\paren{\abs{N_2^{1-\gamma_2 + \varphi}\ip{c\sigma(Z^{2,N_1}(x)),\gamma^{N_1,N_2}_0}}^2}  \le \frac{C}{N_2^{2(\gamma_2-\varphi)}} \sum_{i=1}^{N_2} \E \paren{\abs{C^i_0}^2}\le C.
\end{align*}

Therefore, by applying Gr\"onwall's inequality to equation \eqref{sum_KN} and using Lemma \ref{CLT:lemma:martingale_bound},
\begin{align*}
\sum_{x\in\mathcal{X}} \E\left(\abs{K^{N_1,N_2}_t(x)}^2 \right) \le C(M) T^2 \exp\left[\tilde{C} Tt\right],
\end{align*}
where $C(M), \tilde{C}$ are some finite constants. Hence,
for any $x\in \mathcal{X}$, there exist $C<\infty$ such that
\[\sup_{N_2 \in \mathbb{N}, 0\le t\le T} \E \left[\abs{K^{N_1,N_2}_t(x)}^2 \right] <  C(M) T^2\exp\left[\tilde{C} T^2\right] \le C.\]

By Markov's inequality,  the compact containment condition for $K^{N_1,N_2}_t$ follows, concluding the proof of the lemma.
\end{proof}

We next establish the regularity of the process $K^{N_1,N_2}_t$ in $D_{\R^M}([0,T])$. For the purpose of this lemma, we denote $q(z_1,z_2) = \min\{\norm{ z_1-z_2}_{l^1},1\}$ for $z_1,z_2 \in \R^M$.
\begin{lemma}\label{CLT:lemma:regularity}
For any $\delta \in (0,1)$, there is a constant $C<\infty$ such that for $0\le u \le \delta$, $0\le v\le \delta \wedge t$, and $t \in [0,T]$,
\[\E \left[q\paren{K^{N_1,N_2}_{t+u},K^{N_1,N_2}_{t}}q\paren{K^{N_1,N_2}_{t},K^{N_1,N_2}_{t-v}}\vert \CF^{N_1,N_2}_t\right]\le {C\delta}+\frac{C}{N_2^{1-\varphi}}.\]
\end{lemma}

\begin{proof}For $0\le s<t\le T$, the leading terms in equation \eqref{K_evolution} gives
\begin{align*}
&\abs{K^{N_1,N_2}_t(x) - K^{N_1,N_2}_s(x)}\\
&\le \int^t_s \int_{\CX \times \CY} \abs{y-h^{N_1}_{\tau}(x')} \left(\abs{\ip{B^1_{x,x'}(\theta),\eta^{N_1,N_2}_{\tau}}} + \frac{1}{N_1}\sum_{j=1}^{N_1} \abs{\ip{B^{2,j}_{x,x'}(\theta),\eta^{N_1,N_2}_{\tau}}}\right) \pi(dx',dy) d\tau \\
&\quad +  \frac{1}{N_1}\sum_{j=1}^{N_1}\int_s^t \int_{\CX\times \CY} \abs{y-h^{N_1}_{\tau}(x')}\\
&\qquad \qquad \cdot \left(\abs{\ip{xx'B^{3,j}_{x}(\theta), \eta^{N_1,N_2}_{\tau}}}\abs{\ip{B^{3,j}_{x'}(\theta),\gamma^{N_1}_0}} + \abs{\ip{B^{3,j}_{x}(\theta), \gamma^{N_1}_0}}\abs{\ip{xx'B^{3,j}_{x'}(\theta),\eta^{N_1,N_2}_{\tau}}}\right)\pi(dx',dy)d\tau\\
&\quad +  \int^t_s \int_{\CX \times \CY} \abs{K^{N_1,N_2}_{\tau}(x')} \abs{\ip{B^1_{x,x'}(\theta), \gamma^{N_1}_0}} + \frac{1}{N_1}\sum_{j=1}^{N_1} \abs{\ip{B^{2,j}_{x,x'}(\theta), \gamma^{N_1}_0}}\pi(dx',dy) d\tau\\
&\quad + \frac{1}{N_1}\sum_{j=1}^{N_1}\int^t_s \int_{\CX \times \CY} \abs{K^{N_1,N_2}_{\tau}(x')}\abs{\ip{B^{3,j}_{x}(\theta), \gamma^{N_1}_0}}\abs{\ip{xx'B^{3,j}_{x'}(\theta),\gamma^{N_1}_0}}\pi(dx',dy)d\tau\\
&\quad  +\abs{\Gamma^{N_1,N_2}_t(x) - \Gamma^{N_1,N_2}_s(x)}+ N_2^{\varphi}\abs{M_t^{N_1,N_2}(x)-M_s^{N_1,N_2}(x)}.
\end{align*}

Taking expectation on both sides of the above inequality, by Assumption \ref{assumption}, Lemma \ref{lemma_parameterbound}, and analysis in Lemmas \ref{CLT:lemma:martingale_bound} and \ref{CLT:lemma:bound of ex_Kt}, we have for $0\leq t-s\leq \delta<1$
\begin{align*}
\E\left[\abs{K^{N_1,N_2}_t(x) - K^{N_1,N_2}_s(x)}\big\vert \CF^{N_1,N_2}_s \right]&\le C(t-s)  +  C_1 \int^t_s \int_{\CX \times \CY} \E\left[\abs{K^{N_1,N_2}_\tau(x')}\big\vert \CF^{N_1,N_2}_s\right]  \pi(dx',dy) d\tau\\
&\quad + C\E\left[\abs{N_2^{\varphi} \paren{M_t^{N_1,N_2}(x)-M_s^{N_1,N_2}(x)}}^2 \big\vert \CF^{N_1,N_2}_s \right]^{\frac{1}{2}}\\
&\le C\delta  + \frac{C}{N_2^{1-\varphi}}.
\end{align*}

Note that
\[\E\left[\abs{N_2^{\varphi} \paren{M_t^{N_1,N_2}(x)-M_s^{N_1,N_2}(x)}}^2  \big \vert \CF^{N_1,N_2}_s \right] \le \frac{C \delta}{N_2^{1-2\varphi}} + \frac{C}{N_2^{2-2\varphi}},\]
following an analysis similar to Lemma 3.1 of \cite{SirignanoSpiliopoulosNN1}.
Since $x\in \mathcal{X}$ is arbitrary, the statement of the lemma is then implied.
\end{proof}
By combining Lemmas \ref{CLT:lemma:bound of ex_Kt} and \ref{CLT:lemma:regularity}, we have that the sequence of processes $\{K^{N_1,N_2}_t,t\in[0,T]\}_{N_2\in\mathbb{N}}$ is relatively compact in $D_{\R^M}([0,T])$, which follows from Theorem 8.6 of Chapter 3 of \cite{EthierAndKurtz}.

\subsection{Convergence of $K^{N_1,N_2}_t$}

Denote $l^{N_1,N_2}_{1,t}= l^{N_1,N_2}_t(B^{1}_{x,x'}(\theta))$, $l^{N_1,N_2}_{2,t}$ and $l^{N_1,N_2}_{3,t}$ as $N_1$-dimensional vectors with $j$-th entry being $l^{N_1,N_2}_t(B^{2,j}_{x,x'}(\theta))$ and $l^{N_1,N_2}_t(B^{3,j}_{x}(\theta))$, respectively. We also let $l^{N_1}_{1,t}, l^{N_1}_{2,t}, l^{N_1}_{3,t}$ be the corresponding limits for $l^{N_1,N_2}_{1,t}, l^{N_1,N_2}_{2,t}, l^{N_1,N_2}_{3,t}$ as $N_2 \to \infty$. Recall that from Section \ref{CLT:sec:eta_properties}, for $\gamma_2 \in (1/2,1)$,  if $\varphi < 1-\gamma_2$, $l^{N_1}_{1,t}=0, l^{N_1}_{2,t}= l^{N_1}_{3,t}=0$, and if $\varphi = 1-\gamma_2$, $l^{N_1}_{1,t}, l^{N_1}_{2,t}, l^{N_1}_{3,t}$ are given by \eqref{CLT:evolution_l} for appropriate definitions of the function $f$.

In this section, we show that the processes $(\gamma^{N_1,N_2}_t, h^{N_1,N_2}_t, l^{N_1,N_2}_{1,t}, l^{N_1,N_2}_{2,t}, l^{N_1,N_2}_{3,t}, K^{N_1,N_2}_t)$ converges in distribution in $D_{E_1}([0,T])$ to $(\gamma^{N_1}_0, h^{N_1}_t, l^{N_1}_{1,t}, l^{N_1}_{2,t}, l^{N_1}_{3,t}, K^{N_1}_t)$, where $E_1 = \CM(\R^{1+N_1(1+d)}) \times \R^M \times \R \times \R^{N_1}\times \R^{N_1} \times \R^M$, and $K^{N_1}_t$ satisfies either of the following evolution equations:
\begin{custlist}[Case]
\item When $\gamma_2 \in \paren{\frac{1}{2}, \frac{3}{4}}$ and $\varphi \le \gamma_2 - \frac{1}{2}$, or when $\gamma_2 \in \left[\frac{3}{4}, 1\right)$ and $\varphi < 1-\gamma_2 \le \gamma_2 - \frac{1}{2}$ then $K^{N_1}_t(x)$ is given by (\ref{CLT:evolution}).
\item When $\gamma_2 \in \left[\frac{3}{4}, 1\right)$ and $\varphi = 1-\gamma_2$ then $K^{N_1}_{t}(x)$ satisfies (\ref{K_t for 1-gamma_2}).
\end{custlist}

By Lemmas \ref{LLN:lemma:relative_compact}, \ref{CLT:lemma:eta_compact_contatinment}, \ref{CLT:eta_regularity}, and Section \ref{sec::relative compactness K}, $\{\gamma^{N_1,N_2},h^{N_1,N_2},l^{N_1,N_2}_{1},l^{N_1,N_2}_{2},l^{N_1,N_2}_{3},K^{N_1,N_2}\}_{N_2 \in \mathbb{N}}$ is relatively compact in $D_{E_1}([0,T])$.
Denote $\pi^{N_1,N_2} \in \CM(D_{E_1}([0,T])$ the probability measure corresponding to $(\gamma^{N_1,N_2}, h^{N_1,N_2}, l^{N_1,N_2}_{1}, l^{N_1,N_2}_{2}, l^{N_1,N_2}_{3}, K^{N_1,N_2})$. 
We now show that any limit point $\pi^{N_1}$ of a convergence subsequence $\pi^{N_1,N_{2_k}}$ is a Dirac measure concentrated on $(\gamma^{N_1}, h^{N_1}, l^{N_1}_{1}, l^{N_1}_{2}, l^{N_1}_{3}, K^{N_1})$, where $(\gamma^{N_1},h^{N_1})$ satisfies equation \eqref{h_N1_evolution} and $(l^{N_1}_{1}, l^{N_1}_{2}, l^{N_1}_{3}, K^{N_1})$ satisfies Lemma \ref{CLT:lemma:lt_1-gamma}, equations \eqref{CLT:evolution}, or \eqref{K_t for 1-gamma_2} for different values of $\gamma_2$ and $\varphi$.


\begin{custlist}[Case]
\item When $\gamma_2 \in \paren{\frac{1}{2}, \frac{3}{4}}$ and $\varphi_2 \le \gamma - \frac{1}{2}$, or when $\gamma \in \left[\frac{3}{4}, 1\right)$ and $\varphi < 1-\gamma_2 \le \gamma_2 - \frac{1}{2}$, for any $t \in [0,T]$, $m^1_1,\ldots,m^1_p \in C_b(\R)$, $m^{i,j}_1,\ldots,m^{i,j}_p \in C_b(\R)$ for $i=2,3, j=1,\ldots, N_1$, $z_1, \ldots, z_p \in C_b(\R^{M})$, and $0 \le s_1 < \cdots < s_p \le t$, we define $F_2: D_{E_1}([0,T]) \to \R_+$ as
\begin{align*}
&F_2(\gamma, h, l_1, l_2, l_3, K)\\
 &= F(\gamma^{N_1},h^{N_1}) + \abs{\paren{l^{N_1}_{1,t} - 0} \times m^1_1(l^{N_1}_{1,s_1})\times \cdots \times m^1_p(l^{N_1}_{1,s_p})} + \sum_{i=2}^{3} \sum_{j=1}^{N_1} \abs{\paren{l^{N_1,j}_{i,t} - 0} \times m^{i,j}_1(l^{N_1,j}_{i,s_1})\times \cdots \times m^{i,j}_p(l^{N_1,j}_{i,s_p})}\\
& \quad + \sum_{x\in \mathcal{X}} \left\vert \left\lbrace K^{N_1}_t(x) - K^{N_1}_0(x) - \int_0^t \int_{\CX\times \CY} \left(y-h^{N_1}_s(x')\right)\left[l^{N_1}_t\left(B^1_{x,x'}(\theta)\right)+\frac{1}{N_1}\sum_{j=1}^{N_1}l^{N_1}_t\left(B^{2,j}_{x,x'}(\theta)\right)\right]\pi(dx',dy)ds \right.\right.\\
&\qquad -\frac{1}{N_1}\sum_{j=1}^{N_1}\int_0^t \int_{\CX\times \CY} \left(y-h^{N_1}_s(x')\right)xx'l^{N_1}_t\left(B^{3,j}_{x}(\theta)\right)\ip{B^{3,j}_{x'}(\theta),\gamma^{N_1}_0}\pi(dx',dy)ds\\
&\qquad -\frac{1}{N_1}\sum_{j=1}^{N_1}\int_0^t \int_{\CX\times \CY} \left(y-h^{N_1}_s(x')\right)xx'\ip{B^{3,j}_{x}(\theta), \gamma^{N_1}_0}l^{N_1}_t\left(B^{3,j}_{x'}(\theta)\right)\pi(dx',dy)ds\\
&\qquad + \int^t_0 \int_{\CX \times \CY}K^{N_1}_s(x')\ip{B^1_{x,x'}(\theta)+\frac{1}{N_1}\sum_{j=1}^{N_1}B^{2,j}_{x,x'}(\theta), \gamma^{N_1}_0}\pi(dx',dy) ds\\
&\qquad \left.\left. +\frac{1}{N_1}\sum_{j=1}^{N_1}\int_0^t \int_{\CX\times \CY} K^{N_1}_s(x')xx'\ip{B^{3,j}_{x}(\theta), \gamma^{N_1}_0}\ip{B^{3,j}_{x'}(\theta),\gamma^{N_1}_0}\pi(dx',dy)ds \right\rbrace \times z_1(K^{N_1}_{s_1}) \times \cdots \times z_p(K^{N_1}_{s_p}) \right \vert \\
\end{align*}
where $F(\gamma^{N_1},h^{N_1})$ is as given in equation \eqref{LNN:identify_limit_eq} and $l^{N_1,j}_{i,t}$ is the $j$-th element of the $N_1$-dimensional vector $l^{N_1}_{i,t}$ for $i=2,3$. We now note that for any $x \in \mathcal{X}$, by equation \eqref{K_evolution},
\begin{align*}
&K^{N_1,N_2}_t(x) - K^{N_1,N_2}_0(x) - \int_0^t \int_{\CX\times \CY} \left(y-h^{N_1,N_2}_s(x')\right)\left[l^{N_1,N_2}_t\left(B^1_{x,x'}(\theta)\right)+\frac{1}{N_1}\sum_{j=1}^{N_1}l^{N_1,N_2}_t\left(B^{2,j}_{x,x'}(\theta)\right)\right]\pi(dx',dy)ds \\
&\qquad -\frac{1}{N_1}\sum_{j=1}^{N_1}\int_0^t \int_{\CX\times \CY} \left(y-h^{N_1,N_2}_s(x')\right)xx'l^{N_1,N_2}_t\left(B^{3,j}_{x}(\theta)\right)\ip{B^{3,j}_{x'}(\theta),\gamma^{N_1,N_2}_0}\pi(dx',dy)ds\\
&\qquad -\frac{1}{N_1}\sum_{j=1}^{N_1}\int_0^t \int_{\CX\times \CY} \left(y-h^{N_1,N_2}_s(x')\right)xx'\ip{B^{3,j}_{x}(\theta), \gamma^{N_1,N_2}_0}l^{N_1,N_2}_t\left(B^{3,j}_{x'}(\theta)\right)\pi(dx',dy)ds\\
&\qquad + \int^t_0 \int_{\CX \times \CY}K^{N_1,N_2}_s(x')\ip{B^1_{x,x'}(\theta)+\frac{1}{N_1}\sum_{j=1}^{N_1}B^{2,j}_{x,x'}(\theta), \gamma^{N_1,N_2}_0}\pi(dx',dy) ds\\
&\qquad +\frac{1}{N_1}\sum_{j=1}^{N_1}\int_0^t \int_{\CX\times \CY} K^{N_1,N_2}_s(x')xx'\ip{B^{3,j}_{x}(\theta), \gamma^{N_1,N_2}_0}\ip{B^{3,j}_{x'}(\theta),\gamma^{N_1,N_2}_0}\pi(dx',dy)ds\\
&= (1) + (2) + (3) + (4) + (5) {+ \Gamma^{N_1,N_2}_{t}(x)}+ N_2^{\varphi}M_t^{N_1,N_2}(x) + O (N_2^{-\gamma_2+\varphi}),
\end{align*}
where terms $(1),(2),(3), (4), (5)$ will be specified and analyzed as follows. We see that term $(1)$ satisfies
\begin{equation}\label{eq_temp_Kfirst}
\begin{aligned}
(1) &= \int_0^t \int_{\CX\times \CY} \left[\left(y-h^{N_1}_s(x')\right)- \left(y-h^{N_1,N_2}_s(x')\right)\right]\left[l^{N_1,N_2}_t\left(B^1_{x,x'}(\theta)\right)+\frac{1}{N_1}\sum_{j=1}^{N_1}l^{N_1,N_2}_t\left(B^{2,j}_{x,x'}(\theta)\right)\right]\pi(dx',dy)ds\\
&= \frac{1}{N_1^{\varphi}}\int_0^t \int_{\CX\times \CY} K^{N_1,N_2}_s(x')\left[l^{N_1,N_2}_t\left(B^1_{x,x'}(\theta)\right)+\frac{1}{N_1}\sum_{j=1}^{N_1}l^{N_1,N_2}_t\left(B^{2,j}_{x,x'}(\theta)\right)\right]\pi(dx',dy)ds\\
&= - \left(\Gamma^{N_1,N_2}_{1,t}(x) + \Gamma^{N_1,N_2}_{2,t}(x)\right).
\end{aligned}
\end{equation}
Term $(2)$ can be rearranged into
\begin{equation*}
\begin{aligned}
(2) &= \int^t_0 \int_{\CX \times \CY}K^{N_1,N_2}_s(x')\ip{B^1_{x,x'}(\theta)+\frac{1}{N_1}\sum_{j=1}^{N_1}B^{2,j}_{x,x'}(\theta), \gamma^{N_1,N_2}_0 -\gamma^{N_1}_0}\pi(dx',dy) ds\\
&= \frac{1}{N_2^{\varphi}} \int^t_0 \int_{\CX \times \CY}K^{N_1,N_2}_s(x')\ip{B^1_{x,x'}(\theta)+\frac{1}{N_1}\sum_{j=1}^{N_1}B^{2,j}_{x,x'}(\theta), \eta^{N_1}_0}\pi(dx',dy) ds,
\end{aligned}
\end{equation*}
and by the Cauchy-Schwarz inequality, Lemmas \ref{CLT:lemma:eta_compact_contatinment} and \ref{CLT:lemma:bound of ex_Kt}, for any $t \in [0,T]$,
\begin{equation}\label{eq:K^N:temp}
\begin{aligned}
& \E \paren{\abs{\frac{1}{N_2^{\varphi}} \int^t_0 \int_{\CX \times \CY}K^{N_1,N_2}_s(x')\ip{B^1_{x,x'}(\theta)+\frac{1}{N_1}\sum_{j=1}^{N_1}B^{2,j}_{x,x'}(\theta), \eta^{N_1}_0}\pi(dx',dy) ds}}\\
&\le \frac{C}{N_2^{\varphi}} \int_0^t  \int_{\CX \times \CY}\E \paren{\abs{ K^{N_1,N_2}_s(x')}\abs{\ip{B^1_{x,x'}(\theta)+\frac{1}{N_1}\sum_{j=1}^{N_1}B^{2,j}_{x,x'}(\theta), \eta^{N_1}_0}}} \pi(dx',dy) ds\\
&\le \frac{C}{N_2^{\varphi}} \int_0^t  \int_{\CX \times \CY}\E \paren{\abs{ K^{N_1,N_2}_s(x')}^2 }^{\frac{1}{2}} \E \paren{\abs{\ip{B^1_{x,x'}(\theta)+\frac{1}{N_1}\sum_{j=1}^{N_1}B^{2,j}_{x,x'}(\theta), \eta^{N_1}_0}}^2}^{\frac{1}{2}} \pi(dx',dy) ds\\
& \le \frac{C(T)}{N_2^{\varphi}},
\end{aligned}
\end{equation}
where $C(T)<\infty$ is some finite constant depending on $T$.

We discuss terms (3) and (4) together. Since
\begin{equation*}
\begin{aligned}
(3)&=\frac{1}{N_1}\sum_{j=1}^{N_1}\int_0^t \int_{\CX\times \CY} \left(y-h^{N_1}_s(x')\right)xx'l^{N_1,N_2}_t\left(B^{3,j}_{x}(\theta)\right)\ip{B^{3,j}_{x'}(\theta),\gamma^{N_1}_0}\pi(dx',dy)ds\\
&\quad -\frac{1}{N_1}\sum_{j=1}^{N_1}\int_0^t \int_{\CX\times \CY} \left(y-h^{N_1,N_2}_s(x')\right)xx'l^{N_1,N_2}_t\left(B^{3,j}_{x}(\theta)\right)\ip{B^{3,j}_{x'}(\theta),\gamma^{N_1,N_2}_0}\pi(dx',dy)ds\\
&=-\frac{1}{N_1N_2^{\varphi}}\sum_{j=1}^{N_1}\int_0^t \int_{\CX\times \CY} \left(y-h^{N_1}_s(x')\right)xx'l^{N_1,N_2}_t\left(B^{3,j}_{x}(\theta)\right)l^{N_1,N_2}_t\left(B^{3,j}_{x'}(\theta)\right)\pi(dx',dy)ds\\
&\quad +\frac{1}{N_1N_2^{\varphi}}\sum_{j=1}^{N_1}\int_0^t \int_{\CX\times \CY} K^{N_1,N_2}_s(x')xx'l^{N_1,N_2}_t\left(B^{3,j}_{x}(\theta)\right)\ip{B^{3,j}_{x'}(\theta),\gamma^{N_1,N_2}_s}\pi(dx',dy)ds\\
&\quad +\frac{1}{N_1}\sum_{j=1}^{N_1}\int_0^t \int_{\CX\times \CY} \left(y-h^{N_1}_s(x')\right)xx'l^{N_1,N_2}_t\left(B^{3,j}_{x}(\theta)\right)\ip{B^{3,j}_{x'}(\theta),\gamma^{N_1,N_2}_s-\gamma^{N_1,N_2}_0}\pi(dx',dy)ds,
\end{aligned}
\end{equation*}
and
\begin{equation*}
\begin{aligned}
(4)&= \frac{1}{N_1}\sum_{j=1}^{N_1}\int_0^t \int_{\CX\times \CY} \left(y-h^{N_1}_s(x')\right)xx'\ip{B^{3,j}_{x}(\theta), \gamma^{N_1}_0}l^{N_1,N_2}_t\left(B^{3,j}_{x'}(\theta)\right)\pi(dx',dy)ds\\
&\quad  -\frac{1}{N_1}\sum_{j=1}^{N_1}\int_0^t \int_{\CX\times \CY} \left(y-h^{N_1,N_2}_s(x')\right)xx'\ip{B^{3,j}_{x}(\theta), \gamma^{N_1,N_2}_0}l^{N_1,N_2}_t\left(B^{3,j}_{x'}(\theta)\right)\pi(dx',dy)ds\\
&= \frac{1}{N_1N_2^{\varphi}}\sum_{j=1}^{N_1}\int_0^t \int_{\CX\times \CY} K^{N_1,N_2}_s(x')xx'\ip{B^{3,j}_{x}(\theta),\gamma^{N_1,N_2}_0}l^{N_1,N_2}_t\left(B^{3,j}_{x'}(\theta)\right)\pi(dx',dy)ds\\
&\quad +\frac{1}{N_1}\sum_{j=1}^{N_1}\int_0^t \int_{\CX\times \CY} \left(y-h^{N_1,N_2}_s(x')\right)xx'\ip{B^{3,j}_{x}(\theta),\gamma^{N_1,N_2}_0-\gamma^{N_1}_0}l^{N_1,N_2}_t\left(B^{3,j}_{x'}(\theta)\right)\pi(dx',dy)ds,
\end{aligned}
\end{equation*}
one has
\begin{equation}\label{temp_3_4}
\begin{aligned}
&(3)+(4) =\\
&= -\Gamma^{N_1,N_2}_{3,t}(x) +\frac{1}{N_1}\sum_{j=1}^{N_1}\int_0^t \int_{\CX\times \CY} \left(y-h^{N_1}_s(x')\right)xx'l^{N_1,N_2}_t\left(B^{3,j}_{x}(\theta)\right)\ip{B^{3,j}_{x'}(\theta),\gamma^{N_1,N_2}_s-\gamma^{N_1,N_2}_0}\pi(dx',dy)ds\\
&\quad +\frac{1}{N_1}\sum_{j=1}^{N_1}\int_0^t \int_{\CX\times \CY} \left(y-h^{N_1,N_2}_s(x')\right)xx'\ip{B^{3,j}_{x}(\theta),\gamma^{N_1,N_2}_0-\gamma^{N_1}_0}l^{N_1,N_2}_t\left(B^{3,j}_{x'}(\theta)\right)\pi(dx',dy)ds\\
&=-\Gamma^{N_1,N_2}_{3,t}(x)+\frac{1}{N_1N_2^{\varphi}}\sum_{j=1}^{N_1}\int_0^t \int_{\CX\times \CY} \left(y-h^{N_1}_s(x')\right)xx'l^{N_1,N_2}_t\left(B^{3,j}_{x}(\theta)\right)l^{N_1,N_2}_t\left(B^{3,j}_{x'}(\theta)\right)\pi(dx',dy)ds\\
&\quad - \frac{1}{N_1N_2^{\varphi}}\sum_{j=1}^{N_1}\int_0^t \int_{\CX\times \CY} \left(y-h^{N_1}_s(x')\right)xx'l^{N_1,N_2}_t\left(B^{3,j}_{x}(\theta)\right)l^{N_1,N_2}_0\left(B^{3,j}_{x'}(\theta)\right)\pi(dx',dy)ds\\
&\quad +\frac{1}{N_1N_2^{\varphi}}\sum_{j=1}^{N_1}\int_0^t \int_{\CX\times \CY} \left(y-h^{N_1,N_2}_s(x')\right)xx'l^{N_1,N_2}_0\left(B^{3,j}_{x}(\theta)\right)l^{N_1,N_2}_t\left(B^{3,j}_{x'}(\theta)\right)\pi(dx',dy)ds.
\end{aligned}
\end{equation}
Since by Lemmas \ref{lemma_g} and \ref{CLT:lemma:eta_compact_contatinment},
\begin{equation*}
\begin{aligned}
&\E\left(\abs{\int_0^t \int_{\CX\times \CY} \left(y-h^{N_1}_s(x')\right)xx'l^{N_1,N_2}_t\left(B^{3,j}_{x}(\theta)\right)l^{N_1,N_2}_t\left(B^{3,j}_{x'}(\theta)\right)\pi(dx',dy)ds}\right)\\
&\le\E\left(\abs{\int_0^t \int_{\CX\times \CY} \left(y-h^{N_1}_s(x')\right)xx'l^{N_1,N_2}_t\left(B^{3,j}_{x}(\theta)\right)l^{N_1,N_2}_t\left(B^{3,j}_{x'}(\theta)\right)\pi(dx',dy)ds}^2\right)\\
&\le C(T)\int_0^t \int_{\CX\times \CY}\E\left[\abs{l^{N_1,N_2}_t\left(B^{3,j}_{x}(\theta)\right)l^{N_1,N_2}_t\left(B^{3,j}_{x'}(\theta)\right)}^2\right]\pi(dx',dy)ds\\
&\le C(T)\int_0^t \int_{\CX\times \CY}\E\left[\abs{l^{N_1,N_2}_t\left(B^{3,j}_{x}(\theta)\right)}^4\right]^{\frac{1}{2}} \E\left[ \abs{l^{N_1,N_2}_t\left(B^{3,j}_{x'}(\theta)\right)}^4\right]^{\frac{1}{2}}\pi(dx',dy)ds \\
&\le C(T),
\end{aligned}
\end{equation*}
\begin{equation*}
\begin{aligned}
&\E\left[\abs{\int_0^t \int_{\CX\times \CY} \left(y-h^{N_1,N_2}_s(x')\right)xx'l^{N_1,N_2}_0\left(B^{3,j}_{x}(\theta)\right)l^{N_1,N_2}_t\left(B^{3,j}_{x'}(\theta)\right)\pi(dx',dy)ds}\right]\\
&\le \int_0^t \int_{\CX\times \CY} \E\left[\abs{\left(y-h^{N_1,N_2}_s(x')\right)xx'}\abs{l^{N_1,N_2}_0\left(B^{3,j}_{x}(\theta)\right)}\abs{l^{N_1,N_2}_t\left(B^{3,j}_{x'}(\theta)\right)}\right]\pi(dx',dy)ds\\
&\le \int_0^t \int_{\CX\times \CY} \E\left[\abs{\left(y-h^{N_1,N_2}_s(x')\right)xx'}^2\right]^{\frac{1}{2}}\E\left[\abs{l^{N_1,N_2}_0\left(B^{3,j}_{x}(\theta)\right)}^4\right]^{\frac{1}{4}}\E\left[\abs{l^{N_1,N_2}_t\left(B^{3,j}_{x'}(\theta)\right)}^4\right]^{\frac{1}{4}}\pi(dx',dy)ds\\
&\le C(T),
\end{aligned}
\end{equation*}
the expectation of the last three terms in \eqref{temp_3_4} is bounded by $O(N_2^{-\varphi})$. Lastly, for term (5), we have
\begin{equation*}
\begin{aligned}
(5)&=\frac{1}{N_1}\sum_{j=1}^{N_1}\int_0^t \int_{\CX\times \CY} K^{N_1,N_2}_s(x')xx'\ip{B^{3,j}_{x}(\theta), \gamma^{N_1,N_2}_0}\ip{B^{3,j}_{x'}(\theta),\gamma^{N_1,N_2}_0}\pi(dx',dy)ds\\
&\quad - \frac{1}{N_1}\sum_{j=1}^{N_1}\int_0^t \int_{\CX\times \CY} K^{N_1,N_2}_s(x')xx'\ip{B^{3,j}_{x}(\theta), \gamma^{N_1}_0}\ip{B^{3,j}_{x'}(\theta),\gamma^{N_1}_0}\pi(dx',dy)ds\\
&=\frac{1}{N_1N_2^{2\varphi}}\sum_{j=1}^{N_1}\int_0^t \int_{\CX\times \CY} K^{N_1,N_2}_s(x')xx'l^{N_1,N_2}_0\left(B^{3,j}_{x}(\theta)\right)l^{N_1,N_2}_0\left(B^{3,j}_{x'}(\theta)\right)\pi(dx',dy)ds\\
&\quad + \frac{1}{N_1N_2^{\varphi}}\sum_{j=1}^{N_1}\int_0^t \int_{\CX\times \CY} K^{N_1,N_2}_s(x')xx'l^{N_1,N_2}_0\left(B^{3,j}_{x}(\theta)\right)\ip{B^{3,j}_{x'}(\theta),\gamma^{N_1}_0}\pi(dx',dy)ds\\
&\quad + \frac{1}{N_1N_2^{\varphi}}\sum_{j=1}^{N_1}\int_0^t \int_{\CX\times \CY} K^{N_1,N_2}_s(x')xx'\ip{B^{3,j}_{x}(\theta),\gamma^{N_1}_0}l^{N_1,N_2}_0\left(B^{3,j}_{x'}(\theta)\right)\pi(dx',dy)ds.
\end{aligned}
\end{equation*}
By the Cauchy-Schwarz inequality, Lemmas \ref{CLT:lemma:eta_compact_contatinment} and \ref{CLT:lemma:bound of ex_Kt}, for any $t \in [0,T]$,
\begin{equation*}
\begin{aligned}
&\E\left(\abs{\frac{1}{N_1N_2^{2\varphi}}\sum_{j=1}^{N_1}\int_0^t \int_{\CX\times \CY} K^{N_1,N_2}_s(x')xx'l^{N_1,N_2}_0\left(B^{3,j}_{x}(\theta)\right)l^{N_1,N_2}_0\left(B^{3,j}_{x'}(\theta)\right)\pi(dx',dy)ds}\right)\\
&\le\frac{1}{N_1N_2^{2\varphi}}\sum_{j=1}^{N_1}\int_0^t \int_{\CX\times \CY} \E\left[\abs{K^{N_1,N_2}_s(x')l^{N_1,N_2}_0\left(xx'B^{3,j}_{x}(\theta)\right)l^{N_1,N_2}_0\left(B^{3,j}_{x'}(\theta)\right)}\right] \pi(dx',dy)ds\\
&\le \frac{1}{N_1N_2^{2\varphi}}\sum_{j=1}^{N_1}\int_0^t \int_{\CX\times \CY} \E\left[\abs{K^{N_1,N_2}_s(x')}^2\right]^{\frac{1}{2}}\E\left[\abs{l^{N_1,N_2}_0\left(xx'B^{3,j}_{x}(\theta)\right)}^4\right]^{\frac{1}{4}}\E\left[\abs{l^{N_1,N_2}_0\left(B^{3,j}_{x'}(\theta)\right)}^4\right]^{\frac{1}{4}} \pi(dx',dy)ds\\
&\le \frac{C(T)}{N_2^{2\varphi}},
\end{aligned}
\end{equation*}
\begin{equation*}
\begin{aligned}
&\E\left(\abs{\frac{1}{N_1N_2^{\varphi}}\sum_{j=1}^{N_1}\int_0^t \int_{\CX\times \CY} K^{N_1,N_2}_s(x')xx'l^{N_1,N_2}_0\left(B^{3,j}_{x}(\theta)\right)\ip{B^{3,j}_{x'}(\theta),\gamma^{N_1}_0}\pi(dx',dy)ds}\right)\\
&\le \E\left(\frac{1}{N_1N_2^{\varphi}}\sum_{j=1}^{N_1}\sup_{x'}\abs{\ip{xx'B^{3,j}_{x'}(\theta),\gamma^{N_1}_0}}\int_0^t \int_{\CX\times \CY}\abs{ K^{N_1,N_2}_s(x')l^{N_1,N_2}_0\left(B^{3,j}_{x}(\theta)\right)}\pi(dx',dy)ds\right)\\
&\le \frac{C}{N_1N_2^{\varphi}}\sum_{j=1}^{N_1}\int_0^t \int_{\CX\times \CY}\E\left[\abs{ K^{N_1,N_2}_s(x')}^2\right]^{\frac{1}{2}} \E\left[\abs{l^{N_1,N_2}_0\left(B^{3,j}_{x}(\theta)\right)}^2\right]^{\frac{1}{2}}\pi(dx',dy)ds\\
&\le \frac{C(T)}{N_2^{\varphi}},
\end{aligned}
\end{equation*}
and similarly,
\begin{equation}\label{eq_temp_Klast}
\begin{aligned}
\E\left(\abs{\frac{1}{N_1N_2^{\varphi}}\sum_{j=1}^{N_1}\int_0^t \int_{\CX\times \CY} K^{N_1,N_2}_s(x')xx'\ip{B^{3,j}_{x}(\theta),\gamma^{N_1}_0}l^{N_1,N_2}_0\left(B^{3,j}_{x'}(\theta)\right)\pi(dx',dy)ds}\right)\le \frac{C(T)}{N_2^{\varphi}}.
\end{aligned}
\end{equation}


By equations \eqref{eq_temp_Kfirst} to \eqref{eq_temp_Klast}, the analysis in Sections \ref{sec::Identification of the Limit} and  \ref{CLT:sec:eta_properties}, and Lemma \ref{CLT:lemma:martingale_bound}, we have
\begin{equation*}
\begin{aligned}
&\E_{\pi^{N_1,N_2}} \left[ F_2(\gamma^{N_1}, h^{N_1}, l^{N_1}_1, l^{N_1}_2, l^{N_1}_3, K^{N_1})\right]\\
&= \E_{\pi^{N_1,N_2}} \left[ F(\gamma^{N_1}, h^{N_1})\right] + \E\left[\abs{\paren{l^{N_1,N_2}_{1,t} - 0} \times\prod_{n=1}^p m^1_n(l^{N_1,N_2}_{1,s_n})}\right] + \sum_{i=2}^{3} \sum_{j=1}^{N_1}\E \left[ \abs{\paren{l^{N_1,N_2,j}_{i,t} - 0} \times \prod_{n=1}^p m^{i,j}_n(l^{N_1,N_2,j}_{i,s_n})}\right]\\
&  + \sum_{x\in \mathcal{X}}\E\left\lbrace \left\vert \left( K^{N_1,N_2}_t(x) - K^{N_1,N_2}_0(x) - \int_0^t \int_{\CX\times \CY} \left(y-h^{N_1,N_2}_s(x')\right)l^{N_1,N_2}_t\left(B^1_{x,x'}(\theta)\right)\pi(dx',dy)ds \right.\right.\right.\\
&\qquad -\frac{1}{N_1}\sum_{j=1}^{N_1}\int_0^t \int_{\CX\times \CY} \left(y-h^{N_1,N_2}_s(x')\right)l^{N_1,N_2}_t\left(B^{2,j}_{x,x'}(\theta)\right)\pi(dx',dy)ds\\
&\qquad -\frac{1}{N_1}\sum_{j=1}^{N_1}\int_0^t \int_{\CX\times \CY} \left(y-h^{N_1,N_2}_s(x')\right)xx'l^{N_1,N_2}_t\left(B^{3,j}_{x}(\theta)\right)\ip{B^{3,j}_{x'}(\theta),\gamma^{N_1,N_2}_0}\pi(dx',dy)ds\\
&\qquad -\frac{1}{N_1}\sum_{j=1}^{N_1}\int_0^t \int_{\CX\times \CY} \left(y-h^{N_1,N_2}_s(x')\right)xx'\ip{B^{3,j}_{x}(\theta), \gamma^{N_1,N_2}_0}l^{N_1,N_2}_t\left(B^{3,j}_{x'}(\theta)\right)\pi(dx',dy)ds\\
&\qquad + \int^t_0 \int_{\CX \times \CY}K^{N_1,N_2}_s(x')\ip{B^1_{x,x'}(\theta)+\frac{1}{N_1}\sum_{j=1}^{N_1}B^{2,j}_{x,x'}(\theta), \gamma^{N_1,N_2}_0}\pi(dx',dy) ds\\
&\qquad \left.\left.\left. +\frac{1}{N_1}\sum_{j=1}^{N_1}\int_0^t \int_{\CX\times \CY} K^{N_1,N_2}_s(x')xx'\ip{B^{3,j}_{x}(\theta), \gamma^{N_1,N_2}_0}\ip{B^{3,j}_{x'}(\theta),\gamma^{N_1,N_2}_0}\pi(dx',dy)ds \right) \times  \prod_{i=1}^p z_i(K^{N_1,N_2}_{s_i})\right \vert \right\rbrace\\
&\le C\paren{\frac{1}{N_2^{1-\gamma_2}}}  + C \paren{\frac{1}{N_2^{1-\gamma_2 - \varphi}} + \frac{1}{N_2^{1-\varphi}} + \frac{1}{N_2^{\frac{1}{2}-\varphi}}}  + C\E\left[\abs{N_2^{\varphi}M_t^{N_1,N_2}}^2\right]^{\frac{1}{2}} +C\paren{\frac{1}{N_2^{\gamma_2-\varphi}} + \frac{1}{N_2^{\varphi}}}\\
&\le C \paren{\frac{1}{N_2^{1-\gamma_2 - \varphi}} + \frac{1}{N_2^{\frac{1}{2}-\varphi}} +\frac{1}{N_2^{\varphi}}}.
\end{aligned}
\end{equation*}

Therefore, $\lim_{N_2 \to \infty} \E_{\pi^{N_1,N_2}} [ F_2(\gamma, h, l_1, l_2, l_3, K)] = 0$. Since $F(\cdot)$ is continuous and $F(\gamma^{N_1,N_2},h^{N_1,N_2})$ is uniformly bounded, together with analysis in Sections \ref{CLT:sec:eta_properties} and \ref{sec::relative compactness K}, we have that $F_2(\cdot)$ is continuous and $F_2(\gamma^{N_1,N_2}, h^{N_1,N_2}, l_1^{N_1,N_2}, l_2^{N_1,N_2}, l_3^{N_1,N_2}, K^{N_1,N_2})$ is uniformly bounded. Hence, by weak convergence we have
\[\lim_{N_2 \to \infty} \E_{\pi^{N_1,N_2}} [ F_2(\gamma, h, l_1, l_2, l_3, K)] = \E_{\pi^{N_1}} [ F_2(\gamma, h, l_1, l_2, l_3, K)] = 0.\]

We have shown that any limit point $\pi^{N_1}$ of a convergence sequence must be a Dirac measure concentrated $(\gamma^{N_1}, h^{N_1}, l^{N_1}_{1}, l^{N_1}_{2}, l^{N_1}_{3}, K^{N_1})$, which  satisfies equation \eqref{h_N1_evolution}, $l^{N_1}_{i} = 0$ for $i=1,2,3$, and equation \eqref{CLT:evolution}. Since the solutions to equations \eqref{h_N1_evolution} and \eqref{CLT:evolution} are unique, the processes in consideration converges in distribution to$(\gamma^{N_1}_0, h^{N_1}, 0, 0, 0, K^{N_1})$ by Prokhorov's theorem.

\item  When $\gamma_2 \in \left[\frac{3}{4}, 1\right)$ and $\varphi = 1-\gamma_2$, for any $t \in [0,T]$, $m^1_1,\ldots,m^1_p \in C_b(\R)$, $m^{i,j}_1,\ldots,m^{i,j}_p \in C_b(\R)$ for $i=2,3, j=1,\ldots, N_1$, $z_1, \ldots, z_p \in C_b(\R^{M})$, and $0 \le s_1 < \cdots < s_p \le t$, we define $F_3(\gamma, h, l_1, l_2, l_3, K): D_{E_1}([0,T]) \to \R_+$ as
\begin{equation}\label{F_3}
\begin{aligned}
&F_3(\gamma, h, l_1, l_2, l_3, K)\\
&= F(\gamma^{N_1},h^{N_1}) + \left\vert F_{\eta}(l^{N_1}_{1,t}) \times m^1_1(l^{N_1}_{1,s_1})\times \cdots \times m^1_p(l^{N_1}_{1,s_p})\right\vert + \sum_{i=2}^{3} \sum_{j=1}^{N_1} \abs{F_{\eta}(l^{N_1,j}_{i,t}) \times m^{i,j}_1(l^{N_1,j}_{i,s_1})\times \cdots \times m^{i,j}_p(l^{N_1,j}_{i,s_p})}\\
& \quad + \sum_{x\in \mathcal{X}} \left\vert \left\lbrace K^{N_1}_t(x) - K^{N_1}_0(x) - \int_0^t \int_{\CX\times \CY} \left(y-h^{N_1}_s(x')\right)\left[l^{N_1}_t\left(B^1_{x,x'}(\theta)\right)+\frac{1}{N_1}\sum_{j=1}^{N_1}l^{N_1}_t\left(B^{2,j}_{x,x'}(\theta)\right)\right]\pi(dx',dy)ds \right.\right.\\
&\qquad -\frac{1}{N_1}\sum_{j=1}^{N_1}\int_0^t \int_{\CX\times \CY} \left(y-h^{N_1}_s(x')\right)xx'l^{N_1}_t\left(B^{3,j}_{x}(\theta)\right)\ip{B^{3,j}_{x'}(\theta),\gamma^{N_1}_0}\pi(dx',dy)ds\\
&\qquad -\frac{1}{N_1}\sum_{j=1}^{N_1}\int_0^t \int_{\CX\times \CY} \left(y-h^{N_1}_s(x')\right)xx'\ip{B^{3,j}_{x}(\theta), \gamma^{N_1}_0}l^{N_1}_t\left(B^{3,j}_{x'}(\theta)\right)\pi(dx',dy)ds\\
&\qquad + \int^t_0 \int_{\CX \times \CY}K^{N_1}_s(x')\ip{B^1_{x,x'}(\theta)+\frac{1}{N_1}\sum_{j=1}^{N_1}B^{2,j}_{x,x'}(\theta), \gamma^{N_1}_0}\pi(dx',dy) ds\\
&\qquad \left.\left. +\frac{1}{N_1}\sum_{j=1}^{N_1}\int_0^t \int_{\CX\times \CY} K^{N_1}_s(x')xx'\ip{B^{3,j}_{x}(\theta), \gamma^{N_1}_0}\ip{B^{3,j}_{x'}(\theta),\gamma^{N_1}_0}\pi(dx',dy)ds \right\rbrace \times z_1(K^{N_1}_{s_1}) \times \cdots \times z_p(K^{N_1}_{s_p}) \right \vert \\
\end{aligned}
\end{equation}
where $F(\gamma^{N_1},h^{N_1})$ is as given in equation \eqref{LNN:identify_limit_eq}, $l^{N_1,j}_{i,t}$ is the $j$-th element of the $N_1$-dimensional vector $l^{N_1}_{i,t}$ for $i=2,3$, and
\begin{equation*}
\begin{aligned}
&F_{\eta}(l^{N_1}_{t}(f)) =  l^{N_1}_{t}(f)-\int_0^t \int_{\CX\times \CY} \left(y-h_s^{N_1}(x')\right) \ip{\partial_{c}f(\theta)  \sigma(Z^{2,N_1}(x')),{\gamma}^{N_1}_s} \pi(dx',dy)ds  \\
&\qquad \qquad -\frac{1}{N_1^{1-\gamma_1}}\int_0^t\int_{\CX\times \CY}\left(y-h_s^{N_1}(x')\right) \ip{c\sigma'(Z^{2,N_1}(x'))\sigma(w^1x')\cdot \partial_{w^2}f(\theta) ,{\gamma}^{N_1}_s}\pi(dx',dy)ds\\
&\qquad \qquad - \frac{1}{N_1^{1-\gamma_1}}\int_0^t\int_{\CX\times \CY}\left(y-h_s^{N_1}(x')\right)\ip{\ip{c\sigma'(Z^{2,N_1}(x'))\sigma'(w^1x')w^{2},{\gamma}_s^{N_1}}\cdot  \nabla_{w^1}f(\theta)x',{\gamma}^{N_1}_s}\pi(dx',dy)ds.\\
\end{aligned}
\end{equation*}

By equations \eqref{eq_temp_Kfirst} to \eqref{eq_temp_Klast}, Lemmas \ref{CLT:lemma:lt_1-gamma} and \ref{CLT:lemma:martingale_bound}, and the analysis in Section \ref{sec::Identification of the Limit}, we obtain
\begin{align*}
\E_{\pi^{N_1,N_2}}\left[F_3(\gamma, h, l_1, l_2, l_3, K)\right]&\le C\paren{\frac{1}{N_2^{1-\gamma_2}}+ \frac{1}{N_2^{\frac{1}{2}-\varphi}}+\frac{1}{N_2^{1-\varphi}}+ \frac{1}{N_2^{\gamma_2-\varphi}}+ \frac{1}{N_2^{\varphi}}}.
\end{align*}

Therefore, $\lim_{N_2\to \infty} \E_{\pi^{N_1,N_2}}[F_3(\gamma, h, l_1, l_2, l_3, K)] = 0$. By analysis in Sections \ref{sec:relative_compactness}, \ref{CLT:sec:eta_properties} and \ref{sec::relative compactness K}, we have that $F_3(\cdot)$ is continuous and $F_3(\gamma^{N_1,N_2}, h^{N_1,N_2}, l_1^{N_1,N_2}, l_2^{N_1,N_2}, l_3^{N_1,N_2}, K^{N_1,N_2})$ is uniformly bounded. Hence,
\[\lim_{N_2\to \infty} \E_{\pi^{N_1,N_2}}\left[F_3(\gamma, h, l_1, l_2, l_3, K)\right] = \E_{\pi^{N_1}}\left[F_3(\gamma, h, l_1, l_2, l_3, K)\right] = 0.\]

We have shown that any limit point $\pi^{N_1}$ of a convergence sequence must be a Dirac measure concentrated $(\gamma^{N_1}, h^{N_1}, l_1^{N_1}, l_2^{N_1}, l_3^{N_1}, K^{N_1})\in D_{E_1}([0,T])$, which satisfies equation \eqref{h_N1_evolution}, \eqref{CLT:evolution_l}, and \eqref{K_t for 1-gamma_2}. Since the solutions to equations \eqref{h_N1_evolution} and \eqref{K_t for 1-gamma_2} are unique, by Prokhorov's theorem, the processes

$(\gamma^{N_1,N_2}, h^{N_1,N_2}, l_1^{N_1,N_2}, l_2^{N_1,N_2}, l_3^{N_1,N_2}, K^{N_1,N_2})$ converges in distribution to $(\gamma^{N_1}, h^{N_1}, l_1^{N_1}, l_2^{N_1}, l_3^{N_1}, K^{N_1})$.
\end{custlist}

\section{Proof of Theorem \ref{thm::Psi}}\label{sec::Psi}
For $\gamma_2 \in \paren{\frac{3}{4},1}, \varphi = 1-\gamma_2$, we can further look at the fluctuation process $\Psi^{N_1,N_2}_t = N_2^{\zeta - \varphi} (K^{N_1,N_2}_t - K^{N_1}_t)$, for $\zeta>\varphi$. The evolution of $\Psi^{N_1,N_2}_t(x)$ can be written as
\begin{equation}\label{Psi^N_t}
\begin{aligned}
&\Psi^{N_1,N_2}_t(x)
= \int^t_0 \int_{\CX \times \CY} \paren{y-h^{N_1}_s(x')} N_2^{\zeta - \varphi}\left[ l^{N_1,N_2}_s(B^1_{x,x'}(\theta))-l^{N_1}_s(B^1_{x,x'}(\theta))\right] \pi(dx',dy) ds\\
&\quad + \frac{1}{N_1}\sum_{j=1}^{N_1} \int^t_0 \int_{\CX \times \CY} \paren{y-h^{N_1}_s(x')} N_2^{\zeta - \varphi}\left[ l^{N_1,N_2}_s(B^{2,j}_{x,x'}(\theta))-l^{N_1}_s(B^{2,j}_{x,x'}(\theta))\right] \pi(dx',dy) ds\\
&\quad + \frac{1}{N_1}\sum_{j=1}^{N_1}\int_0^t \int_{\CX\times \CY} \left(y-h^{N_1}_s(x')\right)N_2^{\zeta-\varphi}\left[l^{N_1,N_2}_s\left(B^{3,j}_{x}(\theta)\right)-l^{N_1}_s\left(B^{3,j}_{x}(\theta)\right)\right]\ip{xx'B^{3,j}_{x'}(\theta),\gamma^{N_1}_0}\pi(dx',dy)ds\\
&\quad + \frac{1}{N_1}\sum_{j=1}^{N_1}\int_0^t \int_{\CX\times \CY} \left(y-h^{N_1}_s(x')\right)\ip{xx'B^{3,j}_{x}(\theta), \gamma^{N_1}_0}N_2^{\zeta-\varphi}\left[l^{N_1,N_2}_s\left(B^{3,j}_{x'}(\theta)\right)-l^{N_1}_s\left(B^{3,j}_{x'}(\theta)\right)\right]\pi(dx',dy)ds\\
&\quad - \int^t_0 \int_{\CX \times \CY}\Psi^{N_1,N_2}_s(x')\ip{B^1_{x,x'}(\theta)+\frac{1}{N_1}\sum_{j=1}^{N_1}B^{2,j}_{x,x'}(\theta), \gamma^{N_1}_0}\pi(dx',dy) ds\\
&\quad -\frac{1}{N_1}\sum_{j=1}^{N_1}\int_0^t \int_{\CX\times \CY} \Psi^{N_1,N_2}_s(x')xx'\ip{B^{3,j}_{x}(\theta), \gamma^{N_1}_0}\ip{B^{3,j}_{x'}(\theta),\gamma^{N_1}_0}\pi(dx',dy)ds\\
&\quad +N_2^{\zeta - \varphi}\Gamma^{N_1,N_2}_t(x)+ \Psi^{N_1,N_2}_0(x) + N_2^{\zeta}M^{N_1,N_2}_{t} +O(N_2^{-\gamma_2+\zeta}),
\end{aligned}
\end{equation}
where $\Psi^{N_1,N_2}_0(x)= N_2^{1-\gamma_2 + \zeta}\ip{c\sigma(Z^{2,N_1}(x)),\tilde{\gamma}^{N_1,N_2}_0}$, and $\Gamma^{N_1,N_2}_t$ and $M^{N_1,N_2}_{t}$ are as given in Sections \ref{sec::convergence_of_K} and \ref{sec::prelimit_evolution}.
%
We see that if $\zeta \le \gamma_2 - \frac{1}{2}$, the last two remainder terms in equation \eqref{Psi^N_t} converge to zero as $N_2\to \infty$ by the similar analysis in Lemma \ref{CLT:lemma:martingale_bound}. In addition, if $\zeta = \gamma_2 -\frac{1}{2}$, $\Psi^{N_1,N_2}_0(x)= \ip{c\sigma(Z^{2,N_1}(x)),\sqrt{N_2}\tilde{\gamma}^{N_1,N_2}_0}\xrightarrow{d} \mathcal{G}^{N_1}(x)$
where $\mathcal{G}^{N_1}(x)$ is the Gaussian random variable defined in (\ref{limit_gaussian}).
For any fixed $f\in C^3_b(\R^{1+N_1(1+d)})$, let $L^
{N_1,N_2}_t(f) = N_2^{\zeta - \varphi}\left[ l^{N_1,N_2}_t(f)-l^{N_1}_t(f)\right]$. Its the evolution can be written as
\begin{equation}\label{L^N_fluc}
\begin{aligned}
&L^
{N_1,N_2}_t(f)
= N_2^{\zeta - \varphi}\left[ l^{N_1,N_2}_t(f)-l^{N_1,N_2}_0(f)-l^{N_1}_t(f)+l^{N_1,N_2}_0(f)\right]\\
&=N_2^{\zeta - \varphi}\int_0^t \int_{\CX \times \CY} \paren{y-h^{N_1,N_2}_s(x')}  \ip{\partial_{c}f(\theta)  \sigma(Z^{2,N_1}(x')),{\gamma}^{N_1,N_2}_s} \pi(dx',dy) ds\\
&\quad - N_2^{\zeta - \varphi}\int_0^t \int_{\CX\times \CY} \left(y-h_s^{N_1}(x')\right) \ip{\partial_{c}f(\theta)  \sigma(Z^{2,N_1}(x')),{\gamma}^{N_1}_0} \pi(dx',dy)ds\\
&\quad + \frac{N_2^{\zeta - \varphi}}{N_1^{1-\gamma_1}}\int_0^t \int_{\CX \times \CY} \left(y-h_s^{N_1,N_2}(x')\right) \ip{c\sigma'(Z^{2,N_1}(x'))\sigma(w^1x')\cdot \partial_{w^2}f(\theta) ,{\gamma}^{N_1,N_2}_s} \pi(dx',dy) ds\\
&\quad - \frac{N_2^{\zeta - \varphi}}{N_1^{1-\gamma_1}}\int_0^t\int_{\CX\times \CY}\left(y-h_s^{N_1}(x')\right) \ip{c\sigma'(Z^{2,N_1}(x'))\sigma(w^1x')\cdot \partial_{w^2}f(\theta) ,{\gamma}^{N_1}_0}\pi(dx',dy)ds\\
&\quad + \frac{N_2^{\zeta - \varphi}}{N_1^{1-\gamma_1}}\int_0^t\int_{\CX\times \CY}\left(y-h_s^{N_1,N_2}(x')\right)\ip{\ip{c\sigma'(Z^{2,N_1}(x'))\sigma'(w^1x')w^{2},{\gamma}_s^{N_1,N_2}}\cdot  \nabla_{w^1}f(\theta)x',{\gamma}^{N_1,N_2}_s}\pi(dx',dy)ds\\
&\quad - \frac{N_2^{\zeta - \varphi}}{N_1^{1-\gamma_1}}\int_0^t\int_{\CX\times \CY}\left(y-h_s^{N_1}(x')\right)\ip{\ip{c\sigma'(Z^{2,N_1}(x'))\sigma'(w^1x')w^{2},{\gamma}_0^{N_1}}\cdot  \nabla_{w^1}f(\theta)x',{\gamma}^{N_1}_0}\pi(dx',dy)ds\\
&\quad + N_2^{\zeta}\ip{f,\gamma^{N_1,N_2}_0 - \gamma^{N_1}_0}+ N_2^{\zeta - \varphi}M^{N_1,N_2}_{\eta,1,t} + N_2^{\zeta - \varphi}M^{N_1,N_2}_{\eta,2,t} +  N_2^{\zeta - \varphi}M^{N_1,N_2}_{\eta,3,t}+ O\paren{N_2^{-\gamma_2+\zeta-\varphi}}\\
&=(I)_L + (II)_L + (III)_L + \Gamma^{N_1,N_2}_{L,t}+ N_2^{\zeta}\ip{f,\gamma^{N_1,N_2}_0 - \gamma^{N_1}_0}\nonumber\\
&\quad+ N_2^{\zeta - \varphi}\left(M^{N_1,N_2}_{\eta,1,t}+M^{N_1,N_2}_{\eta,2,t}+M^{N_1,N_2}_{\eta,3,t}\right)+ O\paren{N_2^{-\gamma_2+\zeta-\varphi}},
\end{aligned}
\end{equation}
where
\begin{equation*}
\begin{aligned}
(I)_L &= N_2^{\zeta - 2\varphi}\int_0^t \int_{\CX \times \CY} \paren{y-h^{N_1}_s(x')}  \ip{\partial_{c}f(\theta)  \sigma(Z^{2,N_1}(x')),{\eta}^{N_1,N_2}_s} \pi(dx',dy) ds\\
&\quad - N_2^{\zeta - 2\varphi}\int_0^t \int_{\CX\times \CY} K^{N_1,N_2}_s(x') \ip{\partial_{c}f(\theta)  \sigma(Z^{2,N_1}(x')),{\gamma}^{N_1}_0} \pi(dx',dy)ds,
\end{aligned}
\end{equation*}

\begin{equation*}
\begin{aligned}
(II)_L &=\frac{N_2^{\zeta - 2\varphi}}{N_1^{1-\gamma_1}}\int_0^t \int_{\CX \times \CY} \left(y-h_s^{N_1}(x')\right) \ip{c\sigma'(Z^{2,N_1}(x'))\sigma(w^1x')\cdot \partial_{w^2}f(\theta) ,{\eta}^{N_1,N_2}_s} \pi(dx',dy) ds\\
&\quad - \frac{N_2^{\zeta - 2\varphi}}{N_1^{1-\gamma_1}}\int_0^t\int_{\CX\times \CY}K^{N_1,N_2}_s(x')\ip{c\sigma'(Z^{2,N_1}(x'))\sigma(w^1x')\cdot \partial_{w^2}f(\theta) ,{\gamma}^{N_1}_0}\pi(dx',dy)ds,
\end{aligned}
\end{equation*}

\begin{equation*}
\begin{aligned}
(III)_L &= \frac{N_2^{\zeta - 2\varphi}}{N_1^{1-\gamma_1}}\int_0^t\int_{\CX\times \CY}\left(y-h_s^{N_1}(x')\right)\ip{\ip{c\sigma'(Z^{2,N_1}(x'))\sigma'(w^1x')w^{2},{\gamma}_0^{N_1}}\cdot  \nabla_{w^1}f(\theta)x',{\eta}^{N_1,N_2}_s}\pi(dx',dy)ds\\
&\quad +\frac{N_2^{\zeta - 2\varphi}}{N_1^{1-\gamma_1}}\int_0^t\int_{\CX\times \CY}\left(y-h_s^{N_1}(x')\right)\ip{\ip{c\sigma'(Z^{2,N_1}(x'))\sigma'(w^1x')w^{2},{\eta}_s^{N_1,N_2}}\cdot  \nabla_{w^1}f(\theta)x',{\gamma}^{N_1}_0}\pi(dx',dy)ds\\
&\quad - \frac{N_2^{\zeta - 2\varphi}}{N_1^{1-\gamma_1}}\int_0^t\int_{\CX\times \CY}K^{N_1,N_2}_s(x')\ip{\ip{c\sigma'(Z^{2,N_1}(x'))\sigma'(w^1x')w^{2},{\gamma}_0^{N_1}}\cdot  \nabla_{w^1}f(\theta)x',{\gamma}^{N_1}_0}\pi(dx',dy)ds,\\
\end{aligned}
\end{equation*}
and
\begin{equation*}
\begin{aligned}
&\Gamma^{N_1,N_2}_{L,t}=\\
&= -\frac{1}{N_2^{2\varphi-\zeta}}\int_0^t \int_{\CX \times \CY} K^{N_1,N_2}_s(x')\ip{\partial_{c}f(\theta)  \sigma(Z^{2,N_1}(x')),(\gamma^{N_1,N_2}_s-\gamma^{N_1}_0)}\pi(dx',dy)ds\\
&\quad -\frac{1}{N_1^{1-\gamma_1}N_2^{2\varphi-\zeta}}\int_0^t \int_{\CX \times \CY} K^{N_1,N_2}_s(x')\ip{c\sigma'(Z^{2,N_1}(x'))\sigma(w^1x')\cdot \partial_{w^2}f(\theta),(\gamma^{N_1,N_2}_s-\gamma^{N_1}_0)}\pi(dx',dy)ds\\
&\quad -\frac{1}{N_1^{1-\gamma_1}N_2^{2\varphi-\zeta}}\int_0^t \int_{\CX \times \CY} K^{N_1,N_2}_s(x')\ip{\ip{c\sigma'(Z^{2,N_1}(x'))\sigma'(w^1x')w^{2},{\gamma}_0^{N_1}}\cdot  \nabla_{w^1}f(\theta)x',(\gamma^{N_1,N_2}_s-\gamma^{N_1}_0)}\pi(dx',dy)ds\\
&\quad -\frac{1}{N_1^{1-\gamma_1}N_2^{2\varphi-\zeta}}\int_0^t \int_{\CX \times \CY} K^{N_1,N_2}_s(x')\ip{\ip{c\sigma'(Z^{2,N_1}(x'))\sigma'(w^1x')w^{2},(\gamma^{N_1,N_2}_s-{\gamma}_0^{N_1})}\cdot  \nabla_{w^1}f(\theta)x',\gamma^{N_1,N_2}_s}\pi(dx',dy)ds\\
&\quad +\frac{1}{N_1^{1-\gamma_1}N_2^{2\varphi-\zeta}}\int_0^t \int_{\CX \times \CY} \left(y-h_s^{N_1}(x')\right)\ip{\ip{c\sigma'(Z^{2,N_1}(x'))\sigma'(w^1x')w^{2},(\gamma^{N_1,N_2}_s-{\gamma}_0^{N_1})}\cdot  \nabla_{w^1}f(\theta)x',\eta^{N_1,N_2}_s}\pi(dx',dy)ds\\
\end{aligned}
\end{equation*}

The following lemmas show compact containment and regularity of $L^{N_1,N_2}_t(f)$ for any fixed $f\in C^3_b(\R^{1+ N_1(1+d)})$.
\begin{lemma}\label{lemma:L^N_compact containment}
When $\zeta \le 2-2\gamma_2$, for any fixed $f \in C^3_b(\R^{1+ N_1(1+d)})$, there exists a constant $C<\infty$, such that
\[\sup_{N_2 \in \mathbb{N}, 0\le t\le T} \E \left[\abs{L^
{N_1,N_2}_t(f)}^2 \right] < C.\]
Thus, for any $\epsilon>0$, there exist a compact interval  $U \subset \R$, such that
\[\sup_{N_2 \in \mathbb{N}, 0\le t\le T} \P\paren{L^{N_1,N_2}_t(f) \notin U} < \epsilon. \]
\end{lemma}
\begin{proof}By equation \eqref{L^N_fluc} and the Cauchy-Schartz inequality, we have
\begin{equation*}
\begin{aligned}
&\abs{L^{N_1,N_2}_t(f)}^2\\
&\le \frac{C}{N^{2(2\varphi-\zeta)}}\int_0^t \int_{\CX \times \CY} \abs{y-h^{N_1}_s(x')}^2 \pi(dx',dy)ds   \\
&\qquad \times \int_0^t \int_{\CX \times \CY} \abs{\ip{\partial_{c}f(\theta)  \sigma(Z^{2,N_1}(x')),{\eta}^{N_1,N_2}_s}}^2 + \frac{1}{N_1^{2(1-\gamma_1)}}\abs{\ip{c\sigma'(Z^{2,N_1}(x'))\sigma(w^1x')\cdot \partial_{w^2}f(\theta) ,{\eta}^{N_1,N_2}_s}}^2 \pi(dx',dy) ds\\
&\quad + \frac{C}{N^{2(2\varphi-\zeta)}} \int_0^t \int_{\CX\times \CY}\abs{K^{N_1,N_2}_s(x')}^2 \pi(dx',dy)ds \\
&\qquad \times   \int_0^t \int_{\CX\times \CY}\abs{\ip{\partial_{c}f(\theta)  \sigma(Z^{2,N_1}(x')),{\gamma}^{N_1}_0}}^2 + \frac{1}{N_1^{2(1-\gamma_1)}} \abs{\ip{c\sigma'(Z^{2,N_1}(x'))\sigma(w^1x')\cdot \partial_{w^2}f(\theta) ,{\gamma}^{N_1}_0}}^2 \pi(dx',dy)ds\\
&\quad + \frac{C}{N_1^{2(1-\gamma_1)}N_2^{2(2\varphi-\zeta)}}\int_0^t\int_{\CX\times \CY}\abs{y-h_s^{N_1}(x')}^2 \pi(dx',dy)ds\\
&\qquad \times \left\lbrace \int_0^t \int_{\CX \times \CY} \abs{\ip{\ip{c\sigma'(Z^{2,N_1}(x'))\sigma'(w^1x')w^{2},{\gamma}_0^{N_1}}\cdot  \nabla_{w^1}f(\theta)x',{\eta}^{N_1,N_2}_s}}^2\pi(dx',dy)ds \right.\\
&\qquad \quad + \left. \int_0^t\int_{\CX\times \CY}\abs{\ip{\ip{c\sigma'(Z^{2,N_1}(x'))\sigma'(w^1x')w^{2},{\eta}_s^{N_1,N_2}}\cdot  \nabla_{w^1}f(\theta)x',{\gamma}^{N_1}_0}}^2\pi(dx',dy)ds\right\rbrace \\
&\quad + \frac{C}{N_1^{2(1-\gamma_1)}N_2^{2(2\varphi-\zeta)}}\int_0^t\int_{\CX\times \CY}\abs{K^{N_1,N_2}_s(x')}^2 \pi(dx',dy)ds\\
&\qquad \times \int_0^t\int_{\CX \times \CY} \abs{\ip{\ip{c\sigma'(Z^{2,N_1}(x'))\sigma'(w^1x')w^{2},{\gamma}_0^{N_1}}\cdot  \nabla_{w^1}f(\theta)x',{\gamma}^{N_1}_0}}^2\pi(dx',dy)ds\\
&\quad + C \brac{\abs{\Gamma^{N_1,N_2}_{L,t}}^2 + \abs{N_2^{\zeta}\ip{f,\gamma^{N_1,N_2}_0 - \gamma^{N_1}_0}}^2 + \abs{N_2^{\zeta - \varphi}M^{N_1,N_2}_{\eta,1,t}}^2 + \abs{N_2^{\zeta - \varphi}M^{N_1,N_2}_{\eta,2,t}}^2 + \abs{ N_2^{\zeta - \varphi}M^{N_1,N_2}_{\eta,3,t}}^2 + O\paren{N_2^{-2+2\zeta}}}\\
\end{aligned}
\end{equation*}

When $\zeta \le 2\varphi = 2-2\gamma_2$, $0\le t\le T$, the expectation of the first five terms and $\abs{\Gamma^N_{L,t}}^2$  are bounded by  Assumption \ref{assumption}, Lemmas \ref{lemma_parameterbound}, \ref{CLT:lemma:eta_compact_contatinment} and \ref{CLT:lemma:bound of ex_Kt}.
 Since $\gamma_2 > \frac{3}{4}$, $\zeta < \frac{1}{2}$ and $\zeta + \gamma_2 - 1 \le 1-\gamma_2 <\frac{1}{2}$, by similar analysis as in Section \ref{sec::convergence_of_K}, the remainder terms all converges to 0 as $N_2\to \infty$. The result of the lemma follows.
\end{proof}

\begin{lemma}\label{lemma:L^N_regularity}
When $\zeta \le 2 - 2\gamma_2$, for any  $f \in C^3_b(\R^{1+ N_1(1+d)})$, $\delta \in (0,1)$, there is a constant $C<\infty$ such that for $0\le u \le \delta$, $0\le v\le \delta \wedge t$, and $t \in [0,T]$,
\[\E \left[q\paren{L^{N_1,N_2}_{t+u}(f),L^{N_1,N_2}_{t}(f)}q\paren{L^{N_1,N_2}_{t}(f),L^{N_1,N_2}_{t-v}(f)}\vert \CF^{N_1,N_2}_t\right]\le \frac{C}{N_2^{2-2\gamma_2-\zeta}}\delta  + \frac{C}{N_2^{2-\zeta-\gamma_2}}.\]
\end{lemma}
\begin{proof} The proof is identical to that of Lemma C.2 of \cite{SpiliopoulosYu2021} and thus it is omitted from here.
\end{proof}

%

Denote $\mathfrak{K}^{N_1,N_2}_t= (\gamma^{N_1,N_2}_t, h^{N_1,N_2}_t, l^{N_1,N_2}_{1,t}, l^{N_1,N_2}_{2,t}, l^{N_1,N_2}_{3,t}, K^{N_1,N_2}_t)$. In the next lemma, we prove the convergence of the processes $(\mathfrak{K}^{N_1,N_2}_t,L^{N_1,N_2}_t(f))$ in distribution in the space $D_{E_2}([0,T])$, where $E_2 = \CM(\R^{1+N_1(1+d)}) \times \R^M \times \R \times \R^{N_1}\times \R^{N_1} \times \R^M \times \R$.
\begin{lemma}\label{lemma:limit_L}
When $\gamma_2 \in \paren{\frac{3}{4},1}$, $\varphi = 1-\gamma_2$ and $\zeta\le 2\varphi$, for any fixed $f\in C^3_b(\R^{1+N_1(1+d)})$, the processes $(\mathfrak{K}^{N_1,N_2}_t,L^{N_1,N_2}_t(f))$ converge in distribution in the space $D_{E_2}([0,T])$ to $(\mathfrak{K}^{N_1}_t,L^{N_1}_t(f))$, where $\mathfrak{K}^{N_1}_t = (\gamma^{N_1}_t, h^{N_1}_t, l^{N_1}_{1,t}, l^{N_1}_{2,t}, l^{N_1}_{3,t}, K^{N_1}_t)$ satisfying equations \eqref{h_N1_evolution}, \eqref{CLT:evolution_l}, and \eqref{K_t for 1-gamma_2}. When $\zeta <2\varphi$, $L^{N_1}_t(f)=0$. When $\zeta = 2\varphi$, $L^{N_1}_t(f)$ satisfies (\ref{limit_Lt_0}).

\end{lemma}
\begin{proof}Recall that $\{\mathfrak{K}^{N_1,N_2}\}_{N_2\in\mathbb{N}}$ is relatively compact in $D_{E_1}([0,T])$, where $E_1= \CM(\R^{1+N_1(1+d)}) \times \R^M \times \R \times \R^{N_1}\times \R^{N_1} \times \R^M $. By Lemmas \ref{lemma:L^N_compact containment} and \ref{lemma:L^N_regularity}, $\{L^{N_1,N_2}(f)\}_{N_2\in\mathbb{N}}$ is relatively compact in $D_{\R}([0,T])$. These implies that the probability measures of the family of processes $\{\mathfrak{K}^{N_1,N_2}\}_{N_2\in\mathbb{N}}$ and the probability measures of the family of processes $\{L^{N_1,N_2}(f)\}_{N_2\in\mathbb{N}}$ are tight. Therefore, $\{\mathfrak{K}^{N_1,N_2},L^{N_1,N_2}(f)\}_{N_2\in\mathbb{N}}$ is tight. Hence, $\{\mathfrak{K}^{N_1,N_2},L^{N_1,N_2}(f)\}_{N_2\in\mathbb{N}}$ is relatively compact in $D_{E_2}([0,T])$.

Denote $\pi^{N_1,N_2} \in \CM(D_{E_2}([0,T])$ the probability measure corresponding to $(\mathfrak{K}^{N_1,N_2},L^{N_1,N_2}(f))$. Relative compactness implies that there is a subsequence $\pi^{N_1,N_{2_k}}$ that converges weakly. One can show that any limit point $\pi^{N_1}$ of a convergence subsequence $\pi^{N_1,N_{2_k}}$ is a Dirac measure concentrated on $(\mathfrak{K}^{N_1},L^{N_1}(f))\in D_{E_2}([0,T])$.

\begin{custlist}[Case]
\item When $\zeta < 2\varphi$, for any $t \in [0,T]$, $b_1, \ldots, b_p \in C_b(\R)$, and $0 \le s_1 < \cdots < s_p \le t$, we define $F_4(\mathfrak{K},L(f)): D_{E_2}([0,T]) \to \R_+$ as
\begin{equation}\label{F4}
\begin{aligned}
F_4(\mathfrak{K},L(f)) &= F_3(\gamma^{N_1}, h^{N_1}, l_1^{N_1}, l_2^{N_1}, l_3^{N_1}, K^{N_1}) + \left\vert \left(L^{N_1}_t(f)-0 \right) \times b_1(L^{N_1}_{s_1}(f)) \times \cdots \times b_p(L^{N_1}_{s_p}(f)) \right \vert,
\end{aligned}
\end{equation}
where $F_3$ is as given in equation \eqref{F_3}. By equation \eqref{L^N_fluc}, Lemma \ref{CLT:lemma:martingale2}, and similar analysis as in Lemma \ref{lemma:L^N_compact containment}, we have
\begin{align*}
\E_{\pi^{N_1,N_2}} \left[F_4(\mathfrak{K},L(f))\right] &= E_{\pi^{N_1,N_2}} \left[F_3(\mu, h, l_B, K)\right] + \E \left[ \left\vert \left(L^{N_1,N_2}_t(f)-0 \right) \times \prod_{i=1}^{p} b_i(L^{N_1,N_2}_{s_i}(f))  \right \vert\right] \\
&\le C\paren{\frac{1}{N_2^{1-\gamma_2}}+ \frac{1}{N_2^{\frac{1}{2}-\varphi}}+\frac{1}{N_2^{1-\varphi}}+ \frac{1}{N_2^{\gamma_2-\varphi}}} + \frac{C}{N_2^{2\varphi-\zeta}} + \frac{C}{N_1^{1-\gamma_2}N_2^{2\varphi-\zeta}}+ \frac{C}{N_2^{\frac{1}{2}-\zeta}} \\
&\quad  + \E\left[\abs{N_2^{\zeta+\gamma_2-1}M^{N_1,N_2}_{\eta,1,t}}^2\right]^{\frac{1}{2}} +\E\left[\abs{N_2^{\zeta+\gamma_2-1}M^{N_1,N_2}_{\eta,2,t}}^2\right]^{\frac{1}{2}}+ \frac{C}{N_2^{1-\zeta}}\\
&\le C \paren{\frac{1}{N_2^{1-\gamma_2}} + \frac{1}{N_2^{2\varphi-\zeta}}+\frac{1}{N_2^{\frac{3}{2}-\zeta-\gamma_2}} }.
\end{align*}
Therefore, $\lim_{N_2\to \infty} \E_{\pi^{N_1,N_2}}[F_4(\mathfrak{K},L(f))] = 0$. Since $F_4(\cdot)$ is continuous and uniformly bounded,
\[\lim_{N_2\to \infty} \E_{\pi^{N_1,N_2}}\left[F_4(\mathfrak{K},L(f))\right] = \E_{\pi^{N_1}}\left[F_4(\mathfrak{K},L(f))\right] = 0.\]
Since relative compactness implies that every subsequence $\pi^{N_1,N_{2_k}}$ has a further sub-subsequence that converges weakly. And we have show that any limit point $\pi^{N_1}$ of a convergence sequence must be a Dirac measure concentrated $(\mathfrak{K}^{N_1},L^{N_1}(f))\in D_{E_2}([0,T])$, where  $L_t(f)=0$. Since the solutions to equations \eqref{h_N1_evolution} and \eqref{K_t for 1-gamma_2} are unique, by Prokhorov's theorem, the processes $(\mathfrak{K}^{N_1,N_2}_t,L^{N_1,N_2}_t(f))$ converges in distribution to $(\mathfrak{K}^{N_1}_t,0)$.
\item When $\zeta = 2\varphi$, for any $t \in [0,T]$, $b_1, \ldots, b_p \in C_b(\R)$, and $0 \le s_1 < \cdots < s_p \le t$, we define $F_4(\mathfrak{K},L(f)): D_{E_2}([0,T]) \to \R_+$ as

\begin{equation}\label{F5}
\begin{aligned}
&F_5(\mathfrak{K},L(f))\\
 &= F_3(\gamma^{N_1}, h^{N_1}, l_1^{N_1}, l_2^{N_1}, l_3^{N_1}, K^{N_1}) + \left\vert \left(L^{N_1}_t(f)-\int_0^t \int_{\CX \times \CY} \paren{y-h^{N_1}_s(x')}  l^{N_1}_s\left(\partial_{c}f(\theta)  \sigma(Z^{2,N_1}(x'))\right) \pi(dx',dy) ds\right.\right.\\
&\qquad+  \int_0^t \int_{\CX\times \CY} K^{N_1}_s(x') \ip{\partial_{c}f(\theta)  \sigma(Z^{2,N_1}(x')),{\gamma}^{N_1}_0} \pi(dx',dy)ds\\
&\qquad  - \frac{1}{N_1^{1-\gamma_1}}\int_0^t \int_{\CX \times \CY} \left(y-h_s^{N_1}(x')\right) l^{N_1}_s\left(c\sigma'(Z^{2,N_1}(x'))\sigma(w^1x')\cdot \partial_{w^2}f(\theta) \right) \pi(dx',dy) ds\\
&\qquad + \frac{1}{N_1^{1-\gamma_1}}\int_0^t\int_{\CX\times \CY}K^{N_1}_s(x')\ip{c\sigma'(Z^{2,N_1}(x'))\sigma(w^1x')\cdot \partial_{w^2}f(\theta) ,{\gamma}^{N_1}_0}\pi(dx',dy)ds\\
&\qquad  - \frac{1}{N_1^{1-\gamma_1}}\int_0^t\int_{\CX\times \CY}\left(y-h_s^{N_1}(x')\right)l^{N_1}_s\left(\ip{c\sigma'(Z^{2,N_1}(x'))\sigma'(w^1x')w^{2},{\gamma}_0^{N_1}}\cdot  \nabla_{w^1}f(\theta)x'\right)\pi(dx',dy)ds\\
&\qquad -\frac{1}{N_1^{1-\gamma_1}}\int_0^t\int_{\CX\times \CY}\left(y-h_s^{N_1}(x')\right)\ip{l^{N_1}_s\left(c\sigma'(Z^{2,N_1}(x'))\sigma'(w^1x')w^{2}\right) \cdot  \nabla_{w^1}f(\theta)x',{\gamma}^{N_1}_0}\pi(dx',dy)ds\\
&\qquad \left. + \frac{1}{N_1^{1-\gamma_1}}\int_0^t\int_{\CX\times \CY}K^{N_1}_s(x')\ip{\ip{c\sigma'(Z^{2,N_1}(x'))\sigma'(w^1x')w^{2},{\gamma}_0^{N_1}}\cdot  \nabla_{w^1}f(\theta)x',{\gamma}^{N_1}_0}\pi(dx',dy)ds\right) \\
&\qquad \left. \times b_1(L_{s_1}(f)) \times \cdots \times b_p(L_{s_p}(f)) \right \vert,
\end{aligned}
\end{equation}
where $F_3$ is as given in equation \eqref{F_3}. We first note that by equation \eqref{L^N_fluc}
\begin{equation*}
\begin{aligned}
&L^{N_1,N_2}_t(f)-\int_0^t \int_{\CX \times \CY} \paren{y-h^{N_1,N_2}_s(x')}  l^{N_1,N_2}_s\left(\partial_{c}f(\theta)  \sigma(Z^{2,N_1}(x'))\right) \pi(dx',dy) ds\\
&\qquad+  \int_0^t \int_{\CX\times \CY} K^{N_1,N_2}_s(x') \ip{\partial_{c}f(\theta)  \sigma(Z^{2,N_1}(x')),{\gamma}^{N_1,N_2}_0} \pi(dx',dy)ds\\
&\qquad  - \frac{1}{N_1^{1-\gamma_1}}\int_0^t \int_{\CX \times \CY} \left(y-h_s^{N_1,N_2}(x')\right) l^{N_1,N_2}_s\left(c\sigma'(Z^{2,N_1}(x'))\sigma(w^1x')\cdot \partial_{w^2}f(\theta) \right) \pi(dx',dy) ds\\
&\qquad + \frac{1}{N_1^{1-\gamma_1}}\int_0^t\int_{\CX\times \CY}K^{N_1,N_2}_s(x')\ip{c\sigma'(Z^{2,N_1}(x'))\sigma(w^1x')\cdot \partial_{w^2}f(\theta) ,{\gamma}^{N_1,N_2}_0}\pi(dx',dy)ds\\
&\qquad  - \frac{1}{N_1^{1-\gamma_1}}\int_0^t\int_{\CX\times \CY}\left(y-h_s^{N_1,N_2}(x')\right)l^{N_1,N_2}_s\left(\ip{c\sigma'(Z^{2,N_1}(x'))\sigma'(w^1x')w^{2},{\gamma}_0^{N_1,N_2}}\cdot  \nabla_{w^1}f(\theta)x'\right)\pi(dx',dy)ds\\
&\qquad -\frac{1}{N_1^{1-\gamma_1}}\int_0^t\int_{\CX\times \CY}\left(y-h_s^{N_1,N_2}(x')\right)\ip{l^{N_1,N_2}_s\left(c\sigma'(Z^{2,N_1}(x'))\sigma'(w^1x')w^{2}\right) \cdot  \nabla_{w^1}f(\theta)x',{\gamma}^{N_1,N_2}_0}\pi(dx',dy)ds\\
&\qquad  + \frac{1}{N_1^{1-\gamma_1}}\int_0^t\int_{\CX\times \CY}K^{N_1,N_2}_s(x')\ip{\ip{c\sigma'(Z^{2,N_1}(x'))\sigma'(w^1x')w^{2},{\gamma}_0^{N_1,N_2}}\cdot  \nabla_{w^1}f(\theta)x',{\gamma}^{N_1,N_2}_0}\pi(dx',dy)ds \nonumber
\end{aligned}
\end{equation*}
\begin{equation*}
\begin{aligned}
&= \frac{1}{N_2^{\varphi}} \int_0^t \int_{\CX \times \CY}   K^{N_1,N_2}_s(x') \ip{\partial_{c}f(\theta)  \sigma(Z^{2,N_1}(x'))+ \frac{1}{N_1^{1-\gamma_1}}c\sigma'(Z^{2,N_1}(x'))\sigma(w^1x')\cdot \partial_{w^2}f(\theta), \eta^{N_1,N_2}_0} \pi(dx',dy) ds\\
&\quad + \frac{1}{N_2^{2\varphi}N_1^{1-\gamma_1}}\int_0^t\int_{\CX\times \CY}K^{N_1,N_2}_s(x')\ip{\ip{c\sigma'(Z^{2,N_1}(x'))\sigma'(w^1x')w^{2},{\eta}_0^{N_1,N_2}}\cdot  \nabla_{w^1}f(\theta)x',{\eta}^{N_1,N_2}_s}\pi(dx',dy)ds \\
&\quad - \frac{1}{N_2^{2\varphi}N_1^{1-\gamma_1}}\int_0^t\int_{\CX\times \CY}K^{N_1,N_2}_s(x')\ip{\ip{c\sigma'(Z^{2,N_1}(x'))\sigma'(w^1x')w^{2},{\eta}_s^{N_1,N_2}}\cdot  \nabla_{w^1}f(\theta)x',{\eta}^{N_1,N_2}_s}\pi(dx',dy)ds \\
&\quad + \frac{1}{N_2^{2\varphi}N_1^{1-\gamma_1}}\int_0^t\int_{\CX\times \CY}K^{N_1,N_2}_s(x')\ip{\ip{c\sigma'(Z^{2,N_1}(x'))\sigma'(w^1x')w^{2},{\eta}_s^{N_1,N_2}}\cdot  \nabla_{w^1}f(\theta)x',{\eta}^{N_1,N_2}_0}\pi(dx',dy)ds\\
&\quad + \frac{1}{N_2^{\varphi}N_1^{1-\gamma_1}}\int_0^t\int_{\CX\times \CY}K^{N_1,N_2}_s(x')\ip{\ip{c\sigma'(Z^{2,N_1}(x'))\sigma'(w^1x')w^{2},{\gamma}_0^{N_1}}\cdot  \nabla_{w^1}f(\theta)x',{\eta}^{N_1,N_2}_0}\pi(dx',dy)ds\\
&\quad + \frac{1}{N_2^{\varphi}N_1^{1-\gamma_1}}\int_0^t\int_{\CX\times \CY}K^{N_1,N_2}_s(x')\ip{\ip{c\sigma'(Z^{2,N_1}(x'))\sigma'(w^1x')w^{2},{\eta}_0^{N_1}}\cdot  \nabla_{w^1}f(\theta)x',{\gamma}^{N_1,N_2}_0}\pi(dx',dy)ds\\
&\quad - \frac{1}{N_2^{\varphi}N_1^{1-\gamma_1}}\int_0^t\int_{\CX\times \CY}\left(y-h_s^{N_1}(x')\right)\ip{\ip{c\sigma'(Z^{2,N_1}(x'))\sigma'(w^1x')w^{2},{\eta}_0^{N_1,N_2}}\cdot  \nabla_{w^1}f(\theta)x',\eta^{N_1,N_2}_s}\pi(dx',dy)ds\\
&\quad + \frac{1}{N_2^{\varphi}N_1^{1-\gamma_1}}\int_0^t\int_{\CX\times \CY}\left(y-h_s^{N_1}(x')\right)\ip{\ip{c\sigma'(Z^{2,N_1}(x'))\sigma'(w^1x')w^{2},{\eta}_s^{N_1,N_2}}\cdot  \nabla_{w^1}f(\theta)x',\eta^{N_1,N_2}_s}\pi(dx',dy)ds\\
&\quad - \frac{1}{N_2^{\varphi}N_1^{1-\gamma_1}}\int_0^t\int_{\CX\times \CY}\left(y-h_s^{N_1}(x')\right)\ip{\ip{c\sigma'(Z^{2,N_1}(x'))\sigma'(w^1x')w^{2},{\eta}_s^{N_1,N_2}}\cdot  \nabla_{w^1}f(\theta)x',\eta^{N_1,N_2}_0}\pi(dx',dy)ds\\
&\quad + N_2^{\zeta-\frac{1}{2}}\ip{f,\sqrt{N_2}(\gamma^{N_1,N_2}_0 - \gamma^{N_1}_0)}+ N_2^{\zeta +\gamma_2 -1}M^{N_1,N_2}_{\eta,1,t} + N_2^{\zeta +\gamma_2 -1}M^{N_1,N_2}_{\eta,2,t}+ O\paren{N_2^{-1+\zeta}}
\end{aligned}
\end{equation*}
By similar analysis as for equations \eqref{eq:K^N:temp} to \eqref{eq_temp_Klast}, the expectation of the absolute value of the first nine terms above are bounded by $O(N_2^{-\varphi})$.
Then by Lemma \ref{CLT:lemma:martingale2}, we have
\begin{align*}
\E_{\pi^{N_1,N_2}} \left[F_5(\mathfrak{K},L(f))\right]
&\le  C\paren{\frac{1}{N_2^{1-\gamma_2}}+ \frac{1}{N_2^{\frac{1}{2}-\varphi}}+\frac{1}{N_2^{1-\varphi}}+ \frac{1}{N_2^{\gamma_2-\varphi}}}+ \frac{C}{N_2^{\varphi}}+ \frac{C}{N_2^{\frac{1}{2}-\zeta}}\\
&\quad   + \E\left[\abs{N_2^{\zeta+\gamma_2-1}M^{N_1,N_2}_{\eta,1,t}}^2\right]^{\frac{1}{2}} +\E\left[\abs{N_2^{\zeta+\gamma_2-1}M^{N_1,N_2}_{\eta,2,t}}^2\right]^{\frac{1}{2}}+ \frac{C}{N_2^{1-\zeta}}\\
&\le C\paren{\frac{1}{N_2^{1-\gamma_2}}+ \frac{1}{N_2^{\frac{1}{2}-\varphi}}+\frac{1}{N_2^{1-\varphi}} + \frac{1}{N_2^{\frac{1}{2}-\zeta}} + \frac{1}{N_2^{\frac{3}{2}-\zeta-\gamma_2}}+ \frac{1}{N_2^{1-\zeta}}}\\
&\le C\paren{\frac{1}{N_2^{1-\gamma_2}}+ \frac{1}{N_2^{\frac{1}{2}-\zeta}}}.
\end{align*}
Therefore, $\lim_{N_2\to \infty} \E_{\pi^{N_1,N_2}}[F_5(\mathfrak{K},L(f))] = 0$. Since $F_5(\cdot)$ is continuous and uniformly bounded,
\[\lim_{N_2\to \infty} \E_{\pi^{N_1,N_2}}\left[F_5(\mathfrak{K},L(f))\right] = \E_{\pi^{N_1}}\left[F_5(\mathfrak{K},L(f))\right] = 0.\]
The result then follows.
\end{custlist}
\end{proof}

Moving back to the analysis of $\Psi^{N_1,N_2}_t$, we first show compact containment of $\Psi^N_t$ in the next lemma.
\begin{lemma}\label{Psi_compact_containment}When $\zeta \le \min\{\gamma_2 -\frac{1}{2}, 2-2\gamma_2\}$,
there exit a constant $C<\infty$, such that
\[\sup_{N_2 \in \mathbb{N}, 0\le t\le T} \E \left[\abs{\Psi^{N_1,N_2}_t(x)}^2 \right] < C.\]
Thus, for any $\epsilon>0$, there exist a compact subset $U \subset \R^M$, such that
\[\sup_{N_2\in \mathbb{N}, 0\le t\le T} \P\paren{\Psi^{N_1,N_2}_t \notin U} < \epsilon. \]
\end{lemma}
\begin{proof}In the proof below, $C<\infty$ represents some positive constant, which may be different from line to line. We first rewrite the term $N_2^{\zeta - \varphi}\Gamma^{N_1,N_2}_t(x) = N_2^{\zeta - \varphi}\Gamma^{N_1,N_2}_{1,t}(x) + N_2^{\zeta - \varphi}\Gamma^{N_1,N_2}_{2,t}(x)+N_2^{\zeta - \varphi}\Gamma^{N_1,N_2}_{3,t}(x)$ as
\begin{equation*}
\begin{aligned}
N_2^{\zeta - \varphi}\Gamma^{N_1,N_2}_{1,t}(x) &= -\frac{1}{N_2^{\varphi}}\int_0^t \int_{\CX\times \CY} \Psi^{N_1,N_2}_s(x')\ip{B^1_{x,x'}(\theta), \eta^{N_1,N_2}_s}\pi(dx',dy)ds\\
&\quad - \frac{1}{N_2^{2\varphi-\zeta}}\int_0^t \int_{\CX\times \CY} K^{N_1}_s(x')\ip{B^1_{x,x'}(\theta), \eta^{N_1,N_2}_s}\pi(dx',dy)ds,
\end{aligned}
\end{equation*}
\begin{equation*}
\begin{aligned}
N_2^{\zeta - \varphi}\Gamma^{N_1,N_2}_{2,t}(x) &= -\frac{1}{N_1N_2^{\varphi}}\sum_{j=1}^{N_1} \int_0^t \int_{\CX\times \CY} \Psi^{N_1,N_2}_s(x')\ip{B^{2,j}_{x,x'}(\theta), \eta^{N_1,N_2}_s}\pi(dx',dy)ds\\
&\quad -\frac{1}{N_1N_2^{2\varphi-\zeta}}\sum_{j=1}^{N_1} \int_0^t \int_{\CX\times \CY} K^{N_1}_s(x')\ip{B^{2,j}_{x,x'}(\theta), \eta^{N_1,N_2}_s}\pi(dx',dy)ds,
\end{aligned}
\end{equation*}
\begin{equation*}
\begin{aligned}
N_2^{\zeta - \varphi}\Gamma^{N_1,N_2}_{3,t}(x) &= -\frac{1}{N_1N_2^{\varphi}} \sum_{j=1}^{N_1} \int_0^t \int_{\CX\times \CY} \Psi^{N_1,N_2}_s(x')xx'\ip{B^{3,j}_{x}(\theta), \eta^{N_1,N_2}_s}\ip{B^{3,j}_{x'}(\theta),\gamma^{N_1,N_2}_s}\pi(dx',dy)ds \\
&\quad  -\frac{1}{N_1N_2^{2\varphi-\zeta}} \sum_{j=1}^{N_1} \int_0^t \int_{\CX\times \CY} K^{N_1}_s(x')xx'\ip{B^{3,j}_{x}(\theta), \eta^{N_1,N_2}_s}\ip{B^{3,j}_{x'}(\theta),\gamma^{N_1,N_2}_s}\pi(dx',dy)ds \\
&\quad -\frac{1}{N_1N_2^{\varphi}} \sum_{j=1}^{N_1} \int_0^t \int_{\CX\times \CY} \Psi^{N_1,N_2}_s(x')xx'\ip{B^{3,j}_{x}(\theta), \gamma^{N_1}_0}\ip{B^{3,j}_{x'}(\theta),\eta^{N_1,N_2}_s}\pi(dx',dy)ds\\
&\quad -\frac{1}{N_1N_2^{2\varphi-\zeta}} \sum_{j=1}^{N_1} \int_0^t \int_{\CX\times \CY} K^{N_1}_s(x')xx'\ip{B^{3,j}_{x}(\theta), \gamma^{N_1}_0}\ip{B^{3,j}_{x'}(\theta),\eta^{N_1,N_2}_s}\pi(dx',dy)ds\\
&\quad + \frac{1}{N_1N_2^{\varphi}} \sum_{j=1}^{N_1}\int_0^t \int_{\CX\times \CY} \left(y-h^{N_1}_s(x')\right)xx'\ip{B^{3,j}_{x}(\theta), \eta^{N_1,N_2}_s}\ip{B^{3,j}_{x'}(\theta),\eta^{N_1,N_2}_s}\pi(dx',dy)ds.
\end{aligned}
\end{equation*}

Since the above terms involves the term $K^{N_1}_t(x)$, we first look at the bound for $K^{N_1}_t(x)$. By the Cauchy-Schwarz inequality, equations \eqref{h_N1_evolution}, \eqref{CLT:evolution_l}, \eqref{K_t for 1-gamma_2} and the analysis in Lemma \ref{CLT:lemma:bound of ex_Kt}, for any $t \in [0,T]$, we have
\begin{equation*}
\abs{K^{N_1}_t(x)}^2 \le Ct^2 + Ct \int_0^t \int_{\CX \times \CY} \abs{K^{N_1}_s(x')}^2 \pi(dx',dy)ds \le CT^2 + \frac{CT}{M} \int_0^t  \sum_{x'\in\mathcal{X}}\abs{K^{N_1}_s(x')}^2 \pi(dx',dy)ds.
\end{equation*}
Summing over $x\in \mathcal{X}$ on both sides gives
\begin{equation*}
\sum_{x\in \mathcal{X}} \abs{K^{N_1}_t(x)}^2  \le CT^2M + CT \int_0^t  \sum_{x'\in\mathcal{X}}\abs{K^{N_1}_s(x')}^2 \pi(dx',dy)ds.
\end{equation*}
By applying Gr\"onwall's inequality, we have
\begin{equation*}
\sup_{0\le t\le T} \sum_{x\in \mathcal{X}} \abs{K^{N_1}_t(x)}^2  \le \sup_{0\le t\le T} CT^2M \exp\paren{CTt} < C,
\end{equation*}
which implies that $\sup_{0\le t\le T}\abs{K^{N_1}_t(x)}^2<C$ for any $x\in \mathcal{X}$. Using this uniform bound for $K^{N_1}_t(x)$ together with Lemma \ref{CLT:lemma:eta_compact_contatinment}, similar analysis as for equation \eqref{K_Gamma_bound} gives
\begin{equation*}
\begin{aligned}
&\E\left[\abs{N_2^{\zeta - \varphi}\Gamma^{N_1,N_2}_t(x)}^2\right]\le Ct\int_0^t \int_{\CX \times \CY} \E \left[\abs{\Psi^{N_1,N_2}_s(x')}^2\right] \pi(dx',dy)ds +Ct^2 + \frac{Ct^2}{N_2^{2(2\varphi-\zeta)}}
\end{aligned}
\end{equation*}

By equations \eqref{B0_bound}, \eqref{CLT:network_bound}, \eqref{Psi^N_t} and Lemmas \ref{CLT:lemma:martingale_bound} and  \ref{lemma:L^N_compact containment}, we have
\begin{equation*}
\begin{aligned}
&\E \left[\abs{\Psi^{N_1,N_2}_t(x)}^2\right]
\le Ct \int_0^t \int_{\CX \times \CY} \E\left[\abs{L^{N_1,N_1}_s(B^1_{x,x'}(\theta))}^2 + \frac{1}{N_1}\sum_{j=1}^{N_1}\abs{L^{N_1,N_1}_s(B^{2,j}_{x,x'}(\theta))}^2 \right]\pi(dx',dy)ds\\
&\quad + Ct\int_0^t \int_{\CX \times \CY} \E \left[ \abs{L^{N_1,N_1}_s(B^{3,j}_{x}(\theta))}^2+\abs{L^{N_1,N_1}_s(B^{3,j}_{x'}(\theta))}^2+\abs{\Psi^{N_1,N_2}_s(x')}^2\right] \pi(dx',dy)ds\\
&\quad + \E\left[\abs{N_2^{\zeta - \varphi}\Gamma^{N_1,N_2}_t(x)}^2\right] + C\E \left[\abs{\Psi^{N_1,N_2}_0(x)}^2\right]+ C\E \left[\abs{N_2^{\zeta}M^{N_1,N_2}_t(x)}^2\right]\\
&\le Ct^2 + Ct\int_0^t \int_{\CX \times \CY} \E \left[\abs{\Psi^{N_1,N_2}_s(x')}^2\right] \pi(dx',dy)ds + \frac{Ct^2}{N_2^{2(2\varphi-\zeta)}}+ C + \frac{C}{N_2^{1-2\zeta}} \\
\end{aligned}
\end{equation*}
Summing over $x\in \mathcal{X}$ on both sides gives
\begin{equation*}
\begin{aligned}
\sum_{x\in \mathcal{X}}\E \left[\abs{\Psi^{N_1,N_2}_s(x)}^2\right] &\le CMT^2 + CT\int_0^t \sum_{x'\in\mathcal{X}} \E \left[\abs{\Psi^{N_1,N_2}_s(x')}^2\right] \pi(dx',dy)ds.
\end{aligned}
\end{equation*}
By Gr\"onwall's inequality, we get
\begin{equation*}
\sup_{0\le t\le T} \sum_{x\in \mathcal{X}} \E\left[\abs{\Psi^{N_1,N_2}_t(x)}^2\right]  \le \sup_{0\le t\le T} CT^2M \exp\paren{CTt} < C,
\end{equation*}
which implies that $\sup_{0\le t\le T}\E\left[\abs{\Psi^{N_1,N_2}_t(x)}^2\right] < C$ for any $x\in \mathcal{X}$. The result of the lemma then follows.
\end{proof}

The next lemma establishes the regularity of the process $\Psi^{N_1,N_2}_t$ in $D_{\R^M}([0,T])$. For the purpose of this lemma, we denote $q(z_1,z_2) = \min\{\norm{ z_1-z_2}_{l^1},1\}$ for $z_1,z_2 \in \R^M$. The proof of the lemma is similar to that for Lemma \ref{CLT:lemma:regularity}, which we omit here.
\begin{lemma}\label{Psi_regularity}
For any   $\delta \in (0,1)$, there is a constant $C<\infty$ such that for $0\le u \le \delta$, $0\le v\le \delta \wedge t$, and $t \in [0,T]$,
\[\E \left[q\paren{\Psi^{N_1,N_2}_{t+u},\Psi^N_{t}}q\paren{\Psi^{N_1,N_2}_{t},\Psi^{N_1,N_2}_{t-v}}\vert \CF^{N_1,N_2}_t\right]\le {C\delta}+\frac{C}{N_2^{1-\zeta}}.\]
\end{lemma}
Combining these with our analysis of $L^{N_1,N_2}_t(f)$, we can now identify the limit for $\Psi^{N_1,N_2}_t$.
We denote
$\mathfrak{L}^{N_1,N_2}_t=(\mathfrak{K}^{N_1,N_2}_t,L^{N_1,N_2}_{1,t},L^{N_1,N_2}_{2,t},L^{N_1,N_2}_{3,t})$,
where $L^{N_1,N_2}_{1,t}= L^{N_1,N_2}_t(B^{1}_{x,x'}(\theta))$, $L^{N_1,N_2}_{2,t}$ and $L^{N_1,N_2}_{3,t}$ are $N_1$-dimensional vectors with $j$-th entry being $L^{N_1,N_2}_t(B^{2,j}_{x,x'}(\theta))$ and $L^{N_1,N_2}_t(B^{3,j}_{x}(\theta))$, respectively.
 In the next lemma, we prove the convergence of the processes $(\mathfrak{L}^{N_1,N_2}_t,\Psi^{N_1,N_2}_t)$ in distribution in the space $D_{E_3}([0,T])$, where $E_3 = \CM(\R^{1+N_1(1+d)}) \times \R^M \times \R \times \R^{N_1} \times \R^{N_1}\times \R^M \times \R \times \R^{N_1} \times \R^{N_1} \times \R^M$.
\begin{lemma}When $\gamma_2 \in \paren{\frac{3}{4},1}$, $\varphi = 1-\gamma_2$ and $\zeta\le \gamma_2 -\frac{1}{2}$, the processes $(\mathfrak{L}^{N_1,N_2}_t,\Psi^{N_1,N_2}_t)$ in distribution in the space $D_{E_3}([0,T])$ to $(\mathfrak{L}^{N_1}_t,\Psi^{N_1}_t)$. In particular, $\mathfrak{L}^{N_1}_t = (\gamma^{N_1}_0, h^{N_1}_t, l^{N_1}_{1,t}, l^{N_1}_{2,t}, l^{N_1}_{3,t}, K^{N_1}_t, L^{N_1}_{1,t}, L^{N_1}_{2,t}, L^{N_1}_{3,t})$ satisfies equations \eqref{h_N1_evolution}, \eqref{CLT:evolution_l}, and \eqref{K_t for 1-gamma_2}, $L^{N_1}_{1,t}, L^{N_1}_{2,t}, L^{N_1}_{3,t}, \Psi^{N_1}_t$ satisfy either of the following case:
\begin{custlist}[Case]
\item When $\gamma_2 \in \paren{\frac{3}{4},\frac{5}{6}}$ and $\zeta \le \gamma_2 - \frac{1}{2}$, or when $\gamma_2 \in \left[\frac{5}{6},1\right)$ and $\zeta < 2-2\gamma_2 \le \gamma_2-\frac{1}{2}$,  one has $L^{N_1}_{1,t}=0$, $L^{N_1}_{2,t}= L^{N_1}_{3,t}=0$ and $\Psi^{N_1}_t$ satisfies (\ref{limit_Psi}).
\item When $\gamma_2 \in \left[\frac{5}{6},1\right)$ and $\zeta = 2-2\gamma_2 \le \gamma_2-\frac{1}{2}$, $L^{N_1}_{1,t}, L^{N_1}_{2,t}, L^{N_1}_{3,t}$ satisfy equation \eqref{limit_Lt_0} and $\Psi_t$ satisfies (\ref{Psi_2_0}).
\end{custlist}
\end{lemma}
\begin{proof}
By analysis in Lemma \ref{lemma:limit_L}, $\{\mathfrak{L}^{N_1,N_2}\}_{N_2\in\mathbb{N}}$ is relatively compact in $D_{E_4}([0,T])$, where $E_4= \CM(\R^{1+N_1(1+d)}) \times \R^M \times \R \times \R^{N_1} \times \R^{N_1} \times \R^M \times \R \times \R^{N_1}\times \R^{N_1}$. By Lemmas \ref{Psi_compact_containment} and \ref{Psi_regularity}, $\{\Psi^{N_1,N_2}\}_{N_2\in\mathbb{N}}$ is relatively compact in $D_{\R^M}([0,T])$. These implies that the probability measures of the family of processes $\{\mathfrak{L}^{N_1,N_2}\}_{N_2\in\mathbb{N}}$ and the probability measures of the family of processes $\{\Psi^{N_1,N_2}\}_{N_2\in\mathbb{N}}$ are tight. Therefore, $\{\mathfrak{L}^{N_1,N_2},\Psi^{N_1,N_2}\}_{N_2\in\mathbb{N}}$ is tight. Hence, $\{\mathfrak{L}^{N_1,N_2},\Psi^{N_1,N_2}\}_{N_2\in\mathbb{N}}$ is relatively compact in $D_{E_3}([0,T])$.

Denote $\pi^{N_1,N_2} \in \CM(D_{E_3}([0,T])$ the probability measure corresponding to $(\mathfrak{L}^{N_1,N_2},\Psi^{N_1,N_2})$. Relative compactness implies that there is a subsequence $\pi^{N_1,N_{2_k}}$ that converges weakly. We now show that any limit point $\pi^{N_1}$ of a convergence subsequence $\pi^{N_1,N_{2_k}}$ is a Dirac measure concentrated on $(\mathfrak{L}^{N_1},\Psi^{N_1})\in D_{E_3}([0,T])$.

\begin{custlist}[Case]
\item When $\gamma_2 \in \paren{ \frac{3}{4}, \frac{5}{6}}$ and $\zeta \le \gamma_2 - \frac{1}{2} < 2\varphi$, or when $\gamma_2 \in \left[\frac{5}{6}, 1\right)$ and $\zeta < 2\varphi \le \gamma_2 - \frac{1}{2}$, for any $t \in [0,T]$, $b^{i,j}_1,\ldots, b^{i,j}_p \in C_b(\R)$, $d_1, \ldots, d_p \in C_b(\R^M)$, and $0 \le s_1 < \cdots < s_p \le t$, we define $F_6(\mathfrak{L},\Psi): D_{E_3}([0,T]) \to \R_+$ as
\begin{equation*}
\begin{aligned}
&F_6(\mathfrak{L},\Psi) = F_4(\mathfrak{K}^{N_1},L^{N_1}(B^1_{x,x'}(\theta))) + \sum_{i=2}^{3} \sum_{j=1}^{N_1} \abs{\paren{L^{N_1,j}_{i,t} - 0} \times b^{i,j}_1(L^{N_1,j}_{i,s_1})\times \cdots \times b^{i,j}_p(L^{N_1,j}_{i,s_p})}\\
&\quad + \sum_{x\in\mathcal{X}} \left\vert \left(\Psi^{N_1}_t(x)-\Psi^{N_1}_0(x) + \int^t_0 \int_{\CX \times \CY}\Psi^{N_1}_s(x')\ip{B^1_{x,x'}(\theta)+\frac{1}{N_1}\sum_{j=1}^{N_1}B^{2,j}_{x,x'}(\theta), \gamma^{N_1}_0}\pi(dx',dy) ds\right.\right.\\
& \qquad \qquad  + \frac{1}{N_1}\sum_{j=1}^{N_1}\int_0^t \int_{\CX\times \CY} \Psi^{N_1}_s(x')xx'\ip{B^{3,j}_{x}(\theta), \gamma^{N_1}_0}\ip{B^{3,j}_{x'}(\theta),\gamma^{N_1}_0}\pi(dx',dy)ds \\
&\qquad \qquad -\int^t_0 \int_{\CX \times \CY} \paren{y-h^{N_1}_s(x')}\left[ L^{N_1}_{s}(B^1_{x,x'}(\theta)) + \frac{1}{N_1}\sum_{j=1}^{N_1}   L^{N_1}_{s}((B^{2,j}_{x,x'}(\theta))\right] \pi(dx',dy) ds\\
&\qquad \qquad - \frac{1}{N_1}\sum_{j=1}^{N_1}\int_0^t \int_{\CX\times \CY} \left(y-h^{N_1}_s(x')\right)L^{N_1}_{s}(B^{3,j}_{x}(\theta))\ip{xx'B^{3,j}_{x'}(\theta),\gamma^{N_1}_0}\pi(dx',dy)ds\\
&\qquad \qquad\left.\left. - \frac{1}{N_1}\sum_{j=1}^{N_1}\int_0^t \int_{\CX\times \CY} \left(y-h^{N_1}_s(x')\right)\ip{xx'B^{3,j}_{x}(\theta), \gamma^{N_1}_0}L^{N_1}_{s}(B^{3,j}_{x'}(\theta))\pi(dx',dy)ds \right)\times \cdots \times d_p(\Psi_{s_p}) \right \vert,
\end{aligned}
\end{equation*}
where  $F_4(\mathfrak{K},L(B_{x,x'}(c,w)))$ is as given in equation \eqref{F4} and $L^{N_1,j}_{i,t}$ is the $j$-th element of the $N_1$-dimensional vector $L^{N_1}_{i,t}$ for $i=2,3$. Note that by equation \eqref{Psi^N_t}, we have
\begin{equation*}
\begin{aligned}
&\Psi^{N_1,N_2}_t(x) - \Psi^{N_1,N_2}_0(x) + \int^t_0 \int_{\CX \times \CY}\Psi^{N_1,N_2}_s(x')\ip{B^1_{x,x'}(\theta)+\frac{1}{N_1}\sum_{j=1}^{N_1}B^{2,j}_{x,x'}(\theta), \gamma^{N_1,N_2}_0}\pi(dx',dy) ds\\
&\quad +\frac{1}{N_1}\sum_{j=1}^{N_1}\int_0^t \int_{\CX\times \CY} \Psi^{N_1,N_2}_s(x')xx'\ip{B^{3,j}_{x}(\theta), \gamma^{N_1,N_2}_0}\ip{B^{3,j}_{x'}(\theta),\gamma^{N_1,N_2}_0}\pi(dx',dy)ds\\
&\quad - \int^t_0 \int_{\CX \times \CY} \paren{y-h^{N_1,N_2}_s(x')}\left[ L^{N_1,N_2}_{s}(B^1_{x,x'}(\theta)) + \frac{1}{N_1}\sum_{j=1}^{N_1} L^{N_1,N_2}_{s}((B^{2,j}_{x,x'}(\theta))\right] \pi(dx',dy) ds\\
&\quad - \frac{1}{N_1}\sum_{j=1}^{N_1}\int_0^t \int_{\CX\times \CY} \left(y-h^{N_1,N_2}_s(x')\right)L^{N_1,N_2}_{s}(B^{3,j}_{x}(\theta))\ip{xx'B^{3,j}_{x'}(\theta),\gamma^{N_1,N_2}_0}\pi(dx',dy)ds\\
&\quad - \frac{1}{N_1}\sum_{j=1}^{N_1}\int_0^t \int_{\CX\times \CY} \left(y-h^{N_1,N_2}_s(x')\right)\ip{xx'B^{3,j}_{x}(\theta), \gamma^{N_1,N_2}_0}L^{N_1,N_2}_{s}(B^{3,j}_{x'}(\theta))\pi(dx',dy)ds\\
&= (1)_{\Psi} + (2)_{\Psi} + (3)_{\Psi}+(4)_{\Psi}+(5)_{\Psi} + N_2^{\zeta-\varphi} \Gamma^{N_1,N_2}_{t}(x) + N_2^{\zeta}M_t^{N_1,N_2}(x) + O(N_2^{-\gamma_2+\zeta}),\\
\end{aligned}
\end{equation*}
where
\begin{equation}\label{begin_psi}
\begin{aligned}
(1)_{\Psi}= \frac{1}{N_2^{\varphi}} \int^t_0 \int_{\CX \times \CY}\Psi^{N_1,N_2}_s(x')\ip{B^1_{x,x'}(\theta)+\frac{1}{N_1}\sum_{j=1}^{N_1}B^{2,j}_{x,x'}(\theta), \eta^{N_1,N_2}_0}\pi(dx',dy) ds,
\end{aligned}
\end{equation}

\begin{equation*}
\begin{aligned}
(2)_{\Psi}&= \frac{1}{N_1N_2^{\varphi}}\sum_{j=1}^{N_1}\int_0^t \int_{\CX\times \CY} \Psi^{N_1,N_2}_s(x')xx'\ip{B^{3,j}_{x}(\theta), \eta^{N_1,N_2}_0}\ip{B^{3,j}_{x'}(\theta),\gamma^{N_1}_0}\pi(dx',dy)ds\\
&\quad + \frac{1}{N_1N_2^{\varphi}}\sum_{j=1}^{N_1}\int_0^t \int_{\CX\times \CY} \Psi^{N_1,N_2}_s(x')xx'\ip{B^{3,j}_{x}(\theta), \gamma^{N_1,N_2}_0}\ip{B^{3,j}_{x'}(\theta),\eta^{N_1,N_2}_0}\pi(dx',dy)ds,
\end{aligned}
\end{equation*}

\begin{equation*}
\begin{aligned}
(3)_{\Psi} = \frac{1}{N_2^{\varphi}} \int^t_0 \int_{\CX \times \CY} K^{N_1,N_2}_s(x')\left[ L^{N_1,N_2}_{s}(B^1_{x,x'}(\theta)) + \frac{1}{N_1}\sum_{j=1}^{N_1} L^{N_1,N_2}_{s}((B^{2,j}_{x,x'}(\theta))\right] \pi(dx',dy) ds,
\end{aligned}
\end{equation*}

\begin{equation*}
\begin{aligned}
(4)_{\Psi}&=\frac{1}{N_1N_2^{\varphi}}\sum_{j=1}^{N_1}\int_0^t \int_{\CX\times \CY} K^{N_1,N_2}_s(x')L^{N_1,N_2}_{s}(B^{3,j}_{x}(\theta))\ip{xx'B^{3,j}_{x'}(\theta),\gamma^{N_1}_0}\pi(dx',dy)ds\\
&\quad - \frac{1}{N_1N_2^{\varphi}}\sum_{j=1}^{N_1}\int_0^t \int_{\CX\times \CY} \left(y-h^{N_1,N_2}_s(x')\right)L^{N_1,N_2}_{s}(B^{3,j}_{x}(\theta))\ip{xx'B^{3,j}_{x'}(\theta),\eta^{N_1,N_2}_0}\pi(dx',dy)ds,
\end{aligned}
\end{equation*}

\begin{equation}\label{end_psi}
\begin{aligned}
(5)_{\Psi}&=\frac{1}{N_1N_2^{\varphi}}\sum_{j=1}^{N_1}\int_0^t \int_{\CX\times \CY} K^{N_1,N_2}_s(x')\ip{xx'B^{3,j}_{x}(\theta), \gamma^{N_1}_0}L^{N_1,N_2}_{s}(B^{3,j}_{x'}(\theta))\pi(dx',dy)ds\\
&\quad - \frac{1}{N_1N_2^{\varphi}}\sum_{j=1}^{N_1}\int_0^t \int_{\CX\times \CY} \left(y-h^{N_1,N_2}_s(x')\right)\ip{xx'B^{3,j}_{x}(\theta), \eta^{N_1,N_2}_0}L^{N_1,N_2}_{s}(B^{3,j}_{x'}(\theta))\pi(dx',dy)ds.
\end{aligned}
\end{equation}
We now analyze each of these five terms. By the Cauchy-Schwartz inequality, Lemmas \ref{CLT:lemma:eta_compact_contatinment}, \ref{CLT:lemma:bound of ex_Kt}, \ref{lemma:L^N_compact containment} and \ref{Psi_compact_containment}, we have
\begin{equation}\label{bound_(1)_Psi}
\E\left[\abs{(1)_{\Psi} + (3)_{\Psi}}\right] \le \frac{C}{N_2^{\varphi}}.
\end{equation}
For term $(2)_{\Psi}$, since
\begin{equation*}
\begin{aligned}
&\frac{1}{N_1N_2^{\varphi}}\sum_{j=1}^{N_1}\E\left[\abs{\int_0^t \int_{\CX\times \CY} \Psi^{N_1,N_2}_s(x')xx'\ip{B^{3,j}_{x}(\theta), \eta^{N_1,N_2}_0}\ip{B^{3,j}_{x'}(\theta),\gamma^{N_1}_0}\pi(dx',dy)ds} \right]\\
&\le \frac{1}{N_1N_2^{\varphi}}\sum_{j=1}^{N_1} \E\left[\sup_{x'} \abs{\ip{xx'B^{3,j}_{x'}(\theta),\gamma^{N_1}_0}} \int_0^t \int_{\CX\times \CY} \abs{\Psi^{N_1,N_2}_s(x')\ip{B^{3,j}_{x}(\theta), \eta^{N_1,N_2}_0}}\pi(dx',dy)ds \right]\\
&\le \frac{C}{N_1N_2^{\varphi}}\sum_{j=1}^{N_1}\int_0^t \int_{\CX\times \CY} \E\left[\abs{\Psi^{N_1,N_2}_s(x')}^2\right]^{\frac{1}{2}} \E\left[\abs{\ip{B^{3,j}_{x}(\theta), \eta^{N_1,N_2}_0}}^2\right]^{\frac{1}{2}} \pi(dx',dy)ds \\
&\le \frac{C}{N_2^{\varphi}}
\end{aligned}
\end{equation*}
and similar bound can be obtained for the second term in $(2)_{\Psi}$, we have $\E\left[\abs{(2)_{\Psi} }\right] \le {C}/{N_2^{\varphi}}$. For term $(4)_{\Psi}$, we see that
\begin{equation*}
\begin{aligned}
&\E\left[\abs{(4)_{\Psi} }\right] \le \frac{C}{N_1N_2^{\varphi}}\sum_{j=1}^{N_1}\int_0^t \int_{\CX\times \CY} \E\left[\abs{K^{N_1,N_2}_s(x')}^2 \right]^{\frac{1}{2}} \E\left[ \abs{L^{N_1,N_2}_{s}(B^{3,j}_{x}(\theta))}^2\right]^{\frac{1}{2}}\pi(dx',dy)ds\\
& + \frac{1}{N_1N_2^{\varphi}}\sum_{j=1}^{N_1}\int_0^t \int_{\CX\times \CY} \E\left[\abs{L^{N_1,N_2}_{s}(B^{3,j}_{x}(\theta))}^2\right]^{\frac{1}{2}}\E\left[\abs{y-h^{N_1,N_2}_s(x')}^4\right]^{\frac{1}{4}} \E\left[\abs{\ip{xx'B^{3,j}_{x'}(\theta),\eta^{N_1,N_2}_0}}^4\right]^{\frac{1}{4}}\pi(dx',dy)ds \\
&\le \frac{C(T)}{N_2^{\varphi}}.
\end{aligned}
\end{equation*}

Similarly, $\E\left[\abs{(5)_{\Psi} }\right]\le {C(T)}/{N_2^{\varphi}}$. By the Cauchy-Schwartz inequality, Lemmas \ref{lemma_parameterbound}, \ref{CLT:lemma:eta_compact_contatinment},  \ref{lemma:L^N_compact containment}, and \ref{Psi_compact_containment}, we have
\begin{equation}\label{bound_N^zeta-varphi_Gamma}
\begin{aligned}
&\E\left[\abs{N_2^{\zeta-\varphi} \Gamma^{N_1,N_2}_t}\right]\\
&\le \frac{C}{N_2^{\varphi}}\int_0^t \int_{\CX \times \CY} \E\left[\abs{\Psi_s^{N_1,N_2}(x')}^2\right]^{\frac{1}{2}}  \E\left[\abs{l^{N_1,N_2}_s(B^1_{x,x'}(\theta))+\frac{1}{N_1}\sum_{j=1}^{N_1}l^{N_1,N_2}_s(B^{2,j}_{x,x'}(\theta))}^2\right]^{\frac{1}{2}}\pi(dx',dy) ds \\
&\quad + \frac{C}{N_2^{2\varphi-\zeta}}\int_0^t \int_{\CX \times \CY}\E\left[\abs{ l^N_s(B_{x,x'}(c,w))+\frac{1}{N_1}\sum_{j=1}^{N_1}l^{N_1,N_2}_s(B^{2,j}_{x,x'}(\theta))}\right]  \pi(dx',dy) ds\\
& \quad +\frac{C}{N_1N_2^{\varphi}} \sum_{j=1}^{N_1} \int_0^t \int_{\CX\times \CY} xx'\E\left[\abs{\Psi^{N_1,N_2}_s(x')}^2\right]^{\frac{1}{2}} \E\left[\abs{\ip{B^{3,j}_{x}(\theta), \eta^{N_1,N_2}_s}}^4\right]^{\frac{1}{4}} \E\left[\abs{\ip{B^{3,j}_{x'}(\theta),\gamma^{N_1,N_2}_s}}^2\right]^{\frac{1}{4}}\pi(dx',dy)ds \\
&\quad  +\frac{C}{N_1N_2^{2\varphi-\zeta}} \sum_{j=1}^{N_1} \int_0^t \int_{\CX\times \CY} xx'\E\left[\abs{\ip{B^{3,j}_{x}(\theta), \eta^{N_1,N_2}_s}}^2\right]^{\frac{1}{2}} \E\left[\abs{\ip{B^{3,j}_{x'}(\theta),\gamma^{N_1,N_2}_s}}^2\right]^{\frac{1}{2}}\pi(dx',dy)ds \\
&\quad + \frac{C}{N_1N_2^{\varphi}} \sum_{j=1}^{N_1} \int_0^t \int_{\CX\times \CY}xx' \E\left[\abs{\Psi^{N_1,N_2}_s(x')}^2\right]^{\frac{1}{2}} \E\left[ \abs{\ip{B^{3,j}_{x'}(\theta),\eta^{N_1,N_2}_s}}^2\right]^{\frac{1}{2}}\pi(dx',dy)ds\\
&\quad +\frac{C}{N_1N_2^{2\varphi-\zeta}} \sum_{j=1}^{N_1} \int_0^t \int_{\CX\times \CY}xx' \E\left[\abs{\ip{B^{3,j}_{x}(\theta), \gamma^{N_1}_0}}^2\right]^{\frac{1}{2}} \E\left[\abs{\ip{B^{3,j}_{x'}(\theta),\eta^{N_1,N_2}_s}}^2\right]^{\frac{1}{2}}\pi(dx',dy)ds\\
&\quad + \frac{C}{N_1N_2^{\varphi}} \sum_{j=1}^{N_1}\left(\int_0^t \int_{\CX\times \CY}xx' \E\left[\abs{\ip{B^{3,j}_{x}(\theta), \eta^{N_1,N_2}_s}}^4\right]^{\frac{1}{2}} \E\left[\abs{\ip{B^{3,j}_{x'}(\theta),\eta^{N_1,N_2}_s}}^4\right]^{\frac{1}{2}}\pi(dx',dy)ds\right)^{\frac{1}{2}}\\
&\le C \paren{\frac{1}{N_2^{\varphi}} + \frac{1}{N_2^{2\varphi-\zeta}}}.
\end{aligned}
\end{equation}

Putting everything together, by equation \eqref{Psi^N_t}, Lemmas Lemmas \ref{CLT:lemma:eta_compact_contatinment}, \ref{CLT:lemma:bound of ex_Kt}, \ref{lemma:L^N_compact containment}, \ref{lemma:limit_L}, \ref{Psi_compact_containment}, and the analysis in Section \ref{sec::convergence_of_K}, we have
\begin{align*}
&\E_{\pi^{N_1,N_2}}  \left[F_6(\mathfrak{L},\Psi)\right] \\
&= \E_{\pi^{N_1,N_2}} \left[F_4(\mathfrak{K}^{N_1},L^{N_1}(B^1_{x,x'}(\theta)))\right] + \sum_{i=2}^{3} \sum_{j=1}^{N_1} \E\left[\abs{\paren{L^{N_1,N_2, j}_{i,t} - 0} \times b^{i,j}_1(L^{N_1,N_2,j}_{i,s_1})\times \cdots \times b^{i,j}_p(L^{N_1,N_2,j}_{i,s_p})}\right]\\
&\quad + \sum_{x\in\mathcal{X}} \E\left[\left\vert \left(\Psi^{N_1,N_2}_t(x)-\Psi^{N_1,N_2}_0(x) + \int^t_0 \int_{\CX \times \CY}\Psi^{N_1,N_2}_s(x')\ip{B^1_{x,x'}(\theta)+\frac{1}{N_1}\sum_{j=1}^{N_1}B^{2,j}_{x,x'}(\theta), \gamma^{N_1,N_2}_0}\pi(dx',dy) ds\right.\right.\right.\\
& \qquad \qquad  + \frac{1}{N_1}\sum_{j=1}^{N_1}\int_0^t \int_{\CX\times \CY} \Psi^{N_1,N_2}_s(x')xx'\ip{B^{3,j}_{x}(\theta), \gamma^{N_1,N_2}_0}\ip{B^{3,j}_{x'}(\theta),\gamma^{N_1,N_2}_0}\pi(dx',dy)ds \\
&\qquad \qquad -\int^t_0 \int_{\CX \times \CY} \paren{y-h^{N_1,N_2}_s(x')}\left[ L^{N_1,N_2}_{s}(B^1_{x,x'}(\theta)) + \frac{1}{N_1}\sum_{j=1}^{N_1}   L^{N_1,N_2}_{s}((B^{2,j}_{x,x'}(\theta))\right] \pi(dx',dy) ds\\
&\qquad \qquad - \frac{1}{N_1}\sum_{j=1}^{N_1}\int_0^t \int_{\CX\times \CY} \left(y-h^{N_1,N_2}_s(x')\right)L^{N_1,N_2}_{s}(B^{3,j}_{x}(\theta))\ip{xx'B^{3,j}_{x'}(\theta),\gamma^{N_1,N_2}_0}\pi(dx',dy)ds\\
&\qquad \qquad\left.\left.\left. - \frac{1}{N_1}\sum_{j=1}^{N_1}\int_0^t \int_{\CX\times \CY} \left(y-h^{N_1,N_2}_s(x')\right)\ip{xx'B^{3,j}_{x}(\theta), \gamma^{N_1,N_2}_0}L^{N_1,N_2}_{s}(B^{3,j}_{x'}(\theta))\pi(dx',dy)ds \right)\times \cdots \times d_p(\Psi^{N_1,N_2}_{s_p}) \right \vert\right]\\
&\le C \paren{\frac{1}{N_2^{1-\gamma_2}} + \frac{1}{N_2^{2\varphi-\zeta}} } + C\paren{\frac{1}{N_1^{\varphi}} + \frac{1}{N_1^{2\varphi-\zeta}}} \\
&\quad + C \E \left[\abs{N_2^{\zeta}M_t^{N_1,N_2}}\right] + O(N_2^{-\gamma_2+\zeta})\\
&\le C \paren{\frac{1}{N_2^{1-\gamma_2}} + \frac{1}{N_2^{2\varphi-\zeta}} }.
\end{align*}
Therefore, $\lim_{N_2\to \infty} \E_{\pi^{N_1,N_2}}[F_6(\mathfrak{L},\Psi)] = 0$. Since $F_6(\cdot)$ is continuous and uniformly bounded,
\[\lim_{N_2\to \infty} \E_{\pi^{N_1,N_2}}\left[F_6(\mathfrak{L},\Psi)\right] = \E_{\pi^{N_1}}\left[F_6(\mathfrak{L},\Psi)\right] = 0.\]
Since relative compactness implies that every subsequence $\pi^{N_1,N_{2_k}}$ has a further sub-subsequence that converges weakly. And we have show that any limit point $\pi^{N_1}$ of a convergence sequence must be a Dirac measure concentrated $(\mathfrak{L}^{N_1},\Psi^{N_1})\in D_{E_3}([0,T])$. In particular, $K^{N_1}_t$ satisfies \eqref{K_t for 1-gamma_2}, $L^{N_1}_{1,t}=0$, $L^{N_1}_{2,t}=L^{N_1}_{3,t}=0$ and $\Psi^{N_1}_t$ satisfies equation \eqref{limit_Psi}. Since the solutions to equations \eqref{h_N1_evolution}, \eqref{K_t for 1-gamma_2} and  \eqref{limit_Psi} are unique, by Prokhorov's theorem, the processes $(\mathfrak{L}^{N_1,N_2}_t,\Psi^{N_1,N_2}_t)$ converges in distribution to $(\mathfrak{L}^{N_1}_t,\Psi^{N_1}_t)$.

\item  When $\gamma_2 \in \left[\frac{5}{6}, 1\right)$ and $\zeta = 2-2\gamma_2 = 2\varphi$, for any $t \in [0,T]$, $b^{i,j}_1,\ldots, b^{i,j}_p \in C_b(\R)$, $d_1, \ldots, d_p \in C_b(\R^M)$, and $0 \le s_1 < \cdots < s_p \le t$, we define $F_7(\mathfrak{L},\Psi): D_{E_3}([0,T]) \to \R_+$ as
\begin{equation*}
\begin{aligned}
F_7(\mathfrak{L},\Psi) &= F_5(\mathfrak{K},L(B^1_{x,x'}(\theta))) + \sum_{i=2}^{3} \sum_{j=1}^{N_1} \abs{F_{L}(L^{N_1,j}_{i,t}) \times b^{i,j}_1(L^{N_1,j}_{i,s_1})\times \cdots \times b^{i,j}_p(L^{N_1,j}_{i,s_p})}\\
&+\sum_{x\in\mathcal{X}} \left\vert \left(\Psi^{N_1}_t(x)-\Psi^{N_1}_0(x) + \int^t_0 \int_{\CX \times \CY}\Psi^{N_1}_s(x')\ip{B^1_{x,x'}(\theta)+\frac{1}{N_1}\sum_{j=1}^{N_1}B^{2,j}_{x,x'}(\theta), \gamma^{N_1}_0}\pi(dx',dy) ds\right.\right.\\
&\qquad +\frac{1}{N_1}\sum_{j=1}^{N_1}\int_0^t \int_{\CX\times \CY} \Psi^{N_1}_s(x')xx'\ip{B^{3,j}_{x}(\theta), \gamma^{N_1}_0}\ip{B^{3,j}_{x'}(\theta),\gamma^{N_1}_0}\pi(dx',dy)ds\\
&\qquad -\int^t_0 \int_{\CX \times \CY} \paren{y-h^{N_1}_s(x')}\left[ L^{N_1}_{s}(B^1_{x,x'}(\theta)) + \frac{1}{N_1}\sum_{j=1}^{N_1}   L^{N_1}_{s}((B^{2,j}_{x,x'}(\theta))\right] \pi(dx',dy) ds\\
&\qquad - \frac{1}{N_1}\sum_{j=1}^{N_1}\int_0^t \int_{\CX\times \CY} \left(y-h^{N_1}_s(x')\right)L^{N_1}_{s}(B^{3,j}_{x}(\theta))\ip{xx'B^{3,j}_{x'}(\theta),\gamma^{N_1}_0}\pi(dx',dy)ds\\
&\qquad - \frac{1}{N_1}\sum_{j=1}^{N_1}\int_0^t \int_{\CX\times \CY} \left(y-h^{N_1}_s(x')\right)\ip{xx'B^{3,j}_{x}(\theta), \gamma^{N_1}_0}L^{N_1}_{s}(B^{3,j}_{x'}(\theta))\pi(dx',dy)ds\\
&\qquad + \int_0^t \int_{\CX\times \CY} K^{N_1}_s(x')\left[l^{N_1}_s\left(B^1_{x,x'}(\theta)\right)+\frac{1}{N_1}\sum_{j=1}^{N_1}l^{N_1}_s \left(B^{2,j}_{x,x'}(\theta)\right)\right]\pi(dx',dy)ds\\
&\qquad +\frac{1}{N_1} \sum_{j=1}^{N_1} \int_0^t \int_{\CX\times \CY} K^{N_1}_s(x')xx'l^{N_1}_s\left(B^{3,j}_{x}(\theta)\right)\ip{B^{3,j}_{x'}(\theta),\gamma^{N_1}_0}\pi(dx',dy)ds \\
&\qquad \left.\left.+\frac{1}{N_1} \sum_{j=1}^{N_1} \int_0^t \int_{\CX\times \CY} K^{N_1}_s(x')xx'\ip{B^{3,j}_{x}(\theta), \gamma^{N_1}_0}l^{N_1}_s\left(B^{3,j}_{x'}(\theta)\right)\pi(dx',dy)ds\right)\times d_1(\Psi^{N_1}_{s_1}) \times \cdots \times d_p(\Psi^{N_1}_{s_p}) \right \vert,
\end{aligned}
\end{equation*}
where $F_5(\mathfrak{K},L(B_{x,x'}(c,w)))$ is as given in equation \eqref{F5}, $L^{N_1,j}_{i,t}$ is the $j$-th element of the $N_1$-dimensional vector $L^{N_1}_{i,t}$ for $i=2,3$, and $F_L(f)$ is equal to $L^{N_1}_t(f)$ minus the right-hand side of \eqref{limit_Lt_0}. Note that by equation \eqref{Psi^N_t},

\begin{equation}\label{Psi_2_temp}
\begin{aligned}
&\Psi^{N_1,N_2}_t(x)-\Psi^{N_1,N_2}_0(x) + \int^t_0 \int_{\CX \times \CY}\Psi^{N_1,N_2}_s(x')\ip{B^1_{x,x'}(\theta)+\frac{1}{N_1}\sum_{j=1}^{N_1}B^{2,j}_{x,x'}(\theta), \gamma^{N_1,N_2}_0}\pi(dx',dy) ds\\
&\qquad +\frac{1}{N_1}\sum_{j=1}^{N_1}\int_0^t \int_{\CX\times \CY} \Psi^{N_1,N_2}_s(x')xx'\ip{B^{3,j}_{x}(\theta), \gamma^{N_1,N_2}_0}\ip{B^{3,j}_{x'}(\theta),\gamma^{N_1,N_2}_0}\pi(dx',dy)ds\\
&\qquad -\int^t_0 \int_{\CX \times \CY} \paren{y-h^{N_1,N_2}_s(x')}\left[ L^{N_1,N_2}_{s}(B^1_{x,x'}(\theta)) + \frac{1}{N_1}\sum_{j=1}^{N_1}   L^{N_1,N_2}_{s}((B^{2,j}_{x,x'}(\theta))\right] \pi(dx',dy) ds\\
&\qquad - \frac{1}{N_1}\sum_{j=1}^{N_1}\int_0^t \int_{\CX\times \CY} \left(y-h^{N_1,N_2}_s(x')\right)L^{N_1,N_2}_{s}(B^{3,j}_{x}(\theta))\ip{xx'B^{3,j}_{x'}(\theta),\gamma^{N_1,N_2}_0}\pi(dx',dy)ds\\
&\qquad - \frac{1}{N_1}\sum_{j=1}^{N_1}\int_0^t \int_{\CX\times \CY} \left(y-h^{N_1,N_2}_s(x')\right)\ip{xx'B^{3,j}_{x}(\theta), \gamma^{N_1,N_2}_0}L^{N_1,N_2}_{s}(B^{3,j}_{x'}(\theta))\pi(dx',dy)ds\\
&\qquad + \int_0^t \int_{\CX\times \CY} K^{N_1,N_2}_s(x')\left[l^{N_1,N_2}_s\left(B^1_{x,x'}(\theta)\right)+\frac{1}{N_1}\sum_{j=1}^{N_1}l^{N_1}_s \left(B^{2,j}_{x,x'}(\theta)\right)\right]\pi(dx',dy)ds\\
&\qquad +\frac{1}{N_1} \sum_{j=1}^{N_1} \int_0^t \int_{\CX\times \CY} K^{N_1,N_2}_s(x')xx'l^{N_1,N_2}_s\left(B^{3,j}_{x}(\theta)\right)\ip{B^{3,j}_{x'}(\theta),\gamma^{N_1,N_2}_0}\pi(dx',dy)ds \\
&\qquad +\frac{1}{N_1} \sum_{j=1}^{N_1} \int_0^t \int_{\CX\times \CY} K^{N_1,N_2}_s(x')xx'\ip{B^{3,j}_{x}(\theta), \gamma^{N_1,N_2}_0}l^{N_1,N_2}_s\left(B^{3,j}_{x'}(\theta)\right)\pi(dx',dy)ds\\
\end{aligned}
\end{equation}
\begin{equation*}
\begin{aligned}
&=\frac{1}{N_1N_2^{\varphi}} \sum_{j=1}^{N_1} \int_0^t \int_{\CX\times \CY} K^{N_1,N_2}_s(x')xx'l^{N_1,N_2}_s(B^{3,j}_{x}(\theta))l^{N_1,N_2}_0\left(B^{3,j}_{x'}(\theta)\right)\pi(dx',dy)ds\\
&\quad + \frac{1}{N_1N_2^{\varphi}} \sum_{j=1}^{N_1} \int_0^t \int_{\CX\times \CY} K^{N_1,N_2}_s(x')xx'l^{N_1,N_2}_0(B^{3,j}_{x}(\theta))l^{N_1,N_2}_s\left(B^{3,j}_{x'}(\theta)\right)\pi(dx',dy)ds\\
&\quad - \frac{1}{N_1N_2^{\varphi}} \sum_{j=1}^{N_1} \int_0^t \int_{\CX\times \CY} K^{N_1,N_2}_s(x')xx'l^{N_1,N_2}_s(B^{3,j}_{x}(\theta))l^{N_1,N_2}_s\left(B^{3,j}_{x'}(\theta)\right)\pi(dx',dy)ds\\
&\quad + \frac{1}{N_1N_2^{\varphi}} \sum_{j=1}^{N_1}\int_0^t \int_{\CX\times \CY} \left(y-h^{N_1}_s(x')\right)xx'\ip{B^{3,j}_{x}(\theta), \eta^{N_1,N_2}_s}\ip{B^{3,j}_{x'}(\theta),\eta^{N_1,N_2}_s}\pi(dx',dy)ds\\
&\quad + (1)_{\Psi} + (2)_{\Psi} + (3)_{\Psi}+(4)_{\Psi}+(5)_{\Psi} + N_2^{\zeta}M_t^{N_1,N_2}(x) + O(N_2^{-\gamma_2+\zeta}),
\end{aligned}
\end{equation*}
where $(1)_{\Psi}$ to $(5)_{\Psi}$ are given in \eqref{begin_psi} to \eqref{end_psi}. By Lemmas \ref{CLT:lemma:eta_compact_contatinment} and \ref{CLT:lemma:bound of ex_Kt},
\begin{equation*}
\begin{aligned}
&\E\left[\abs{\frac{1}{N_1N_2^{\varphi}} \sum_{j=1}^{N_1} \int_0^t \int_{\CX\times \CY} K^{N_1,N_2}_s(x')xx'l^{N_1,N_2}_s(B^{3,j}_{x}(\theta))l^{N_1,N_2}_0\left(B^{3,j}_{x'}(\theta)\right)\pi(dx',dy)ds}\right]\\
&\le\frac{C}{N_1N_2^{\varphi}} \sum_{j=1}^{N_1} \int_0^t \int_{\CX\times \CY} \E\left[\abs{K^{N_1,N_2}_s(x')}^2\right]^{\frac{1}{2}} \E\left[\abs{l^{N_1,N_2}_s(B^{3,j}_{x}(\theta))}^4\right]^{\frac{1}{4}} \E\left[\abs{l^{N_1,N_2}_0\left(B^{3,j}_{x'}(\theta)\right)}^4\right]^{\frac{1}{4}}\pi(dx',dy)ds.\\
&\le \frac{C}{N_2^{\varphi}}
\end{aligned}
\end{equation*}
Similarly, the expectation of the absolute value of the first three terms on the right-hand side of \eqref{Psi_2_temp} are bounded by $O(N_2^{-\varphi})$. The analysis for the forth term and $(1)_{\Psi}$ to $(5)_{\Psi}$ are given in \eqref{bound_(1)_Psi} to \eqref{bound_N^zeta-varphi_Gamma}.

Therefore, we have
\begin{equation*}
\begin{aligned}
&\E_{\pi^{N_1,N_2}}\left[F_7(\mathfrak{L},\Psi)\right]\\
 &= \E_{\pi^{N_1,N_2}}\left[F_5(\mathfrak{K},L(B^1_{x,x'}(\theta)))\right] + \sum_{i=2}^{3} \sum_{j=1}^{N_1} \E\left[\abs{F_{L}(L^{N_1,N_2,j}_{i,t}) \times b^{i,j}_1(L^{N_1,N_2,j}_{i,s_1})\times \cdots \times b^{i,j}_p(L^{N_1,N_2,j}_{i,s_p})}\right]\\
&\quad +\sum_{x\in\mathcal{X}} \E\left[\left\vert \left(\Psi^{N_1,N_2}_t(x)-\Psi^{N_1,N_2}_0(x) + \int^t_0 \int_{\CX \times \CY}\Psi^{N_1,N_2}_s(x')\ip{B^1_{x,x'}(\theta)+\frac{1}{N_1}\sum_{j=1}^{N_1}B^{2,j}_{x,x'}(\theta), \gamma^{N_1,N_2}_0}\pi(dx',dy) ds\right.\right.\right.\\
&\qquad +\frac{1}{N_1}\sum_{j=1}^{N_1}\int_0^t \int_{\CX\times \CY} \Psi^{N_1,N_2}_s(x')xx'\ip{B^{3,j}_{x}(\theta), \gamma^{N_1,N_2}_0}\ip{B^{3,j}_{x'}(\theta),\gamma^{N_1,N_2}_0}\pi(dx',dy)ds\\
&\qquad -\int^t_0 \int_{\CX \times \CY} \paren{y-h^{N_1,N_2}_s(x')}\left[ L^{N_1,N_2}_{s}(B^1_{x,x'}(\theta)) + \frac{1}{N_1}\sum_{j=1}^{N_1}   L^{N_1,N_2}_{s}((B^{2,j}_{x,x'}(\theta))\right] \pi(dx',dy) ds\\
&\qquad - \frac{1}{N_1}\sum_{j=1}^{N_1}\int_0^t \int_{\CX\times \CY} \left(y-h^{N_1,N_2}_s(x')\right)L^{N_1,N_2}_{s}(B^{3,j}_{x}(\theta))\ip{xx'B^{3,j}_{x'}(\theta),\gamma^{N_1,N_2}_0}\pi(dx',dy)ds\\
&\qquad - \frac{1}{N_1}\sum_{j=1}^{N_1}\int_0^t \int_{\CX\times \CY} \left(y-h^{N_1,N_2}_s(x')\right)\ip{xx'B^{3,j}_{x}(\theta), \gamma^{N_1,N_2}_0}L^{N_1,N_2}_{s}(B^{3,j}_{x'}(\theta))\pi(dx',dy)ds\\
&\qquad + \int_0^t \int_{\CX\times \CY} K^{N_1,N_2}_s(x')\left[l^{N_1,N_2}_s\left(B^1_{x,x'}(\theta)\right)+\frac{1}{N_1}\sum_{j=1}^{N_1}l^{N_1,N_2}_s \left(B^{2,j}_{x,x'}(\theta)\right)\right]\pi(dx',dy)ds\\
&\qquad +\frac{1}{N_1} \sum_{j=1}^{N_1} \int_0^t \int_{\CX\times \CY} K^{N_1,N_2}_s(x')xx'l^{N_1,N_2}_s\left(B^{3,j}_{x}(\theta)\right)\ip{B^{3,j}_{x'}(\theta),\gamma^{N_1,N_2}_0}\pi(dx',dy)ds \\
&\qquad \left. +\frac{1}{N_1} \sum_{j=1}^{N_1} \int_0^t \int_{\CX\times \CY} K^{N_1,N_2}_s(x')xx'\ip{B^{3,j}_{x}(\theta), \gamma^{N_1,N_2}_0}l^{N_1,N_2}_s\left(B^{3,j}_{x'}(\theta)\right)\pi(dx',dy)ds\right)\\
&\qquad \left.\left.  \times d_1(\Psi^{N_1,N_2}_{s_1}) \times \cdots \times d_p(\Psi^{N_1,N_2}_{s_p}) \right \vert\right]\\
&\le C \paren{\frac{1}{N_2^{1-\gamma_2}}+ \frac{1}{N_2^{\frac{1}{2}-\zeta}}} + \frac{C}{N_2^{\varphi}} + C \E \left[N_2^{\zeta}M_t^{N_1,N_2}(x) \right] + \frac{C}{N_2^{\gamma_2-\zeta}}\\
&\le C \paren{\frac{1}{N_2^{1-\gamma_2}} + \frac{1}{N_2^{\gamma_2-\zeta}} }.
\end{aligned}
\end{equation*}
Hence, $\lim_{N_2\to \infty} \E_{\pi^{N_1,N_2}}[F_7(\mathfrak{L},\Psi)] = 0$. Since $F_7(\cdot)$ is continuous and uniformly bounded,
\[\lim_{N_2\to \infty} \E_{\pi^{N_1,N_2}}\left[F_7(\mathfrak{L},\Psi)\right] = \E_{\pi^{N_1}}\left[F_7(\mathfrak{L},\Psi)\right] = 0.\]
The result then follows by Prokhorov's theorem.
\end{custlist}
\end{proof}

\section{Derivation of the asymptotic expansion of $h^{N_1,N_2}_t$ for  $\gamma_2\in(1/2,1)$}\label{sec::higher order}
%


The goal of this section is to provide an inductive argument to derive the asymptotic expansion for $\ip{f,\gamma^{N_1,N_2}_{t}}$ and $h^{N_1,N_2}_{t}$ as $N_2\rightarrow\infty$ as claimed in (\ref{measure_expansion}) and (\ref{network_expansion}) respectively.

%
%

Let $\nu\in \mathbb{N}$ and let  $\mathcal{G}^{N_1}(x)$ be the Gaussian random variable defined in Section \ref{S:MainResults}. Then, when $\gamma_2 \in \left[\frac{2\nu-1}{2\nu}, \frac{2\nu+1}{2\nu+2}\right)$,  we obtain that for any fixed $f\in C^{\infty}_b(\R^{1+N_1(1+d)})$, as $N_2\rightarrow\infty$, we have the expansion given by (\ref{measure_expansion}) where for $n \ge 3$,
\begin{equation}\label{Eq:l_equation}
\begin{aligned}
l^{N_1}_{n,t}(f) &=  \int_0^t \int_{\CX\times \CY}  \paren{y-Q^{N_1}_{0,s}(x')}\left[l^{N_1}_{n-1,s}(C_{x'}^{f,1}(\theta)) + \frac{1}{N_1^{1-\gamma_1}}l^{N_1}_{n-1,s}(C_{x'}^{f,2}(\theta))\right] \pi(dx',dy)ds\\
&\quad + \frac{1}{N_1^{1-\gamma_1}}\int_0^t\int_{\CX\times \CY}\paren{y - Q^{N_1}_{0,s}(x')}\left[\sum_{k=0}^{n-1} l^{N_1}_{k,s}(C^{3}_{x'}(\theta)) \cdot l^{N_1}_{n-1-k,s}(\nabla_{w^1}f(\theta)x')\right]\pi(dx',dy)ds\\
&\quad -\sum_{m=1}^{n-1}\int_0^t\int_{\CX\times \CY}Q^{N_1}_{n-m,s}(x') \left[ l^{N_1}_{m-1,s}(C^{f,1}_{x'}(\theta)) + \frac{1}{N_1^{1-\gamma_2}}l^{N_1}_{m-1,s}(C^{f,2}_{x'}(\theta))\right]\pi(dx',dy)ds\\
&\quad  - \sum_{m=1}^{n-1}\frac{1}{N_1^{1-\gamma_1}}\int_0^t\int_{\CX\times \CY} Q^{N_1}_{n-m,s}(x')\left[\sum_{k=0}^{m-1} l^{N_1}_{k,s}(C^{3}_{x'}(\theta)) \cdot l^{N_1}_{m-1-k,s}(\nabla_{w^1}f(\theta)x')\right]\pi(dx',dy)ds,
\end{aligned}
\end{equation}
where
\begin{equation}\label{C^f_def}
\begin{aligned}
C^{f,1}_{x}(\theta) &= \partial_{c}f(\theta)  \sigma(Z^{2,N_1}(x)),\\
C^{f,2}_{x}(\theta) &=c\sigma'(Z^{2,N_1}(x))\sigma(w^1x)\cdot \partial_{w^2}f(\theta),\\
C^{3}_{x}(\theta) &= c\sigma'(Z^{2,N_1}(x))\sigma'(w^1x)w^{2}.
\end{aligned}
\end{equation}


As $N_2 \rightarrow\infty$ and when $\gamma_2 \in \left(\frac{2\nu-1}{2\nu}, \frac{2\nu+1}{2\nu+2}\right]$, we have the asymptotic expansion (\ref{network_expansion}) for $h^{N_1,N_2}_t(x)$. The terms on the right hand side of the asymptotic expansion (\ref{network_expansion}) satisfy  the deterministic evolution equations (\ref{Eq:Qj_formula1}), (\ref{Eq:Qk_formula1}) and (\ref{Eq:Qj_formula2}).
\begin{equation}\label{Eq:Qj_formula1}
\begin{aligned}
Q^{N_1}_{n,t}(x)
&= \int_0^t  \int_{\CX \times \CY} \paren{y-Q^{N_1}_{0,s}(x')}  l^{N_1}_{n,s}\left(B^1_{x,x'}(\theta)+ \frac{1}{N_1}\sum_{j=1}^{N_1}B^{2,j}_{x,x'}(\theta)\right)\pi(dx',dy) ds\\
&\quad +\frac{1}{N_1} \sum_{j=1}^{N_1} \int_0^t \int_{\CX\times \CY} \paren{y- Q^{N_1}_{0,s}(x')}xx' \left[\sum_{k=0}^{n} l^{N_1}_{k,s}\left(B^{3,j}_{x}(\theta)\right) l^{N_1}_{n-k,s}\left(B^{3,j}_{x'}(\theta)\right)\right]\pi(dx',dy)ds\\
&\quad  - \sum_{m=0}^{n-1} \int_0^t  \int_{\CX \times \CY}Q^{N_1}_{n-m,s}(x')  l^{N_1}_{m,s}\left(B^1_{x,x'}(\theta)+ \frac{1}{N_1}\sum_{j=1}^{N_1}B^{2,j}_{x,x'}(\theta)\right)\pi(dx',dy) ds\\
&\quad -\frac{1}{N_1} \sum_{j=1}^{N_1} \sum_{m=0}^{n-1} \int_0^t \int_{\CX\times \CY} Q^{N_1}_{n-m,s}(x') xx'\left[\sum_{k=0}^{m}l^{N_1}_{k,s}\left(B^{3,j}_{x}(\theta)\right) l^{N_1}_{m-k,s}\left(B^{3,j}_{x'}(\theta)\right)\right]\pi(dx',dy)ds,
\end{aligned}
\end{equation}


When $\gamma_2 \in \paren{\frac{2\nu-1}{2\nu}, \frac{2\nu+1}{2\nu+2}}$,
\begin{equation}\label{Eq:Qk_formula1}
\begin{aligned}
Q^{N_1}_{\nu,t}(x) &= \mathcal{G}(x)-\int_0^t  \int_{\CX \times \CY} Q^{N_1}_{\nu,s}(x') l^{N_1}_{0,s}\left(B^1_{x,x'}(\theta)+ \frac{1}{N_1}\sum_{j=1}^{N_1}B^{2,j}_{x,x'}(\theta)\right) \pi(dx',dy) ds\\
&\quad  -\frac{1}{N_1} \sum_{j=1}^{N_1} \int_0^t \int_{\CX\times \CY} Q^{N_1}_{\nu,s}(x') xx' l^{N_1}_{0,s}\left(B^{3,j}_{x}(\theta)\right)  l^{N_1}_{0,s}\left(B^{3,j}_{x'}(\theta)\right) \pi(dx',dy)ds.
\end{aligned}
\end{equation}
and when $\gamma_2 = \frac{2\nu+1}{2\nu+2}$,
\begin{equation}\label{Eq:Qj_formula2}
\begin{aligned}
Q^{N_1}_{\nu,t}(x) &= \mathcal{G}^{N_1}(x) + \int_0^t  \int_{\CX \times \CY} \paren{y-Q^{N_1}_{0,s}(x')}  l^{N_1}_{\nu,s}\left(B^1_{x,x'}(\theta)+ \frac{1}{N_1}\sum_{j=1}^{N_1}B^{2,j}_{x,x'}(\theta)\right)\pi(dx',dy) ds\\
&\quad +\frac{1}{N_1} \sum_{j=1}^{N_1} \int_0^t \int_{\CX\times \CY} \paren{y- Q^{N_1}_{0,s}(x')}xx' \left[\sum_{k=0}^{n} l^{N_1}_{k,s}\left(B^{3,j}_{x}(\theta)\right) l^{N_1}_{\nu-k,s}\left(B^{3,j}_{x'}(\theta)\right)\right]\pi(dx',dy)ds\\
&\quad  - \sum_{m=0}^{\nu-1} \int_0^t  \int_{\CX \times \CY}Q^{N_1}_{\nu-m,s}(x')  l^{N_1}_{m,s}\left(B^1_{x,x'}(\theta)+ \frac{1}{N_1}\sum_{j=1}^{N_1}B^{2,j}_{x,x'}(\theta)\right)\pi(dx',dy) ds\\
&\quad -\frac{1}{N_1} \sum_{j=1}^{N_1} \sum_{m=0}^{\nu-1} \int_0^t \int_{\CX\times \CY} Q^{N_1}_{\nu-m,s}(x') xx'\left[\sum_{k=0}^{m}l^{N_1}_{k,s}\left(B^{3,j}_{x}(\theta)\right) l^{N_1}_{m-k,s}\left(B^{3,j}_{x'}(\theta)\right)\right]\pi(dx',dy)ds.
\end{aligned}
\end{equation}


It is interesting to note that the approach that is presented in this section also recovers the rigorously derived formulas for $\nu=1$ and $\nu=2$ as presented in the main theoretical results of Section \ref{S:MainResults}. Below we focus on presenting the argument for the case $\nu> 2$.

\subsection{General $\nu>2$ case}
To find an expression for $Q^{N_1}_{\nu,t}$ for any $\nu >2$, we use an inductive argument. Assuming $Q^{N_1}_{0,t}(x) = h^{N_1}_t(x)$ and $l^{N_1}_{0,t}(f) = \ip{f,\gamma^{N_1}_0}$, we have already rigorously shown that the statement holds for $\nu=1$ and $\nu=2$. For $n=3,\ldots,\nu-1$, we will assume that $Q^{N_1}_{n,t}$ and $l^{N_1}_{n,t}(f)$ satisfy the following deterministic evolution equations,
\begin{equation*}
\begin{aligned}
Q^{N_1}_{n,t}(x)
&= \int_0^t  \int_{\CX \times \CY} \paren{y-Q^{N_1}_{0,s}(x')}  l^{N_1}_{n,s}\left(B^1_{x,x'}(\theta)+ \frac{1}{N_1}\sum_{j=1}^{N_1}B^{2,j}_{x,x'}(\theta)\right)\pi(dx',dy) ds\\
&\quad +\frac{1}{N_1} \sum_{j=1}^{N_1} \int_0^t \int_{\CX\times \CY} \paren{y- Q^{N_1}_{0,s}(x')}xx' \left[\sum_{k=0}^{n} l^{N_1}_{k,s}\left(B^{3,j}_{x}(\theta)\right) l^{N_1}_{n-k,s}\left(B^{3,j}_{x'}(\theta)\right)\right]\pi(dx',dy)ds\\
&\quad  - \sum_{m=0}^{n-1} \int_0^t  \int_{\CX \times \CY}Q^{N_1}_{n-m,s}(x')  l^{N_1}_{m,s}\left(B^1_{x,x'}(\theta)+ \frac{1}{N_1}\sum_{j=1}^{N_1}B^{2,j}_{x,x'}(\theta)\right)\pi(dx',dy) ds\\
&\quad -\frac{1}{N_1} \sum_{j=1}^{N_1} \sum_{m=0}^{n-1} \int_0^t \int_{\CX\times \CY} Q^{N_1}_{n-m,s}(x') xx'\left[\sum_{k=0}^{m}l^{N_1}_{k,s}\left(B^{3,j}_{x}(\theta)\right) l^{N_1}_{m-k,s}\left(B^{3,j}_{x'}(\theta)\right)\right]\pi(dx',dy)ds,
\end{aligned}
\end{equation*}
and
\begin{equation*}
\begin{aligned}
l^{N_1}_{n,t}(f) &=  \int_0^t \int_{\CX\times \CY}  \paren{y-Q^{N_1}_{0,s}(x')}\left[l^{N_1}_{n-1,s}(C_{x'}^{f,1}(\theta)) + \frac{1}{N_1^{1-\gamma_1}}l^{N_1}_{n-1,s}(C_{x'}^{f,2}(\theta))\right] \pi(dx',dy)ds\\
&\quad + \frac{1}{N_1^{1-\gamma_1}}\int_0^t\int_{\CX\times \CY}\paren{y - Q^{N_1}_{0,s}(x')}\left[\sum_{k=0}^{n-1} l^{N_1}_{k,s}(C^{3}_{x'}(\theta)) \cdot l^{N_1}_{n-1-k,s}(\nabla_{w^1}f(\theta)x')\right]\pi(dx',dy)ds\\
&\quad -\sum_{m=1}^{n-1}\int_0^t\int_{\CX\times \CY}Q^{N_1}_{n-m,s}(x') \left[ l^{N_1}_{m-1,s}(C^{f,1}_{x'}(\theta)) + \frac{1}{N_1^{1-\gamma_2}}l^{N_1}_{m-1,s}(C^{f,2}_{x'}(\theta))\right]\pi(dx',dy)ds\\
&\quad  - \sum_{m=1}^{n-1}\frac{1}{N_1^{1-\gamma_1}}\int_0^t\int_{\CX\times \CY} Q^{N_1}_{n-m,s}(x')\left[\sum_{k=0}^{m-1} l^{N_1}_{k,s}(C^{3}_{x'}(\theta)) \cdot l^{N_1}_{m-1-k,s}(\nabla_{w^1}f(\theta)x')\right]\pi(dx',dy)ds\\
\end{aligned}
\end{equation*}
We now derive the formulas for $Q^{N_1}_{\nu,t}$ and $l^{N_1}_{\nu,t}(f)$ for any $\nu \in \mathbb{N}$.

\begin{itemize}
\item When $\gamma \in \left(\frac{2\nu-1}{2\nu}, \frac{2\nu+1}{2\nu+2}\right)$,
plugging equations \eqref{network_expansion} and \eqref{measure_expansion} into the left hand side of equation \eqref{h_N1N2_evolution} gives (the symbol $\approx$ is used to ignore the remainder terms in (\ref{measure_expansion}) and (\ref{network_expansion}))
\begin{equation*}
\begin{aligned}
&h^{N_1,N_2}_t(x)-h^{N_1,N_2}_0(x)\\
&\approx \int_0^t  \int_{\CX \times \CY} \paren{y-\sum_{k=1}^{\nu-1} \frac{1}{N_2^{k(1-\gamma_2)}}Q^{N_1}_{k,s}(x') - \frac{1}{N_2^{\gamma_2-\frac{1}{2}}}Q^{N_1}_{\nu,s}(x')}  \\
&\qquad \qquad \times \sum_{k=0}^{\nu-1} \frac{1}{N_2^{k(1-\gamma_2)}} l^{N_1}_{k,s}\left(B^1_{x,x'}(\theta)+ \frac{1}{N_1}\sum_{j=1}^{N_1}B^{2,j}_{x,x'}(\theta)\right) \pi(dx',dy) ds\\
&\quad +\frac{1}{N_1} \sum_{j=1}^{N_1} \int_0^t \int_{\CX\times \CY} \paren{y-\sum_{k=0}^{\nu-1} \frac{1}{N_2^{k(1-\gamma_2)}}Q^{N_1}_{k,s}(x') - \frac{1}{N_2^{\gamma_2-\frac{1}{2}}}Q^{N_1}_{\nu,s}(x')}xx'\\
&\qquad \qquad \times \left[\sum_{k=0}^{\nu-1} \frac{1}{N_2^{k(1-\gamma_2)}} l^{N_1}_{k,s}\left(B^{3,j}_{x}(\theta)\right) \right]\left[\sum_{k=0}^{\nu-1} \frac{1}{N_2^{k(1-\gamma_2)}} l^{N_1}_{k,s}\left(B^{3,j}_{x'}(\theta)\right) \right]\pi(dx',dy)ds\\
&=\left\lbrace \int_0^t  \int_{\CX \times \CY} \paren{y- Q^{N_1}_{0,s}(x')}  l^{N_1}_{0,s}\left(B^1_{x,x'}(\theta)+ \frac{1}{N_1}\sum_{j=1}^{N_1}B^{2,j}_{x,x'}(\theta)\right) \pi(dx',dy) ds\right.\\
&\qquad\left. +\frac{1}{N_1} \sum_{j=1}^{N_1} \int_0^t \int_{\CX\times \CY} \paren{y-Q^{N_1}_{0,s}(x')}xx' l^{N_1}_{0,s}\left(B^{3,j}_{x}(\theta)\right) l^{N_1}_{0,s}\left(B^{3,j}_{x'}(\theta)\right)\pi(dx',dy)ds \right\rbrace\\
&\quad +\sum_{n=1}^{\nu-1} \frac{1}{N_2^{n(1-\gamma_2)}} \left\lbrace \int_0^t  \int_{\CX \times \CY} \paren{y- Q^{N_1}_{0,s}(x')} l^{N_1}_{n,s}\left(B^1_{x,x'}(\theta)+ \frac{1}{N_1}\sum_{j=1}^{N_1}B^{2,j}_{x,x'}(\theta)\right) \pi(dx',dy) ds \right.\\
&\qquad - \sum_{m=0}^{n-1}\int_0^t  \int_{\CX \times \CY} Q^{N_1}_{n-m,s}(x') l^{N_1}_{m,s}\left(B^1_{x,x'}(\theta)+ \frac{1}{N_1}\sum_{j=1}^{N_1}B^{2,j}_{x,x'}(\theta)\right) \pi(dx',dy) ds\\
&\qquad +\frac{1}{N_1} \sum_{j=1}^{N_1} \int_0^t \int_{\CX\times \CY} \paren{y-Q^{N_1}_{0,s}(x') }xx' \left[\sum_{k=0}^{n} l^{N_1}_{k,s}\left(B^{3,j}_{x}(\theta)\right)  l^{N_1}_{n-k,s}\left(B^{3,j}_{x'}(\theta)\right)\right] \pi(dx',dy)ds\\
&\qquad\left. -  \frac{1}{N_1} \sum_{j=1}^{N_1} \sum_{m=0}^{n-1}\int_0^t \int_{\CX\times \CY} Q^{N_1}_{n-m,s}(x') xx' \left[\sum_{k=0}^{m} l^{N_1}_{k,s}\left(B^{3,j}_{x}(\theta)\right)  l^{N_1}_{m-k,s}\left(B^{3,j}_{x'}(\theta)\right)\right] \pi(dx',dy)ds\right\rbrace\\
&\quad - \frac{1}{N_2^{\gamma_2-\frac{1}{2}}}\left\lbrace\int_0^t  \int_{\CX \times \CY} Q^{N_1}_{\nu,s}(x') l^{N_1}_{0,s}\left(B^1_{x,x'}(\theta)+ \frac{1}{N_1}\sum_{j=1}^{N_1}B^{2,j}_{x,x'}(\theta)\right) \pi(dx',dy) ds\right.\\
&\qquad \left. +\frac{1}{N_1} \sum_{j=1}^{N_1} \int_0^t \int_{\CX\times \CY} Q^{N_1}_{\nu,s}(x') xx' l^{N_1}_{0,s}\left(B^{3,j}_{x}(\theta)\right)  l^{N_1}_{0,s}\left(B^{3,j}_{x'}(\theta)\right) \pi(dx',dy)ds\right\rbrace + O(N_2^{-\Omega_2})\\
&= \sum_{n=0}^{\nu-1} \frac{1}{N_2^{n(1-\gamma_2)}} Q^{N_1}_{n,t}(x) - \frac{1}{N_2^{\gamma_2-\frac{1}{2}}}\left\lbrace\int_0^t  \int_{\CX \times \CY} Q^{N_1}_{\nu,s}(x') l^{N_1}_{0,s}\left(B^1_{x,x'}(\theta)+ \frac{1}{N_1}\sum_{j=1}^{N_1}B^{2,j}_{x,x'}(\theta)\right) \pi(dx',dy) ds\right.\\
&\qquad \left. +\frac{1}{N_1} \sum_{j=1}^{N_1} \int_0^t \int_{\CX\times \CY} Q^{N_1}_{\nu,s}(x') xx' l^{N_1}_{0,s}\left(B^{3,j}_{x}(\theta)\right)  l^{N_1}_{0,s}\left(B^{3,j}_{x'}(\theta)\right) \pi(dx',dy)ds\right\rbrace + O(N^{-\Omega_{\nu}}),
\end{aligned}
\end{equation*}
for some $\Omega_{\nu}  > \gamma - \frac{1}{2}$.
Adding $h^{N_1,N_2}_0(x)$ and subtracting $\sum_{k=0}^{\nu-1} \frac{1}{N_2^{k(1-\gamma_2)}} Q^{N_1}_{k,t}(x)$ on both sides, we have
\begin{equation*}
\begin{aligned}
&\frac{1}{N_2^{\gamma_2-\frac{1}{2}}}Q^{N_1}_{\nu,t}(x)  \\
&= \frac{1}{N_2^{\gamma_2-\frac{1}{2}}}(N_2^{\gamma_2-\frac{1}{2}}h^{N_1,N_2}_0(x))-\frac{1}{N_2^{\gamma_2-\frac{1}{2}}}\left\lbrace\int_0^t  \int_{\CX \times \CY} Q^{N_1}_{\nu,s}(x') l^{N_1}_{0,s}\left(B^1_{x,x'}(\theta)+ \frac{1}{N_1}\sum_{j=1}^{N_1}B^{2,j}_{x,x'}(\theta)\right) \pi(dx',dy) ds\right.\\
&\qquad \left. +\frac{1}{N_1} \sum_{j=1}^{N_1} \int_0^t \int_{\CX\times \CY} Q^{N_1}_{\nu,s}(x') xx' l^{N_1}_{0,s}\left(B^{3,j}_{x}(\theta)\right)  l^{N_1}_{0,s}\left(B^{3,j}_{x'}(\theta)\right) \pi(dx',dy)ds\right\rbrace.
\end{aligned}
\end{equation*}
Since $N_2^{\gamma_2-\frac{1}{2}}h^{N_1,N_2}_0(x)$ converges in distribution to the Gaussian random variable $\mathcal{G}^{N_1}(x)$ defined in (\ref{limit_gaussian}), we have an expression for $Q^{N_1}_{{\nu},t}$:
\begin{equation*}
\begin{aligned}
Q^{N_1}_{\nu,t}(x) &= \mathcal{G}(x)-\int_0^t  \int_{\CX \times \CY} Q^{N_1}_{\nu,s}(x') l^{N_1}_{0,s}\left(B^1_{x,x'}(\theta)+ \frac{1}{N_1}\sum_{j=1}^{N_1}B^{2,j}_{x,x'}(\theta)\right) \pi(dx',dy) ds\\
&\quad  -\frac{1}{N_1} \sum_{j=1}^{N_1} \int_0^t \int_{\CX\times \CY} Q^{N_1}_{\nu,s}(x') xx' l^{N_1}_{0,s}\left(B^{3,j}_{x}(\theta)\right)  l^{N_1}_{0,s}\left(B^{3,j}_{x'}(\theta)\right) \pi(dx',dy)ds.
\end{aligned}
\end{equation*}
which coincides with (\ref{Eq:Qk_formula1}).
\item When $\gamma \ge \frac{2\nu+1}{2\nu+2}$, we first derive an expression for $l^{\nu}_t(f)$ by plugging \eqref{network_expansion} and \eqref{measure_expansion} into equation \eqref{mu_N1N2_evolution},
\begin{equation*}
\begin{aligned}
&\ip{f,\gamma^{N_1,N_2}_{t}} - \ip{ f,\gamma^{N_1,N_2}_{0}}\\
&\approx \frac{1}{N_2^{1-\gamma_2}}\int_0^t\int_{\CX\times \CY}  \paren{y  - \sum_{k=0}^{\nu} \frac{1}{N_2^{k(1-\gamma_2)}} Q^{N_1}_{k,s}(x')- O(N_2^{-(\nu+1)(1-\gamma_2)})} \\
& \qquad \qquad  \times \left[\sum_{k}^{\nu}\frac{1}{N_2^{k(1-\gamma_2)}}l^{N_1}_{k,s}(C^{f,1}_{x'}(\theta))+ O(N_2^{-(\nu+1)(1-\gamma_2)})\right]\pi(dx',dy)ds\\
&\quad +\frac{1}{N_1^{1-\gamma_1}N_2^{1-\gamma_2}}\int_0^t\int_{\CX\times \CY}\paren{y  - \sum_{k=0}^{\nu} \frac{1}{N_2^{k(1-\gamma_2)}} Q^{N_1}_{k,s}(x')- O(N_2^{-(\nu+1)(1-\gamma_2)})}\\
&\qquad \qquad  \times \left[\sum_{k}^{\nu}\frac{1}{N_2^{k(1-\gamma_2)}}l^{N_1}_{k,s}(C^{f,2}_{x'}(\theta))+ O(N_2^{-(\nu+1)(1-\gamma_2)})\right]\pi(dx',dy)ds\\
&\quad + \frac{1}{N_1^{1-\gamma_1}N_2^{1-\gamma_2}}\int_0^t\int_{\CX\times \CY}\paren{y  - \sum_{k=0}^{\nu} \frac{1}{N_2^{k(1-\gamma_2)}} Q^{N_1}_{k,s}(x')- O(N_2^{-(\nu+1)(1-\gamma_2)})}\\
&\qquad \qquad  \times \left[\sum_{k}^{\nu}\frac{1}{N_2^{k(1-\gamma_2)}}l^{N_1}_{k,s}(C^{f,3}_{x'}(\theta))+ O(N_2^{-(\nu+1)(1-\gamma_2)})\right]\\
& \qquad \qquad  \qquad \cdot \left[\sum_{k}^{\nu}\frac{1}{N_2^{k(1-\gamma_2)}}\frac{1}{k(N_2^{1-\gamma_2})}l^{N_1}_{k,s}(\nabla_{w^1}f(\theta)x') + O(N_2^{-(\nu+1)(1-\gamma_2)})\right] \pi(dx',dy)ds
\end{aligned}
\end{equation*}
\begin{equation*}
\begin{aligned}
&=\sum_{n=1}^{\nu-1}\frac{1}{N_2^{n(1-\gamma_2)}}l^{N_1}_{n,t}(f) + \frac{1}{N_2^{\nu(1-\gamma_2)}}\int_0^t\int_{\CX\times \CY}  \paren{y - Q^{N_1}_{0,s}(x')} \left[ l^{N_1}_{\nu-1,s}(C^{f,1}_{x'}(\theta)) + \frac{1}{N_1^{1-\gamma_2}}l^{N_1}_{\nu-1,s}(C^{f,2}_{x'}(\theta))\right]\pi(dx',dy)ds\\
&\quad -\frac{1}{N_2^{\nu(1-\gamma_2)}}\sum_{m=1}^{\nu-1}\int_0^t\int_{\CX\times \CY}Q^{N_1}_{\nu-m,s}(x') \left[ l^{N_1}_{m-1,s}(C^{f,1}_{x'}(\theta)) + \frac{1}{N_1^{1-\gamma_2}}l^{N_1}_{m-1,s}(C^{f,2}_{x'}(\theta))\right]\pi(dx',dy)ds\\
&\quad + \frac{1}{N_1^{1-\gamma_1}N_2^{\nu(1-\gamma_2)}}\int_0^t\int_{\CX\times \CY}\paren{y - Q^{N_1}_{0,s}(x')}\left[\sum_{k=0}^{\nu-1}l^{N_1}_{k,s}(C^{3}_{x'}(\theta)) \cdot l^{N_1}_{\nu-1-k,s}(\nabla_{w^1}f(\theta)x')\right]\pi(dx',dy)ds\\
&\quad - \frac{1}{N_1^{1-\gamma_1}N_2^{\nu(1-\gamma_2)}}\sum_{m=1}^{\nu-1}\int_0^t\int_{\CX\times \CY}Q^{N_1}_{\nu-m,s}(x')\left[\sum_{k=0}^{m-1} l^{N_1}_{k,s}(C^{3}_{x'}(\theta)) \cdot l^{N_1}_{m-1-k,s}(\nabla_{w^1}f(\theta)x')\right]\pi(dx',dy)ds\\
&\quad + O(N_2^{-(\nu+1)(1-\gamma_2)})
\end{aligned}
\end{equation*}

Subtracting $\ip{f,\gamma^{N_1}_0} + \sum_{n=1}^{\nu-1}\frac{1}{N_2^{n(1-\gamma_2)}}l^{N_1}_{n,t}(f)$, multiplying $N_2^{\nu(1-\gamma_2)}$ on both sides of the above equation, and using the fact that
 $N_2^{\nu(1-\gamma_2)}\paren{\ip{f,\gamma^{N_1,N_2}_0}-\ip{f,\gamma^{N_1}_0}}$ converges to 0 in distribution when $\gamma_2 \ge \frac{2\nu+1}{2\nu+2}$, we can get the following evolution equation for $l^{N_1}_{\nu,t}(f)$,
\begin{align*}
l^{N_1}_{\nu,t}(f)
&=\int_0^t\int_{\CX\times \CY}  \paren{y - Q^{N_1}_{0,s}(x')} \left[ l^{N_1}_{\nu-1,s}(C^{f,1}_{x'}(\theta)) + \frac{1}{N_1^{1-\gamma_2}}l^{N_1}_{\nu-1,s}(C^{f,2}_{x'}(\theta))\right]\pi(dx',dy)ds\\
&\quad -\sum_{m=1}^{\nu-1}\int_0^t\int_{\CX\times \CY}Q^{N_1}_{\nu-m,s}(x') \left[ l^{N_1}_{m-1,s}(C^{f,1}_{x'}(\theta)) + \frac{1}{N_1^{1-\gamma_2}}l^{N_1}_{m-1,s}(C^{f,2}_{x'}(\theta))\right]\pi(dx',dy)ds\\
&\quad + \frac{1}{N_1^{1-\gamma_1}}\int_0^t\int_{\CX\times \CY}\paren{y - Q^{N_1}_{0,s}(x')}\left[\sum_{k=0}^{\nu-1}l^{N_1}_{k,s}(C^{3}_{x'}(\theta)) \cdot l^{N_1}_{\nu-1-k,s}(\nabla_{w^1}f(\theta)x')\right]\pi(dx',dy)ds\\
&\quad - \frac{1}{N_1^{1-\gamma_1}}\sum_{m=1}^{\nu-1}\int_0^t\int_{\CX\times \CY}Q^{N_1}_{\nu-m,s}(x')\left[\sum_{k=0}^{m-1} l^{N_1}_{k,s}(C^{3}_{x'}(\theta)) \cdot l^{N_1}_{m-1-k,s}(\nabla_{w^1}f(\theta)x')\right]\pi(dx',dy)ds,
\end{align*}
which concludes the inductive step for $l^{N_1}_{\nu,t}(f)$.

Next, we derive $Q^{N_1}_{\nu,t}$ by plugging equations \eqref{network_expansion} and \eqref{measure_expansion} into the left hand side of equation \eqref{h_N1N2_evolution}:
\begin{equation*}
\begin{aligned}
&h^{N_1,N_2}_t(x)-h^{N_1,N_2}_0(x)\\
&\approx \int_0^t  \int_{\CX \times \CY} \paren{y-\sum_{k=0}^{\nu} \frac{1}{N_2^{k(1-\gamma_2)}}Q^{N_1}_{k,s}(x') - O(N_2^{(\nu+1)(1-\gamma_2)})}  \\
& \quad \times \left[\sum_{k=0}^{\nu} \frac{1}{N_2^{k(1-\gamma_2)}} l^{N_1}_{k,s}\left(B^1_{x,x'}(\theta)+ \frac{1}{N_1}\sum_{j=1}^{N_1}B^{2,j}_{x,x'}(\theta)\right) + O(N_2^{-(\nu+1)(1-\gamma_2)})\right]\pi(dx',dy) ds\\
&\quad +\frac{1}{N_1} \sum_{j=1}^{N_1} \int_0^t \int_{\CX\times \CY} \paren{y-\sum_{k=0}^{\nu} \frac{1}{N_2^{k(1-\gamma_2)}}Q^{N_1}_{k,s}(x') - O(N_2^{-(\nu+1)(1-\gamma_2)})}xx'\\
&\quad \times \left[\sum_{k=0}^{\nu} \frac{1}{N_2^{k(1-\gamma_2)}} l^{N_1}_{k,s}\left(B^{3,j}_{x}(\theta)\right)+O(N_2^{-(\nu+1)(1-\gamma_2)}) \right]\nonumber\\
&\quad\times\left[\sum_{k=0}^{\nu} \frac{1}{N_2^{k(1-\gamma_2)}} l^{N_1}_{k,s}\left(B^{3,j}_{x'}(\theta)\right) +O(N_2^{-(\nu+1)(1-\gamma_2)})\right]\pi(dx',dy)ds
\end{aligned}
\end{equation*}
\begin{equation*}
\begin{aligned}
&= \sum_{n=0}^{\nu-1} \frac{1}{N_2^{n(1-\gamma_2)}} Q^{N_1}_{n,t}(x) + \frac{1}{N_2^{\nu(1-\gamma_2)}}\int_0^t  \int_{\CX \times \CY} \paren{y-Q^{N_1}_{0,s}(x')}  l^{N_1}_{\nu,s}\left(B^1_{x,x'}(\theta)+ \frac{1}{N_1}\sum_{j=1}^{N_1}B^{2,j}_{x,x'}(\theta)\right)\pi(dx',dy) ds\\
&\quad +\frac{1}{N_1N_2^{\nu(1-\gamma_2)}}\sum_{j=1}^{N_1} \int_0^t \int_{\CX\times \CY} \paren{y- Q^{N_1}_{0,s}(x')}xx' \left[\sum_{k=0}^{\nu} l^{N_1}_{k,s}\left(B^{3,j}_{x}(\theta)\right) l^{N_1}_{\nu-k,s}\left(B^{3,j}_{x'}(\theta)\right)\right]\pi(dx',dy)ds\\
&\quad  - \frac{1}{N_2^{\nu(1-\gamma_2)}}\sum_{m=0}^{\nu-1} \int_0^t  \int_{\CX \times \CY}Q^{N_1}_{\nu-m,s}(x')  l^{N_1}_{m,s}\left(B^1_{x,x'}(\theta)+ \frac{1}{N_1}\sum_{j=1}^{N_1}B^{2,j}_{x,x'}(\theta)\right)\pi(dx',dy) ds\\
&\quad -\frac{1}{N_1N_2^{\nu(1-\gamma_2)}} \sum_{j=1}^{N_1} \sum_{m=0}^{\nu-1} \int_0^t \int_{\CX\times \CY} Q^{N_1}_{\nu-m,s}(x') xx'\left[\sum_{k=0}^{m}l^{N_1}_{k,s}\left(B^{3,j}_{x}(\theta)\right) l^{N_1}_{m-k,s}\left(B^{3,j}_{x'}(\theta)\right)\right]\pi(dx',dy)ds\\
&\quad+ O(N^{-\Omega_{\nu+1}}),
\end{aligned}
\end{equation*}
where $\Omega_{\nu+1}>(\nu+1)(1-\gamma_2 )$.
Following the same idea as earlier, when $\gamma_2 = \frac{2\nu+1}{2\nu+2}$, we note that $\nu(1-\gamma_2) = \gamma_2 -\frac{1}{2}=\frac{\nu}{2\nu+2}$, we can obtain an expression for $Q^{N_1}_{\nu,t}$ (which coincides with (\ref{Eq:Qj_formula2})):
\begin{equation*}
\begin{aligned}
Q^{N_1}_{\nu,t}(x) &= \mathcal{G}^{N_1}(x) + \int_0^t  \int_{\CX \times \CY} \paren{y-Q^{N_1}_{0,s}(x')}  l^{N_1}_{\nu,s}\left(B^1_{x,x'}(\theta)+ \frac{1}{N_1}\sum_{j=1}^{N_1}B^{2,j}_{x,x'}(\theta)\right)\pi(dx',dy) ds\\
&\quad +\frac{1}{N_1} \sum_{j=1}^{N_1} \int_0^t \int_{\CX\times \CY} \paren{y- Q^{N_1}_{0,s}(x')}xx' \left[\sum_{k=0}^{n} l^{N_1}_{k,s}\left(B^{3,j}_{x}(\theta)\right) l^{N_1}_{\nu-k,s}\left(B^{3,j}_{x'}(\theta)\right)\right]\pi(dx',dy)ds\\
&\quad  - \sum_{m=0}^{\nu-1} \int_0^t  \int_{\CX \times \CY}Q^{N_1}_{\nu-m,s}(x')  l^{N_1}_{m,s}\left(B^1_{x,x'}(\theta)+ \frac{1}{N_1}\sum_{j=1}^{N_1}B^{2,j}_{x,x'}(\theta)\right)\pi(dx',dy) ds\\
&\quad -\frac{1}{N_1} \sum_{j=1}^{N_1} \sum_{m=0}^{\nu-1} \int_0^t \int_{\CX\times \CY} Q^{N_1}_{\nu-m,s}(x') xx'\left[\sum_{k=0}^{m}l^{N_1}_{k,s}\left(B^{3,j}_{x}(\theta)\right) l^{N_1}_{m-k,s}\left(B^{3,j}_{x'}(\theta)\right)\right]\pi(dx',dy)ds,
\end{aligned}
\end{equation*}
where $\mathcal{G}(x)$ is the Gaussian random variable. And when $\gamma_2 >\frac{2\nu+1}{2\nu+2}$, $Q^{N_1}_{\nu,t}$ is driven by the deterministic equation
\begin{equation*}
\begin{aligned}
Q^{N_1}_{\nu,t}(x) &=  \int_0^t  \int_{\CX \times \CY} \paren{y-Q^{N_1}_{0,s}(x')}  l^{N_1}_{\nu,s}\left(B^1_{x,x'}(\theta)+ \frac{1}{N_1}\sum_{j=1}^{N_1}B^{2,j}_{x,x'}(\theta)\right)\pi(dx',dy) ds\\
&\quad +\frac{1}{N_1} \sum_{j=1}^{N_1} \int_0^t \int_{\CX\times \CY} \paren{y- Q^{N_1}_{0,s}(x')}xx' \left[\sum_{k=0}^{n} l^{N_1}_{k,s}\left(B^{3,j}_{x}(\theta)\right) l^{N_1}_{\nu-k,s}\left(B^{3,j}_{x'}(\theta)\right)\right]\pi(dx',dy)ds\\
&\qquad  - \sum_{m=0}^{\nu-1} \int_0^t  \int_{\CX \times \CY}Q^{N_1}_{\nu-m,s}(x')  l^{N_1}_{m,s}\left(B^1_{x,x'}(\theta)+ \frac{1}{N_1}\sum_{j=1}^{N_1}B^{2,j}_{x,x'}(\theta)\right)\pi(dx',dy) ds\\
&\qquad -\frac{1}{N_1} \sum_{j=1}^{N_1} \sum_{m=0}^{\nu-1} \int_0^t \int_{\CX\times \CY} Q^{N_1}_{\nu-m,s}(x') xx'\left[\sum_{k=0}^{m}l^{N_1}_{k,s}\left(B^{3,j}_{x}(\theta)\right) l^{N_1}_{m-k,s}\left(B^{3,j}_{x'}(\theta)\right)\right]\pi(dx',dy)ds,
\end{aligned}
\end{equation*}

This concludes the inductive step for the derivation of $Q^{N_1}_{\nu,t}(x)$.
\end{itemize}
\bibliographystyle{abbrv}

\end{document}